%% file: main.tex
\documentclass[journal]{IEEEtran} 

\usepackage[version=0.96]{pgf}
\usepackage[noadjust]{cite}

\usepackage{pgfplots}
\usepackage{graphicx}
\pgfplotsset{compat = newest}
\usetikzlibrary{decorations.fractals,spy}
\usepackage{lettrine} 
\newcommand{\MyPARstart}[2]{\lettrine[lines=2,nindent=0em]{#1}{#2}}
\usepackage{tikz}
\usetikzlibrary{spy}
\usepackage{filecontents}
\usepackage{tikzducks}
\usepackage{import}
\usepackage{tikz,pgfplots}
\usepgfplotslibrary{fillbetween}
\usepackage{bigints}
\usepackage{abraces}
\usepackage{multirow}
\usepackage{tcolorbox}
\usepackage[noadjust]{cite}
\usepackage{amsfonts}
\usepackage{algorithmic}
\usepackage{graphicx}
\usepackage{textcomp}
\usepackage{float}
\definecolor{darkgreen}{rgb}{0,0.6,0}
\usepackage{authblk}
\usepackage[colorlinks,linkcolor=blue,citecolor=blue]{hyperref}
\usepackage[latin1]{inputenc}
\usepackage{graphicx}
\usepackage{verbatim}
\usepackage{epsfig}
\usepackage[labelfont=bf]{caption}
\usepackage{enumerate}
\usepackage{subfig}
\usepackage{epstopdf}
\usepackage{relsize,exscale}
\usepackage[export]{adjustbox}
\usepackage{algorithm,algorithmic}

\usepackage{booktabs} 

\newtheorem{definition}{Definition}
\newtheorem{assumption}{Assumption}

\newtheorem{remark}{Remark}

\newtheorem{proposition}{Proposition}

\definecolor{amethyst}{rgb}{0.6, 0.4, 0.8}
\definecolor{carminepink}{rgb}{0.92, 0.3, 0.26}
\definecolor{color1}{rgb}{0.196, 0.722, 0.592}
\definecolor{color2}{HTML}{D76364}

\def\begquo{\begin{quote}}
\def\endquo{\end{quote}}
\def\begequarr{\begin{eqnarray}}
\def\endequarr{\end{eqnarray}}
\def\begequarrs{\begin{eqnarray*}}
\def\endequarrs{\end{eqnarray*}}
\def\begarr{\begin{array}}
\def\endarr{\end{array}}
\def\begequ{\begin{equation}}
\def\endequ{\end{equation}}

\def\begdes{\begin{description}}
\def\enddes{\end{description}}
\def\begenu{\begin{enumerate}}
\def\begite{\begin{itemize}}
\def\endite{\end{itemize}}
\def\endenu{\end{enumerate}}

\def\lef[{\left[\begin{array}}
\def\rig]{\end{array}\right]}
\def\begcen{\begin{center}}
\def\endcen{\end{center}}
\def\endrem{\end{remark}}
\def\begdef{\begin{definition}}
\def\enddef{\end{definition}}
\def\begpro{\begin{proposition}}
\def\endpro{\end{proposition}}
\def\begfac{\begin{fact}}
\def\endfac{\end{fact}}
\def\begsubequ{\begin{subequations}}
\def\endsubequ{\end{subequations}}

\def\begmat#1{\begin{bmatrix}#1\end{bmatrix}}

\def\bfg{\mathbf{g}}
\def\caln{{\cal N}}

\def\cale{{\cal E}}
\def\cali{{\cal I}}

\def\cale{{\cal E}}

\def\call{{\cal L}}

\def\calx{{\cal X}}

\def\bff{{\bf f}}

\def\bff{{\bf f}}

\def\bfx{{\bf x}}
\def\bfq{{\bf q}}
\def\bfp{{\bf p}}

\def\bfu{{\bf u}}

\def\ug{U_{\tt G}}
\def\ue{U_{\tt E}}

\def\L2e{{\cal L}_{2e}}

\def\rea{\mathbb{R}}

\def\diag{\mbox{diag}}

\def\col{\mbox{col}}
\def\hal{{1 \over 2}}

\def\diag{\mbox{diag}}

\def\min{{\mbox{min}}}

\def\QED{\hfill $\square$}
\def\qed{\hfill $\triangleleft$}

\usepackage{tikz,xcolor,hyperref}

\definecolor{lime}{HTML}{A6CE39}
\DeclareRobustCommand{\orcidicon}{
	\begin{tikzpicture}
	\draw[lime, fill=lime] (0,0) 
	circle [radius=0.16] 
	node[white] {{\fontfamily{qag}\selectfont \tiny ID}};
	\draw[white, fill=white] (-0.0625,0.095) 
	circle [radius=0.007];
	\end{tikzpicture}
	\hspace{-2mm}
}

\foreach \x in {A, ..., Z}{\expandafter\xdef\csname orcid\x\endcsname{\noexpand\href{https://orcid.org/\csname orcidauthor\x\endcsname}
			{\noexpand\orcidicon}}
}

\usepackage[prependcaption,colorinlistoftodos]{todonotes}

\usepackage[normalem]{ulem}
\input epsf

\def\BibTeX{{\rm B\kern-.05em{\sc i\kern-.025em b}\kern-.08em
    T\kern-.1667em\lower.7ex\hbox{E}\kern-.125emX}}

\begin{document}

\title{Simultaneous Position-and-Stiffness Control of Underactuated Antagonistic Tendon-Driven Continuum Robots}
\author{Bowen Yi\orcidA{}, Yeman Fan\orcidB{}, Dikai Liu\orcidC{}, and Jose Guadalupe Romero\orcidD{} 
%
\thanks{This work was supported in part by the Australian Research Council (ARC) Discovery Project under Grant DP200102497, the University of Technology Sydney, the Natural Sciences and Engineering Research Council of Canada (NSERC), and the Programme PIED. B. Yi and Y. Fan contributed equally to this work. {\it (Corresponding author: Y. Fan)}}
\thanks{B. Yi is with the Department of Electrical Engineering, Polytechnique Montreal \& GERAD, Montreal, QC, Canada. The work has been done in part when B. Yi was with University of Technology Sydney, Australia. (email: bowen.yi@polymtl.ca) 
}
\thanks{Y. Fan and D. Liu are with Robotics Institute, Faculty of Engineering and Information Technology, University of Technology Sydney, Sydney, NSW 2006, Australia. (email: \{yeman.fan,dikai.liu\}@uts.edu.au) 
}
\thanks{J.G. Romero is with Departamento Acad\'emico de Sistemas Digitales, ITAM, R\'io Hondo 1,  01080, Ciudad de M\'exico, M\'exico (email: jose.romerovelazquez@itam.mx)}
}

\maketitle

\begin{abstract}
Continuum robots have gained widespread popularity due to their inherent compliance and flexibility, particularly their adjustable levels of stiffness for various application scenarios. Despite efforts to dynamic modeling and control synthesis over the past decade, few studies have incorporated stiffness regulation into their feedback control design; however, this is one of the initial motivations to develop continuum robots. This paper addresses the crucial challenge of controlling both the position and stiffness of underactuated continuum robots actuated by antagonistic tendons. We begin by presenting a rigid-link dynamical model that can analyze the open-loop stiffening of tendon-driven continuum robots. Based on this model, we propose a novel passivity-based position-and-stiffness controller that adheres to the non-negative tension constraint. Comprehensive experiments on our continuum robot validate the theoretical results and demonstrate the efficacy and precision of this approach. 
\end{abstract}

\begin{keywords}
This paper is motivated by our experience and practical needs in building continuum robotic platforms. Stiffness, flexibility, and accurate configuration regulation are practically important properties for this class of robots. Even though simultaneous position-and-stiffness control is a mature topic for rigid and softly-actuated robots, it remains an open problem for continuum robots with theoretically guaranteed performance. The intention of this paper is to address this situation by proposing a model-based solution for a class of underactuated tendon-driven continuum robots. Hence, we propose an energy-based model and a passivity-based controller to regulate stiffness and configuration concurrently. We believe that this work can be beneficial for academic and industrial research in the context of control algorithms for continuum robots.
\end{keywords}

%
\section{Introduction}
\label{sec1}
%
\MyPARstart{C}{ontinuum} robots are a novel class of robotic systems that have made significant progress in the past few years. Their unique properties, such as scaled dexterity and mobility, make them well facilitated important and suitable for human-robot interaction and manipulation tasks in uncertain and complex environments. For example, they can be used for manipulating objects with unknown shapes, performing search and rescue operations, and whole-arm grasping \cite{KAPetal,LIUetal2020}. 

Despite the above advantages, rigid-body robots still outperform continuum robots in tasks requiring adaptable movement and compliant interactions  \cite{DELetal}. Consequently, many efforts have been devoted to addressing the challenges for real-time control of continuum robots that facilitate fast, efficient, and reliable operation \cite{DELetalSURVEY,WANCHO,THUetal}. The existing control approaches for soft robots can be broadly classified into two categories: data-driven and model-based design. Initially, data-driven approaches dominated the research in this specific field, as obtaining reliable models of a continuum robot was believed to be overwhelmingly complex \cite{DELetalSURVEY}. Various learning methodologies have been applied to control soft robots, such as Koopman operator \cite{BRUetal}, Gaussian process temporal difference learning \cite{ENGetal,MOetal2024}, supervised learning via recurrent neural networks \cite{THUetalTRO,XIAetal2023,TANetal2023TII, TANetal2024}, and feedforward neural networks \cite{BRAetal}. However, these approaches have some key limitations, including stringent date set requirements and no guarantee of stability or safety  \cite{FANetal,YIetalKOOPMAN,TSUetal}. Conversely, the recent resurgence of interest in model-based approaches has made them particularly appealing for soft robots due to their robustness, interpretability, and manageable properties \cite{DELetalSURVEY}.

Elastic deformation of continuum robots theoretically leads to infinite degrees-of-freedom (DoF) motion, which renders them particularly suitable to be modelled by partial differential equations (PDEs) \cite{CAAetal}. In particular, there are two prevalent categories of modelling approaches: mechanics-based and geometry-based. The former focuses on studying the elastic behaviour of the constitutive materials and solving the boundary conditions problem, such as the methods using Cosserat rod theory and Euler-Bernoulli beam theory \cite{TUMetal, chen2024chained}. They need to be solved numerically to obtain a closed formulation for each material subdomain that has proven successful in the design and analysis of continuum robots with high accuracy \cite{ARMetal}, but, due to the extremely heavy computational burden, they are not adapted to real-time control \cite{BIEetal}. In contrast, geometrical models assume that the soft body can be represented by a specific geometric shape, \textit{e.g.}, piecewise constant curvature (PCC). As these modelling approaches often lead to kinematic models rather than dynamical models, they enable the design of kinematic or quasi-static controllers \cite{JONWAL,zhao2024controller}. It has been shown that such types of kinematic controllers are likely to yield poor closed-loop performance \cite{KAPetal}. 

To address these challenges, recent research has been focused on the dynamic modelling and model-based control of continuum robots.\footnote{In this paper, we use the term ``dynamic controllers'' to refer to feedback laws designed from dynamical and kinematic models. This differs from the terminology in control theory, which typically refers to feedback control with dynamics extension (\textit{e.g.} adaptive and observer-based control) \cite{ORTetal01}.} Several dynamical models have been adopted for controller synthesis, including the geometrically exact dynamic Cosserat model \cite{RENetal}, port-Hamiltonian Cosserat model \cite{CAAetal,CHAetalRoyal}, rigid-link models \cite{FRAGAR,DELetal}, and reduced-order Euler-Lagrangian model \cite{FALetal,DEUetal}; see also \cite{BREMCC,TILRUC} for stability analysis of equilibria in continuum robots. These works employ model-based control approaches such as passivity-based control (PBC) \cite{CHAetalCDC}, partial feedback linearisation, proportional derivative (PD) control, and immersion and invariance (I\&I) adaptive control. Among these, \cite{DELetal} reports probably the earliest solution in the literature to the design and experimental validation of \textit{dynamic} feedback control for soft robots.

As illustrated above, one of the primary motivations for developing continuum robots is enhance agility, adaptability, and compliant interactions \cite{FANetal2}. Consequently, there is an urgent and rapidly growing need to develop high-performance control algorithms to regulate position and stiffness simultaneously, particularly in certain applications involving human interaction or in complicated environments, such as search and rescue, industrial inspection, medical service, and home living care. The problems of stiffness control and impedance control are well established for \emph{rigid} and \emph{softly-actuated} robotics \cite{BESetal,KIMetal,MENetal}. In contrast, simultaneous position-and-stiffness control of continuum robots remains an open research area. The first stiffness controller for continuum robots in the literature may refer to \cite{MAHDUP}, which extends a simple Cartesian impedance controller using a kinematic model. In \cite{bajo2016hybrid}, the authors tailor the classic hybrid motion/force controller for a static model of multi-backbone continuum robots, requiring estimation of external wrenches. Note that \cite{bajo2016hybrid,MAHDUP} are concerned with static/kinematic models, thus limiting their transient performance. In \cite{DELetal}, a Cartesian stiffness controller was proposed for dynamic control of a fully-actuated soft robot, facilitating interaction with environment. Note that these approaches are not applicable to \textit{underactuated} dynamical models of continuum robots.

This paper aims to address the above gap by proposing a novel dynamical model and a real-time control approach that regulates both position and stiffness concurrently for underactuated antagonistic tendon-driven continuum robots. Note that continuum robots have infinite DoF with only finite actuation inputs, the fact that makes them intrinsically underactuated systems \cite{ORTetal}. The main contributions of the paper are:
\begin{itemize}\setlength\itemsep{.2em}
    \item[1)] We propose a port-Hamiltonian dynamical model for a class of antagonistic tendon-driven continuum robots, which features a \textit{configuration-dependent input matrix} that enables us to interpret the underlying mechanism for open-loop stiffening.
    \item[2)] Stiffness flexibility is one of the motivations for developing continuum robots. Using the derived \textit{underactuated} dynamical model, we propose a novel potential energy shaping controller. Though simultaneous position-and-stiffness control has been widely studied for rigid and softly-actuated robots, to the best of the authors' knowledge, this work is probably among the earliest to design a controller capable of simultaneously controlling an underactuated continuum robot.
    \item[3)] We analyse the set of assignable equilibria for the proposed model class, which is actuated by tendons providing only \textit{non-negative} tensions. Furthermore, we demonstrate how to integrate input constraints into the controller design via an input transformation.
\end{itemize}

We conducted experiments in a variety of scenarios on a robotic platform OctRobot-I to validate the theoretical results presented in the paper. However, due to 2) and the lack of applicable control approaches in the existing literature, we were unable to include a fair experimental comparison with previous work in this study.


\textit{Notation.} All functions and mappings are assumed to be $C^2$-continuous. $I_n$ is the $n \times n$ identity matrix, $0_{n \times s}$ is an
$n \times s$ matrix of zeros, the vector $\mathbf{0}_{n}$ represents $\col(0,\ldots, 0) \in \rea^n$, and $\mathbf{1}_n := \col(1, \ldots, 1) \in \rea^n$. Throughout the paper, we adopt the convention of using bold font for variables denoting vectors, while scalars and matrices are represented in normal font. For $\bfx \in \rea^n$, $S \in \rea^{n \times n}$, $S=S^\top
>0$, we denote the Euclidean norm $\|\bfx\|^2:=\bfx^\top \bfx$, and the weighted--norm $\|\bfx\|^2_S:=\bfx^\top S \bfx$. Given a function $f:  \rea^n \to \rea$ we define the differential operators
$
\nabla f:=(\frac{\partial f }{ \partial x})^\top,\;\nabla_{x_i} f:=(\frac{
\partial f }{ \partial x_i})^\top,
$
where $x_i \in \rea^p$ is an element of the vector $\bfx$. The set $\caln$ is defined as $\caln:= \{1,\ldots,n\}$. For a full rank matrix $g\in \rea^{n\times m}$ ($m<n$), we denote the generalised inverse as $g^\dagger  = [g^\top g]^{-1} g^\top$ and $g^\bot$ a full rank left annihilator. When clear from the context, the arguments of functions and mappings may be omitted.

%
\section{Model and Problem Set}
\label{sec2}
%

\subsection{Modelling of A Class of Continuum Robots}

In this section, we present a \emph{control-oriented} rigid-link dynamical model specifically designed for a class of underactuated continuum robots driven by tendons. This model class encompasses a wide range of recently reported continuum robots in the literature, including the elephant trunk-inspired robot \cite{Zhangetal}, the deployable soft robotic arm \cite{FATetal}, the push puppet-inspired robot \cite{BERetal}, the dexterous tip-extending robot \cite{WANZHAetal}, and our own developed OctRobot-I \cite{FANLIU}, alongside other notable examples \cite{FRAGAR,CHEMISetal,MOetal}. By employing this versatile model, we aim to provide a general framework that can effectively describe and analyse a variety of underactuated continuum robots, enabling a deeper understanding of their stiffening mechanisms and facilitating control design.

In order to visualise the modelling process, we take OctRobot-I as an example to introduce the proposed dynamical model but keep its generality in mind that the model is not limited to this specific robotic platform. This robot imitates an octopus tentacle's structure and motion mechanism, as shown in Fig. \ref{fig:1-2}. The whole continuum manipulator consists of several sections in order to be able to deform in three-dimensional space. Each of them is made of connected spine segments that are driven by a pair of cables. More details of the continuum robot OctRobot-I are given in Section \ref{sec:61}, as well as in \cite{FANLIU}.

\begin{figure}[!htp]
    \centering
    \includegraphics[width = 0.42\textwidth, angle = 0]{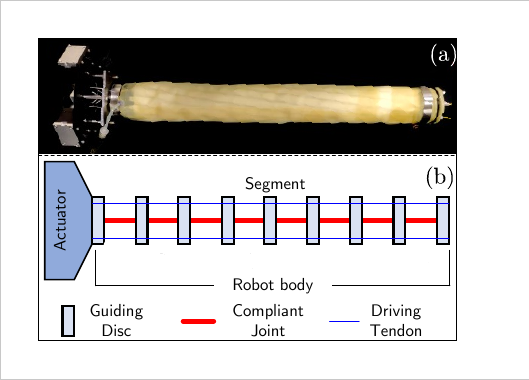}
    \caption{(a) A photo of the continuum robot OctRobot-I (b) Mechanical structure of a class of continuum robots}
    \label{fig:1-2}
\end{figure}

In this paper, we use a rigid-link model to approximate the dynamical behaviours of continuum robots for three reasons: 1) the considered class of continuum robots naturally partitions into spine segments; 2) rigid-link models simplify \emph{control-oriented} tasks; and 3) it is convenient to account for external loading.

\begin{table}[!htb]
\begin{tcolorbox}[
colback=white!10,
coltitle=blue!20!black,  
]
\begin{center}
 {\small \textsc{Nomenclature}}
\end{center}
\vspace{0.2cm}
  \renewcommand\arraystretch{1.4}
\small
\begin{tabular}{ll}
$\bfq = [q_1 ~ \ldots ~ q_n]^\top$ & Configuration variable \\
$\bfp \in \rea^n$ & Generalised momenta\\
$\bfq_\star \in \rea^n$ & Desired configuration \\
$U_{\tt G}, U_{\tt E}, U \in \rea$ & Gravitational, elastic, total potential\\
& energy functions\\
$H(\bfq,\bfp) \in \rea$&  Total Hamiltonian \\
$\tau_{\tt ext} \in \rea^n$ &  External torque \\
$\bfu \in \rea^2$ & Cable tensions  \\
$K_{\tt C} := \diag(K_{\tt T}, K_{\tt A})$ &  Stiffness matrix\\
$D(\bfq) \in \rea^{n\times n}_{\succ 0}$ &  Damping matrix
\end{tabular}
\end{tcolorbox}
\end{table}

To obtain the dynamical model, we make the following assumptions.
\begin{assumption}
\label{assp:1}
The continuum robot satisfies the properties:
\begin{itemize}

    \item[({\bf a})] The actuator dynamics is negligible, \textit{i.e.}, the motor is operating in the torque control mode with sufficiently short transient stages;

    \item[({\bf b})] The sections have a piecewise constant curvature (PCC), conforming to the segments\footnote{Each spine segment has constant curvatures but is variable in time.}, and the curvatures consistently have the same sign.
    
    \item[(\bf c)] The continuum robot allows for axial extension, but the axial deformation resulting from antagonistic tensions is negligible compared to the bending.\qed
\end{itemize}
\end{assumption}


 In this paper, we specifically concentrate on the two-dimensional case, limiting our analysis to a single section, in order to effectively illustrate the underlying mechanism.\footnote{It is promising to extend the main results to the three-dimensional case with multi-sections. We will consider it as a valuable avenue for further exploration.}

In the rigid-link dynamical model, we use a serial chain of rigid links with $n$ rotational joints to approximate one section of the continuum robot. Then, the configuration variable can be defined as 
$$
\bfq = \begmat{q_1&  \ldots&  q_n}^\top \in \calx \subset  \rea^n,
$$
with $q_i$ representing the approximate link angles, where $\calx$ is the feasible configuration space; see Fig. \ref{fig:1-1} for an illustration. Practically, all angles $q_i$ are within some subsets of $[-{\pi \over 2}, {\pi \over 2}]$ due to physical constraints. 

We model the continuum robot as a port-Hamiltonian system in the form of \cite{ORTetal,VANbook}
\begin{equation}
\label{model:pH}
	\begmat{ ~\dot \bfq ~\\ ~\dot \bfp~} 
	=
	\begmat{0_{n\times n} & I_n \\ - I_n & -D(\bfq)} \begmat{\nabla	_{\bfq} H \\ \nabla_{\bfp} H }+ \begmat{~\mathbf{0}_n ~ \\ ~G(\bfq) \bfu + \tau_{\tt ext}~}
\end{equation}
with the generalised momenta $\bfp \in \rea^n$, the damping matrix $D(\bfq) \in \rea^{n\times n}_{\succ 0}$, $\tau_{\tt ext} \in \rea^n$ the external torque, and the input matrix $G(\bfq) \in \rea^{n\times m}$ with $m<n$. The total energy of the robotic system is given by
\begin{equation}
\label{H}
H(\bfq,\bfp) = {1\over2} \bfp^\top M^{-1}(\bfq)\bfp + U(\bfq),
\end{equation}
with the inertia matrix $M(\bfq) \succ 0$, and the potential energy $U(\bfq)$, which consists of the gravitational part $U_{\tt G}$ and the elastic one $\ue(\bfq)$, \textit{i.e.}
$
U(\bfq) ~=~ \ug(\bfq) + \ue(\bfq).
$
The potential energy function has an isolated local minimum at its open-loop equilibrium $\bfq = \mathbf{0}_n$.

\begin{figure}[!hpt]
    \centering
    \includegraphics[width = 0.4\textwidth]{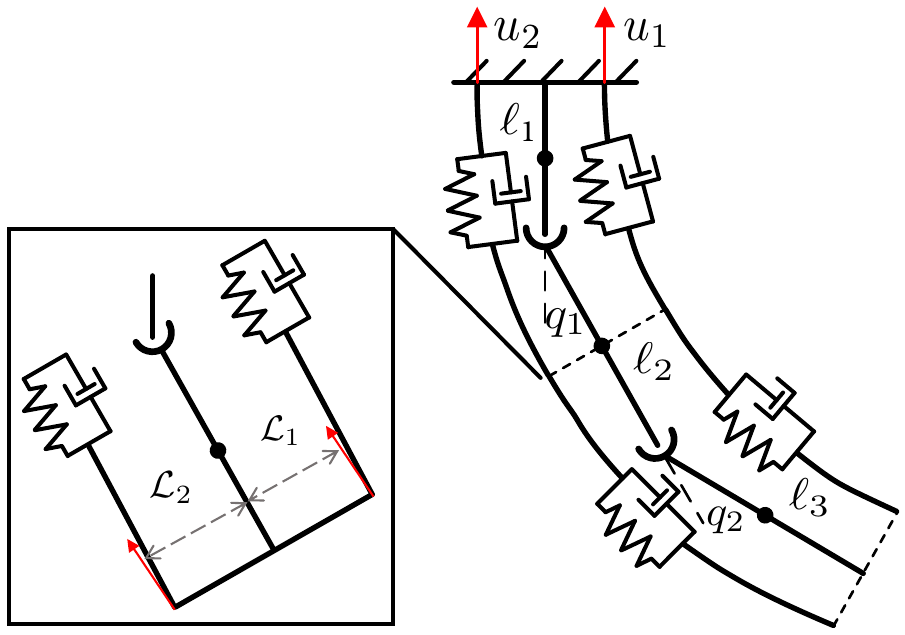}
    \caption{An illustration of configuration variables (The symbol ``$\bullet$'' indicates the lumped mass $m_i$; the angle $q_i$ assumes positive values for small counter-clockwise rotations.)
    }
\label{fig:1-1}
\end{figure}

Fig. \ref{fig:1-1} shows the above rigid-link model of continuum robots. Similar to \cite{DELetal}, we adopt the assumption that a lumped mass with the value $m_i~(i\in \caln)$ is virtually placed in the middle of each link, and the link lengths are $\ell_i$. We additionally make the following assumption on the mass and length.

\begin{assumption}\rm\label{ass:uniform}
  The continuum robot satisfies the uniformity assumptions:
  \begin{itemize}
      \item[({\bf a})] The masses verify $  m_i = m_j, ~ \forall i,j \in \caln.$
      \item[({\bf b})] The lengths satisfy the relation: $l_0 =\ell$ and $l_i = 2\ell$ ($i \in \caln$) for some $\ell>0$. The radius of the beam is $r$. \qed
  \end{itemize}
\end{assumption}

The gravitational and elastic potential energy functions $U_{\tt G}, U_{\tt E}$ can be derived according to the geometric deformation under the uniformity assumptions of the materials. To make the paper self-contained, the details on the modelling of the potential energy functions are given in Appendix.

The variable $\bfu \in \rea^m$ represents the control input, denoting the tensions along the cables generated by actuators. In the planar case, we have $m=2$ with two cables. Due to the specific structure, the tensions are one-directional, \textit{i.e.},
\begin{equation}
\label{constraint:direction}
u_i\ge 0 \quad (i=1,2).
\end{equation}

For the studied case, the input matrix $G(\bfq): \rea^n \to \rea^{n\times 2}$ can be conformally partitioned as
\begin{equation}
\label{Gpart}
G(\bfq) ~=~ \begmat{~ G_1(\bfq) & G_2(\bfq) ~}.
\end{equation}
In the following assumption, some key properties of the matrix $G(\bfq)$ are underlined when modelling the continuum robot.
\begin{assumption}\label{ass:1}\rm
The input matrix $G(\bfq)$ of the continuum robot model \eqref{model:pH} -- or equivalently in \eqref{Gpart} -- satisfies
\begin{itemize}
\item[({\bf a})] $G(\bfq)$ is state-dependent and $C^1$-continuous.

\item[({\bf b})] $G_1(\textbf{0}_n) = - G_2(\textbf{0}_n)$.

\item[({\bf c})] $\|G_1(\bfq) + G_2(\bfq)\| \neq 0$ for $\bfq\in {\calx} \backslash \{ \textbf{0}_n\}$. \qed
\end{itemize}	
\end{assumption}

\vspace{0.2cm}

Among the above three items, the state-dependency of the input matrix is a key feature of the proposed model, which is instrumental in showing the tunability of open-loop stiffness of tendon-driven continuum robots. We will give more details about the input matrix in the subsequent sections of the paper. The second point ({\bf b}) means that at the open-loop equilibrium (\emph{i.e.} $\bfq = \textbf{0}_n$), the tensions in the two cables are equal in magnitude but opposite in direction.

\begin{remark}
Let us now consider a single link in the zoomed-in subfigure in Fig. \ref{fig:1-1}. If the forces along the cables are assumed lossless, their directions are nonparallel to the centroid of the continuum robot. The torques imposed on the first approximate link are given by $u_1 \mathcal{L}_1(q_1)$ and $u_2 \call_2(q_1)$ with $\mathcal{L}_1, \mathcal{L}_2$ the lever's fulcrums, which are nonlinear functions of the configuration $q_1$. From some basic geometric relations, it satisfies $\call_1(0)= \call_2(0)$. This illustrates the rationality of Assumption \ref{ass:1}. \qed
\end{remark}

\begin{remark}
  In Assumption \ref{ass:1}(c), we assume that the geometry of the continuum robot satisfies the PCC condition. Additionally, the assumption that the \emph{curvatures} have an unchanged sign excludes the second type of configuration (S-shape) shown in Fig. \ref{fig:new}(b); see \cite{kim2013stiffness} for a comprehensive analysis of this case. This assumption is reasonable, as it covers the majority of scenarios for our platform. In contrast, the S-shape configuration necessitates precise symmetry in both the mechanical design and the wire arrangement. 
\end{remark}
\begin{figure}[!htp]
    \centering
    \includegraphics[width = 0.32\textwidth, angle = 0]{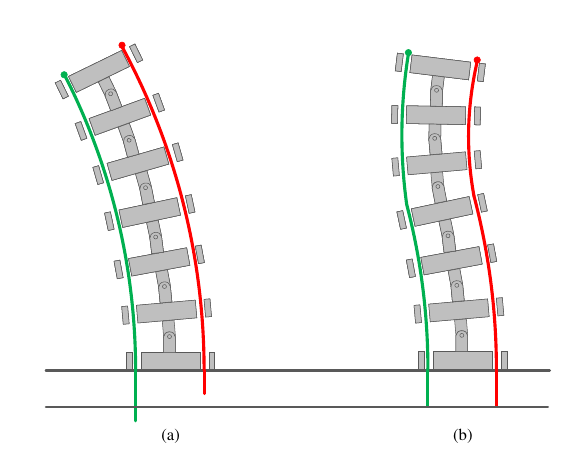}
    \caption{One-dimensional illustration of traditional continuum robots \cite{kim2013stiffness}}
    \label{fig:new}
\end{figure}

\subsection{Problem Set}
\label{sec:2-2}

In this paper, we study how to design a feedback controller that is capable of regulating the continuum robot deformation and achieving the variable stiffness capability. To be precise, the closed loop complies with the input constraint \eqref{constraint:direction} and achieves the following aims:

\begin{itemize}
\item[{\bf A1}:] In the absence of the external torque (\emph{i.e.} $\tau_{\tt ext} = \mathbf{0}_n$), it achieves the asymptotically accurate regulation of the position, that is
\begin{equation}
	\lim_{t\to\infty} \bfq(t)  = \bfq_\star,
\end{equation}
with a desired assignable configuration $\bfq_\star \in \calx$.
\item[{\bf A2}:] We are able to control the stiffness at the closed-loop equilibrium concurrently.
\end{itemize}

%
\section{Open-loop Stiffening}
\label{sec:3}
%

It is widely recognised that the antagonism mechanism is a popular method for stiffening of tendon-driven continuum robots \cite{FATetal}. By arranging a pair of cables on two sides of the robot and adjusting their tensions simultaneously, the robot body compresses or expands, and generates a reaction that counteracts the tension. As a consequence, the stiffness of the robot changes.

In this section, it is shown, via an intuitive example, that there is a redundant degree of freedom of input in the proposed model, which provides the possibility to regulate the stiffness in the open-loop system. This property is instrumental for the controller design to regulate position and stiffness simultaneously.

Now consider the case of the open-loop system with a pair of identical constant inputs 
\begequ
\label{u1u2}
u_1= u_2= \mu.
\endequ
This balance will keep the configuration variable $\bfq$ at the open-loop equilibrium $\bfq_\star=\mathbf{0}_n$ for any feasible $\mu\ge 0$. Especially, when $\mu =0$, the robot is in a slack state where its stiffness corresponds to its inherent properties determined by the materials and mechanical structures. Intuitively, as $\mu$ increases, the manipulator progressively transitions towards a state of higher rigidity or inflexibility. The proposed dynamical model should be capable of interpreting the above phenomenon -- physical common senses tell us that a larger value of $\mu >0$ implies a larger transverse stiffness.

In order to study the stiffness at the end-effector, we need the Jacobian $J(\bfq) \in \rea^{2\times n}$ from the contact force $\bff_{\tt ext}\in \rea^2$, \emph{i.e.}, satisfying
\begin{equation}
 \tau_{\tt ext} = J^\top(\bfq)  \bff_{\tt ext},
\end{equation}
in which $\tau_{\tt ext} \in \rea^n$ is the external torque vector acting on each link.

In the following proposition, we aim to demonstrate the property of stiffness tunability by changing the value $\mu$ in the proposed dynamical model.

\begin{proposition}
\rm \label{prop:1}
Consider the antagonistic tendon-driven model \eqref{model:pH} for continuum robots, with the constant inputs \eqref{u1u2} under Assumption \ref{ass:1}. If the following assumptions are satisfied:

\begin{itemize}
\item[{\bf H1}:] For $j=1,2$
\begin{equation}
\label{eq:diag}
\begin{aligned}
	{\partial G_{n,j} \over \partial q_n} (\bfq)  < 0, \quad 
 {\partial G_{n,j} \over \partial q_k} (\bfq)  = 0, ~k \in \caln \backslash\{n\}
\end{aligned}
\end{equation}
in a neighbourhood of the origin.

\item[{\bf H2}:] The forward kinematics (the mapping from the configuration $\bfq \in \calx$ to the end-effector Cartesian coordinate $\bfx \in \rea^2$) is a \textit{locally} injective immersion. 
\end{itemize}
Then, the input \eqref{u1u2} guarantees the origin $\mathbf{0}_n$ an equilibrium in the absence of the external perturbation, \emph{i.e.} $\tau_{\tt ext}=\mathbf{0}$. Furthermore, a large input value $\mu>0$ implies a larger transverse stiffness $K_{\tt T}$ of the end-effector at this open-loop equilibrium.
\end{proposition}
\begin{proof}
From Assumption \ref{ass:1}(b), the input term 
$$
G(\bfq)\bfu \big|_{\bfq = \mathbf{0}_n} = \mathbf{0}_n
$$ 
with \eqref{u1u2} guarantees that the origin $\bfq_\star = \mathbf{0}_n$ is an equilibrium in the case of $\tau_{\tt ext} = \mathbf{0}$.

The Cartesian coordinate $\bfx \in \rea^2$ of the end-effector can be uniquely determined by the configuration $\bfq$ as 
\begin{equation}
    \bfx= T(\bfq) 
\end{equation}
for some smooth function $T: \rea^n \to \rea^2$, with the open-loop equilibrium $\bfx_\star:= T(\bfq_\star)$. Note that the function $T$ depends on the coordinate selection. Without loss of generality, we assume $\bfx_\star = \mathbf{0}_2$ and the local coordinate of $\bfx :=\col(x_1, x_2)$ are selected as the tangential and the axial directions of the $n$-th link.

For a non-zero constant $\tau_{\tt ext}$, let us denote the \textit{shifted equilibrium} as $(\bar\bfq, \mathbf{0}_n)$ in the presence of the external perturbation, with the corresponding end-effector coordinate 
$
\bar\bfx = T(\bar\bfq).
$

In order to study the transverse stiffness at the end of the robot, we assume that the external force $\bff_{\tt ext}$ is only applied to the $n$-th link \cite{SAL}. Substituting it into the port-Hamiltonian model \eqref{model:pH}, it satisfies the following equations at the shifted equilibrium
\begin{equation}
\label{eq:nabla_U_f}
\begin{aligned}
    \nabla U(\bar\bfq) & ~=~ [G_1(\bar\bfq) + G_2(\bar\bfq)] \mu + J^\top(\bar\bfq)\bff_{\tt ext}
    \\
    \bff_{\tt ext} & ~=~ K_{\tt C}(\bar\bfx - \bfx_\star),
\end{aligned}
\end{equation}
in which $K_{\tt C} \in \rea^{2\times 2}$ is the stiffness matrix with the partition
$$
\begin{aligned}
K_{\tt C} &:= \diag(K_{\tt T}, K_{\tt A})
\end{aligned}
$$
with
$$
J(\bfq)  = \begmat{J_1(\bfq) \\ J_2(\bfq)},
$$
the axial stiffness $K_{\tt A}$ and the transverse stiffness $K_{\tt T}$. 
For the particular coordinate selection as mentioned above, the Jacobian matrix $J_1(\bfq)$ is in the form
\begin{equation}
J_1 = \begmat{ 0 & \ldots & 0 & l_n},
\end{equation}
in which $l_n>0$ represents the distance from the contact point to the centre of the $n$-th link.

Now, we have
\begin{equation}
\label{nabla_U:K_T}
    \nabla U(\bar\bfq) = [G_1(\bar\bfq) + G_2(\bar\bfq)]\mu + K_{\tt T} J_1^\top [\bar x_1 - x_{\star 1}] + J_2^\top[\bar x_2 - {x_\star}_{2}].
\end{equation}
For convenience of presentation and analysis, we define a function $f_\mu : \rea^n \to \rea^n$ as
\begin{equation}
    f_{\mu}(\bfq) := \nabla U(\bfq) - [G_1(\bfq) + G_2(\bfq)]\mu,
\end{equation}
which is parameterised by the constant force $\mu \ge 0$. Invoking the fact that $\bfq_\star = \mathbf{0}_n$ is an open-loop equilibrium, we have 
\begin{equation}
f_\mu(\bfq_\star) =0,
\end{equation}
and thus
\begin{equation}
\label{eq:f_mu'_star}
    f_\mu(\bar\bfq) - f_\mu(\bfq_\star) = K_{\tt T} J_1^\top(\bar\bfq)  [\bar x_1 - x_{\star 1}]  + J_2^\top (\bar\bfq)[\bar x_2 - {x_\star}_{2}].
\end{equation}

Noting the local injectivity assumption {\bf H2}, there exists a left inverse function $T^{\tt L}$ -- which is defined \textit{locally} -- of the function $T$ such that 
\begin{equation}
\bfq = T^{\tt L} (T(\bfq))
\end{equation}
in a small neighborhood of $\bfq_\star$.

Given that the axial deformation due to antagonistic tensions is negligible compared to the transverse deformation, for sufficiently small \(\bar{x}_1 - {x_\star}_1\), we have \(\bar{x}_2 - {x_\star}_2 = o(\bar{x}_1 - {x_\star}_1)\), where $o(\cdot)$ represents terms of higher-order smallnesss. Therefore, when considering $\bar x_1 \to {x_\star}_1$, we can drop the term $J_2^\top (\bar\bfq)[\bar x_2 - {x_\star}_{2}]$ in the limit analysis. From the equation \eqref{eq:f_mu'_star}, the transverse stiffness $K_{\tt T}$ at the open-loop equilibrium $\bfq_\star$ can be defined by taking the limit $\bar\bfq \to \bfq_\star$, \emph{i.e.}
\begin{equation}
\label{eq:K_T}
\begin{aligned}
K_{\tt T} &~=~   {J_1 (\bfq_\star) \over \|J_1(\bfq_\star)\|^2 } \lim_{\bar x_1 \to x_{\star 1}} {  f_\mu(\bar\bfq) - f_\mu(\bfq_\star) \over \bar x_1 - x_{\star1}}
\\
&~=~ {J_1 (\bfq_\star) \over \|J_1(\bfq_\star)\|^2 } [\nabla f_\mu(\bfq_\star)]^\top \nabla_{x_1} T^{\tt L}(\bfx_\star)
\\
&~=~ {J_1 \over \|J_1\|^2 }\big[\nabla^2 U - \mu\big( \nabla G_1  + \nabla G_2 \big)^\top\big]
\nabla_{x_1} T^{\tt L} \bigg|_{\bfq_\star}.
\end{aligned}
\end{equation}

The assumption {\bf H1} guarantees that the $(n,n)$-element of $\nabla G_j(\bfq_\star)$ for $j=1,2$ is negative. On the other hand, from the local coordinate selection, the variation of $x_1$ implies that $q_n$ will also change accordingly. As a consequence, the last element of $\nabla_{x_1} T^{\tt L}$ is non-zero, and indeed, it is positive. It is straightforward to see that $K_{\tt T}$ is increasing by selecting a larger $\mu>0$.
\end{proof}

With the proposed dynamical model, the above calculation shows the underlying mechanism of the tunability of open-loop stiffness. Later on, we will illustrate that the tension difference $(u_1 - u_2)$ provides another degree of freedom to regulate the robot configuration. 

\begin{remark}\rm 
Assumption {\bf H1} means that $G_{n,j}$ only depends on the state $q_n$ rather than other configuration variables. It is used to simplify the presentation and analysis. Indeed, from the above proof, it may be replaced by a weaker condition $J_1(\nabla G_1 + \nabla G_2) \nabla_{x_1} T^{\tt L} \neq 0$, for which we are still able to show the ability to tune the open-loop stiffness via changing the tendon force $\mu$. The assumption {\bf H2} means that we can determine a unique inverse kinematic solution in a small neighborhood of a given configuration $\bfq_\star$. While it is generally not true to ensure the existence of a \textit{global} inverse, achieving it within a local context is feasible.
\end{remark}


\begin{remark}\rm
In practice, when the value $\mu$ is increased beyond a certain threshold, the continuum robot may be observed with the phenomenon of \emph{buckling} \cite[Sec. III]{FANLIU}. However, the critical values are usually very large, and cannot be generated by actuators in many robotic platforms. In this paper, we do not take the buckling behaviour into account.
Also note that the unchanged sign of curvature condition in Assumption \ref{assp:1} roles out the scenario discussed in \cite{kim2013stiffness} where ``the manipulator can be moved to various configurations without changing the wire length.'' As a result, the control input is capable of adjusting the stiffness in our case.
\end{remark}

%
\section{Control Design}
\label{sec:control}
%

In this section, we will study how to design a state feedback law, based on the proposed model in Section \ref{sec2}, to regulate position and stiffness simultaneously.

To facilitate the controller design, we additionally assume the following for the input matrix $G(\bfq)$ in terms of the geometric constraints.

\begin{assumption}\rm \label{ass:G}
The matrix $G(\bfq)$ in \eqref{Gpart} is parameterised as
\begin{equation}
    G_1(\bfq)  =  \bfg_1(\bfq) + \bfg_0 , \quad
    G_2(\bfq)  =  \bfg_1(\bfq) - \bfg_0
\end{equation}
with a constant vector $\bfg_0 \in \rea^n$ and a $C^1$-continuous function $\bfg_1:\rea^n \to \rea^n$ satisfying the following:
\begin{itemize}
    \item[({\bf a})] $\bfg_1(\bfq)$ is a smooth \textit{odd} function;
    \item[({\bf b})] $\bfg_1(\bfq)$ is full column rank for $\bfq\in \calx \slash \{\mathbf{0}_n\}$;
    \item[({\bf c})]
    The constant vector $\bfg_0$ and the vector field $\bfg_1(\bfq)$ can be re-parameterised as
\begin{equation}
\label{eq:g0gq}
    \bfg_0 = g_0 \mathbf{1}_n , \quad \bfg_1(\bfq) = g_1(\bfq) \mathbf{1}_n,
\end{equation}
and $g_0+ g_1(\bfq) \neq 0$ for all $\bfq$.   \qed
\end{itemize}
\end{assumption}

\vspace{.1cm}

Clearly, the above is compatible with Assumption \ref{ass:1}. The vector-valued function $\bfg_1(\bfq)$ is related to the open-loop stiffness tunability outlined in Proposition \ref{prop:1}. At the end of Appendix, we provide some details on how to model the terms $g_0$ and $g_1$.

\subsection{Assignable Equilibria}

For underactuated mechanical systems, it is essential to identify the set of assignable equilibria, also referred to as achievable or feasible equilibria. Although this has been extensively explored, tendon-driven robots face a significant \emph{obstacle} in the form of the one-directional input constraint \eqref{constraint:direction}. 

In order to facilitate the control design, we make the following input transformation:
\begin{equation}
\label{trans:input}
\tau = T_{\tt u} \bfu, \quad  T_{\tt u}:=\begmat{1 & -1 \\ 0 & 1}
\end{equation}
with new input control 
$
\tau= \col(\tau_1,\; \tau_2) \in \rea^2.
$
For $\tau_1$, there is no sign constraint; the other input channel verifies $\tau_2 \ge 0$, and thus we define the admissible input set as
\begin{equation}
\label{E_tau}
\cale_\tau := \{\tau \in \rea^2: \tau_1 \in \rea, ~\tau_2 \ge 0\}.
\end{equation}
Invoking the intuitive idea in Section \ref{sec:3}, we may use these two inputs $\tau_1,\tau_2$ to regulate the position and stiffness concurrently. 

For convenience, we define the new input matrix as
\begin{equation}
\begin{aligned}
 G_\tau(\bfq)  = \begmat{\rho_1(\bfq) & \rho_2(\bfq)} 
=: G(\bfq)T_{\tt u}^{-1}
 \end{aligned}
\end{equation}
with
$$
\begin{aligned}
\rho_1(\bfq)   ~=~ \bfg_0 + \bfg_1(\bfq) 
, \quad
\rho_2(\bfq)   ~=~ 2\bfg_1(\bfq).
\end{aligned}
$$
With the above input transformation, the controlled model \eqref{model:pH} now becomes
\begin{equation}
\label{model:pH2}
	\begmat{ ~\dot \bfq ~\\ ~\dot \bfp~} 
	=
	\begmat{0_{n\times n} & I_n \\ - I_n & -D(\bfq)} \begmat{\nabla	_{\bfq} H \\ \nabla_{\bfp} H }+ \begmat{~\mathbf{0}_{n} ~ \\ ~G_\tau(\bfq) \tau ~}
\end{equation}
in the absence of the external perturbation $\tau_{\tt ext}$.

\begin{remark}\rm
    Indeed, the real constraint for $\tau_1$ should be $\tau_1 \ge -\tau_2$ rather than $\tau_1\in \rea$ in order to guarantee the constraint $u_1,u_2\ge0$. Since we are able to set the value of $\tau_2$ \textit{arbitrarily}, we consider the admissible input set $\cale_\tau$ defined above for convenience in the subsequent analysis.  \qed
\end{remark}

According to \cite{ORTetal} and invoking the full rankness of $T_u$, if \textit{there were not} input constraints, the assignable equilibrium set would be given by
$
\{ \bfq\in \rea^n : G(\bfq)^\bot \nabla U (\bfq) = \mathbf{0} \}.
$
Clearly, this \textit{does not} hold true for our case, because the feasible solution cannot be guaranteed to live within the set $\cale_\tau$ rather than $\tau\in \rea^2$. 

To address this point, in the following proposition we present the assignable equilibria set for the studied case with constrained inputs.

\begin{proposition}\rm \label{prop:assign} (\textit{Assignable Equilibria})
Consider the unperturbed model \eqref{model:pH} and the input transformation \eqref{trans:input} with $\tau_{\tt ext}=\mathbf{0}$ under the input constraint $\cale_\tau$ in \eqref{E_tau}. All the assignable equilibria are given by the set $\cale_q \cap \calx$, with the definition
\begin{equation*}
\begin{aligned}
& \cale_q  := 
 \\
&~~~~ \left\{ \bfq\in \rea^n \Bigg|
 \begin{aligned}
 (\bfg_1^\bot \bfg_0)^\bot \bfg_1^\bot \nabla U =0  
 \\
 \bfg_1^\top \nabla U - \bfg_1^\top(\bfg_0+\bfg_1) (\bfg_1^\bot \bfg_0)^\dagger \bfg_1^\bot \nabla U \ge 0
 \end{aligned}
 \right\}.
 \end{aligned}
\end{equation*}
\end{proposition}

\vspace{0.2cm}

\begin{proof}
In terms of Assumption \ref{ass:G}(\textbf{b}), we can always find a left annihilator $\bfg_1^\bot(\bfq) \in \rea^{(n-1)\times n}$, which is full rank for all $\bfq \in \calx\slash\{\mathbf{0}_n\}$. For an equilibrium $(\bfq,{\bf 0})$, there should exist $\tau_1$ and $\tau_2$ satisfying 
\begin{equation}
\label{id:eq1}
\begin{aligned}
    \nabla U (\bfq) & ~= ~ G_\tau(\bfq)\tau
    \\
    & ~= ~
    [\bfg_0 + \bfg_1(\bfq)]\tau_1 + 2 \bfg_1(\bfq)\tau_2.
\end{aligned}
\end{equation}
Considering the full-rankness of the square matrix $\col(\bfg_1^\bot (\bfq), \bfg_1^\top (\bfq)) \in \rea^{n\times n}$
for $\bfq \in  \calx\slash\{\mathbf{0}_n\}$, we have
\begin{equation}
\eqref{id:eq1} ~\iff ~ \begmat{\bfg_1^\bot  \\ \bfg_1^\top } \nabla U
=
\begmat{ \bfg_1^\bot  \bfg_0 \tau_1 \\ \bfg_1^\top(\bfg_0 + \bfg_1) \tau + 2\|\bfg_1\|^2\tau_2 }.
\end{equation}
Its solvability relies on finding all the points $\bfq\in \calx \subset \rea^n$ satisfying
\begin{align}
    \bfg_1^\bot \nabla U & = \bfg_1^\bot \bfg_0 \tau_1 \label{c1}
    \\
    \bfg_1^\top \nabla U & = \bfg_1^\top(\bfg_0 + \bfg_1) \tau_1 + 2\|\bfg_1\|^2\tau_2  \label{c2}
\end{align}
at the same time under the constraint $\tau \in \cale_\tau$. Clearly, all the feasible equilibria satisfying \eqref{c1} live in the set 
$$
\{\bfq\in \rea^n:(\bfg_1^\bot(\bfq) \bfg_0)^\bot \bfg_1^\bot(\bfq) \nabla U(\bfq) =0  \},
$$
and the corresponding input $\tau_1$ is given by
\begin{equation}
\label{solution:tau1}
    \tau_1  = (\bfg_1^\bot \bfg_0)^\dagger \bfg_1^\bot \nabla U.
\end{equation}

On the other hand, Assumption \ref{ass:G}(\textbf{b}) imposes the constraint $\|\bfg_1\|^2>0$, thus \eqref{c2} admits a positive solution to $\tau_2 \ge 0$ if and only if
\begin{equation}
\bfg_1^\top \nabla U  -  \bfg_1^\top(\bfg_0+\bfg_1) \tau_1 \ge 0.
\end{equation}
Inserting \eqref{solution:tau1} into the above equation completes the proof.
\end{proof}

\vspace{0.1cm}

After imposing Assumption \ref{ass:G}(\textbf{c}) to the input matrix $G(\bfq)$, we are interested in a class of particular equilibria. In this paper, we call them the \textit{homogeneous equilibria} that are characterised by the set
\begin{equation}
 \cale_\theta :=\{ \bfq\in \rea^n :  q_i =\theta, ~ \forall i\in \caln \}
\end{equation}
for some constant $\theta$. This definition is tailored for the proposed continuum robot model under the assumptions in the paper.

In the following, we show all homogeneous equilibria belong to the assignable equilibrium set $\cale_q$ in Proposition \ref{prop:assign}.

\begin{proposition}\rm 
Consider the model \eqref{model:pH} of the continuum robot under Assumptions \ref{ass:G}-\ref{ass:Ue}. Then, all homogeneous equilibria are assignable, \emph{i.e.} $\cale_\theta \subset \cale_q$.
\end{proposition}
\begin{proof}
For the case with $\theta = 0$, since $g_1(\mathbf{0}_{n}) =0$, the equilibrium $\bfq_\star= \theta \mathbf{1}_n$ makes the equation \eqref{id:eq1} solvable with $\tau_1 = 0$ and any $\tau_2  \ge 0$.

For the case with $\theta \neq 0$ and any \textit{fixed} $\tau_1\ge 0$, the determination of the set $\cale_q$ is equivalent to solving \eqref{id:eq1}, which can be written as
$
\nabla U(\bfq)  =  [\bfg_0 + \bfg_1(\bfq)]\tau_1 + 2 \bfg_1(\bfq)\tau_2.
$
We compactly formulate the above as
\begin{equation}\label{pde:tau_n}
\begin{aligned}
        \nabla U(\bfq) = G_{\tt N}\tau_{\tt N}
\end{aligned}
\end{equation}
with the new definitions
\begin{equation}
\label{GnTn}
\begin{aligned}
    G_{\tt N}  := \mathbf{1}_n
    , ~
    \tau_{\tt N}(\tau)  := [g_0+g_1(\bfq)]\tau_1 + 2g_1(\bfq) \tau_2.
\end{aligned}
\end{equation}
For any fixed $\tau_2\ge 0$, invoking \eqref{eq:g0gq} from Assumption \ref{ass:G}, the mapping $\tau_1 \to \tau_{\tt N}$ is a diffeomorphism from $\rea \to \rea$. It implies that there is no constraint for $\tau_{\tt N}$. As a consequence, the PDE \eqref{pde:tau_n} becomes 
\begin{equation}
\label{GN_nabla_U}
G_{\tt N}^\bot \nabla U(\bfq_\star) =0
\end{equation}
at the desired equilibrium $\bfq_\star$.

A feasible full-rank annihilator of $G_{\tt N}$ is given by
\begin{equation}
    G_{\tt N}^\bot = \begmat{1& -1 & 0& \ldots & 0 \\ 0 & 1& -1 & \ldots &0 \\  && \ddots&\ddots  \\ 0 & \ldots& 0& 1& -1} \in \rea^{(n-1)\times n},
\end{equation}
and the Jacobian $\nabla U $ at the desired equilibrium $\bfq_\star$ is in the form
\begin{equation}
    \nabla U (\bfq_\star) = \alpha_1 \underbrace{\begmat{~\sin(q_\Sigma)~ \\ \vdots\\ ~\sin(q_\Sigma)~}}_{ \sin(n\theta) \mathbf{1}_n }
    + \alpha_2 \underbrace{\begmat{~\bfq_{\star,1} ~\\ ~\vdots~ \\ ~\bfq_{\star,n} ~}}_{\theta \mathbf{1}_n}
\end{equation}
with 
$
q_\Sigma := \sum_{i\in \caln} \bfq_i.
$
It is straightforward to verify that \eqref{GN_nabla_U} holds true for any $\tau_2\ge 0$ with a homogeneous equilibrium $\bfq_\star = \theta\mathbf{1}_n$. Since there is no constraint for the input variable $\tau_1$, the equilibrium for this case is also assignable under the constraint \eqref{E_tau}. We complete the proof.
\end{proof}

In the sequel of the paper, our focus will be on control algorithm design aimed at regulating certain homogeneous equilibria that have been demonstrated to be assignable within the proposed class of models for continuum robots.

\subsection{Simultaneous Position-and-Stiffness Control}

We now aim at stabilising an arbitrary homogeneous equilibrium $\bfq_\star$ in the subset of $\cale_\theta$ with a tunable stiffness of the closed loop. To the end, we employ the passivity-based control (PBC) method since it has a clear energy interpretation and simplifies both modelling and controller design. This makes it suitable for continuum robots to preserve the system compliance \cite{FRAGAR}.

Our basic idea is to fix $\tau_2$ at some constant value $\tau_2^{\star}\ge 0$. We then utilise the input $\tau_1$ to achieve potential energy shaping for the regulation task. Compared to the more general approach of interconnection and damping assignment (IDA) PBC \cite{ORTetalAUT}, on one hand, potential energy shaping may provide a simpler controller form, and on the other hand, as pointed out in \cite{KEPetal} changing the inertia is prone to fail in practice -- albeit being theoretically sound with additional degrees of freedom.  

For a given input $\tau_2 = \tau_2^\star \ge 0$, the actuation into the dynamics is given by
\begin{equation}
\label{tau_N}
\begin{aligned}
    G_\tau(\bfq) \tau & ~= ~\rho_1(\bfq) \tau_1 + \rho_2(\bfq) \tau_2^\star 
    \\
    &:=~ G_{\tt N} \tau_{\tt N}(\col(\tau_1, \tau_2^\star)),
\end{aligned}
\end{equation}
with the function $\tau_{\tt N}: \rea^2 \to \rea$ defined in \eqref{GnTn}. From Assumption \ref{ass:G}, the vector field $\rho_1(\bfq) \neq 0$ for all $\bfq \in \calx$. Now the design target becomes using the control input $\tau_1$ (with a fixed $\tau_2^\star$) to shape the potential energy function $U(\bfq)$ into a new one -- the desired potential energy function $U_{\tt d}(\bfq)$. To this end, we need to solve the PDE \cite{ORTetal}
\begin{equation}
\label{pde:shaping}
G_{\tt N}^\bot \big[ \nabla U(\bfq) - \nabla U_{\tt d}(\bfq) \big] = 0.
\end{equation}
Note that the solution to the function $U_{\tt d}$ must adhere to the constraints
\begin{align}
    \nabla U_{\tt d}(\bfq_\star) & ~=~ 0 
    \\
    \nabla^2 U_{\tt d}(\bfq_\star) & ~\succ ~ 0,
\end{align}
in order to make the desired configuration $\bfq_\star$ an asymptotically stable equilibrium.

We are now in the position to propose the controller for simultaneous control of position and stiffness.

\begin{proposition}
\label{prop:control}\rm
Consider the continuum robotic model \eqref{model:pH}, \eqref{trans:input} with the constraint \eqref{E_tau} satisfying Assumptions \ref{ass:1}-\ref{ass:Ue}, and the full-rank damping matrix $D(\bfq)$ is uniformly positive definite. The feedback controller 
\begin{equation}
\label{u:control}
    u = T_u^{-1} \tau
\end{equation}
with the transformed input
\begin{equation}
\label{tau+}
    \tau =  \tau_{\tt es}+ \tau_{\tt da}  + \tau_{\tt st} 
\end{equation}
and the terms
\begin{equation}
\label{tau:3}
\begin{aligned}
    \tau_{\tt st} & = \begmat{
   -{ 2g_1(\bfq) \over g_0+g_1(\bfq)}  \\  1
     }\tau_2^\star
     \\
    \tau_{\tt es} & = \begmat{
   {1\over g_0+g_1(\bfq)} G_{\tt N}^\dagger (\nabla U_{\tt d} - \nabla U)  \\  0
     }
     \\
     \tau_{\tt da} & = \begmat{ - {1\over g_0+g_1(\bfq)} G_{\tt N}^\top K_{\tt d} M^{-1}(\bfq)\bfp  \\ 0},
\end{aligned}
\end{equation}
where $G_{\tt N} = \mathbf{1}_n$, $\tau_2^\star >0$, $K_{\tt d} \succ 0$ is a gain matrix, and the desired potential energy function is given by
\begin{equation}
\label{Ud}
\begin{aligned}
    U_{\tt d}(\bfq) = - \gamma \cos(q_\Sigma - q_{\Sigma}^\star) + {\alpha_2 \over 2}\|\bfq - \bfq_\star\|^2
    ,~
    q_{\Sigma}^\star = \sum_{i \in \caln} \bfq_{\star,i},
\end{aligned}
\end{equation}
the gain $\gamma>0$, and some desired regulation configuration $\bfq_\star \in \cale_\theta$, achieves the following closed-loop properties:
\vspace{.5em}
\begin{itemize}
  \setlength\itemsep{1em}
    \item[{\bf P1}:] (\textit{Position regulation in free motion}) If the external force $\tau_{\tt ext} =0$ and $\gamma <\alpha_2 $, then the desired equilibrium point $\bfq_\star$ is globally asymptotically stable (GAS) with
    \begin{equation}
        \lim_{t\to + \infty} \bfq(t) = \bfq_\star.
    \end{equation}

    \item[{\bf P2}:] (\textit{Compliant behavior}) The overall closed-loop stiffness (\emph{i.e.}, from the external torque $\tau_{\tt ext} \in \rea^n$ to the configuration $\bfq \in \rea^n$) is 
\begin{equation}
\label{K_O}
    K_{\tt O} = \gamma \mathbf{1}_{n\times n} + \alpha_2 I_n,
\end{equation}
where $\mathbf{1}_{n\times n}$ is an $n\times n$ matrix of ones.
\qed
\end{itemize}
\end{proposition}


\begin{proof}
First, it is straightforward to verify that the vector $\tau_{\tt st}$ in \eqref{tau:3} is in the null space of $\tau_{\tt N}(\tau)$ for any $\tau_2^\star$, \emph{i.e.},
\begin{equation}
\label{tntst=0}
\tau_{\tt N}(\tau_{\tt st}) = 0, \quad \forall \tau_2^\star \in \rea_{\ge 0}.
\end{equation}
Hence, the term $\tau_{\tt st}$ does not change the closed-loop dynamics. 

Now, let us study the effect of the potential energy shaping term $\tau_{\tt es}$. The Jacobian of the desired potential energy function $U_{\tt d}$ is given by
\begin{equation}
    \nabla U_{\tt d}(\bfq) = \gamma \sin(q_\Sigma - q_\Sigma^\star) \mathbf{1}_n + \alpha_2(\bfq - \bfq_\star).
\end{equation}
It satisfies the following:
\begin{align}
\label{Jacobian:Ud}
    \nabla U_{\tt d} (\bfq_\star) & ~=~0 
    \\
    \nabla^2 U_{\tt d} (\bfq) & ~=~  \gamma \cos(q_\Sigma - q_\Sigma^\star)\mathbf{1}_{n\times n} 
    + \alpha_2 I_n \succ 0,
\end{align}
where the second line holds true for all $\bfq\in \rea^n$ by noting that the eigenvalues of the symmetric matrix $\nabla^2 U_{\tt d}$ are given by
$$
\{\underbrace{\alpha_2, \ldots, \alpha_2}_{n-1}, \alpha_2 + \gamma  |\cos(q_\Sigma - q_\Sigma^\star)|\}
$$
with all elements positive from the condition $\gamma < \alpha_2$ in {\bf P1}. This implies that the desired potential energy function $U_{\tt d}$ is convex and achieves its global minimum at $\bfq_\star$.

For the function $U_{\tt d}$, we have
\begin{equation*}
\begin{aligned}
        & G_{\tt N}^\bot [\nabla U - \nabla U_{\tt d}] 
        \\
        ~=~ & G_{\tt N}^\bot [ \alpha_1 \sin(q_\Sigma)\mathbf{1}_n + \alpha_2 \bfq - \gamma \sin(q_\Sigma - q_\Sigma^\star) \mathbf{1}_n \\
        & \quad ~ - \alpha_2 (\bfq -
 \bfq_\star ) ]
        \\
        ~=~ & G_{\tt N}^\bot [ \alpha_1 \sin(q_\Sigma)\mathbf{1}_n  - \gamma \sin(q_\Sigma - q_\Sigma^\star) \mathbf{1}_n + \alpha_2 \bfq_\star  ]
 \\
 ~=~& 0,
\end{aligned}
\end{equation*}
where in the last equation we have used the fact $\bfq_\star \in \cale_\theta$, so that  the PDE \eqref{pde:shaping} is verified. Together with \eqref{tntst=0}, the controller \eqref{tau+} makes the closed-loop dynamics take the form
\begin{equation}
\label{cl_dyn:ext}
\begmat{ ~\dot \bfq ~\\ ~\dot \bfp~} 
	=
	\begmat{0_{n \times n} & I_n \\ - I_n & -\mathsf{D}} \begmat{ \nabla_\bfq H_{\tt d} \\  \nabla_\bfp H_{\tt d}} + \begmat{~\mathbf{0}_n ~ \\ ~ \tau_{\tt ext}~},
\end{equation}
with 
\begin{equation}
\begin{aligned}
    H_{\tt d}(\bfq,\bfp) & ~:=~ {1\over2} \bfp^\top M^{-1}(\bfq) \bfp + U_{\tt d}(\bfq)
    \\
    \mathsf{D}(\bfq) & ~:=~ D(\bfq) + G_{\tt N} K_{\tt d} G_{\tt N}^\top \succ 0.
\end{aligned}
\end{equation}
We use the function $\mathsf{D}(\bfq)$ to represent the closed-loop damping. For free motion (\emph{i.e.} $\tau_{\tt ext} = 0$), following the standard Lyapunov analysis we have
\begin{equation}
    \dot {H}_{\tt d} \le -(\nabla H_{\tt d})^\top  \mathsf{D}(\bfq) \nabla H_{\tt d} \le 0.
\end{equation}
For the closed-loop system \eqref{cl_dyn:ext}, the set
$
\{(\bfq, \bfp): (\nabla H_{\tt d})^\top  \mathsf{D}(\bfq) \nabla H_{\tt d} =0\}
$
contains only a single point $(\bfq_\star, \mathbf{0}_n)$. According to LaSalle's invariance principle \cite[Sec 4.2]{KHA}, we are able to show the global asymptotic stability of the desired equilibrium $(\bfq_\star, \mathbf{0}_n)$. Hence, we have proven {\bf P1}.

The next step is to verify the stiffness property {\bf P2} in a small neighborhood $B_\varepsilon(\bfq_\star)$ of $\bfq_\star$ with a sufficiently small $\varepsilon>0$. For a constant external force $\tau_{\tt ext}$, the shifted equilibrium $(\bar\bfq,\mathbf{0}_n)$ should satisfy
\begin{equation}
    - \nabla_\bfq U_{\tt d} + \tau_{\tt ext} =0,
\end{equation}
or equivalently
\begin{equation}
\label{phi:q}
  \phi(\bar\bfq):=  \gamma \sin(\bar q_\Sigma - q_\Sigma^\star) \mathbf{1}_n + \alpha_2 (\bar\bfq - \bfq_\star) = \tau_{\tt ext},
\end{equation}
with the definition
$
\bar q_\Sigma:=\sum_{i \in \caln} \bar \bfq_i.
$
Note that $\nabla \phi = \nabla^2 U_{\tt d} \succ 0$, which means that $
\phi: \bfq\mapsto \tau_{\tt ext}
$
is a (locally) injective immersion. Hence, in the small neighborhood $B_\varepsilon(\bfq_\star)$ of $\bfq_\star$, there is a \emph{unique} solution $\bar\bfq$ to \eqref{phi:q} for a given $\tau_{\tt ext}$.

We show that the shifted equilibrium $(\bar\bfq, \mathbf{0}_n )$ is asymptotically stable by considering the Lyapunov function
\begin{equation}
	V(\bfq,\bfp) = H_{\tt d}(\bfq,\bfp) - \bfq^\top \tau_{\tt ext}.
\end{equation}
From the above analysis, it is clear that
\begin{equation}
\begin{aligned}
    \nabla V(\bar\bfq,\mathbf{0})  = \begmat{\phi(\bar\bfq) - \tau_{\tt ext} \\ \mathbf{0}}  = \mathbf{0}
    , \quad
    \nabla^2 V(\bar\bfq,\mathbf{0}) \succ 0.
\end{aligned}
\end{equation}
Hence, $V$ qualifies as a Lyapunov function. Its time derivative along the system trajectory is given by
\begin{equation}
    \begin{aligned}
        \dot V &~=~ - \|\nabla_\bfp H_{\tt d}\|_{\mathsf{D}}^2 - (\nabla_\bfp H_{\tt d})^\top \tau_{\tt ext} - \dot \bfq^\top \tau_{\tt ext}
        \\
        &~= ~ - \|\nabla_\bfp H_{\tt d}\|_{\mathsf{D}}^2
        ~\le ~ 0.
    \end{aligned}
\end{equation}
It yields the Lyapunov stability of the closed-loop dynamics \eqref{cl_dyn:ext} in the presence of a constant external torque $\tau_{\tt ext}$, and all the system states are bounded. On the other hand, the set 
$
\cale_{u}:= \{(\bfq,\bfp):\|\nabla_\bfp H_{\tt d}(\bfq)\| = \mathbf{0} \}
$
only contains a single isolated equilibrium $(\bar\bfq,\mathbf{0})$. According to LaSalle's invariance principle, $(\bar\bfq,\mathbf{0})$ is an asymptotically stable equilibrium, in which $\bar\bfq$ depends on the (arbitrary) constant torque $\tau_{\tt ext}$ -- in terms of the unique solution to the algebraic equation \eqref{phi:q}.

The overall stiffness $K_{\tt O}$ is defined by $\tau_{\tt ext} =  K_{\tt O}(\bar\bfq-\bfq_\star)$. Substituting it into \eqref{phi:q}, we have
\begin{equation}
     \gamma \sin(\mathbf{1}_n^\top \bar\bfq - q_\Sigma^\star) \mathbf{1}_n + \alpha_2 (\bar\bfq - \bfq_\star) =  K_{\tt O}(\bar\bfq-\bfq_\star).
\end{equation}
The stiffenss at the desired equilibrium $\bfq_\star$ is calculated from any direction of the limit $\bar\bfq\to \bfq_\star$, thus obtaining
\begin{equation}
\begin{aligned}
    K_{\tt O}\Big|_{\bfq_\star} & ~=~ {\partial \phi\over \partial \bfq}(\bfq_\star)
    \\
    & ~=~ \lim_{\bar\bfq\to \bfq_\star} \gamma \cos(\mathbf{1}_n^\top \bar\bfq - q_\Sigma^\star) \mathbf{1}_{n\times n} + \alpha_2 I_n
    \\
    & ~=~ \gamma \mathbf{1}_{n\times n} + \alpha_2 I_n.
\end{aligned}
\end{equation}
This verifies the property {\bf P2}, and we complete the proof.
\end{proof}

The above shows that the proposed controller \eqref{u:control}-\eqref{tau:3} can achieve the position regulation with the closed-loop stiffness $K_{\tt O}$ given by \eqref{K_O}. It implies our ability to set a prescribed stiffness by selecting the control gain $\gamma>0$ properly. More discussions are provided in the next section.

%
\section{Discussions}
\label{sec:5}
%

The following remarks about the proposed controller are given in order.

\begin{itemize}
    \item[1)]
In {\bf P2}, we study the \textit{overall} stiffness -- from the external torque vector $\tau_{\tt ext}$ to the configuration $\bfq \in \calx$, rather than the transverse stiffness at the end-effector. Consider the external force $f_{\tt ext}\in \rea$ at the end-effector along the transverse direction of the $n$-th link with the Jacobian $J = [0,\ldots, 0, \ell]$. With a \textit{small} force $f_{\tt ext}$, the coordinate of the end-effector would shift from
$
\left(
\ell\left(\sum_{k \in \caln}\sin(k\theta_\star)\right),
\ell  \left(\sum_{k \in \caln}\cos(k\theta_\star)\right)
\right)
$
to
$
(F_x( f_{\tt ext}), F_y(f_{\tt ext}))
$
with
$$
\begin{bmatrix}  ~F_x ~\\ ~F_y~ \end{bmatrix}
 = 
\begin{bmatrix}
    \ell \sin(\beta) + \ell\sum_{k \in \caln\backslash \{n\}}\sin(k\theta_\star + \gamma\ell f_{\tt ext})
    \\
    \ell \cos(\beta) + \ell\sum_{k \in \caln\backslash \{n\}}\cos(k\theta_\star + \gamma\ell f_{\tt ext})
\end{bmatrix}
$$
and $\beta := n\theta_\star + (\gamma + \alpha_2) \ell f_{\tt ext}$. Hence, the transverse stiffness is given by 
\begin{equation*}\small
    \begin{aligned}
        K_{\tt T} = \lim_{ f_{\tt ext} \to 0 }
        { \sqrt{[F_x(f_{\tt ext}) - F_x(0)]^2 + [F_y(f_{\tt ext}) - F_y(0)]^2} \over  f_{\tt ext}} .
    \end{aligned}
\end{equation*}
As a result, we have
\begin{equation}
\label{K_T:prop}
\boxed{
  ~ K_{\tt T} ~\propto~  \kappa_1\gamma + \kappa_2 ~
   }
\end{equation}
with some non-zero constants $\kappa_1, \kappa_2$ for $\theta_\star \neq 0$. This important affine relationship will be experimentally verified in the next section. It means that for a given desired equilibrium $\bfq_\star \in \calx\backslash\{\mathbf{0}\}$, the transverse stiffness is affine in the gain $\gamma$, thus providing a way to tune the closed-loop stiffness \textit{linearly}. 


\item[2)]  

The proposed controller can be roughly viewed as a nonlinear PD controller. The first term $\tau_{\tt st}$ is used to compensate the ``anisotropy'' in the input matrix $G(\bfq)$ due to its state-dependency property; the potential energy shaping term $\tau_{\tt es}$ and the damping injection term $\tau_{\tt da}$, indeed, play the role of nonlinear PD control. To be precise, the term $\tau_{\tt es}$ is the error between the nonlinear functions of the position $\bfq$ and its desired value $\bfq_\star$; and the term $\tau_{\tt da}$ can be viewed as the negative feedback of velocity errors. This is not surprising, since the original idea of energy shaping has its roots in the pioneering work of Takegaki and Arimoto in robot manipulator control \cite{TAKARI}, in which they proposed a very well-known ``PD + gravity compensation'' feedback \cite{ORTetal01}.


\item[3)] 

To ensure that $\nabla^2 U_{\tt d}(\bfq_\star) \succ 0$, it is necessary to impose the condition $\gamma< \alpha_2$ on the control gains. However, this condition may restrict the range of closed-loop stiffness values within an interval. If this condition is not imposed, it is only possible to guarantee the positive definiteness of $\nabla U_{\tt d}$ in the vicinity of $\bfq_\star$, which would result in \emph{local} asymptotic stability. We provide some experimental evidence regarding this point in the next section.


\item[4)] 

Let us now look at the proposed controller \eqref{u:control}-\eqref{tau+}. Note that the term $M^{-1} (\bfq)\bfp$ corresponds to the generalised velocity of $\bfp$. Thus, the controller depends on only three \textit{plant parameters} ($\alpha_1, \alpha_2$ and $g_0$) and a nonlinear function $g_1$ -- which need to be identified in advance -- along with two adaptation gains (\textit{i.e.} $K_{\tt d}$ and $\gamma$). This means that it is unnecessary to identify all parameters and functions in the plant model. This makes the resulting controller robust \textit{vis-\`a-vis} different types of uncertainties. 

\item[5)] Continuum robots inherently admit infinite degrees of freedom, and thus increasing the link number $n \in \mathbb{N}_+$ in the rigid-link model will enhance precision. On the other hand, a higher dimension $n$ will bring computational challenges to obtain the real-time detection/estimation of the configuration $\mathbf{q}$. Consequently, it is necessary to make a tradeoff between accuracy and computation burden, for the selection of the number $n$.
\end{itemize}


%
\section{Experimental Results}
\label{sec:6}
%

\subsection{Experimental setup}
\label{sec:61}

\begin{figure}[!htp]
    \centering
    \includegraphics[width = 0.75\linewidth]{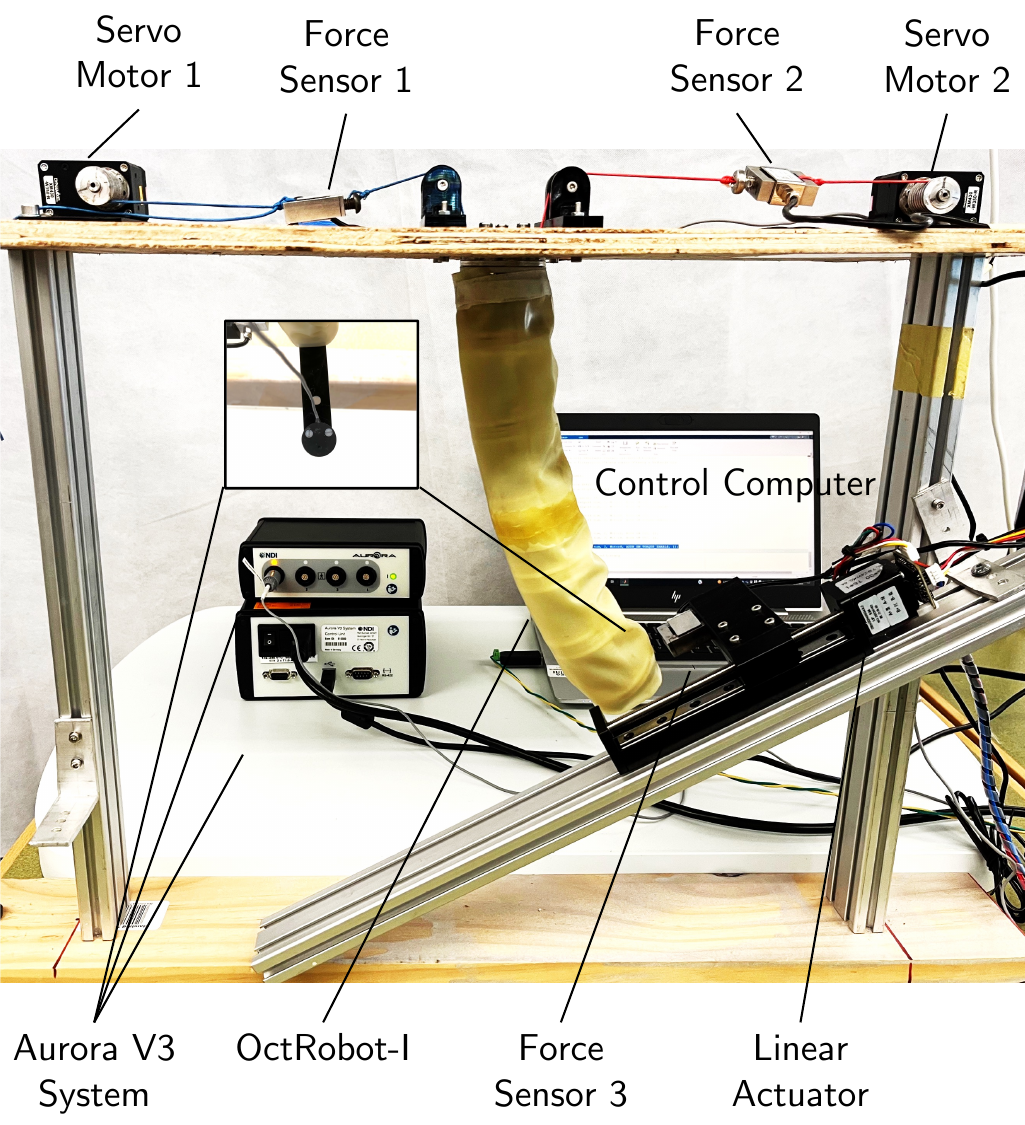}
    \caption{Photo of the entire experimental platform}
    \label{fig:platform}
\end{figure}

The proposed approach was tested using the OctRobot-I, a continuum robot developed in our lab at the University of Technology Sydney \cite{FANLIU}. We considered the planar case with six segments (\emph{i.e.} $n=6$) with an overall length of 252 mm, and a diameter of approximately 50 mm, which meets the critical assumptions outlined in the paper. Notably, the OctRobot-I has a jamming sheath that can provide an extra degree of freedom for stiffening, though the present paper does not delve into this feature's stiffening capabilities. Additional details of the OctRobot-I can be found in \cite{FANLIU}.

As shown in Fig. \ref{fig:platform}, the test platform used in the experiments consists of the one-section robot (OctRobot-I), two servo motors (XM430-W350, DYNAMIXEL) with customized aluminum spools, three force sensors (JLBS-M2-10kg), a linear actuator, and an electromagnetic tracking system (Aurora V3, NDI). The control experiments and data collections were conducted using the software MATLAB\textsuperscript{\texttrademark}. We installed the Aurora sensor at the distal point of the robot to provide its real-time coordinate $z_e$. In the position regulation tasks, no external load was applied to the robot, and we observed that it nearly satisfied the constant curvature condition with the approximation $q_i = q_j ~(i,j \in \caln)$ available.\footnote{Note that in the theoretical analysis, we do not assume $q_i = q_j ~(i,j\in \caln)$.} Together with the coordinate $z_e$ and some basic geometric relations, we are able to estimate the configuration vector $\bfq$ in real-time. In our experimental setup, where we assumed $q_i=q_j$, we use $\theta(t)$ to represent the estimated value $q_i$ in the sequel of this section.

The servo motors in the platform can provide accurate position information with high accuracy, making it easier to control cable lengths between the servo motors and the actuator unit. Using Hooke's law, it is possible to consider the cable length proportional to the force for each cable, with a few coefficients to be identified off-line using collected data sets. To verify the linear relationship between the cable length and the tension force, as well as to obtain the coefficients, we conducted a group of experiments with different configurations and recorded the cable lengths and the corresponding forces. Each configuration was repeatedly conducted three times under identical conditions, and all the data were utilized for identification. In Fig. \ref{fig:lvsf}, we plot the relation between the right cable length $L_2$ and the corresponding force, and the one between the length difference $\Delta L:= L_1 - L_2$ and the force difference $\tau_1 := u_1 - u_2$ of these two cables. The correlation coefficients are 0.9977 and 0.9987, which imply the strong linearity between cable lengths and forces. Thus, it is reasonable to use the cable lengths -- driven by motors -- as the ``real'' input signals. Note that the above-mentioned linearity only holds in static or low-speed conditions. To satisfy this, we used a high gain for the cable length control loop to yield a very short transient stage.

The force sensors were used in the open-loop stiffening experiments to provide the real-time force signals, and then we were able to study the relation between the value $\mu$ and the open-loop stiffness. Additionally, these sensors have proven important for examining the relation between cable lengths and applied forces as mentioned above. However, in closed-loop control, we removed force sensors on the platform and directly regulated the cable lengths. Note that these sensors may cause significant inertial disturbances to the loop.

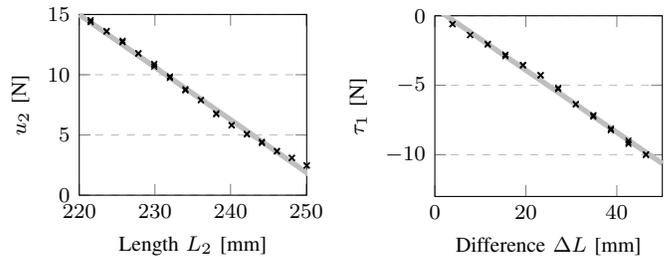
\begin{figure}
    \centering
    \import{fig/}{plot_force_vs_length.tex}
    \caption{The linearity between forces and lengths: The length $L_1$ vs the force $u_1$ of the right cable; and the length difference $\Delta L:= L_1 - L_2$ vs the force difference $\tau_1 := u_1 - u_2$ (``$\times$'' represents test data, and the dash lines are the fitted functions.)}
    \label{fig:lvsf}
\end{figure}

\subsection{Open-loop stiffening experiments}

In this subsection, we aim to validate the results regarding open-loop stiffening presented in Section \ref{sec:3}. For this purpose, we utilised a linear actuator placed at the end-effector to generate a small displacement $\delta x > 0$, as illustrated in Fig. \ref{fig:OpenStiff}(a). The actuator was connected to the force sensors for measuring the external force, denoted as $f_{\tt ext}$, in relation to the displacement. By calculating the ratio of the measured force to the applied displacement, \emph{i.e.}, ${f_{\tt ext} \over \delta x}$, we were able to estimate the transverse stiffness, given that $\delta x$ was sufficiently small. This procedure allowed us to verify the findings related to open-loop stiffening as outlined in Section \ref{sec:3}.
\begin{figure}[!htp]
 \centering
\setkeys{Gin}{width=0.35\linewidth}
\subfloat[Experimental setup]{\includegraphics[width = 0.28\linewidth]{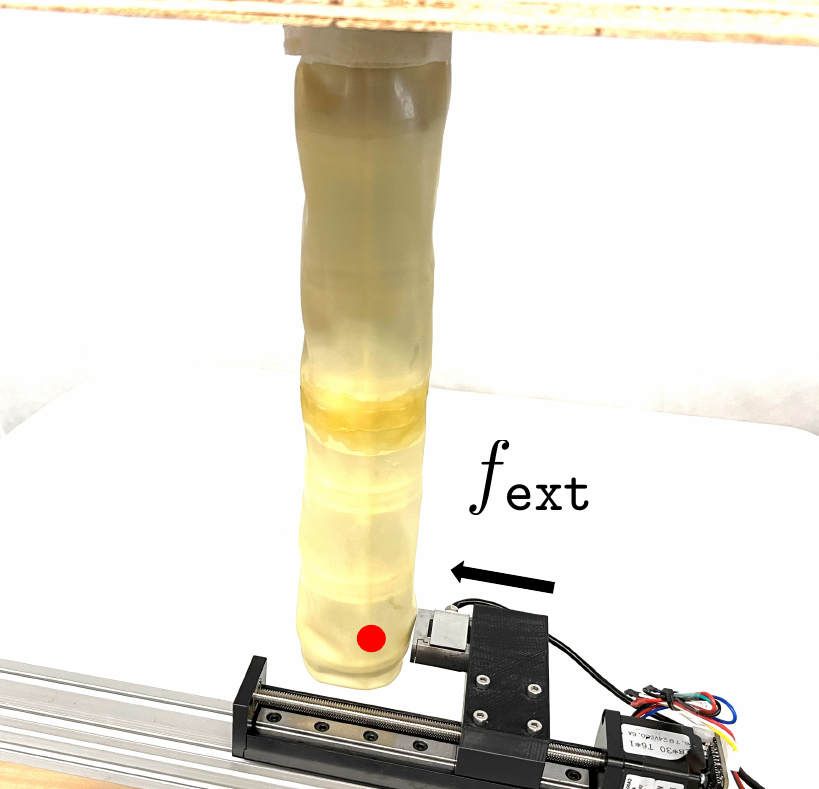}}%
\hfill
\subfloat[Experimental results of transverse stiffness]{\import{fig/}{plot_fig1.tex}}
\hfill
\caption{Experiments for open-loop stiffening}
\label{fig:OpenStiff}
\end{figure}

We measured the stiffness values under different open-loop tendon forces $\mu>0$ in the interval $[0,45]$ N. Each experiment was repeatedly conducted three times under the same conditions in order to improve reliability. The experimental results are shown in Fig. \ref{fig:OpenStiff}(b), where ``$\times$'' represents the mean values of the calculated stiffness for all $\mu$, and the error bars are $\pm 1$ standard deviation. This clearly verifies the theoretical results in Proposition \ref{prop:1}. The correlation coefficient between $\mu$ and the stiffness is 0.989, which illustrates the strong linearity -- exactly coinciding with the equation \eqref{eq:K_T}. 

\subsection{Closed-loop experiments}

In order to apply the proposed real-time control algorithm to the experimental platform, we first conducted the identification procedure to estimate the parameters outlined in the fourth discussing point in Section \ref{sec:5}. It relies on the fact that at any \emph{static} configuration $\bfq$ (\textit{i.e.} with $\bfp =0$) the identity $\nabla_\bfq U(\bfq) = G(\bfq) \bfu$  holds true, and thus $J(\theta) =0$ with the cost function
\begin{align}
  & J(\theta,u_1,u_2) \\
  & \hspace{.3cm} := \Big|\alpha_1 \sin(n\theta) + \alpha_2\theta - [g_0 +g_1(\theta)] u_1 + [g_0 - g_1(\theta)]u_2\Big|^2,\nonumber
\end{align}
that contains all the quantities to be identified. According to the modelling procedure in Appendix, we simply parameterised $g_0 + g_1(\theta) = c_1 + c_2 \sin(\theta)$ with two constants $\alpha_1>0$ and $\alpha_2<0$, complying with the assumptions on the input matrix. We regulated the continuum robot to different equilibria $\theta^j$ ($j=1,\ldots, w$ with some $w\in \mathbb{N}_+$) by driving the cables, and recorded the corresponding forces $(u_1^j, u_2^j)$. 

The identification procedure boils down to the optimisation 
\begin{equation}
    \underset{c_1, \alpha_1,\alpha_2>0, c_2<0}{\arg\min} ~\sum_{j\in \{1,\ldots, w\}}  J(\theta^j, u_1^j, u_2^j) .
\end{equation}
We ran the identification experiments to collect data at 15 equilibria points (\emph{i.e.} $w = 15$) and repeated for six times. Using this data set, the identified parameters were
$
    c_1= 1.2143,  c_2 = -2.9015,  \alpha_1 = 8.6114
$
and $\alpha_2 = 0.001$.

To evaluate the performance of position control, we first considered a desired configuration $\mathbf{q}_\star = [\theta_\star, \ldots, \theta_\star]^\top$ with $\theta_\star = 5$ deg for the proposed control scheme. We conducted experiments for the cases without external forces under various values of the gains $\gamma$ and $K_{\tt d}$, as shown in Figs. \ref{fig:ctrl_agl1_Kd}-\ref{fig:ctrl_agl1_gamma}, respectively. The second row of these figures depicts the configuration variable at the steady-state stage during $[2,8]$ s. It is worth noting that the control inputs $u_i$ ($i=1,2$) are mapped to the cable length $L_i$, as explained in Section \ref{sec:61}. In all these scenarios, the transient stages lasted for less than 1.5 seconds, and the configuration variable quickly converged to small neighborhoods of the desired angle, demonstrating the high accuracy of the proposed control approach. There were no apparent overshooting in configuration variables. Our results indicate that selecting either a sufficiently small or large $K_{\tt d}$ can negatively affect the control performance during the transient stage. On the other hand, setting a large $\gamma>0$ may lead to chattering due to measurement noise at the steady-state stage, which is well understood as the deleterious effect of high-gain design in the control literature \cite{ORTetal01}. 

\begin{figure*}[!htp]
    \centering
    \includegraphics[width = 0.95\textwidth, angle = 0]{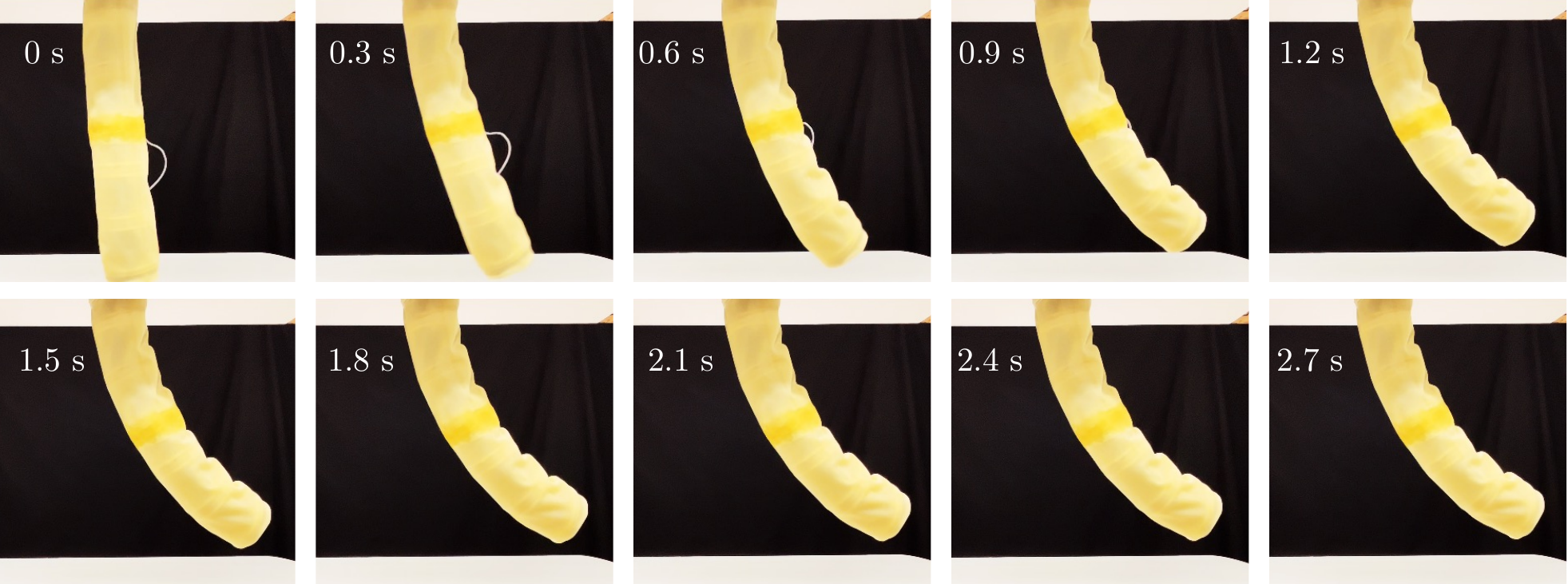}
    \caption{Photo sequence of the position control with $\theta_\star = 10$ deg (overall bending as 60 deg) and $\gamma=1, K_{\tt d}=0.1$}
    \label{fig:photo_sequence}
\end{figure*}

\begin{figure*}[!htp]
 \centering
\setkeys{Gin}{width=0.23\linewidth}
\subfloat{\import{FigPlot/}{CtrlAgl_5deg_a.tex} }
\hfill
\subfloat{\import{FigPlot/}{CtrlAgl_5deg_b.tex}}
\hfill
\subfloat{\import{FigPlot/}{CtrlAgl_5deg_c.tex}}
\hfill
\subfloat{\import{FigPlot/}{CtrlAgl_5deg_d.tex}}
\hfill

\vspace{-0.5 cm}

\subfloat{\import{FigPlot/}{5deg_ss-1.tex}}
\hfill
\subfloat{\import{FigPlot/}{5deg_ss-2.tex}}
\hfill
\subfloat{\import{FigPlot/}{5deg_ss-3.tex}}
\hfill
\subfloat{\import{FigPlot/}{5deg_ss-4.tex}}
\hfill

\vspace{-0.37 cm}

\subfloat[$K_{\tt d} = 10, ~ \gamma =1$]{\import{FigPlot/}{Inputs_5_deg_a}}
\hfill
\subfloat[$K_{\tt d} = 5, ~ \gamma =1$]{\import{FigPlot/}{Inputs_5_deg_b}}
\hfill
\subfloat[$K_{\tt d} = 0.1, ~ \gamma =1$]{\import{FigPlot/}{Inputs_5_deg_c}}
\hfill
\subfloat[$K_{\tt d} = 0.01, ~ \gamma =1$]{\import{FigPlot/}{Inputs_5_deg_d}}
\hfill

\caption{Position control performance for the desired configuration $\theta_\star = 5 ~{\rm deg}$ with different $K_{\tt d}$ }
\label{fig:ctrl_agl1_Kd}
\end{figure*}

\begin{figure*}[htp]
 \centering
\setkeys{Gin}{width=0.24\linewidth}

\subfloat{\import{FigPlot/}{CtrlAgl_5deg_i.tex}}
\hfill
\subfloat{\import{FigPlot/}{CtrlAgl_5deg_j.tex}}
\hfill
\subfloat{\import{FigPlot/}{CtrlAgl_5deg_k.tex}}
\hfill
\subfloat{\import{FigPlot/}{CtrlAgl_5deg_l.tex}}

\vspace{-0.5 cm}

\subfloat{\import{FigPlot/}{5deg_ss-5.tex}}
\hfill
\subfloat{\import{FigPlot/}{5deg_ss-6.tex}}
\hfill
\subfloat{\import{FigPlot/}{5deg_ss-7.tex}}
\hfill
\subfloat{\import{FigPlot/}{5deg_ss-8.tex}}
\hfill

\vspace{-0.37 cm}

\subfloat[$K_{\tt d} = 1, ~ \gamma =0.01$]{\import{FigPlot/}{Inputs_5_deg_i}}
\hfill
\subfloat[$K_{\tt d} = 1, ~ \gamma =0.1$]{\import{FigPlot/}{Inputs_5_deg_j}}
\hfill
\subfloat[$K_{\tt d} = 1, ~ \gamma =1$]{\import{FigPlot/}{Inputs_5_deg_k}}
\hfill
\subfloat[$K_{\tt d} = 1, ~ \gamma =5$]{\import{FigPlot/}{Inputs_5_deg_l}}

\caption{Position control performance for the desired configuration $\theta_\star = 5 ~{\rm deg}$ with different $\gamma$. }
\label{fig:ctrl_agl1_gamma}
\end{figure*}

We conducted additional experiments to test the proposed approach in different scenarios, including the desired configurations of $\theta_\star = 10$ and $15$ deg, shown in Fig. \ref{fig:ctrl_10_11_deg}. These results demonstrate that the algorithm is capable of achieving high accuracy and performance for position control. To quantify the steady performance, we study the configuration trajectories during $\cali_{\tt s}:=[4,8]$ s, since for all these scenarios the system states arrive at the steady-state stage. For these two desired equilibria, the proposed design achieved high accuracy, verifying the property {\bf P1} in Proposition \ref{prop:control}.

We summarise the accuracy achieved in these experiments with different equilibria ($5$ deg, $10$ deg and $15$ deg) and gains of $\gamma$ and $K_{\tt d}$ in Table \ref{tab:1}, where $[\theta_{\tt min}, \theta_{\tt max}]$ represents the minimal and the maximal values during the interval $\cali_{\tt s}$. We also give the root mean square (RMS) and the mean absolute error (MAE) for each scenario in the same table. For $\theta_\star = 5$ deg, it achieved the highest accuracy among the three equilibria, for which the selections of $\gamma$ as $1$ and $5$ degraded the steady-state accuracy a little bit. In Fig. \ref{fig:photo_sequence}, we present a photo sequence of one of the scenarios with the desired configuration $\theta_\star = 10$ deg, and the gains $K_{\tt d} =0.1$ and $\gamma=1$. This sequence serves as an intuitive illustration of the dynamic behaviour of the closed loop.


In addition, we report the result with a \emph{large} $\gamma=10$. However, as explained in the discussion point 3) in Section \ref{sec:5}, a large $\gamma>0$ may make the desired potential energy function $U_{\tt d}$ \textit{non-convex}, resulting in instability. This is consistent with the experimental results, as we observe the neutral stability with oscillating behaviours at the steady-state stage for $\gamma=10$; see Fig. \ref{fig:LargeGamma}. Evaluating the closed-loop performance in Fig. \ref{fig:ctrl_agl1_Kd}, we observe that a smaller gain $K_{\tt d}$ was likely to cause poor \emph{transient} performance with longer time. Whereas, the experimental results in Table \ref{tab:1} show that the value of $K_{\tt d}$ within the interval $[0.1, 10]$ has limited effects on the \emph{steady-state} performance for position control.

%
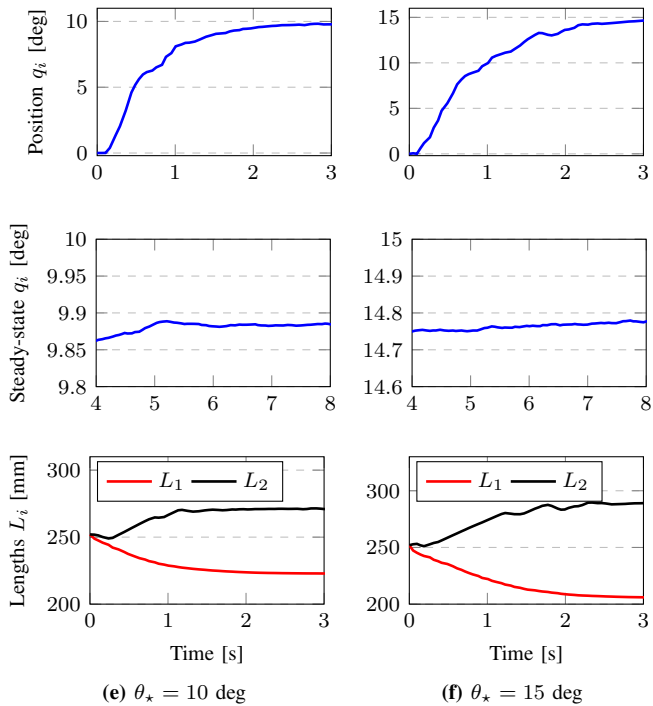
\begin{figure}
\centering
\setkeys{Gin}{width=0.25\linewidth}
~~\subfloat{\import{FigPlot/}{CtrlAgl_10deg}}
\hfill
\subfloat{\import{FigPlot/}{CtrlAgl_15deg}}
\hfill
\\
\subfloat{\import{FigPlot/}{10deg_ss}}
\subfloat{\import{FigPlot/}{15deg_ss}}
\\
\subfloat[$\theta_\star=10$ deg]{\import{FigPlot/}{Inputs_10_deg}}
\hfill
\subfloat[$\theta_\star=15$ deg]{\import{FigPlot/}{Inputs_15_deg}}
\caption{Position control performance for the desired configurations $\theta_\star = 10 ~{\rm deg}$ and $15$ deg with $K_{\tt d} = 1$ and $\gamma =0.1$}
\label{fig:ctrl_10_11_deg}
\end{figure}




\begin{table*}[]
\caption{The steady-state errors in the time interval $[4,8]$ s of different scenarios (Unit: deg) }
\label{tab:1}\footnotesize
\centering
\begin{tabular}{l|l|c|c|c|c|c|c|c|c|c}
\specialrule{1pt}{0pt}{0pt}
\multicolumn{2}{l}{}                 & \multicolumn{3}{|c|}{$K_{\tt d}=10$}                                             & \multicolumn{3}{c|}{$K_{\tt d}=1$}                                                                                      & \multicolumn{3}{c}{$K_{\tt d}=0.1$}                                          \\\cmidrule(l){3-11}    
\multicolumn{2}{l|}{\qquad\qquad $\gamma$ }              & {[}$\theta_{\tt min}, \theta_{\tt max}${]} & RMS  & MAE & {[}$\theta_{\tt min}, \theta_{\tt max}${]} & RMS &MAE & {[}$\theta_{\tt min}, \theta_{\tt max}${]} & RMS &MAE \\
   \hline   
\multirow{3}{*}{$5$ deg} & $0.01$ &  [4.8876, 4.9403] & 4.9336 &	0.0664
&
[4.9492, 4.9532]  & 4.9506&	0.0494
& [4.9236, 4.9515] & 4.9413 &	0.0587
      
\\
& $0.1$ &   [4.9538, 4.9687]  & 4.9605	 & 0.0395
 & 
[4.9387, 4.9556]  &4.9505  &	0.0495
 & 
[4.9480, 4.9528] & 4.9503& 	0.0497
                     
\\
& $1$ &  [4.9148, 4.9265] &  4.9570	&0.0430
 &
[4.9148, 4.9265]& 4.9219&	0.0781
 &  [4.9156, 4.9195] &  4.9179 & 	0.0821
                
\\
& $5$ &  [4.9325, 4.9367] & 4.9351 &	0.0649
 & 
[4.9333, 4.9425]& 4.9383&	0.0617
 &  [4.9316, 4.9473] &  4.9408	& 0.0592 \\
\hline
\multirow{3}{*}{$ 10$ deg}  &$0.01$  & [9.5437, 10.2709]& 9.8684	  & 0.1460
& 
[9.8328, 9.9597] & 9.8824 &	0.1176
 &  [9.8458, 9.9596] & 9.8902 &	0.1098
                                 
\\
& $0.1$ &   [9.8320, 9.9048]  & 9.8803	 &0.1197
 &
[9.8328, 9.9597]& 9.8810 &	0.1190
& 
[9.8440, 9.9758]& 9.8991	 &   0.1009
                  \\
& $ 1$ &  [9.9074, 9.9295] & 9.9216	&0.0784 
 &  [9.7814, 9.9044]&  9.8455&	0.1546
& 
[9.8204, 9.8457]& 9.8324 &	0.1676

\\
& $5$ & [9.8706, 9.8857] & 9.8756	&0.1244 & 

[9.8955, 9.8701] & 9.8884 & 	0.1116
& 
[9.8942, 9.8442] & 9.8812 &	0.1188

\\ \hline
\multirow{3}{*}{$ 15$ deg} & $0.01$ & [12.0252, 14.9968] & 14.5881 & 0.4266
&
[14.7067, 14.7753] & 14.7651&	0.2350
 &  [14.7333, 14.7798]    &14.7711	  &0.2289 \\
& $0.1$ &  [14.3324, 14.9968] & 14.7970 & 	0.2044
& [14.5142, 14.7350] & 14.6812 &	0.3190 &
[14.7174, 17.7939] & 14.7710 & 	0.2290
                       \\
& $ 1$ &  [14.6276, 14.9968] &14.8059  &	0.1942
& [14.7499, 14.7794]& 14.7632 &	0.2368
&
[14.7228, 14.7597] &14.7482  &	0.2519
\\
& $5$ &  [14.6863, 14.7590] & 14.7209& 0.2792
& [14.7061, 14.7307] & 14.7236 &	0.2764&
[14.7051, 14.7301]& 14.7186 &    	0.2814\\
                                  \hline
\specialrule{1pt}{0pt}{0pt}
\end{tabular}
\end{table*}

We present experimental results demonstrating closed-loop stiffness regulation around the desired equilibria, which is related to \textbf{P2} in Proposition \ref{prop:control}. While measuring the overall stiffness is generally not manageable, we can test the transverse stiffness as outlined in Item 1) of Section \ref{sec:5}. To this end, we equipped a linear actuator perpendicularly to the tangential direction of the continuum robot at the end-effector, as shown in Fig. \ref{fig:platform}. We repeated the experiments for two different desired equilibria, namely 8 deg and 10 deg. We collected stiffness data using different gains $\gamma$ and plotted the results in Figs. \ref{fig:stiffness_8deg} and \ref{fig:stiffness_10deg}. The results match the equation \eqref{K_T:prop} in Section \ref{sec:5} that the closed-loop stiffness is affine in the control gain $\gamma$. This implied that we were able to identify the parameters $\kappa_1$ and $\kappa_2$, and use them to tune the controller for a prescribed stiffness around the desired configuration.

Finally, to investigate the robustness, we conducted experiments for two supplementary scenarios: one for examining the robot response in the presence of external disturbances, and the other for studying passive environmental interaction via encountering a semi-rigid foam obstruction. These setups are shown in Fig. \ref{fig:robus}, and the corresponding experimental results are presented in Fig. \ref{fig:dist1}. Specifically, Fig. \ref{fig:dist1}(a) provides evidence of the remarkable robustness of the proposed controller {\it vis-\`a-vis} external disturbances, as it effectively made the system back to its desired configuration after the vanishing of disturbances. We also note that a larger gain $\gamma$ yielded a shorter recovery response. In the encounter experiments, as shown in Figs. \ref{fig:robus}(b) and \ref{fig:dist1}(b), a larger value of $\gamma=5$ ensured that the robot passed the foam obstruction, resulting in a recorded trajectory that was approximately \emph{monotonic} over time -- indicating the robot's stiff behaviour. In contrast, using the smaller value of $\gamma=0.1$ did not yield this effect, causing the robot to exhibit some deformation when encountering the foam and displaying real-time position fluctuations in the transient stage due to its softness. 

Note that although various control approaches have been proposed for continuum robotics, their suitability for achieving simultaneous control of position and stiffness in underactuated robots is limited, particularly considering the variations in actuation mechanisms across different continuum robotic platforms. Given the absence of applicable control strategies in the existing literature, this study did not provide experimental comparisons to previous works. However, our objective is to lay the groundwork for future exploration and development of experimental studies in this area.



\begin{figure*}
\centering
\begin{minipage}[t]{0.3\linewidth}
{\import{FigPlot/}{LargeGamma.tex}}
\caption{A large $\gamma =10$ leads to unstable performance}
\label{fig:LargeGamma}
\end{minipage}
~~
\begin{minipage}[t]{0.3\linewidth}
\centering
{\import{FigPlot/}{8degstiffness.tex}}
\caption{Stiffness regulation with the desired configuration $\theta_\star =$ 8 deg ($\kappa_1 = 0.0288$ and $\kappa_2 = 0.2332$)
}
\label{fig:stiffness_8deg}
\end{minipage}
~~
\begin{minipage}[t]{0.3\linewidth}
\centering
{\import{FigPlot/}{10degstiffness.tex}}
\caption{Stiffness regulation with the desired configuration $\theta_\star =$ 10 deg ($\kappa_1 = 0.0805$ and $\kappa_2 = 0.4305$)}
\label{fig:stiffness_10deg}
\end{minipage}
\end{figure*}

\begin{figure}[!htp]
    \centering
    \includegraphics[width = .99\linewidth]{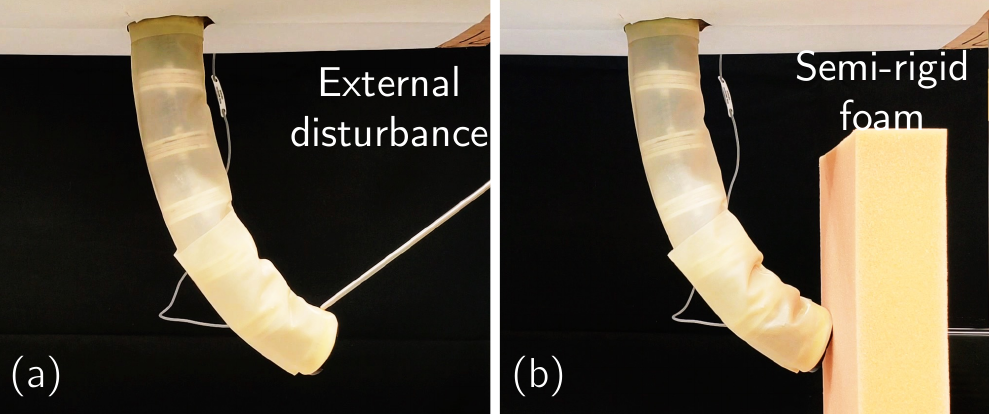}
    \caption{Experiments to evaluate the robustness of the closed-loop robotic system: (a) An external disturbance was added to the distal point of the robot; (b) The robot was controlled to encounter a semi-rigid foam obstruction.}
    \label{fig:robus}
\end{figure}


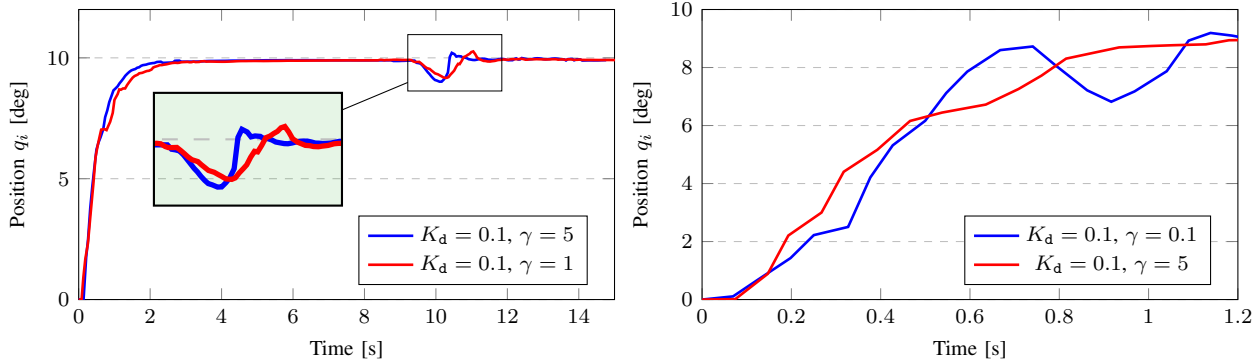
\begin{figure*}[!htp]
 \centering
{\import{FigPlot/}{dist_1.tex}}
\caption{Results of the robustness evaluation experiments: (Left) Applying external disturbances; (Right) Encountering a semi-rigid foam obstruction.}
\label{fig:dist1}
\end{figure*}

\section{Concluding Remarks and Future Works}
\label{sec:7}

In this paper, we studied the modelling and control of underactuated antagonistic tendon-driven continuum robots. The proposed model possesses a configuration-dependent input matrix, which effectively captures the mechanism for open-loop stiffening through cable tension regulation. We have thoroughly analysed the assignable equilibria set and devised a potential shaping feedback controller that enables simultaneous position-and-stiffness regulation while adhering to the non-negative input constraint. To the best of the authors' knowledge, this is the first design for such a problem. The experimental results on the robotic platform OctRobot-I demonstrate the effectiveness and reliability of the proposed approach. Our approach relies on only a few intrinsic parameters of the model, rather than depending on the complete dynamical model. This grants it remarkable robustness against modelling errors.

Along the research line, the following problems are considered as potential future works:

\begin{itemize}
\item[1)] As per Proposition \ref{prop:control}, we impose $\gamma<\alpha_2 $ to guarantee the convexity of the desired potential energy function $U_{\tt d}$. Note that the parameter $\alpha_2$ is an intrinsic characteristic of the continuum robot, and consequently the range of choices for the gain $\gamma$ is limited, which restricts our ability to control the stiffness in a relatively narrow interval. Our experimental results support this assertion. To enlarge the closed-loop stiffness range, a potential way is to make full use of \textit{jamming} in the continuum robot via changing the compression level of jamming flaps \cite{YIFAN23}.

\item[2)] Exploring alternative desired potential energy functions may offer a promising way to enhance closed-loop performance. In addition, applying state-of-the-art energy shaping methodologies, \textit{e.g.} \cite{YIetalAUT,ROMetalSCL}, could prove valuable for solving more complex tasks, such as path following and robust simultaneous position-and-stiffness control.

\item[3)] Similar to the recent works \cite{FRAGAR,DELetal}, our approach is developed for the robot with one section used in planar case, which is quite simple for continuum robots being utilised in the real word. It is underway to extend to multiple sections in the spatial case. A promising way is to change the mechanical structure and actuate the overall robot in a sagittal plane for each section to prevent sections from twisting about their neutral axis \cite{MARRUS}. Then, we may use the proposed approach to control sections separately in different planes, and the manipulator is capable of three-dimensional Cartesian positioning.

\item[4)] Our proposed approach does not consider the \textit{actuation dynamics}, opting instead to utilise a high-gain design to enforce time-scale separation and disregard these dynamics. It would be advantageous to take the actuation dynamics in to the controller synthesis by incorporating advanced robustification techniques \cite{ORTetalSCL}. 

\item[5)] In this paper, the proposed method is applicable only to cases where the signs of curvatures in the PCC structure remain unchanged. Extending this approach to more complex configurations, such as the S-shape illustrated in Fig. \ref{fig:new}(b), would be a valuable direction for future work.

\end{itemize}

\section*{Acknowledgement}

The authors are grateful to Dr. Liang Zhao and Tiancheng Li from UTS for their support during experiments, and to the Associate Editor and three anonymous reviewers for their thoughtful comments.
  
\appendix

\section*{A. Supplementary Details on Modelling}

We provide additional details on the model, in particular the potential energy functions of the continuum robotic platform.

\subsubsection{Gravitational energy} In order to approximate the gravitational potential energy $U_{\tt G}$, we make the assumption that the mass is lumped at the centre of each link with the link lengths $l_i>0$ and the masses $m_{i}>0$ for $i\in \caln$. From some basic geometric relations, we have
\begin{equation}
    U_{\tt G}(\bfq) = \sum_{i \in \caln} {l_i m_i \over 2} \left[ \cos(q_0 +\ldots+  q_{i-1})  - \cos(q_0+ \ldots +q_i) \right]
\end{equation}
with the parameter $q_0 =0$, which satisfies $U_{\tt G}(\mathbf{0}_n) = 0$.

In addition, we impose Assumption \ref{ass:uniform} about the mass and length, under which the potential energy becomes
\begin{equation}
\begin{aligned}
    U_{\tt G}(\bfq) &~ = ~\alpha_1 (1-\cos(q_\Sigma))
\end{aligned}
\end{equation}
with
$
    q_\Sigma ~ :=~ \sum_{i\in \caln} \bfq_i,
$
and some coefficient $\alpha_1>0$.

\subsubsection{Elastic energy} In the designed continuum robot, each spine segment contains two pair of helical compression springs. Since we limit ourselves to the 2-dimensional case, we only consider a pair of springs as illustrated in Fig. \ref{fig:elastic}, and make the assumption below.

\begin{assumption}\rm
The deformable part of the continuum manipulator consists of a fixed number of segments with constant curvature and differentiable curves everywhere \cite{DELetal}.
\qed
\end{assumption}


\begin{figure}
    \centering
    \includegraphics[width = 0.48\linewidth]{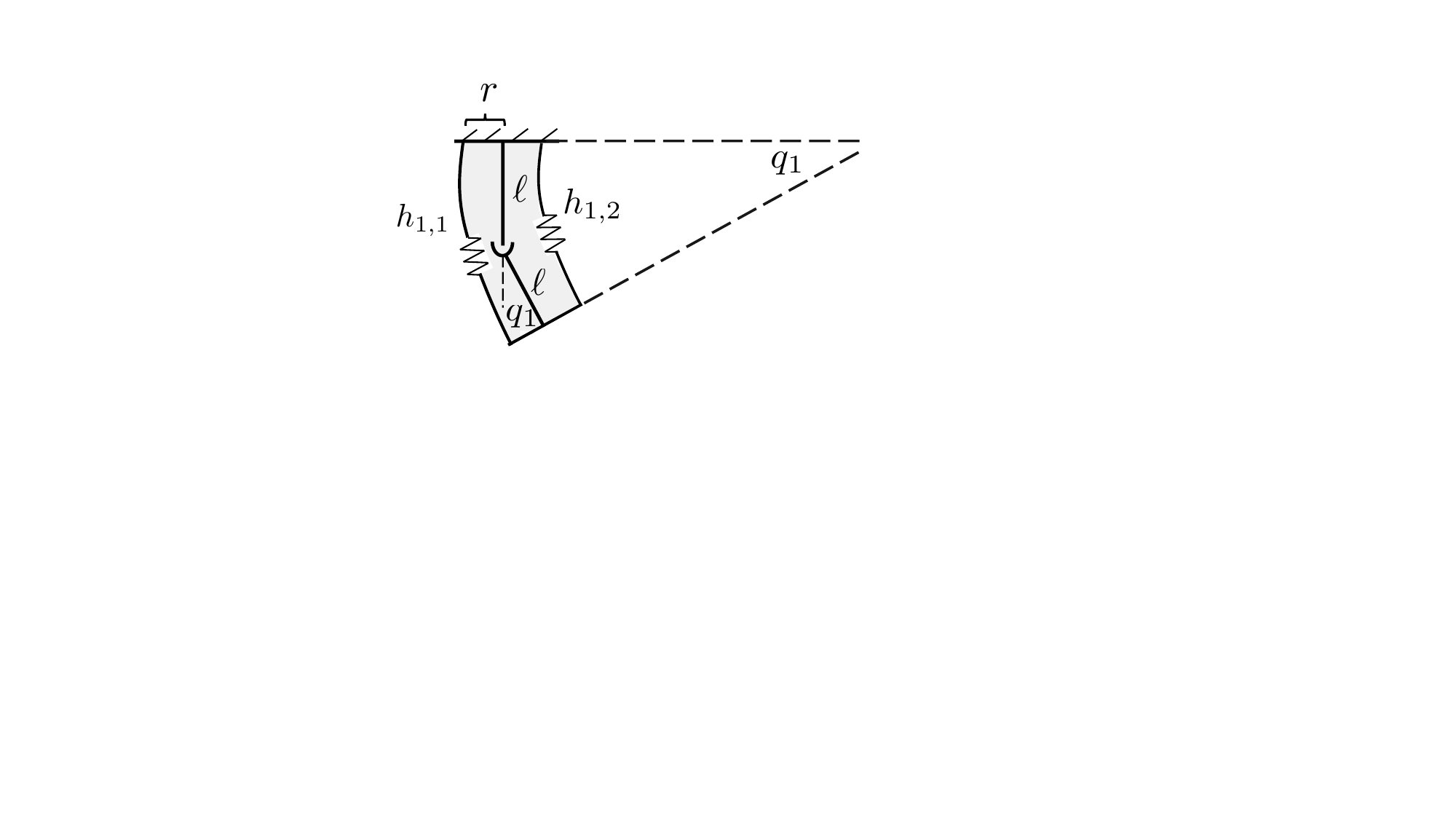}
    \caption{Illustration of the elastic potential energy in springs}
    \label{fig:elastic}
\end{figure}

In terms of the above assumptions, the boundary lengths in the $i$-th segment are given by
$$
\begin{aligned}
h_{i,1} = q_i\Big[ \ell \cot\Big({q_i \over 2}\Big) + r \Big], \quad
h_{i,2} = q_i\Big[ \ell \cot\Big({q_i \over 2}\Big) - r \Big].
\end{aligned}
$$
Note that the above functions are well-posed when $q_i \to 0$, \emph{i.e.},
$
\lim_{q_i \to 0} h_{i,1} =2\ell, ~\lim_{q_i \to 0} h_{i,2} =2\ell.
$
Hence, the elastic energy can be modelled as
\begin{equation}
\label{UE:real}
    U_{\tt E} = \sum_{i \in \caln} {k_i} \left[ q_i^2\Big(\ell^2 \cos\Big({q_i \over 2}\Big)^2 + r^2\Big) - \ell^2 \right] + k_i' q_i^2,
\end{equation}
in which $k_i>0$ and $k_i'>0$ are some elastic coefficients to characterise the elastic energies caused by the elongation and bending of springs.

In the proposed rigid-link model, each configuration variable $q_i$ would generally be small, \emph{i.e.}, $q_i\in [- {\pi\over 12}, {\pi\over 12}]$, for which the term $\cos({q_i \over 2})^2$ takes values within $[0.983,1]$. Then, it is reasonable to make the following quadratic assumption to approximate the highly nonlinear function in \eqref{UE:real}.

\begin{assumption}\rm \label{ass:Ue}
The elastic energy $U_{\tt E}$ has the quadratic form
\begin{equation}
	U_{\tt E}(\bfq) = \hal \bfq^\top \Lambda \bfq + U_0
\end{equation}
with a constant coefficient $U_0$ and a diagonalisable matrix $\Lambda:=\diag(\alpha_2, \ldots \alpha_2) \succ 0$. \qed
\end{assumption}

\subsubsection{Inertia and kinematic energy} The analytic form of the inertia matrix $M(\bfq)$ can be obtained following the standard way for rigid-link robotic models. The interested reader may find detailed procedures in \cite[Chapter 8.4]{LYNPAR}. Then, the kinematic energy is given by $\hal \bfp M(\bfq)^{-1} \bfp$.

Note that the specific formulation of the inertia $M(\bfq)$ is not involved in the controller design. This makes the closed loop relatively robust, and it is unnecessary to obtain the analytic formulation of $M(\bfq)$ for experimental implementation.

\subsubsection{Input matrix} Let us use a single link to discuss the modelling of the input matrix $G(\bfq)$. We assume that the tension is uniformly distributed along the cables, and lumped forces are along the tangential directions at the middle points of the constant-curvature outline. Then, the lever's fulcrums $\call_1$ and $\call_2$ in Fig. \ref{fig:1-1} are given by 
$
\call_1 = r - {\ell \over 2} \sin(q_i)
$ and $
\call_2 = r+ {\ell \over 2}\sin(q_i).
$
According to some basic geometric relations, we can obtain the $i$-th row of the input matrix as
$$
\begmat{G_{i,1} & G_{i, 2}} = k_u 
\begmat{r - {\ell \over 2} \sin(q_i)
&
- r - {\ell \over 2} \sin(q_i)}
,
$$
in which $k_u >0$ is a coefficient to denote the loss of inputs at the virtual joint, subject to friction and viscoelastic effects \cite{PALetal}. Correspondingly, $g_0 = k_u r$ and $g_1 = - {1\over 2}\ell k_u \sin(q_i)$. This model verifies the key assumption \eqref{eq:diag} with $\partial G_{i,j}/\partial q_i(\bfq )<0$.

\bibliographystyle{IEEEtranS}
\bibliography{reference}

\begin{biography}
[{\includegraphics[width=1.1in,height=1.25in,clip,keepaspectratio,trim={.4in .8in .4in 0.4in}]{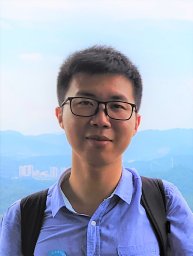}}]{Bowen Yi} (Member, IEEE) obtained his Ph.D. degree in Control Engineering from Shanghai Jiao Tong University, China in 2019. 

Between 2017 and 2019, he was a Visiting Student at Laboratoire des Signaux et Syst\`emes, CNRS-CentraleSup\'elec, Gif-sur-Yvette, France. He has held postdoctoral positions at the Australian Centre for Robotics (ACFR), The University of Sydney, NSW, Australia (2019 -- 2022), and the Robotics Institute, University of Technology Sydney, NSW, Australia (Sept. 2022 -- 2023). Currently, he is an Assistant Professor in the Department of Electrical Engineering, Polytechnique Montreal and is affiliated with GERAD, Queb\'ec, Canada. His research interests involve nonlinear systems (estimation, control, and learning) and robotics. Dr. Yi was the recipient of the 2019 CCTA Best Student Paper Award from the IEEE Control Systems Society, and the Australian Research Council (ARC) Discovery Early Career Researcher Award (DECRA).
\end{biography}

\begin{biography}
[{\includegraphics[width=1.1in,height=1.25in,clip,keepaspectratio,trim={0 0.2in 0 0}]{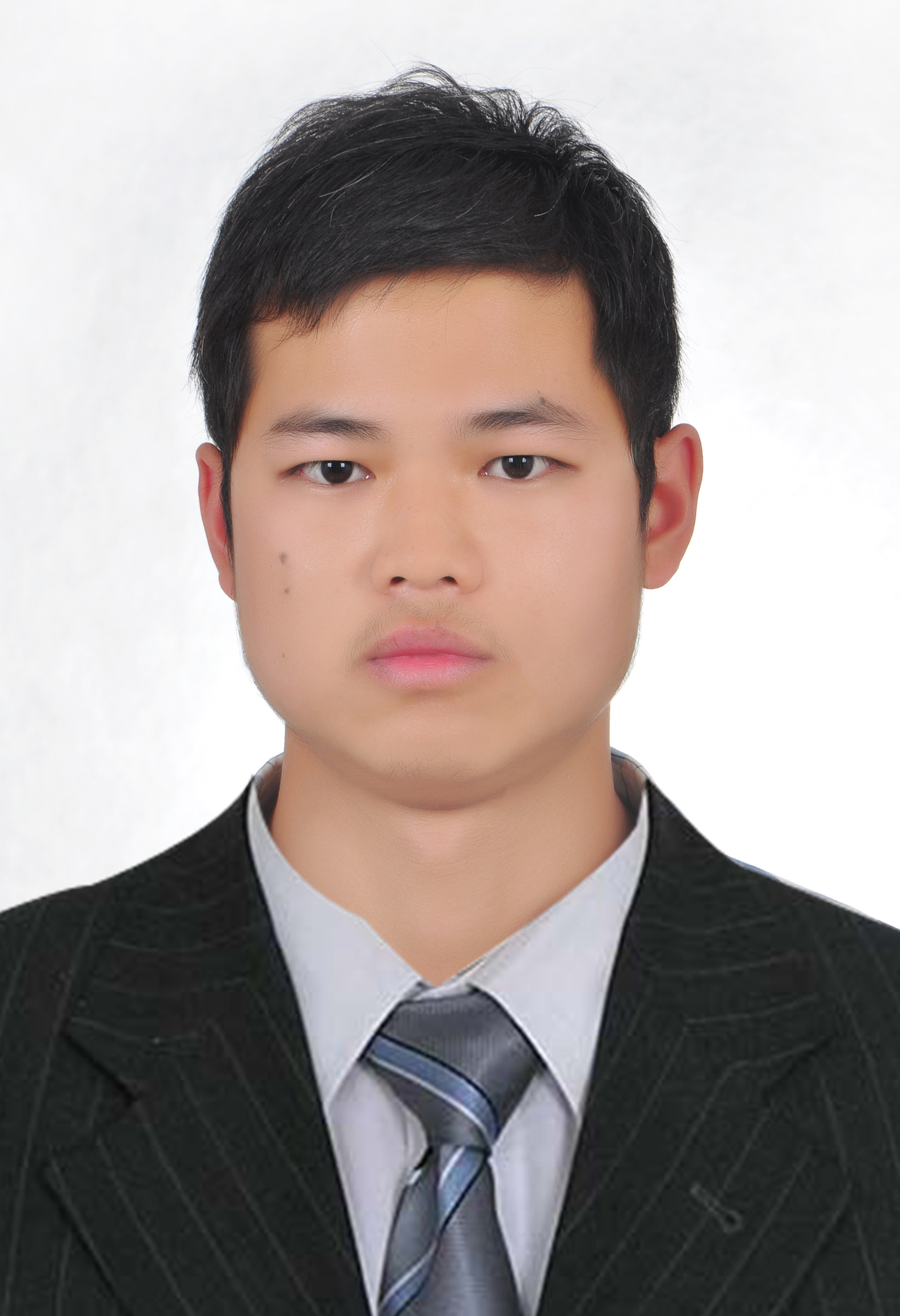}}]{Yeman Fan} (Student Member, IEEE) received the B.E. degree in Machine Designing, Manufacturing and Automation, and the M.E. degree in Agricultural Electrification and Automation from Northwest A\&F University, Yangling, China, in 2016 and 2019, respectively. He is currently pursuing the Ph.D. degree with the Robotics Institute, University of Technology Sydney, Sydney, NSW, Australia. 

His research interests include continuum robots and manipulators, robot control systems, and jamming technology for robotics.
\end{biography}

\begin{biography}
[{\includegraphics[width=1.1in,height=1.25in,clip,keepaspectratio,trim={2in 0 2in 0}]{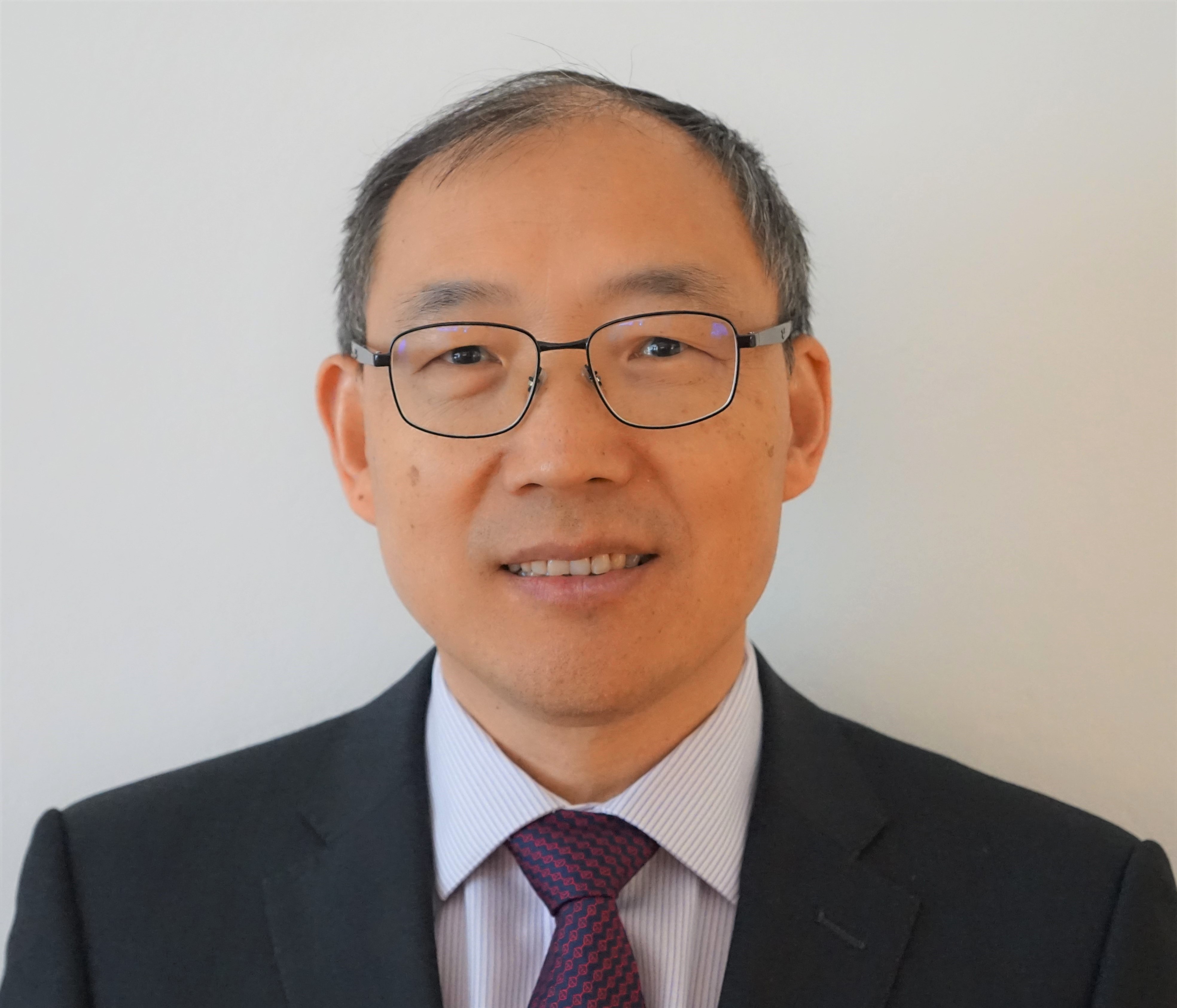}}]{Dikai Liu} (Senior Member, IEEE) received the Ph.D. degree in Dynamics and Control from the Wuhan University of Technology, Wuhan, China, in 1997. 

He is currently a Professor in Mechanical and Mechatronic Engineering with the Robotics Institute, University of Technology Sydney, Sydney, NSW, Australia. His main research interest is robotics, including robot perception, planning and control of mobile manipulators operating in complex environments, human-robot collaboration, multi-robot coordination, and bioinspired robotics.
\end{biography}

\begin{biography}
[{\includegraphics[width=1.1in,height=1.25in,clip,keepaspectratio]{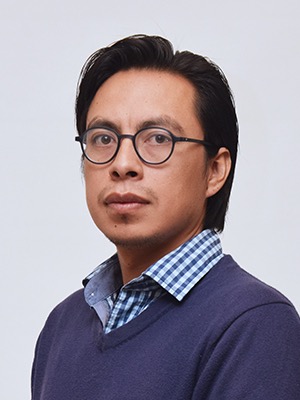}}]{Jose Guadalupe Romero} (Member, IEEE) obtained the Ph.D. degree in Control Theory from the University of Paris-Sud XI, France in 2013. 

 Since 2016, he has been with the Instituto Tecnol'ogico Aut'onomo de M\'exico (ITAM), Mexico, where he is currently  a full time Professor and since July 2023 he is the Chair of the Department of Electrical and Electronic Engineering. From August 2022 to June 2023, he was a Visiting Lecturer at The Hong Kong Polytechnic University (PolyU). He has over 50 papers in peer-reviewed international journals where he has also served as a reviewer.
His research interests are focused on nonlinear and adaptive control, stability analysis  and the state estimation problem, with application to mechanical systems, aerial vehicles, mobile robots and multi-agent systems. 

Dr. Romero serves as an Editor of the \textsc{International Journal of Adaptive Control and Signal Processing}.
\end{biography}

\end{document}

%% file: fig/plot_force_vs_length.tex
\begin{tikzpicture}
    \begin{axis}[
      xmin=220, xmax=250.06,
      ymin=0, ymax=15,
      ymajorgrids=true,
      grid style=dashed,
      legend pos=north west,
          label style={font=\footnotesize},
      tick label style={font=\footnotesize},
      width = 0.52\linewidth,
      height = 0.4\linewidth,
    xlabel={Length $L_2$ [mm]},
    ylabel={$u_2$ [N]}
    ]

\addplot[
        domain = 220:251, color=lightgray, line width = 2pt
    ] {-0.4364*x + 111.0000};

    \addplot[only marks, mark=x, color=black, line width = 0.5pt, mark size=1.5000pt] table[row sep=crcr]{%
x	y\\
250.059167947769	2.47541427612305\\
250.059167947769	2.46364617347717\\
250.059167947769	2.45972347259521\\
248.09723991394	3.09127187728882\\
248.09723991394	3.09519457817078\\
248.09723991394	3.09715580940247\\
246.124763946533	3.68163204193115\\
246.124763946533	3.644366979599\\
246.124763946533	3.61886954307556\\
244.141739959717	4.32102584838867\\
244.141739959717	4.37594318389893\\
244.141739959717	4.44458961486816\\
242.148168182373	5.03691101074219\\
242.1587159729	5.06044721603394\\
242.148168182373	5.10751914978027\\
240.133499908447	5.79006195068359\\
240.144047698975	5.83517265319824\\
240.133499908447	5.80771398544312\\
238.108283843994	6.71188688278198\\
238.108283843994	6.80799198150635\\
238.108283843994	6.74915218353271\\
236.072519989014	7.9043755531311\\
236.083067779541	7.91025972366333\\
236.083067779541	7.8729944229126\\
234.036755218506	8.73205661773682\\
234.036755218506	8.81639385223389\\
234.026208343506	8.69675254821777\\
231.979896697998	9.87551212310791\\
231.979896697998	9.79117488861084\\
231.979896697998	9.72645092010498\\
229.901940765381	10.6620054244995\\
229.901940765381	10.8149890899658\\
229.912487640381	10.9150171279907\\
227.813436126709	11.760350227356\\
227.813436813354	11.8034992218018\\
227.823984832764	11.7426986694336\\
225.724931488037	12.7037506103516\\
225.735480194092	12.8351593017578\\
225.735480194092	12.7076730728149\\
223.615331268311	13.6040010452271\\
223.615331268311	13.6432275772095\\
223.615331268311	13.5667352676392\\
221.505731048584	14.3551902770996\\
221.505731048584	14.5474004745483\\
221.505731048584	14.4238367080688\\
};

    \end{axis}
\end{tikzpicture}
\begin{tikzpicture}
    \begin{axis}[
      xmin=0, xmax=50,
      ymin=-13, ymax=0,
      ymajorgrids=true,
      grid style=dashed,
      legend pos=north west,
      width = 0.52\linewidth,
      height = 0.4\linewidth,
          label style={font=\footnotesize},
      tick label style={font=\footnotesize},
    xlabel={Difference $\Delta L$ [mm]},
    ylabel={${\tau}_1$ [N]}
    ]

\addplot[only marks, mark=x, line width = 0.5pt, mark size=1.5000pt, draw=black] table[row sep=crcr]{%
x	y\\
3.88166396141052	-0.609671711921692\\
3.87111617088318	-0.607710242271423\\
3.88166396141052	-0.584174156188965\\
7.75278053283691	-1.38047432899475\\
7.76332832336426	-1.38831973075867\\
7.75278053283691	-1.38439691066742\\
11.6238963317871	-2.07674634456635\\
11.6344441223145	-2.02771329879761\\
11.6344441223145	-2.0198677778244\\
15.5055601501465	-2.79067063331604\\
15.5055601501465	-2.81616806983948\\
15.4950123596191	-2.91423439979553\\
19.3766761779785	-3.5222464799881\\
19.3555801391602	-3.54774403572083\\
19.3766761779785	-3.58893191814423\\
23.2477922058105	-4.25578415393829\\
23.2372444152832	-4.30481743812561\\
23.2477922058105	-4.27932012081146\\
27.1189083337784	-5.17172503471375\\
27.1083603286743	-5.28352081775665\\
27.1083603286743	-5.21879696846008\\
30.9900241470337	-6.37598168849945\\
30.9794763565063	-6.38774991035461\\
30.9794763565063	-6.3387166261673\\
34.8294969177246	-7.17228150367737\\
34.8294969177246	-7.26838672161102\\
34.8505918121338	-7.14482271671295\\
38.6795155334473	-8.24316775798798\\
38.6795155334473	-8.17452120780945\\
38.6900633239746	-8.10587453842163\\
42.5611793518066	-8.97278249263763\\
42.5506315612793	-9.12576615810394\\
42.5506324768066	-9.2297168970108\\
46.411199798584	-9.96913814544678\\
46.4006513214111	-10.0338617563248\\
46.3901033020019	-9.96913850307465\\
50.2612202453613	-10.8321239948273\\
50.2506715393066	-10.9282287359238\\
50.2506715393066	-10.8183945417404\\
54.1112406921387	-11.5617387294769\\
54.1112406921387	-11.5891971588135\\
54.1112406921387	-11.5460476875305\\
57.9296168518066	-12.2050547599792\\
57.9401655578613	-12.4031488895416\\
57.9296168518066	-12.2678172588348\\
};


\addplot[
        domain = 0:50, color=lightgray, line width = 2pt
    ] {-0.2238*x + 0.5622};


    \end{axis}
\end{tikzpicture}

%% file: fig/plot_fig1.tex
\begin{tikzpicture}
    \begin{axis}[
      width  = 0.65 \linewidth,
      height = 0.4 \linewidth,
      xmin=0, xmax=44.6,
      ymin=-0.2, ymax=1.7,
      ymajorgrids=true,
      grid style=dashed,
      label style={font=\footnotesize},
      tick label style={font=\footnotesize},
      legend pos=north west,
    x label style={above=5mm},
    xlabel={Force $\mu$ [N]},
    ylabel={Stiffness [N/mm]}
    ]

\addplot[domain = 0:50, color=lightgray, line width = 2pt] {0.02491*x  -0.04697 };

    \addplot[only marks, mark=x, color=black, line width = 0.5pt, mark size=2.000pt]
        coordinates {(0.000,0.023)(4.361,0.050)(8.877,0.112)(13.435,0.290)(17.939,0.412)(22.487,0.492)(27.919,0.602)(32.138,0.742)(36.588,0.930)(41.304,0.993)(46.014,1.092)};

    \addplot[name path=us_top,color=carminepink!70] coordinates {(0.000,0.0250)(4.361,0.0670)(8.877,0.1510)(13.435,0.3160)(17.939,0.4720)(22.487,0.6060)(27.919,0.762)(32.138,0.904)(36.588,1.032)(41.304,1.099)(46.014,1.3)};

    \addplot[name path=us_down,color=carminepink!70] 
    coordinates {(0.000,0.02)(4.361,0.033)(8.877,0.073)(13.435,0.264)(17.939,0.352)(22.487,0.378)(27.919,0.442)(32.138,0.58)(36.588,0.828)(41.304,0.887)(46.014,0.884)};

   \addplot[carminepink!50,fill opacity=0.5] fill between[of=us_top and us_down];

    \end{axis}
\end{tikzpicture}

%% file: FigPlot/CtrlAgl_5deg_a.tex
\begin{tikzpicture}[spy using outlines={rectangle, magnification=10,  connect spies}]
    \begin{axis}[
      xmin=0, xmax=3,
      ymin=-0.2, ymax=6,
      ymajorgrids=true,
      grid style=dashed,
      legend pos=north west,
      width = 0.26\linewidth,
      height = 0.2\linewidth,
    label style={font=\footnotesize},
      tick label style={font=\footnotesize},
   y label style={at={(-0.23,.35)}},
    ylabel={Position $q_i$ [deg]}
    ]


    \addplot[color= blue, line width = 1pt] table {
0	0
0.0487324000000000	0.0571194725109317
0.0974648000000000	0.00716201049496108
0.150304600000000	0.591477696607778
0.206128400000000	1.34988812935604
0.253611900000000	1.78215830048176
0.305590900000000	2.27855511893232
0.351815300000000	2.50615689245564
0.402593400000000	3.20712615641229
0.473879300000000	3.45772108911153
0.521701700000000	3.59367461800630
0.576204500000000	3.55209649150558
0.621306600000000	3.46933947732990
0.679257300000000	3.54101966249684
0.753773900000000	3.67746724460794
0.801874300000000	3.88551853741958
0.875632900000000	4.20135059562660
0.920843800000000	4.41580150078830
0.973573500000000	4.49527782301339
1.02277970000000	4.51017780232997
1.09737210000000	4.49152057200293
1.14822380000000	4.48601619082811
1.19960290000000	4.49060108949309
1.26341830000000	4.57283718154546
1.34012220000000	4.65707087439488
1.39105240000000	4.74824147319439
1.44134890000000	4.78523484230877
1.51456050000000	4.83468907686631
1.58755350000000	4.84720482308325
1.66072420000000	4.85030904448449
1.71099380000000	4.85075879714618
1.78502430000000	4.85278309169274
1.84310370000000	4.87547161856021
1.89349410000000	4.85479415771481
1.96739920000000	4.92516244966948
2.01623760000000	4.95141383989991
2.06677470000000	4.97309186674926
2.13957490000000	4.96860339695397
2.18988980000000	4.95325211503427
2.26262350000000	4.96111589362212
2.31145540000000	4.94012373428327
2.38696080000000	4.93922642504921
2.45719020000000	4.94128379142853
2.53128950000000	4.97995321555521
2.58210240000000	4.99136191548834
2.63070030000000	4.98630477143591
2.68108080000000	4.97212414759924
2.72986300000000	4.96843204637111
2.80388460000000	4.96104703927220
2.85519870000000	4.96577522099735
2.92950870000000	4.96408994278639
2.98101730000000	4.96161027650487
3.02886160000000	4.95922184658255
3.10396570000000	4.95911031096950
3.15479550000000	4.95888728931020
3.22768620000000	4.95827001294337
3.27975280000000	4.95835702065788
3.33168450000000	4.95903067374719
3.40513040000000	4.95959358625474
3.45427510000000	4.96018801869675
3.52847130000000	4.96022301177717
3.57386590000000	4.95977494422954
3.62581250000000	4.95981114454899
3.68250200000000	4.95945504253893
3.73351150000000	4.96003835430557
3.78411790000000	4.95992358953660
3.83403450000000	4.95981327805675
3.90944840000000	4.96054451124211
3.95267110000000	4.96018278474058
4.02749720000000	4.96032445497862
4.07862070000000	4.96044616027851
4.13442500000000	4.95961845543031
4.20690800000000	4.95930488618330
4.25480820000000	4.95887750900481
4.32944060000000	4.95869941311035
4.37556050000000	4.95851361169176
4.44908110000000	4.95919188398835
4.49793060000000	4.95973904273607
4.57176340000000	4.95996654694631
4.62296640000000	4.96025776748601
4.69685730000000	4.95938088838603
4.75160990000000	4.95859672675151
4.80374250000000	4.95821929158789
4.87659520000000	4.95812947859219
4.92746240000000	4.95844241390669
4.98141700000000	4.95884243803438
5.05133400000000	4.95904753315599
5.10318640000000	4.95913459027288
5.17699030000000	4.95860774523223
5.22735680000000	4.95818580727223
5.28349340000000	4.95806952137343
5.34107430000000	4.95728174222656
5.39463210000000	4.95740652800301
5.46921630000000	4.95745306165965
5.52100140000000	4.95821237650412
5.59321130000000	4.95880222905750
5.63858930000000	4.95906602110536
5.71161120000000	4.95905109585246
5.76349570000000	4.95855688153853
5.83793340000000	4.95750170479235
5.89117050000000	4.95700846541048
5.94000130000000	4.95669206263883
6.01058080000000	4.95668683175103
6.06368430000000	4.95701250133930
6.13681840000000	4.95788374541570
6.18662990000000	4.95879997144536
6.25728180000000	4.95932078419700
6.30980020000000	4.95946164877826
6.36388870000000	4.95887809891190
6.40977320000000	4.95781884921095
6.46053680000000	4.95686916826692
6.50955030000000	4.95635034665031
6.58343870000000	4.95613259016851
6.63542220000000	4.95653991067689
6.68901350000000	4.95648778895858
6.74046670000000	4.95649367914884
6.79176900000000	4.95629156787002
6.84142170000000	4.95572599575297
6.91480810000000	4.95535549453097
6.96816780000000	4.95484246845093
7.00973040000000	4.95456946332923
7.05867330000000	4.95403415918045
7.11133440000000	4.95397484457256
7.18544300000000	4.95431335182668
7.23620790000000	4.95489237634417
7.29016600000000	4.95557044849587
7.34350370000000	4.95562074585908
7.41738410000000	4.95480810218885
7.46959410000000	4.95418127446411
7.54234150000000	4.95316829015513
7.59555280000000	4.95200708451099
7.64339490000000	4.95209933812858
7.71462220000000	4.95225394152961
7.78803370000000	4.95277589991744
7.84034940000000	4.95316054935106
7.89046980000000	4.95317543459079
7.94242980000000	4.95184447182015
8.01722950000000	4.95083423538828
8.08872840000000	4.94930556082976
8.14115190000000	4.94826582527981
8.19378590000000	4.94766965311757
8.23903440000000	4.94738999576557
8.28772880000000	4.94674764199208
8.33547410000000	4.94614883099921
8.41308920000000	4.94613222689715
8.46307290000000	4.94637824877810
8.53755090000000	4.94651066588923
8.58940150000000	4.94598381440933
8.64484310000000	4.94587767012979
8.71903830000000	4.94593600722323
8.76757640000000	4.94690762969947
8.81907150000000	4.94890977579475
8.86252430000000	4.94963837641798
8.91464090000000	4.95106323057676
8.98859260000000	4.95360430394290
9.03989700000000	4.95502526321103
9.09159680000000	4.95569693788836
9.14326220000000	4.95702684558497
9.21785050000000	4.95683544997571
9.26509340000000	4.95717523213427
9.33814230000000	4.95831082405184
9.39003030000000	4.95872517771563
9.43812240000000	4.95931397170452
9.49193690000000	4.95998777863754
9.54377220000000	4.96035780871557
9.61904710000000	4.96040588832609
9.66864780000000	4.95965339320678
9.71964000000000	4.95905545353369
9.79305270000000	4.95825944091683
9.84562940000000	4.95810582172400
9.91673040000000	4.95855046871308
9.96663690000000	4.95990840014358
10.0373609000000	4.95998914533498
10.0879254000000	4.96095255093484
10.1585437000000	4.96123500857680
10.2067080000000	4.96067668971093
10.2791101000000	4.95987550023394
10.3317213000000	4.95882229010837
10.3828650000000	4.95803741934602
10.4565585000000	4.95817915410135
10.5077424000000	4.95793400030457
10.5646813000000	4.95772603312093
10.6379389000000	4.95840530261400
10.6841437000000	4.95817271595357
10.7594528000000	4.95773050327384
10.8108819000000	4.95746068552987
10.8647451000000	4.95637134611249
10.9372444000000	4.95528293509485
10.9904679000000	4.95554709700483
11.0404265000000	4.95508227825593
11.1148392000000	4.95526942330849
11.1646992000000	4.95629494387902
11.2377483000000	4.95644751379434
11.2868871000000	4.95735026403333
11.3582398000000	4.95709027744847
11.4069668000000	4.95649347536131
11.4808795000000	4.95578386988353
11.5554047000000	4.95511939962895
11.6259108000000	4.95521284798879
11.6789337000000	4.95569092215540

 };

    \end{axis}

\end{tikzpicture}

%% file: FigPlot/CtrlAgl_5deg_b.tex
\begin{tikzpicture}
    \begin{axis}[
      label style={font=\footnotesize},
      tick label style={font=\footnotesize},
      xmin=0, xmax=3,
      ymin=-0.2, ymax=6,
      ymajorgrids=true,
      grid style=dashed,
      legend pos=north west,
      width = 0.26\linewidth,
      height = 0.2\linewidth,
    ]


    \addplot[color=blue, line width = 1pt] table {
0	0
0.0702578000000000	0.00433362481546621
0.140515600000000	0.532892025400373
0.187789100000000	0.973648389621454
0.262076000000000	1.55512253420422
0.308745600000000	1.97253820495782
0.364577100000000	2.33680483965311
0.437808400000000	2.93325149254827
0.487995200000000	3.20000474247410
0.544334000000000	3.30811170314803
0.615377000000000	3.36511098419767
0.668287200000000	3.41431550877347
0.741237400000000	3.53711922921779
0.793452400000000	3.72248421611395
0.863594800000000	3.99804331456495
0.912167200000000	4.22259730731308
0.984637700000000	4.35945067693008
1.03772920000000	4.42996381119947
1.08767120000000	4.40947001005043
1.16004530000000	4.43800279673973
1.23705000000000	4.43881123670197
1.29269880000000	4.55854097072357
1.36504500000000	4.65052645380426
1.43813100000000	4.74086491387146
1.49111870000000	4.77268954060935
1.53716770000000	4.79280440797281
1.61281120000000	4.77582274631906
1.66273970000000	4.79780094243468
1.73878690000000	4.81972038758784
1.79215060000000	4.83740746683777
1.84464480000000	4.85616376789465
1.91861270000000	4.87038601477063
1.96832290000000	4.93162339206160
2.03927980000000	4.93775745092926
2.08907450000000	4.91558537609342
2.13715090000000	4.93089097993523
2.19446000000000	4.93438389402621
2.26731980000000	4.93540763764190
2.31399870000000	4.93474976579536
2.36822950000000	4.93542116582549
2.44395090000000	4.93661251814027
2.51629250000000	4.92762969747306
2.56945110000000	4.93136634921536
2.64091180000000	4.93652628521015
2.68596540000000	4.94043648278148
2.73347740000000	4.95500366380115
2.78084500000000	4.95646942413570
2.82617970000000	4.95826730553405
2.87554220000000	4.95797582733957
2.92072350000000	4.95650370915464
2.96835740000000	4.95510628267875
3.01544280000000	4.95447657321026
3.06263540000000	4.95246139025144
3.10900470000000	4.95156657438156
3.15940700000000	4.95251930532138
3.20359640000000	4.95304252503134
3.25357260000000	4.95288736573008
3.29855520000000	4.95310210024413
3.34626250000000	4.95213531626587
3.40122650000000	4.95140676502475
3.44657730000000	4.95057870370702
3.49138760000000	4.95025815610661
3.57070660000000	4.95032269369161
3.64428250000000	4.95074083465547
3.69181170000000	4.95096174432673
3.73643390000000	4.95088681885993
3.78808380000000	4.95127747574198
3.83399380000000	4.95089001323739
3.88249020000000	4.95058195718698
3.93437800000000	4.95016029446226
4.00928860000000	4.94932651677280
4.05950120000000	4.94896206366936
4.10919350000000	4.94900008165323
4.16018750000000	4.94857491374962
4.20908590000000	4.94850916639453
4.26475120000000	4.94911525967985
4.31328070000000	4.94914434853249
4.35703160000000	4.94990083573402
4.40988930000000	4.95093575896612
4.46627860000000	4.95136287307887
4.54180530000000	4.95079892000712
4.61515840000000	4.94911214513096
4.67064560000000	4.94731633523216
4.71308970000000	4.94671233607467
4.78620080000000	4.94662908999490
4.84098620000000	4.94710852322468
4.91114910000000	4.94669524022625
4.96432380000000	4.94649964307830
5.01258570000000	4.94626857756290
5.09341650000000	4.94505693581467
5.13396530000000	4.94495830518737
5.17799010000000	4.94450019085877
5.23025500000000	4.94486827348201
5.30420640000000	4.94463182280006
5.44382930000000	4.94470362741070
5.49061850000000	4.94447709951594
5.53903200000000	4.94524506817213
5.59404150000000	4.94528974647679
5.65108850000000	4.94578235473745
5.69758840000000	4.94603362159362
5.74587560000000	4.94662662072298
5.79680840000000	4.94645452537930
5.93445340000000	4.94780285275711
5.97571300000000	4.94812781215256
6.04791580000000	4.94839487318039
6.10008900000000	4.94824149559212
6.15062100000000	4.94836545362548
6.19296470000000	4.94827289784847
6.24081700000000	4.94857752153202
6.31437140000000	4.94877140294215
6.38634590000000	4.95013627954779
6.43551720000000	4.95040233663464
6.48984780000000	4.95094500854862
6.56251810000000	4.95100166494727
6.61531960000000	4.94962879082734
6.66567000000000	4.94934656212631
6.71182790000000	4.94867865760224
6.76399320000000	4.94873764633892
6.80845150000000	4.94869008401916
6.85541410000000	4.94890884258659
6.90914850000000	4.94918674883230
6.98126300000000	4.94933038916490
7.03251720000000	4.94927157526108
7.09396700000000	4.94922679023351
7.14209410000000	4.94929023286982
7.18983100000000	4.94914682520746
7.23722040000000	4.94849441886807
7.28343260000000	4.94814524668893
7.33214310000000	4.94759447375422
7.37240680000000	4.94744147485515
7.42941890000000	4.94728075253963
7.48433830000000	4.94740196149418
7.53294640000000	4.94774957574898
7.57727280000000	4.94813321116824
7.62455960000000	4.94834990601281
7.66831660000000	4.94837472418432
7.71711020000000	4.94858841201660
7.76220450000000	4.94768188405754
7.80934000000000	4.94739871981847
7.85646170000000	4.94693209340038
7.90180400000000	4.94662143824202
7.94990750000000	4.94597034355759
7.99197430000000	4.94541858245655
8.06551870000000	4.94491488543953
8.11589060000000	4.94442545468910
8.16591730000000	4.94416657561737
8.21303680000000	4.94422944690088
8.27356010000000	4.94381134476714
8.31711490000000	4.94396077769189
8.37461340000000	4.94389693582737
8.42495330000000	4.94337290025038
8.47407030000000	4.94339478806530
8.52391330000000	4.94287343671722
8.59665120000000	4.94256496626497
8.64879800000000	4.94229370264706
8.69799980000000	4.94242234501201
8.75080210000000	4.94289172778704
8.79785840000000	4.94268624551473
8.84201350000000	4.94249546589017
8.89596130000000	4.94293785301128
8.94100800000000	4.94340440875220
8.98643190000000	4.94301465330969
9.03577600000000	4.94322840883859
9.08357360000000	4.94282129542026
9.12978930000000	4.94253215782993
9.17441290000000	4.94247429288088
9.22127780000000	4.94202712253653
9.27003470000000	4.94220851245795
9.31675200000000	4.94257392616126
9.36169920000000	4.94281093981094
9.40813810000000	4.94284051041270
9.45630980000000	4.94259348168853
9.50201390000000	4.94303583151848
9.54780710000000	4.94263485834691
9.59484180000000	4.94237288473623
9.64328340000000	4.94193935798892
9.69035220000000	4.94164733354636
9.73619580000000	4.94091764330253
9.78097800000000	4.94125860186696
9.82941040000000	4.94123967723039
9.87721770000000	4.94145380680110
9.92084110000000	4.94146486931585
9.96939640000000	4.94192837717599
10.0157046000000	4.94179685461292
10.0627986000000	4.94213360074515
10.1071622000000	4.94190291678920
10.1575748000000	4.94147826064046
10.2134808000000	4.94097111390283
10.2823027000000	4.94053085478868
10.3302605000000	4.94027087839770
10.3752132000000	4.94043635276735
10.4256201000000	4.94097660814021
10.4724145000000	4.94161605094075
10.5283391000000	4.94237701659132
10.5762795000000	4.94283860473035
10.6286338000000	4.94244000051561
10.6863439000000	4.94174629419027
10.7601990000000	4.94068012209927
10.8061174000000	4.93981200523151
10.8810273000000	4.93995658085636
    };

    \end{axis}

\end{tikzpicture}

%% file: FigPlot/CtrlAgl_5deg_c.tex
\begin{tikzpicture}
    \begin{axis}[
      label style={font=\footnotesize},
      tick label style={font=\footnotesize},
      xmin=0, xmax=3,
      ymin=-0.2, ymax=6,
      ymajorgrids=true,
      grid style=dashed,
      legend pos=north west,
      width = 0.26\linewidth,
      height = 0.2\linewidth,
    ]


    \addplot[color=blue, line width = 1pt] table {
0	0
0.0643425000000000	-0.000323599209366166
0.128685000000000	0.367086286677908
0.179990500000000	0.874936005498207
0.252648600000000	1.36734059283998
0.302044600000000	1.85066162914253
0.350144400000000	2.26262733336101
0.396763100000000	2.62883792788476
0.450177800000000	3.14502186161659
0.502453700000000	3.47208519473605
0.551189800000000	3.59697550018217
0.608195100000000	3.64955287437470
0.652493600000000	3.64422074212966
0.728552000000000	3.74515650788916
0.777774400000000	3.90985966690603
0.851622600000000	4.12745686950485
0.905830800000000	4.39035372750780
0.955523500000000	4.48482583902768
1.02737480000000	4.53154598135477
1.08015320000000	4.50545372905943
1.15260420000000	4.49814795994616
1.22881620000000	4.52485108227802
1.29924180000000	4.60236577406786
1.35085350000000	4.66165155754448
1.40091180000000	4.71945865607535
1.45428870000000	4.77270848109475
1.52652020000000	4.80903923292787
1.57532690000000	4.76783618320251
1.62653600000000	4.80176765944201
1.67408940000000	4.79731972508602
1.74732140000000	4.80367317053766
1.79952680000000	4.80362391288903
1.87166830000000	4.80589841907348
1.92292720000000	4.81778940757326
1.97431730000000	4.87240239157427
2.04630120000000	4.91096173793458
2.09779230000000	4.88107712616123
2.17052630000000	4.89137513234770
2.24498890000000	4.90694831453459
2.31860950000000	4.89947253568598
2.38932230000000	4.90347285273103
2.46233360000000	4.89997231602863
2.51546670000000	4.89986346141249
2.59004720000000	4.90098554486797
2.64201290000000	4.90365247969613
2.69392870000000	4.90440758181687
2.76697280000000	4.90396611822388
2.81353890000000	4.90153411767034
2.89141130000000	4.89922957813826
2.94285050000000	4.89690418238230
2.99279240000000	4.89708043628124
3.06632040000000	4.89891012587438
3.11758500000000	4.90447560431376
3.19272160000000	4.90840272527916
3.24291310000000	4.91278866242060
3.29436300000000	4.91412958696296
3.36604880000000	4.91383684174034
3.41098350000000	4.91239229470400
3.46188500000000	4.91152975803171
3.51265330000000	4.91180319850702
3.56054780000000	4.91201986695688
3.63277620000000	4.91292022375470
3.68521280000000	4.91507083301862
3.73704870000000	4.91626524535708
3.80829610000000	4.91728082596052
3.86010160000000	4.91764885798754
3.93342480000000	4.91731911311927
3.99064070000000	4.91588498670877
4.03295790000000	4.91562842523653
4.08445240000000	4.91666991165876
4.13528670000000	4.91664344512778
4.18031730000000	4.91795699369077
4.23250960000000	4.91887843913113
4.28595870000000	4.91875286900078
4.33425200000000	4.91836060972755
4.40759900000000	4.91829423536037
4.48068700000000	4.91647678004461
4.53229340000000	4.91649511565165
4.58475100000000	4.91629256964319
4.63096940000000	4.91629769619460
4.68064090000000	4.91689750948544
4.73215560000000	4.91749909203210
4.80514050000000	4.91745681596649
4.85505230000000	4.91790776012981
4.92631040000000	4.91794875773431
4.97549510000000	4.91760175837730
5.04962110000000	4.91736601941900
5.09581520000000	4.91710270285229
5.16866020000000	4.91754586275799
5.22395340000000	4.91769613627411
5.27664090000000	4.91752508451705
5.35232550000000	4.91740741120368
5.42946230000000	4.91670070464697
5.47994600000000	4.91702219911088
5.55592010000000	4.91760717506031
5.60624260000000	4.91748501607924
5.67921880000000	4.91777359650201
5.73095430000000	4.91823304600814
5.78207800000000	4.91857715183053
5.85512000000000	4.91933611350249
5.92532020000000	4.91953514590745
5.99801270000000	4.91904481455455
6.07309350000000	4.91800505922291
6.12482010000000	4.91824763030747
6.17628030000000	4.91787329906587
6.25184820000000	4.91800068287679
6.29781980000000	4.91796338468584
6.34405570000000	4.91812422877166
6.42389830000000	4.91832299693420
6.46959010000000	4.91866721015377
6.51471690000000	4.91882454688452
6.56369930000000	4.91901468505757
6.60987410000000	4.91827127416870
6.65659500000000	4.91881295185888
6.70683790000000	4.91811131678087
6.74919720000000	4.91809802076700
6.79594750000000	4.91794617544609
6.84371260000000	4.91854138952285
6.88870100000000	4.91811121264613
6.93238680000000	4.91805719289453
7.00794520000000	4.91790054806042
7.05609250000000	4.91776832766219
7.10247280000000	4.91752351441608
7.15088140000000	4.91789610024960
7.19550360000000	4.91790611826728
7.24637400000000	4.91805205900583
7.29426300000000	4.91783281301653
7.34042450000000	4.91814128599729
7.38668400000000	4.91813798553365
7.43277010000000	4.91813594310496
7.47850140000000	4.91820046607847
7.55670270000000	4.91836192982246
7.60610480000000	4.91805846724046
7.65153650000000	4.91836152915754
7.72340840000000	4.91827231857651
7.77281670000000	4.91836646511488
7.82997110000000	4.91856923074653
7.88450030000000	4.91841784268803
7.95760360000000	4.91823318671294
8.01047690000000	4.91884747790762
8.06586830000000	4.91889020880168
8.12036990000000	4.91903871206086
8.16678920000000	4.91924262356962
8.21360230000000	4.91916165429132
8.28674770000000	4.91869612904469
8.33698240000000	4.91887971359696
8.40945730000000	4.91873178884592
8.46729220000000	4.91916118102506
8.51267900000000	4.91941784539689
8.58536440000000	4.91911859385167
8.63650070000000	4.91924601089361
8.70908670000000	4.91868740291754
8.78112380000000	4.91864106856427
8.85341390000000	4.91926357095608
8.90297750000000	4.92011094800543
8.96614270000000	4.92094838055899
9.03744700000000	4.92129536523999
9.08835840000000	4.92191470268786
9.22819730000000	4.92238398734388
9.27726480000000	4.92261513670270
9.31996270000000	4.92267239667696
9.37377590000000	4.92246115655626
9.43374060000000	4.92243306984625
9.47259030000000	4.92263567964922
9.52630490000000	4.92280710842587
9.58112210000000	4.92400113740552
9.62543410000000	4.92469150375642
9.67754220000000	4.92560119168146
9.72372000000000	4.92582283689467
9.76403180000000	4.92633991598941
9.84415370000000	4.92609849062309
9.89161650000000	4.92565128154752
9.93858310000000	4.92587358121525
9.98914990000000	4.92599006385960
10.0465661000000	4.92594251026230
10.0908864000000	4.92644119473783
10.1349466000000	4.92701420839401
10.1833960000000	4.92740850252741
10.2304337000000	4.92783941451260
10.2834075000000	4.92820886030886
10.3432734000000	4.92848987719785
10.3857632000000	4.92851745768151
10.4305597000000	4.92846228190248
10.4768493000000	4.92818547852879
10.5242689000000	4.92839657424089
10.5717099000000	4.92839273116260
10.6223214000000	4.92766670284667
10.6846309000000	4.92751077559445
10.7448955000000	4.92722613952469
10.8004118000000	4.92728902979007
10.8474494000000	4.92722673653006
10.8954420000000	4.92715991484203
10.9468326000000	4.92708213683426
10.9976671000000	4.92725831238574
11.0437357000000	4.92703471813408
11.0897934000000	4.92639813249518
11.1385333000000	4.92682224243959
11.1808989000000	4.92661728113457
11.2546059000000	4.92725577374360
11.3054231000000	4.92692198039331

    };

    \end{axis}
\end{tikzpicture}

%% file: FigPlot/CtrlAgl_5deg_d.tex
\begin{tikzpicture}
    \begin{axis}[
      label style={font=\footnotesize},
      tick label style={font=\footnotesize},
      xmin=0, xmax=3,
      ymin=-0.2, ymax=6,
      ymajorgrids=true,
      grid style=dashed,
      legend pos=north west,
      width = 0.26\linewidth,
      height = 0.2\linewidth,
    ]


    \addplot[color=blue, line width = 1pt] table {
0	0
0.0699137000000000	0.0695829824538597
0.139827400000000	0.00742801435080527
0.194102700000000	0.904543720503210
0.243916900000000	1.28429008089629
0.318287600000000	1.72595253344791
0.367103200000000	2.66045736236351
0.443274200000000	3.00496096643265
0.492855500000000	3.57568939372661
0.545328700000000	3.64838495881620
0.615871000000000	3.44818945139158
0.665425800000000	3.28080238871453
0.737518800000000	3.40807863633447
0.782812400000000	3.63410775437722
0.842091600000000	4.03203461717936
0.915809700000000	4.49375495767409
0.966468700000000	4.67044415737199
1.03915690000000	4.48157926235574
1.09039560000000	4.32252794306582
1.14207730000000	4.41164192651038
1.18624300000000	4.41036839188663
1.23742990000000	4.44648166756981
1.29333670000000	4.55501441413845
1.36637840000000	4.71280366382640
1.41663260000000	4.84957205200941
1.46470550000000	4.83855050905946
1.53808770000000	4.81432982748483
1.61079750000000	4.78490694330609
1.66405270000000	4.76115951075144
1.71628120000000	4.75789627529267
1.78681480000000	4.78730213177791
1.83741000000000	4.82668888397286
1.88954520000000	4.84646765411576
1.94469740000000	4.87401974872372
2.02077230000000	4.88919110199932
2.06874850000000	4.90042475167124
2.14151810000000	4.89690985353244
2.19046000000000	4.88282833434550
2.26387860000000	4.87251993192899
2.31193920000000	4.87943809503698
2.38382540000000	4.88114680672221
2.42911650000000	4.87975968092012
2.50460640000000	4.88164722132076
2.55386040000000	4.93650927441771
2.62781960000000	4.94220166347850
2.67542040000000	4.94610146658095
2.73018150000000	4.93834467604327
2.80759430000000	4.91950616844546
2.85714990000000	4.91559474526264
2.93292070000000	4.91009647216892
3.00625250000000	4.91433393503578
3.08050290000000	4.91276239070212
3.13096320000000	4.92845319054010
3.20954100000000	4.92966086318551
3.28319390000000	4.92427328812564
3.33546770000000	4.92595451851365
3.40944320000000	4.92556159234550
3.45846300000000	4.92152204782944
3.53127440000000	4.91949220720298
3.57886240000000	4.91805065860113
3.62521450000000	4.91874307884969
3.67860200000000	4.92261427949738
3.73051010000000	4.92533143585468
3.77912090000000	4.92853370712067
3.85302030000000	4.93100317161255
3.90209650000000	4.93179722328969
3.97713010000000	4.93027941859968
4.02614210000000	4.92826445033546
4.07917320000000	4.92714242216167
4.15102430000000	4.92632076003954
4.19735050000000	4.92688272596155
4.24341020000000	4.92774870133829
4.29729610000000	4.92850478297130
4.36837350000000	4.92944476201945
4.42018700000000	4.92972965010466
4.47241690000000	4.92858407290506
4.54443800000000	4.92808823434217
4.59277480000000	4.92698388145191
4.66432510000000	4.92584789558756
4.71592330000000	4.92667017965943
4.79045940000000	4.92689148172703
4.83558790000000	4.92821149833239
4.91186060000000	4.92950666274551
4.96080200000000	4.92999150643396
5.03008380000000	4.93027796370892
5.08021380000000	4.93117087249784
5.13469410000000	4.93116835362986
5.18084930000000	4.93148213224100
5.23046010000000	4.93229590395648
5.28126140000000	4.93242290598004
5.35332670000000	4.93307392343141
5.40415820000000	4.93434799075294
5.45042660000000	4.93541593795178
5.50249240000000	4.93602201978368
5.55464570000000	4.93702516216057
5.62790190000000	4.93687798995352
5.67754230000000	4.93667240884940
5.72970100000000	4.93631477335762
5.78081820000000	4.93654890793159
5.83111440000000	4.93706981482628
5.88087670000000	4.93697262681137
5.95273140000000	4.93752028220045
6.00030180000000	4.93755589504437
6.04896470000000	4.93801861073786
6.09961320000000	4.93816182758953
6.15256720000000	4.93828114179663
6.22651190000000	4.93800804875306
6.27682850000000	4.93711487349554
6.34885270000000	4.93725328450543
6.39690600000000	4.93718983312838
6.44872540000000	4.93772511923670
6.52333810000000	4.93769297275693
6.57160050000000	4.93834716852281
6.64542100000000	4.93835545416495
6.69113550000000	4.93823212763622
6.74693520000000	4.93833040268483
6.82225890000000	4.93803750058179
6.87245330000000	4.93801707564258
6.94514420000000	4.93795950716580
7.01563640000000	4.93717498333986
7.09011700000000	4.93674980629963
7.14254890000000	4.93656640374146
7.21533040000000	4.93760097387709
7.26448740000000	4.93902318367436
7.32049080000000	4.94027120203984
7.39445410000000	4.94193133595635
7.44117300000000	4.94266887559995
7.51739930000000	4.94266903810411
7.56822100000000	4.94210521290933
7.63955130000000	4.94211184845650
7.71922830000000	4.94115722880780
7.79069860000000	4.94078317247567
7.84410460000000	4.94163205022576
7.89107560000000	4.94134831809123
7.96379290000000	4.94241657447546
8.01673110000000	4.94257770150424
8.09096280000000	4.94263697170823
8.14213200000000	4.94303342213661
8.18757810000000	4.94261693740586
8.26129620000000	4.94251532775703
8.31260930000000	4.94228891087882
8.38675750000000	4.94320235225498
8.43807680000000	4.94471291943045
8.48338650000000	4.94578948302783
8.53913100000000	4.94722152976082
8.58833550000000	4.94750858452701
8.63805540000000	4.94822369883824
8.69103920000000	4.94821186433935
8.74064900000000	4.94882568696393
8.78978460000000	4.94846777498711
8.86457880000000	4.94847673041092
8.91526890000000	4.94827143984031
8.98770450000000	4.94792257396559
9.03914410000000	4.94829438135668
9.08520050000000	4.94862153665730
9.15839670000000	4.94904525983142
9.21065370000000	4.94999303498586
9.28312770000000	4.95097822389924
9.33757670000000	4.95249849717607
9.39246670000000	4.95341680712646
9.46612550000000	4.95337748929286
9.51787250000000	4.95283599454807
9.59048520000000	4.95212795317848
9.63931710000000	4.95149568440121
9.69090780000000	4.95120415477935
9.76187610000000	4.95192564787996
9.81277130000000	4.95164853438464
9.88740460000000	4.95153562874616
9.93451580000000	4.95148763231204
10.0084573000000	4.95107488319626
10.0592610000000	4.95111161666863
10.1117554000000	4.95128330088890
10.1858978000000	4.95133538794188
10.2367241000000	4.95125757399710
10.2923635000000	4.95048590593296
10.3373338000000	4.94983754197129
10.3872943000000	4.94975050344027
10.4652493000000	4.94967709562230
10.5216448000000	4.95000004248391
10.5694969000000	4.95051742344125
10.6202200000000	4.95056033794297
10.6663138000000	4.95059746208255
10.7123956000000	4.95111174336543
10.7615487000000	4.95141728531145
10.8071981000000	4.95145174934436
10.8557634000000	4.95134616325238
10.9017323000000	4.95139993119884
10.9468140000000	4.95107787157061
10.9931837000000	4.95111824729736
11.0414305000000	4.95136159858582
11.0861015000000	4.95162092167447
11.1329093000000	4.95223723087526
11.1846944000000	4.95203930312746
11.2273296000000	4.95224211604115
11.2721457000000	4.95232078815673
11.3150864000000	4.95262723703068
11.3618545000000	4.95237490064763
11.4100199000000	4.95228482709713
11.4560203000000	4.95286808195191
11.5007451000000	4.95263015832138

    };

    \end{axis}
\end{tikzpicture}

%% file: FigPlot/5deg_ss-1.tex
\begin{tikzpicture}
    \begin{axis}[
        label style={font=\footnotesize},
      tick label style={font=\footnotesize},
      xmin=2, xmax=8,
      ymin=4.85, ymax=5,
      ymajorgrids=true,
      grid style=dashed,
      legend pos=north west,
      width = 0.26\linewidth,
      height = 0.2\linewidth,
   ylabel={Steady-state $q_i$ [deg]}
    ]


    \addplot[color=blue, line width = 1pt] table {
0	0
0.0487324000000000	0.0571194725109317
0.0974648000000000	0.00716201049496108
0.150304600000000	0.591477696607778
0.206128400000000	1.34988812935604
0.253611900000000	1.78215830048176
0.305590900000000	2.27855511893232
0.351815300000000	2.50615689245564
0.402593400000000	3.20712615641229
0.473879300000000	3.45772108911153
0.521701700000000	3.59367461800630
0.576204500000000	3.55209649150558
0.621306600000000	3.46933947732990
0.679257300000000	3.54101966249684
0.753773900000000	3.67746724460794
0.801874300000000	3.88551853741958
0.875632900000000	4.20135059562660
0.920843800000000	4.41580150078830
0.973573500000000	4.49527782301339
1.02277970000000	4.51017780232997
1.09737210000000	4.49152057200293
1.14822380000000	4.48601619082811
1.19960290000000	4.49060108949309
1.26341830000000	4.57283718154546
1.34012220000000	4.65707087439488
1.39105240000000	4.74824147319439
1.44134890000000	4.78523484230877
1.51456050000000	4.83468907686631
1.58755350000000	4.84720482308325
1.66072420000000	4.85030904448449
1.71099380000000	4.85075879714618
1.78502430000000	4.85278309169274
1.84310370000000	4.87547161856021
1.89349410000000	4.85479415771481
1.96739920000000	4.92516244966948
2.01623760000000	4.95141383989991
2.06677470000000	4.97309186674926
2.13957490000000	4.96860339695397
2.18988980000000	4.95325211503427
2.26262350000000	4.96111589362212
2.31145540000000	4.94012373428327
2.38696080000000	4.93922642504921
2.45719020000000	4.94128379142853
2.53128950000000	4.97995321555521
2.58210240000000	4.99136191548834
2.63070030000000	4.98630477143591
2.68108080000000	4.97212414759924
2.72986300000000	4.96843204637111
2.80388460000000	4.96104703927220
2.85519870000000	4.96577522099735
2.92950870000000	4.96408994278639
2.98101730000000	4.96161027650487
3.02886160000000	4.95922184658255
3.10396570000000	4.95911031096950
3.15479550000000	4.95888728931020
3.22768620000000	4.95827001294337
3.27975280000000	4.95835702065788
3.33168450000000	4.95903067374719
3.40513040000000	4.95959358625474
3.45427510000000	4.96018801869675
3.52847130000000	4.96022301177717
3.57386590000000	4.95977494422954
3.62581250000000	4.95981114454899
3.68250200000000	4.95945504253893
3.73351150000000	4.96003835430557
3.78411790000000	4.95992358953660
3.83403450000000	4.95981327805675
3.90944840000000	4.96054451124211
3.95267110000000	4.96018278474058
4.02749720000000	4.96032445497862
4.07862070000000	4.96044616027851
4.13442500000000	4.95961845543031
4.20690800000000	4.95930488618330
4.25480820000000	4.95887750900481
4.32944060000000	4.95869941311035
4.37556050000000	4.95851361169176
4.44908110000000	4.95919188398835
4.49793060000000	4.95973904273607
4.57176340000000	4.95996654694631
4.62296640000000	4.96025776748601
4.69685730000000	4.95938088838603
4.75160990000000	4.95859672675151
4.80374250000000	4.95821929158789
4.87659520000000	4.95812947859219
4.92746240000000	4.95844241390669
4.98141700000000	4.95884243803438
5.05133400000000	4.95904753315599
5.10318640000000	4.95913459027288
5.17699030000000	4.95860774523223
5.22735680000000	4.95818580727223
5.28349340000000	4.95806952137343
5.34107430000000	4.95728174222656
5.39463210000000	4.95740652800301
5.46921630000000	4.95745306165965
5.52100140000000	4.95821237650412
5.59321130000000	4.95880222905750
5.63858930000000	4.95906602110536
5.71161120000000	4.95905109585246
5.76349570000000	4.95855688153853
5.83793340000000	4.95750170479235
5.89117050000000	4.95700846541048
5.94000130000000	4.95669206263883
6.01058080000000	4.95668683175103
6.06368430000000	4.95701250133930
6.13681840000000	4.95788374541570
6.18662990000000	4.95879997144536
6.25728180000000	4.95932078419700
6.30980020000000	4.95946164877826
6.36388870000000	4.95887809891190
6.40977320000000	4.95781884921095
6.46053680000000	4.95686916826692
6.50955030000000	4.95635034665031
6.58343870000000	4.95613259016851
6.63542220000000	4.95653991067689
6.68901350000000	4.95648778895858
6.74046670000000	4.95649367914884
6.79176900000000	4.95629156787002
6.84142170000000	4.95572599575297
6.91480810000000	4.95535549453097
6.96816780000000	4.95484246845093
7.00973040000000	4.95456946332923
7.05867330000000	4.95403415918045
7.11133440000000	4.95397484457256
7.18544300000000	4.95431335182668
7.23620790000000	4.95489237634417
7.29016600000000	4.95557044849587
7.34350370000000	4.95562074585908
7.41738410000000	4.95480810218885
7.46959410000000	4.95418127446411
7.54234150000000	4.95316829015513
7.59555280000000	4.95200708451099
7.64339490000000	4.95209933812858
7.71462220000000	4.95225394152961
7.78803370000000	4.95277589991744
7.84034940000000	4.95316054935106
7.89046980000000	4.95317543459079
7.94242980000000	4.95184447182015
8.01722950000000	4.95083423538828
8.08872840000000	4.94930556082976
8.14115190000000	4.94826582527981
8.19378590000000	4.94766965311757
8.23903440000000	4.94738999576557
8.28772880000000	4.94674764199208
8.33547410000000	4.94614883099921
8.41308920000000	4.94613222689715
8.46307290000000	4.94637824877810
8.53755090000000	4.94651066588923
8.58940150000000	4.94598381440933
8.64484310000000	4.94587767012979
8.71903830000000	4.94593600722323
8.76757640000000	4.94690762969947
8.81907150000000	4.94890977579475
8.86252430000000	4.94963837641798
8.91464090000000	4.95106323057676
8.98859260000000	4.95360430394290
9.03989700000000	4.95502526321103
9.09159680000000	4.95569693788836
9.14326220000000	4.95702684558497
9.21785050000000	4.95683544997571
9.26509340000000	4.95717523213427
9.33814230000000	4.95831082405184
9.39003030000000	4.95872517771563
9.43812240000000	4.95931397170452
9.49193690000000	4.95998777863754
9.54377220000000	4.96035780871557
9.61904710000000	4.96040588832609
9.66864780000000	4.95965339320678
9.71964000000000	4.95905545353369
9.79305270000000	4.95825944091683
9.84562940000000	4.95810582172400
9.91673040000000	4.95855046871308
9.96663690000000	4.95990840014358
10.0373609000000	4.95998914533498
10.0879254000000	4.96095255093484
10.1585437000000	4.96123500857680
10.2067080000000	4.96067668971093
10.2791101000000	4.95987550023394
10.3317213000000	4.95882229010837
10.3828650000000	4.95803741934602
10.4565585000000	4.95817915410135
10.5077424000000	4.95793400030457
10.5646813000000	4.95772603312093
10.6379389000000	4.95840530261400
10.6841437000000	4.95817271595357
10.7594528000000	4.95773050327384
10.8108819000000	4.95746068552987
10.8647451000000	4.95637134611249
10.9372444000000	4.95528293509485
10.9904679000000	4.95554709700483
11.0404265000000	4.95508227825593
11.1148392000000	4.95526942330849
11.1646992000000	4.95629494387902
11.2377483000000	4.95644751379434
11.2868871000000	4.95735026403333
11.3582398000000	4.95709027744847
11.4069668000000	4.95649347536131
11.4808795000000	4.95578386988353
11.5554047000000	4.95511939962895
11.6259108000000	4.95521284798879
11.6789337000000	4.95569092215540
    };

    \end{axis}

\end{tikzpicture}

%% file: FigPlot/5deg_ss-2.tex
\begin{tikzpicture}
   \begin{axis}[
      label style={font=\footnotesize},
      tick label style={font=\footnotesize},
      xmin=2, xmax=8,
      ymin=4.85, ymax=5,
      ymajorgrids=true,
      grid style=dashed,
      legend pos=north west,
      width = 0.26\linewidth,
      height = 0.2\linewidth,
    ]


    \addplot[color=blue, line width = 1pt] table {
0	0
0.0702578000000000	0.00433362481546621
0.140515600000000	0.532892025400373
0.187789100000000	0.973648389621454
0.262076000000000	1.55512253420422
0.308745600000000	1.97253820495782
0.364577100000000	2.33680483965311
0.437808400000000	2.93325149254827
0.487995200000000	3.20000474247410
0.544334000000000	3.30811170314803
0.615377000000000	3.36511098419767
0.668287200000000	3.41431550877347
0.741237400000000	3.53711922921779
0.793452400000000	3.72248421611395
0.863594800000000	3.99804331456495
0.912167200000000	4.22259730731308
0.984637700000000	4.35945067693008
1.03772920000000	4.42996381119947
1.08767120000000	4.40947001005043
1.16004530000000	4.43800279673973
1.23705000000000	4.43881123670197
1.29269880000000	4.55854097072357
1.36504500000000	4.65052645380426
1.43813100000000	4.74086491387146
1.49111870000000	4.77268954060935
1.53716770000000	4.79280440797281
1.61281120000000	4.77582274631906
1.66273970000000	4.79780094243468
1.73878690000000	4.81972038758784
1.79215060000000	4.83740746683777
1.84464480000000	4.85616376789465
1.91861270000000	4.87038601477063
1.96832290000000	4.93162339206160
2.03927980000000	4.93775745092926
2.08907450000000	4.91558537609342
2.13715090000000	4.93089097993523
2.19446000000000	4.93438389402621
2.26731980000000	4.93540763764190
2.31399870000000	4.93474976579536
2.36822950000000	4.93542116582549
2.44395090000000	4.93661251814027
2.51629250000000	4.92762969747306
2.56945110000000	4.93136634921536
2.64091180000000	4.93652628521015
2.68596540000000	4.94043648278148
2.73347740000000	4.95500366380115
2.78084500000000	4.95646942413570
2.82617970000000	4.95826730553405
2.87554220000000	4.95797582733957
2.92072350000000	4.95650370915464
2.96835740000000	4.95510628267875
3.01544280000000	4.95447657321026
3.06263540000000	4.95246139025144
3.10900470000000	4.95156657438156
3.15940700000000	4.95251930532138
3.20359640000000	4.95304252503134
3.25357260000000	4.95288736573008
3.29855520000000	4.95310210024413
3.34626250000000	4.95213531626587
3.40122650000000	4.95140676502475
3.44657730000000	4.95057870370702
3.49138760000000	4.95025815610661
3.57070660000000	4.95032269369161
3.64428250000000	4.95074083465547
3.69181170000000	4.95096174432673
3.73643390000000	4.95088681885993
3.78808380000000	4.95127747574198
3.83399380000000	4.95089001323739
3.88249020000000	4.95058195718698
3.93437800000000	4.95016029446226
4.00928860000000	4.94932651677280
4.05950120000000	4.94896206366936
4.10919350000000	4.94900008165323
4.16018750000000	4.94857491374962
4.20908590000000	4.94850916639453
4.26475120000000	4.94911525967985
4.31328070000000	4.94914434853249
4.35703160000000	4.94990083573402
4.40988930000000	4.95093575896612
4.46627860000000	4.95136287307887
4.54180530000000	4.95079892000712
4.61515840000000	4.94911214513096
4.67064560000000	4.94731633523216
4.71308970000000	4.94671233607467
4.78620080000000	4.94662908999490
4.84098620000000	4.94710852322468
4.91114910000000	4.94669524022625
4.96432380000000	4.94649964307830
5.01258570000000	4.94626857756290
5.09341650000000	4.94505693581467
5.13396530000000	4.94495830518737
5.17799010000000	4.94450019085877
5.23025500000000	4.94486827348201
5.30420640000000	4.94463182280006
5.44382930000000	4.94470362741070
5.49061850000000	4.94447709951594
5.53903200000000	4.94524506817213
5.59404150000000	4.94528974647679
5.65108850000000	4.94578235473745
5.69758840000000	4.94603362159362
5.74587560000000	4.94662662072298
5.79680840000000	4.94645452537930
5.93445340000000	4.94780285275711
5.97571300000000	4.94812781215256
6.04791580000000	4.94839487318039
6.10008900000000	4.94824149559212
6.15062100000000	4.94836545362548
6.19296470000000	4.94827289784847
6.24081700000000	4.94857752153202
6.31437140000000	4.94877140294215
6.38634590000000	4.95013627954779
6.43551720000000	4.95040233663464
6.48984780000000	4.95094500854862
6.56251810000000	4.95100166494727
6.61531960000000	4.94962879082734
6.66567000000000	4.94934656212631
6.71182790000000	4.94867865760224
6.76399320000000	4.94873764633892
6.80845150000000	4.94869008401916
6.85541410000000	4.94890884258659
6.90914850000000	4.94918674883230
6.98126300000000	4.94933038916490
7.03251720000000	4.94927157526108
7.09396700000000	4.94922679023351
7.14209410000000	4.94929023286982
7.18983100000000	4.94914682520746
7.23722040000000	4.94849441886807
7.28343260000000	4.94814524668893
7.33214310000000	4.94759447375422
7.37240680000000	4.94744147485515
7.42941890000000	4.94728075253963
7.48433830000000	4.94740196149418
7.53294640000000	4.94774957574898
7.57727280000000	4.94813321116824
7.62455960000000	4.94834990601281
7.66831660000000	4.94837472418432
7.71711020000000	4.94858841201660
7.76220450000000	4.94768188405754
7.80934000000000	4.94739871981847
7.85646170000000	4.94693209340038
7.90180400000000	4.94662143824202
7.94990750000000	4.94597034355759
7.99197430000000	4.94541858245655
8.06551870000000	4.94491488543953
8.11589060000000	4.94442545468910
8.16591730000000	4.94416657561737
8.21303680000000	4.94422944690088
8.27356010000000	4.94381134476714
8.31711490000000	4.94396077769189
8.37461340000000	4.94389693582737
8.42495330000000	4.94337290025038
8.47407030000000	4.94339478806530
8.52391330000000	4.94287343671722
8.59665120000000	4.94256496626497
8.64879800000000	4.94229370264706
8.69799980000000	4.94242234501201
8.75080210000000	4.94289172778704
8.79785840000000	4.94268624551473
8.84201350000000	4.94249546589017
8.89596130000000	4.94293785301128
8.94100800000000	4.94340440875220
8.98643190000000	4.94301465330969
9.03577600000000	4.94322840883859
9.08357360000000	4.94282129542026
9.12978930000000	4.94253215782993
9.17441290000000	4.94247429288088
9.22127780000000	4.94202712253653
9.27003470000000	4.94220851245795
9.31675200000000	4.94257392616126
9.36169920000000	4.94281093981094
9.40813810000000	4.94284051041270
9.45630980000000	4.94259348168853
9.50201390000000	4.94303583151848
9.54780710000000	4.94263485834691
9.59484180000000	4.94237288473623
9.64328340000000	4.94193935798892
9.69035220000000	4.94164733354636
9.73619580000000	4.94091764330253
9.78097800000000	4.94125860186696
9.82941040000000	4.94123967723039
9.87721770000000	4.94145380680110
9.92084110000000	4.94146486931585
9.96939640000000	4.94192837717599
10.0157046000000	4.94179685461292
10.0627986000000	4.94213360074515
10.1071622000000	4.94190291678920
10.1575748000000	4.94147826064046
10.2134808000000	4.94097111390283
10.2823027000000	4.94053085478868
10.3302605000000	4.94027087839770
10.3752132000000	4.94043635276735
10.4256201000000	4.94097660814021
10.4724145000000	4.94161605094075
10.5283391000000	4.94237701659132
10.5762795000000	4.94283860473035
10.6286338000000	4.94244000051561
10.6863439000000	4.94174629419027
10.7601990000000	4.94068012209927
10.8061174000000	4.93981200523151
10.8810273000000	4.93995658085636
    };

    \end{axis}

\end{tikzpicture}

%% file: FigPlot/5deg_ss-3.tex
\begin{tikzpicture}
    \begin{axis}[
      label style={font=\footnotesize},
      tick label style={font=\footnotesize},
      xmin=2, xmax=8,
      ymin=4.85, ymax=5,
      ymajorgrids=true,
      grid style=dashed,
      legend pos=north west,
      width = 0.26\linewidth,
      height = 0.2\linewidth,
 y label style={at={(0.2,0.5)}},
    ]


    \addplot[color=blue, line width = 1pt] table {
0	0
0.0643425000000000	-0.000323599209366166
0.128685000000000	0.367086286677908
0.179990500000000	0.874936005498207
0.252648600000000	1.36734059283998
0.302044600000000	1.85066162914253
0.350144400000000	2.26262733336101
0.396763100000000	2.62883792788476
0.450177800000000	3.14502186161659
0.502453700000000	3.47208519473605
0.551189800000000	3.59697550018217
0.608195100000000	3.64955287437470
0.652493600000000	3.64422074212966
0.728552000000000	3.74515650788916
0.777774400000000	3.90985966690603
0.851622600000000	4.12745686950485
0.905830800000000	4.39035372750780
0.955523500000000	4.48482583902768
1.02737480000000	4.53154598135477
1.08015320000000	4.50545372905943
1.15260420000000	4.49814795994616
1.22881620000000	4.52485108227802
1.29924180000000	4.60236577406786
1.35085350000000	4.66165155754448
1.40091180000000	4.71945865607535
1.45428870000000	4.77270848109475
1.52652020000000	4.80903923292787
1.57532690000000	4.76783618320251
1.62653600000000	4.80176765944201
1.67408940000000	4.79731972508602
1.74732140000000	4.80367317053766
1.79952680000000	4.80362391288903
1.87166830000000	4.80589841907348
1.92292720000000	4.81778940757326
1.97431730000000	4.87240239157427
2.04630120000000	4.91096173793458
2.09779230000000	4.88107712616123
2.17052630000000	4.89137513234770
2.24498890000000	4.90694831453459
2.31860950000000	4.89947253568598
2.38932230000000	4.90347285273103
2.46233360000000	4.89997231602863
2.51546670000000	4.89986346141249
2.59004720000000	4.90098554486797
2.64201290000000	4.90365247969613
2.69392870000000	4.90440758181687
2.76697280000000	4.90396611822388
2.81353890000000	4.90153411767034
2.89141130000000	4.89922957813826
2.94285050000000	4.89690418238230
2.99279240000000	4.89708043628124
3.06632040000000	4.89891012587438
3.11758500000000	4.90447560431376
3.19272160000000	4.90840272527916
3.24291310000000	4.91278866242060
3.29436300000000	4.91412958696296
3.36604880000000	4.91383684174034
3.41098350000000	4.91239229470400
3.46188500000000	4.91152975803171
3.51265330000000	4.91180319850702
3.56054780000000	4.91201986695688
3.63277620000000	4.91292022375470
3.68521280000000	4.91507083301862
3.73704870000000	4.91626524535708
3.80829610000000	4.91728082596052
3.86010160000000	4.91764885798754
3.93342480000000	4.91731911311927
3.99064070000000	4.91588498670877
4.03295790000000	4.91562842523653
4.08445240000000	4.91666991165876
4.13528670000000	4.91664344512778
4.18031730000000	4.91795699369077
4.23250960000000	4.91887843913113
4.28595870000000	4.91875286900078
4.33425200000000	4.91836060972755
4.40759900000000	4.91829423536037
4.48068700000000	4.91647678004461
4.53229340000000	4.91649511565165
4.58475100000000	4.91629256964319
4.63096940000000	4.91629769619460
4.68064090000000	4.91689750948544
4.73215560000000	4.91749909203210
4.80514050000000	4.91745681596649
4.85505230000000	4.91790776012981
4.92631040000000	4.91794875773431
4.97549510000000	4.91760175837730
5.04962110000000	4.91736601941900
5.09581520000000	4.91710270285229
5.16866020000000	4.91754586275799
5.22395340000000	4.91769613627411
5.27664090000000	4.91752508451705
5.35232550000000	4.91740741120368
5.42946230000000	4.91670070464697
5.47994600000000	4.91702219911088
5.55592010000000	4.91760717506031
5.60624260000000	4.91748501607924
5.67921880000000	4.91777359650201
5.73095430000000	4.91823304600814
5.78207800000000	4.91857715183053
5.85512000000000	4.91933611350249
5.92532020000000	4.91953514590745
5.99801270000000	4.91904481455455
6.07309350000000	4.91800505922291
6.12482010000000	4.91824763030747
6.17628030000000	4.91787329906587
6.25184820000000	4.91800068287679
6.29781980000000	4.91796338468584
6.34405570000000	4.91812422877166
6.42389830000000	4.91832299693420
6.46959010000000	4.91866721015377
6.51471690000000	4.91882454688452
6.56369930000000	4.91901468505757
6.60987410000000	4.91827127416870
6.65659500000000	4.91881295185888
6.70683790000000	4.91811131678087
6.74919720000000	4.91809802076700
6.79594750000000	4.91794617544609
6.84371260000000	4.91854138952285
6.88870100000000	4.91811121264613
6.93238680000000	4.91805719289453
7.00794520000000	4.91790054806042
7.05609250000000	4.91776832766219
7.10247280000000	4.91752351441608
7.15088140000000	4.91789610024960
7.19550360000000	4.91790611826728
7.24637400000000	4.91805205900583
7.29426300000000	4.91783281301653
7.34042450000000	4.91814128599729
7.38668400000000	4.91813798553365
7.43277010000000	4.91813594310496
7.47850140000000	4.91820046607847
7.55670270000000	4.91836192982246
7.60610480000000	4.91805846724046
7.65153650000000	4.91836152915754
7.72340840000000	4.91827231857651
7.77281670000000	4.91836646511488
7.82997110000000	4.91856923074653
7.88450030000000	4.91841784268803
7.95760360000000	4.91823318671294
8.01047690000000	4.91884747790762
8.06586830000000	4.91889020880168
8.12036990000000	4.91903871206086
8.16678920000000	4.91924262356962
8.21360230000000	4.91916165429132
8.28674770000000	4.91869612904469
8.33698240000000	4.91887971359696
8.40945730000000	4.91873178884592
8.46729220000000	4.91916118102506
8.51267900000000	4.91941784539689
8.58536440000000	4.91911859385167
8.63650070000000	4.91924601089361
8.70908670000000	4.91868740291754
8.78112380000000	4.91864106856427
8.85341390000000	4.91926357095608
8.90297750000000	4.92011094800543
8.96614270000000	4.92094838055899
9.03744700000000	4.92129536523999
9.08835840000000	4.92191470268786
9.22819730000000	4.92238398734388
9.27726480000000	4.92261513670270
9.31996270000000	4.92267239667696
9.37377590000000	4.92246115655626
9.43374060000000	4.92243306984625
9.47259030000000	4.92263567964922
9.52630490000000	4.92280710842587
9.58112210000000	4.92400113740552
9.62543410000000	4.92469150375642
9.67754220000000	4.92560119168146
9.72372000000000	4.92582283689467
9.76403180000000	4.92633991598941
9.84415370000000	4.92609849062309
9.89161650000000	4.92565128154752
9.93858310000000	4.92587358121525
9.98914990000000	4.92599006385960
10.0465661000000	4.92594251026230
10.0908864000000	4.92644119473783
10.1349466000000	4.92701420839401
10.1833960000000	4.92740850252741
10.2304337000000	4.92783941451260
10.2834075000000	4.92820886030886
10.3432734000000	4.92848987719785
10.3857632000000	4.92851745768151
10.4305597000000	4.92846228190248
10.4768493000000	4.92818547852879
10.5242689000000	4.92839657424089
10.5717099000000	4.92839273116260
10.6223214000000	4.92766670284667
10.6846309000000	4.92751077559445
10.7448955000000	4.92722613952469
10.8004118000000	4.92728902979007
10.8474494000000	4.92722673653006
10.8954420000000	4.92715991484203
10.9468326000000	4.92708213683426
10.9976671000000	4.92725831238574
11.0437357000000	4.92703471813408
11.0897934000000	4.92639813249518
11.1385333000000	4.92682224243959
11.1808989000000	4.92661728113457
11.2546059000000	4.92725577374360
11.3054231000000	4.92692198039331

    };

    \end{axis}
\end{tikzpicture}

%% file: FigPlot/5deg_ss-4.tex
\begin{tikzpicture}
    \begin{axis}[
      label style={font=\footnotesize},
      tick label style={font=\footnotesize},
      xmin=2, xmax=8,
      ymin=4.85, ymax=5,
      ymajorgrids=true,
      grid style=dashed,
      legend pos=north west,
      width = 0.26\linewidth,
      height = 0.2\linewidth,
 y label style={at={(0.2,0.5)}},
    ]


    \addplot[color=blue, line width = 1pt] table {
0	0
0.0699137000000000	0.0695829824538597
0.139827400000000	0.00742801435080527
0.194102700000000	0.904543720503210
0.243916900000000	1.28429008089629
0.318287600000000	1.72595253344791
0.367103200000000	2.66045736236351
0.443274200000000	3.00496096643265
0.492855500000000	3.57568939372661
0.545328700000000	3.64838495881620
0.615871000000000	3.44818945139158
0.665425800000000	3.28080238871453
0.737518800000000	3.40807863633447
0.782812400000000	3.63410775437722
0.842091600000000	4.03203461717936
0.915809700000000	4.49375495767409
0.966468700000000	4.67044415737199
1.03915690000000	4.48157926235574
1.09039560000000	4.32252794306582
1.14207730000000	4.41164192651038
1.18624300000000	4.41036839188663
1.23742990000000	4.44648166756981
1.29333670000000	4.55501441413845
1.36637840000000	4.71280366382640
1.41663260000000	4.84957205200941
1.46470550000000	4.83855050905946
1.53808770000000	4.81432982748483
1.61079750000000	4.78490694330609
1.66405270000000	4.76115951075144
1.71628120000000	4.75789627529267
1.78681480000000	4.78730213177791
1.83741000000000	4.82668888397286
1.88954520000000	4.84646765411576
1.94469740000000	4.87401974872372
2.02077230000000	4.88919110199932
2.06874850000000	4.90042475167124
2.14151810000000	4.89690985353244
2.19046000000000	4.88282833434550
2.26387860000000	4.87251993192899
2.31193920000000	4.87943809503698
2.38382540000000	4.88114680672221
2.42911650000000	4.87975968092012
2.50460640000000	4.88164722132076
2.55386040000000	4.93650927441771
2.62781960000000	4.94220166347850
2.67542040000000	4.94610146658095
2.73018150000000	4.93834467604327
2.80759430000000	4.91950616844546
2.85714990000000	4.91559474526264
2.93292070000000	4.91009647216892
3.00625250000000	4.91433393503578
3.08050290000000	4.91276239070212
3.13096320000000	4.92845319054010
3.20954100000000	4.92966086318551
3.28319390000000	4.92427328812564
3.33546770000000	4.92595451851365
3.40944320000000	4.92556159234550
3.45846300000000	4.92152204782944
3.53127440000000	4.91949220720298
3.57886240000000	4.91805065860113
3.62521450000000	4.91874307884969
3.67860200000000	4.92261427949738
3.73051010000000	4.92533143585468
3.77912090000000	4.92853370712067
3.85302030000000	4.93100317161255
3.90209650000000	4.93179722328969
3.97713010000000	4.93027941859968
4.02614210000000	4.92826445033546
4.07917320000000	4.92714242216167
4.15102430000000	4.92632076003954
4.19735050000000	4.92688272596155
4.24341020000000	4.92774870133829
4.29729610000000	4.92850478297130
4.36837350000000	4.92944476201945
4.42018700000000	4.92972965010466
4.47241690000000	4.92858407290506
4.54443800000000	4.92808823434217
4.59277480000000	4.92698388145191
4.66432510000000	4.92584789558756
4.71592330000000	4.92667017965943
4.79045940000000	4.92689148172703
4.83558790000000	4.92821149833239
4.91186060000000	4.92950666274551
4.96080200000000	4.92999150643396
5.03008380000000	4.93027796370892
5.08021380000000	4.93117087249784
5.13469410000000	4.93116835362986
5.18084930000000	4.93148213224100
5.23046010000000	4.93229590395648
5.28126140000000	4.93242290598004
5.35332670000000	4.93307392343141
5.40415820000000	4.93434799075294
5.45042660000000	4.93541593795178
5.50249240000000	4.93602201978368
5.55464570000000	4.93702516216057
5.62790190000000	4.93687798995352
5.67754230000000	4.93667240884940
5.72970100000000	4.93631477335762
5.78081820000000	4.93654890793159
5.83111440000000	4.93706981482628
5.88087670000000	4.93697262681137
5.95273140000000	4.93752028220045
6.00030180000000	4.93755589504437
6.04896470000000	4.93801861073786
6.09961320000000	4.93816182758953
6.15256720000000	4.93828114179663
6.22651190000000	4.93800804875306
6.27682850000000	4.93711487349554
6.34885270000000	4.93725328450543
6.39690600000000	4.93718983312838
6.44872540000000	4.93772511923670
6.52333810000000	4.93769297275693
6.57160050000000	4.93834716852281
6.64542100000000	4.93835545416495
6.69113550000000	4.93823212763622
6.74693520000000	4.93833040268483
6.82225890000000	4.93803750058179
6.87245330000000	4.93801707564258
6.94514420000000	4.93795950716580
7.01563640000000	4.93717498333986
7.09011700000000	4.93674980629963
7.14254890000000	4.93656640374146
7.21533040000000	4.93760097387709
7.26448740000000	4.93902318367436
7.32049080000000	4.94027120203984
7.39445410000000	4.94193133595635
7.44117300000000	4.94266887559995
7.51739930000000	4.94266903810411
7.56822100000000	4.94210521290933
7.63955130000000	4.94211184845650
7.71922830000000	4.94115722880780
7.79069860000000	4.94078317247567
7.84410460000000	4.94163205022576
7.89107560000000	4.94134831809123
7.96379290000000	4.94241657447546
8.01673110000000	4.94257770150424
8.09096280000000	4.94263697170823
8.14213200000000	4.94303342213661
8.18757810000000	4.94261693740586
8.26129620000000	4.94251532775703
8.31260930000000	4.94228891087882
8.38675750000000	4.94320235225498
8.43807680000000	4.94471291943045
8.48338650000000	4.94578948302783
8.53913100000000	4.94722152976082
8.58833550000000	4.94750858452701
8.63805540000000	4.94822369883824
8.69103920000000	4.94821186433935
8.74064900000000	4.94882568696393
8.78978460000000	4.94846777498711
8.86457880000000	4.94847673041092
8.91526890000000	4.94827143984031
8.98770450000000	4.94792257396559
9.03914410000000	4.94829438135668
9.08520050000000	4.94862153665730
9.15839670000000	4.94904525983142
9.21065370000000	4.94999303498586
9.28312770000000	4.95097822389924
9.33757670000000	4.95249849717607
9.39246670000000	4.95341680712646
9.46612550000000	4.95337748929286
9.51787250000000	4.95283599454807
9.59048520000000	4.95212795317848
9.63931710000000	4.95149568440121
9.69090780000000	4.95120415477935
9.76187610000000	4.95192564787996
9.81277130000000	4.95164853438464
9.88740460000000	4.95153562874616
9.93451580000000	4.95148763231204
10.0084573000000	4.95107488319626
10.0592610000000	4.95111161666863
10.1117554000000	4.95128330088890
10.1858978000000	4.95133538794188
10.2367241000000	4.95125757399710
10.2923635000000	4.95048590593296
10.3373338000000	4.94983754197129
10.3872943000000	4.94975050344027
10.4652493000000	4.94967709562230
10.5216448000000	4.95000004248391
10.5694969000000	4.95051742344125
10.6202200000000	4.95056033794297
10.6663138000000	4.95059746208255
10.7123956000000	4.95111174336543
10.7615487000000	4.95141728531145
10.8071981000000	4.95145174934436
10.8557634000000	4.95134616325238
10.9017323000000	4.95139993119884
10.9468140000000	4.95107787157061
10.9931837000000	4.95111824729736
11.0414305000000	4.95136159858582
11.0861015000000	4.95162092167447
11.1329093000000	4.95223723087526
11.1846944000000	4.95203930312746
11.2273296000000	4.95224211604115
11.2721457000000	4.95232078815673
11.3150864000000	4.95262723703068
11.3618545000000	4.95237490064763
11.4100199000000	4.95228482709713
11.4560203000000	4.95286808195191
11.5007451000000	4.95263015832138

    };

    \end{axis}
\end{tikzpicture}

%% file: FigPlot/Inputs_5_deg_a.tex
\begin{tikzpicture}[spy using outlines={rectangle, magnification=10,  connect spies}]
    \begin{axis}[
      xmin=0, xmax=3,
      ymin=235, ymax=275,
      ymajorgrids=true,
      grid style=dashed,
      legend pos=north west,
      width = 0.26\linewidth,
      height = 0.2\linewidth,
    label style={font=\footnotesize},
      tick label style={font=\footnotesize},
    xlabel={Time [s]},
    ylabel={Lengths $L_i$ [mm]},
    ]

    \addplot[color= black, line width = 1pt] table[x=time, y=u1]     { time u1 u2
0	252	252
0.0487324000000000	252.706715698242	250.429203673205
0.0974648000000000	254.014668273926	249.160292663389
0.150304600000000	254.246724243164	247.921029643413
0.206128400000000	253.160279846191	246.807193960471
0.253611900000000	254.468231506348	245.779696042344
0.305590900000000	256.050432128906	244.955629557679
0.351815300000000	257.453315734863	244.162256843955
0.402593400000000	258.982776031494	243.495400064568
0.473879300000000	261.124019622803	242.987803858747
0.521701700000000	262.474163818359	242.496938903902
0.576204500000000	264.014171905518	242.117729188295
0.621306600000000	265.258835792542	241.693294964306
0.679257300000000	266.841035842896	241.287568298327
0.753773900000000	268.918991775513	240.845975354163
0.801874300000000	269.826119842529	240.445825268897
0.875632900000000	270.606671905518	240.152556139110
0.920843800000000	271.925171813965	239.903692379345
0.973573500000000	273.507371978760	239.760104645510
1.02277970000000	274.709843444824	239.616304557966
1.09737210000000	274.224635925293	239.483861501085
1.14822380000000	272.684627838135	239.340876135874
1.19960290000000	271.144619750977	239.202022860091
1.26341830000000	269.193239822388	239.066339225049
1.34012220000000	266.988707828522	238.955050431877
1.39105240000000	265.955003843308	238.854937909558
1.44134890000000	266.218703842163	238.799195842626
1.51456050000000	266.450759847164	238.747529078071
1.58755350000000	266.039387855530	238.710325715905
1.66072420000000	265.332671813965	238.672066491432
1.71099380000000	264.720887832642	238.641397424238
1.78502430000000	263.613347854614	238.609894431094
1.84310370000000	263.201975784302	238.589434800927
1.89349410000000	262.853891830444	238.575467030805
1.96739920000000	262.189367980957	238.537326727787
2.01623760000000	262.136627655029	238.538337997638
2.06677470000000	262.980467834473	238.546585115107
2.13957490000000	263.149235916138	238.553395468096
2.18988980000000	262.948823776245	238.562532248723
2.26262350000000	262.126079864502	238.557709501273
2.31145540000000	261.862380065918	238.560179980177
2.38696080000000	261.345527801514	238.553585098821
2.45719020000000	261.092375793457	238.553303200811
2.53128950000000	261.092375793457	238.553949541521
2.58210240000000	261.324431762695	238.555551024397
2.63070030000000	261.809639739990	238.559135173186
2.68108080000000	261.725256042480	238.578642005581
2.72986300000000	261.377171630859	238.584733906214
2.80388460000000	261.039635925293	238.583573998404
2.85519870000000	260.944703979492	238.581253930000
2.92950870000000	261.029087677002	238.582739332097
2.98101730000000	260.976347808838	238.582209886332
3.02886160000000	260.944703979492	238.581430876195
3.10396570000000	260.881415863037	238.580680528805
3.15479550000000	260.881415863037	238.580645488859
3.22768620000000	260.891963653564	238.580575424538
3.27975280000000	260.891963653564	238.580381501448
3.33168450000000	260.902511901855	238.580408835728
3.40513040000000	260.934155731201	238.580620470088
3.45427510000000	260.944703979492	238.580797314267
3.52847130000000	260.965800018311	238.580984060727
3.57386590000000	260.955251770020	238.580995054127
3.62581250000000	260.955251770020	238.580854289556
3.68250200000000	260.955251770020	238.580865662221
3.73351150000000	260.955251770020	238.580753789475
3.78411790000000	260.955251770020	238.580937042272
3.83403450000000	260.955251770020	238.580900987856
3.90944840000000	260.965800018311	238.580866332483
3.95267110000000	260.965800018311	238.581096056163
4.02749720000000	260.965800018311	238.580982416431
4.07862070000000	260.965800018311	238.581026923449
4.13442500000000	260.965800018311	238.581065158296
4.20690800000000	260.955251770020	238.580805127149
4.25480820000000	260.934155731201	238.580706616465
4.32944060000000	260.934155731201	238.580572351965
4.37556050000000	260.934155731201	238.580516401489
4.44908110000000	260.934155731201	238.580458030252
4.49793060000000	260.955251770020	238.580671115779
4.57176340000000	260.955251770020	238.580843010769
4.62296640000000	260.965800018311	238.580914483324
4.69685730000000	260.976347808838	238.581005972955
4.75160990000000	260.955251770020	238.580730493261
4.80374250000000	260.944703979492	238.580484141618
4.87659520000000	260.923607940674	238.580365566865
4.92746240000000	260.934155731201	238.580337351280
4.98141700000000	260.934155731201	238.580435662808
5.05133400000000	260.955251770020	238.580561334094
5.10318640000000	260.955251770020	238.580625766627
5.17699030000000	260.955251770020	238.580653116427
5.22735680000000	260.944703979492	238.580487603176
5.28349340000000	260.944703979492	238.580355047456
5.34107430000000	260.944703979492	238.580318515164
5.39463210000000	260.934155731201	238.580071027046
5.46921630000000	260.934155731201	238.580110229654
5.52100140000000	260.934155731201	238.580124848633
5.59321130000000	260.965800018311	238.580363394427
5.63858930000000	260.965800018311	238.580548702072
5.71161120000000	260.965800018311	238.580631574788
5.76349570000000	260.965800018311	238.580626885881
5.83793340000000	260.955251770020	238.580471623875
5.89117050000000	260.934155731201	238.580140130324
5.94000130000000	260.934155731201	238.579985174602
6.01058080000000	260.923607940674	238.579885773740
6.06368430000000	260.923607940674	238.579884130408
6.13681840000000	260.944703979492	238.579986442527
6.18662990000000	260.944703979492	238.580260151926
6.25728180000000	260.976347808838	238.580547992822
6.30980020000000	260.976347808838	238.580711610973
6.36388870000000	260.976347808838	238.580755864887
6.40977320000000	260.965800018311	238.580572537289
6.46053680000000	260.923607940674	238.580239764182
6.50955030000000	260.923607940674	238.579941413114
6.58343870000000	260.913059692383	238.579778420496
6.63542220000000	260.913059692383	238.579710010280
6.68901350000000	260.923607940674	238.579837973791
6.74046670000000	260.934155731201	238.579821599270
6.79176900000000	260.934155731201	238.579823449728
6.84142170000000	260.934155731201	238.579759954597
6.91480810000000	260.923607940674	238.579582274877
6.96816780000000	260.923607940674	238.579465878485
7.00973040000000	260.913059692383	238.579304706589
7.05867330000000	260.913059692383	238.579218939500
7.11133440000000	260.902511901855	238.579050768742
7.18544300000000	260.902511901855	238.579032134508
7.23620790000000	260.913059692383	238.579138479699
7.29016600000000	260.923607940674	238.579320385616
7.34350370000000	260.955251770020	238.579533408265
7.41738410000000	260.955251770020	238.579549209647
7.46959410000000	260.934155731201	238.579293910109
7.54234150000000	260.923607940674	238.579096986371
7.59555280000000	260.913059692383	238.578778747965
7.64339490000000	260.891963653564	238.578413944453
7.71462220000000	260.891963653564	238.578442926782
7.78803370000000	260.902511901855	238.578491496873
7.84034940000000	260.923607940674	238.589204180991
7.89046980000000	260.944703979492	238.589325022174
7.94242980000000	260.944703979492	238.589329698510
8.01722950000000	260.902511901855	238.588911564224
8.08872840000000	260.902511901855	238.588594189089
8.14115190000000	260.881415863037	238.588113941813
8.19378590000000	260.860319824219	238.587787299256
8.23903440000000	260.860319824219	238.587600006248
8.28772880000000	260.860319824219	238.587512149299
8.33547410000000	260.860319824219	238.587310347910
8.41308920000000	260.860319824219	238.587122225888
8.46307290000000	260.860319824219	238.587117009556
8.53755090000000	260.881415863037	238.587194299609
8.58940150000000	260.881415863037	238.576687193617
8.64484310000000	260.860319824219	238.565974803343
8.71903830000000	260.870868072510	238.565941457134
8.76757640000000	260.870868072510	238.565959784272
8.81907150000000	260.902511901855	238.566265028476
8.86252430000000	260.944703979492	238.566894021222
8.91464090000000	260.955251770020	238.567122917859
8.98859260000000	260.976347808838	238.567570548994
9.03989700000000	260.997443847656	238.568368850736
9.09159680000000	261.018539886475	238.568815258256
9.14326220000000	261.018539886475	238.569026271079
9.21785050000000	261.029087677002	238.569444073904
9.26509340000000	261.018539886475	238.569383945200
9.33814230000000	260.997443847656	238.569490690914
9.39003030000000	261.007991638184	238.569847447636
9.43812240000000	261.007991638184	238.569977620679
9.49193690000000	260.997443847656	238.570162595766
9.54377220000000	261.007991638184	238.570374278457
9.61904710000000	261.007991638184	238.570490526834
9.66864780000000	260.986896057129	238.570505631489
9.71964000000000	260.976347808838	238.570269228176
9.79305270000000	260.955251770020	238.570081379887
9.84562940000000	260.944703979492	238.569831305148
9.91673040000000	260.934155731201	238.569783044255
9.96663690000000	260.934155731201	238.569922734227
10.0373609000000	260.976347808838	238.570349340967
10.0879254000000	260.976347808838	238.570374707817
10.1585437000000	260.997443847656	238.570677370613
10.2067080000000	260.976347808838	238.570766107298
10.2791101000000	260.965800018311	238.570590706253
10.3317213000000	260.955251770020	238.570339005156
10.3828650000000	260.934155731201	238.570008129437
10.4565585000000	260.902511901855	238.569761555015
10.5077424000000	260.913059692383	238.569806082301
10.5646813000000	260.913059692383	238.569729064965
10.6379389000000	260.913059692383	238.569663730147
10.6841437000000	260.923607940674	238.569877128952
10.7594528000000	260.923607940674	238.569804059697
10.8108819000000	260.923607940674	238.569665134487
10.8647451000000	260.923607940674	238.569580368743
10.9372444000000	260.891963653564	238.569238142671
10.9904679000000	260.891963653564	238.568896208266
11.0404265000000	260.891963653564	238.568979197177
11.1148392000000	260.891963653564	238.568833170061
11.1646992000000	260.913059692383	238.568891963413
11.2377483000000	260.934155731201	238.569214140202
11.2868871000000	260.934155731201	238.569262071454
11.3582398000000	260.965800018311	238.569545678806
11.4069668000000	260.955251770020	238.569464001612
11.4808795000000	260.934155731201	238.569276510707
11.5554047000000	260.913059692383	238.569053581571
11.6259108000000	260.902511901855	238.568844832084
11.6789337000000	260.913059692383	238.568874189752

    };

    \addplot[color= red, line width = 1pt] table[x=time, y=u2]     { time u1 u2
0	252	252
0.0487324000000000	252.706715698242	250.429203673205
0.0974648000000000	254.014668273926	249.160292663389
0.150304600000000	254.246724243164	247.921029643413
0.206128400000000	253.160279846191	246.807193960471
0.253611900000000	254.468231506348	245.779696042344
0.305590900000000	256.050432128906	244.955629557679
0.351815300000000	257.453315734863	244.162256843955
0.402593400000000	258.982776031494	243.495400064568
0.473879300000000	261.124019622803	242.987803858747
0.521701700000000	262.474163818359	242.496938903902
0.576204500000000	264.014171905518	242.117729188295
0.621306600000000	265.258835792542	241.693294964306
0.679257300000000	266.841035842896	241.287568298327
0.753773900000000	268.918991775513	240.845975354163
0.801874300000000	269.826119842529	240.445825268897
0.875632900000000	270.606671905518	240.152556139110
0.920843800000000	271.925171813965	239.903692379345
0.973573500000000	273.507371978760	239.760104645510
1.02277970000000	274.709843444824	239.616304557966
1.09737210000000	274.224635925293	239.483861501085
1.14822380000000	272.684627838135	239.340876135874
1.19960290000000	271.144619750977	239.202022860091
1.26341830000000	269.193239822388	239.066339225049
1.34012220000000	266.988707828522	238.955050431877
1.39105240000000	265.955003843308	238.854937909558
1.44134890000000	266.218703842163	238.799195842626
1.51456050000000	266.450759847164	238.747529078071
1.58755350000000	266.039387855530	238.710325715905
1.66072420000000	265.332671813965	238.672066491432
1.71099380000000	264.720887832642	238.641397424238
1.78502430000000	263.613347854614	238.609894431094
1.84310370000000	263.201975784302	238.589434800927
1.89349410000000	262.853891830444	238.575467030805
1.96739920000000	262.189367980957	238.537326727787
2.01623760000000	262.136627655029	238.538337997638
2.06677470000000	262.980467834473	238.546585115107
2.13957490000000	263.149235916138	238.553395468096
2.18988980000000	262.948823776245	238.562532248723
2.26262350000000	262.126079864502	238.557709501273
2.31145540000000	261.862380065918	238.560179980177
2.38696080000000	261.345527801514	238.553585098821
2.45719020000000	261.092375793457	238.553303200811
2.53128950000000	261.092375793457	238.553949541521
2.58210240000000	261.324431762695	238.555551024397
2.63070030000000	261.809639739990	238.559135173186
2.68108080000000	261.725256042480	238.578642005581
2.72986300000000	261.377171630859	238.584733906214
2.80388460000000	261.039635925293	238.583573998404
2.85519870000000	260.944703979492	238.581253930000
2.92950870000000	261.029087677002	238.582739332097
2.98101730000000	260.976347808838	238.582209886332
3.02886160000000	260.944703979492	238.581430876195
3.10396570000000	260.881415863037	238.580680528805
3.15479550000000	260.881415863037	238.580645488859
3.22768620000000	260.891963653564	238.580575424538
3.27975280000000	260.891963653564	238.580381501448
3.33168450000000	260.902511901855	238.580408835728
3.40513040000000	260.934155731201	238.580620470088
3.45427510000000	260.944703979492	238.580797314267
3.52847130000000	260.965800018311	238.580984060727
3.57386590000000	260.955251770020	238.580995054127
3.62581250000000	260.955251770020	238.580854289556
3.68250200000000	260.955251770020	238.580865662221
3.73351150000000	260.955251770020	238.580753789475
3.78411790000000	260.955251770020	238.580937042272
3.83403450000000	260.955251770020	238.580900987856
3.90944840000000	260.965800018311	238.580866332483
3.95267110000000	260.965800018311	238.581096056163
4.02749720000000	260.965800018311	238.580982416431
4.07862070000000	260.965800018311	238.581026923449
4.13442500000000	260.965800018311	238.581065158296
4.20690800000000	260.955251770020	238.580805127149
4.25480820000000	260.934155731201	238.580706616465
4.32944060000000	260.934155731201	238.580572351965
4.37556050000000	260.934155731201	238.580516401489
4.44908110000000	260.934155731201	238.580458030252
4.49793060000000	260.955251770020	238.580671115779
4.57176340000000	260.955251770020	238.580843010769
4.62296640000000	260.965800018311	238.580914483324
4.69685730000000	260.976347808838	238.581005972955
4.75160990000000	260.955251770020	238.580730493261
4.80374250000000	260.944703979492	238.580484141618
4.87659520000000	260.923607940674	238.580365566865
4.92746240000000	260.934155731201	238.580337351280
4.98141700000000	260.934155731201	238.580435662808
5.05133400000000	260.955251770020	238.580561334094
5.10318640000000	260.955251770020	238.580625766627
5.17699030000000	260.955251770020	238.580653116427
5.22735680000000	260.944703979492	238.580487603176
5.28349340000000	260.944703979492	238.580355047456
5.34107430000000	260.944703979492	238.580318515164
5.39463210000000	260.934155731201	238.580071027046
5.46921630000000	260.934155731201	238.580110229654
5.52100140000000	260.934155731201	238.580124848633
5.59321130000000	260.965800018311	238.580363394427
5.63858930000000	260.965800018311	238.580548702072
5.71161120000000	260.965800018311	238.580631574788
5.76349570000000	260.965800018311	238.580626885881
5.83793340000000	260.955251770020	238.580471623875
5.89117050000000	260.934155731201	238.580140130324
5.94000130000000	260.934155731201	238.579985174602
6.01058080000000	260.923607940674	238.579885773740
6.06368430000000	260.923607940674	238.579884130408
6.13681840000000	260.944703979492	238.579986442527
6.18662990000000	260.944703979492	238.580260151926
6.25728180000000	260.976347808838	238.580547992822
6.30980020000000	260.976347808838	238.580711610973
6.36388870000000	260.976347808838	238.580755864887
6.40977320000000	260.965800018311	238.580572537289
6.46053680000000	260.923607940674	238.580239764182
6.50955030000000	260.923607940674	238.579941413114
6.58343870000000	260.913059692383	238.579778420496
6.63542220000000	260.913059692383	238.579710010280
6.68901350000000	260.923607940674	238.579837973791
6.74046670000000	260.934155731201	238.579821599270
6.79176900000000	260.934155731201	238.579823449728
6.84142170000000	260.934155731201	238.579759954597
6.91480810000000	260.923607940674	238.579582274877
6.96816780000000	260.923607940674	238.579465878485
7.00973040000000	260.913059692383	238.579304706589
7.05867330000000	260.913059692383	238.579218939500
7.11133440000000	260.902511901855	238.579050768742
7.18544300000000	260.902511901855	238.579032134508
7.23620790000000	260.913059692383	238.579138479699
7.29016600000000	260.923607940674	238.579320385616
7.34350370000000	260.955251770020	238.579533408265
7.41738410000000	260.955251770020	238.579549209647
7.46959410000000	260.934155731201	238.579293910109
7.54234150000000	260.923607940674	238.579096986371
7.59555280000000	260.913059692383	238.578778747965
7.64339490000000	260.891963653564	238.578413944453
7.71462220000000	260.891963653564	238.578442926782
7.78803370000000	260.902511901855	238.578491496873
7.84034940000000	260.923607940674	238.589204180991
7.89046980000000	260.944703979492	238.589325022174
7.94242980000000	260.944703979492	238.589329698510
8.01722950000000	260.902511901855	238.588911564224
8.08872840000000	260.902511901855	238.588594189089
8.14115190000000	260.881415863037	238.588113941813
8.19378590000000	260.860319824219	238.587787299256
8.23903440000000	260.860319824219	238.587600006248
8.28772880000000	260.860319824219	238.587512149299
8.33547410000000	260.860319824219	238.587310347910
8.41308920000000	260.860319824219	238.587122225888
8.46307290000000	260.860319824219	238.587117009556
8.53755090000000	260.881415863037	238.587194299609
8.58940150000000	260.881415863037	238.576687193617
8.64484310000000	260.860319824219	238.565974803343
8.71903830000000	260.870868072510	238.565941457134
8.76757640000000	260.870868072510	238.565959784272
8.81907150000000	260.902511901855	238.566265028476
8.86252430000000	260.944703979492	238.566894021222
8.91464090000000	260.955251770020	238.567122917859
8.98859260000000	260.976347808838	238.567570548994
9.03989700000000	260.997443847656	238.568368850736
9.09159680000000	261.018539886475	238.568815258256
9.14326220000000	261.018539886475	238.569026271079
9.21785050000000	261.029087677002	238.569444073904
9.26509340000000	261.018539886475	238.569383945200
9.33814230000000	260.997443847656	238.569490690914
9.39003030000000	261.007991638184	238.569847447636
9.43812240000000	261.007991638184	238.569977620679
9.49193690000000	260.997443847656	238.570162595766
9.54377220000000	261.007991638184	238.570374278457
9.61904710000000	261.007991638184	238.570490526834
9.66864780000000	260.986896057129	238.570505631489
9.71964000000000	260.976347808838	238.570269228176
9.79305270000000	260.955251770020	238.570081379887
9.84562940000000	260.944703979492	238.569831305148
9.91673040000000	260.934155731201	238.569783044255
9.96663690000000	260.934155731201	238.569922734227
10.0373609000000	260.976347808838	238.570349340967
10.0879254000000	260.976347808838	238.570374707817
10.1585437000000	260.997443847656	238.570677370613
10.2067080000000	260.976347808838	238.570766107298
10.2791101000000	260.965800018311	238.570590706253
10.3317213000000	260.955251770020	238.570339005156
10.3828650000000	260.934155731201	238.570008129437
10.4565585000000	260.902511901855	238.569761555015
10.5077424000000	260.913059692383	238.569806082301
10.5646813000000	260.913059692383	238.569729064965
10.6379389000000	260.913059692383	238.569663730147
10.6841437000000	260.923607940674	238.569877128952
10.7594528000000	260.923607940674	238.569804059697
10.8108819000000	260.923607940674	238.569665134487
10.8647451000000	260.923607940674	238.569580368743
10.9372444000000	260.891963653564	238.569238142671
10.9904679000000	260.891963653564	238.568896208266
11.0404265000000	260.891963653564	238.568979197177
11.1148392000000	260.891963653564	238.568833170061
11.1646992000000	260.913059692383	238.568891963413
11.2377483000000	260.934155731201	238.569214140202
11.2868871000000	260.934155731201	238.569262071454
11.3582398000000	260.965800018311	238.569545678806
11.4069668000000	260.955251770020	238.569464001612
11.4808795000000	260.934155731201	238.569276510707
11.5554047000000	260.913059692383	238.569053581571
11.6259108000000	260.902511901855	238.568844832084
11.6789337000000	260.913059692383	238.568874189752

    };


    \end{axis}

\end{tikzpicture}

%% file: FigPlot/Inputs_5_deg_b.tex
\begin{tikzpicture}[spy using outlines={rectangle, magnification=10,  connect spies}]
    \begin{axis}[
      xmin=0, xmax=3,
      ymin=235, ymax=275,
      ymajorgrids=true,
      grid style=dashed,
      legend pos=north west,
      width = 0.26\linewidth,
      height = 0.2\linewidth,
    label style={font=\footnotesize},
      tick label style={font=\footnotesize},
    xlabel={Time [s]},
    ]


    \addplot[color=red, line width = 1pt] table[x=time, y=u2] {
time u1 u2
0	252	252
0.0702578000000000	252.559043884277	250.429203673205
0.140515600000000	253.487267761230	248.837817166271
0.187789100000000	253.613843994141	247.516600923868
0.262076000000000	255.797280578613	246.621364090645
0.308745600000000	257.210711975098	245.517184417158
0.364577100000000	258.877296295166	244.793930989897
0.437808400000000	261.092376251221	243.990692641822
0.487995200000000	262.590192260742	243.323683456387
0.544334000000000	264.235680084229	242.827346080557
0.615377000000000	266.260896148682	242.312814621207
0.668287200000000	267.326244163513	241.803320960772
0.741237400000000	266.324184150696	241.354667079589
0.793452400000000	265.079520149231	240.897490780601
0.863594800000000	265.090068168640	240.544352838395
0.912167200000000	265.628016128540	240.230098918069
0.984637700000000	267.220764141083	240.036943463590
1.03772920000000	267.695424156189	239.847882163939
1.08767120000000	267.115284175873	239.690717501798
1.16004530000000	265.849524135590	239.526059579792
1.23705000000000	264.193488235474	239.366255108551
1.29269880000000	263.613348083496	239.208289482963
1.36504500000000	263.434032211304	239.087684083753
1.43813100000000	263.792664184570	238.990005027103
1.49111870000000	263.887596130371	238.934003367133
1.53716770000000	263.834856262207	238.880712794271
1.61281120000000	263.318004226685	238.834292198060
1.66273970000000	263.001564102173	238.786764258546
1.73878690000000	262.315944213867	238.740929044328
1.79215060000000	262.189367980957	238.695075372950
1.84464480000000	262.052243957520	238.668987645664
1.91861270000000	261.946764221191	238.643235824316
1.96832290000000	261.862380065918	238.637156999946
2.03927980000000	262.294848175049	238.645846583354
2.08907450000000	262.347588043213	238.647773654781
2.13715090000000	261.936215972900	238.640808092039
2.19446000000000	261.672516174316	238.635069614298
2.26731980000000	261.577584228516	238.636166945623
2.31399870000000	261.472104034424	238.625939858111
2.36822950000000	261.366624298096	238.625733181575
2.44395090000000	261.250596313477	238.625944108115
2.51629250000000	261.187308197022	238.626318382483
2.56945110000000	261.113472290039	238.623496346141
2.64091180000000	261.060731964111	238.624670249907
2.68596540000000	261.081828002930	238.626291291609
2.73347740000000	261.092376251221	238.627519716405
2.78084500000000	261.208404235840	238.632096131293
2.82617970000000	261.261144104004	238.632556613483
2.87554220000000	261.229500274658	238.633121434582
2.92072350000000	261.187308197022	238.633029864007
2.96835740000000	261.124020080566	238.632567384439
3.01544280000000	261.081828002930	238.632128369964
3.06263540000000	261.050184173584	238.631930540900
3.10900470000000	261.007992095947	238.631297452502
3.15940700000000	260.976348266602	238.631016337806
3.20359640000000	260.986896057129	238.631315647058
3.25357260000000	261.007992095947	238.631480021378
3.29855520000000	260.997444305420	238.631431276645
3.34626250000000	260.997444305420	238.631498737483
3.40122650000000	260.986896057129	238.631195013338
3.44657730000000	260.986896057129	238.630966132216
3.49138760000000	260.955252227783	238.630705989080
3.57070660000000	260.955252227783	238.630605286082
3.64428250000000	260.955252227783	238.630625561162
3.69181170000000	260.976348266602	238.630756924020
3.73643390000000	260.976348266602	238.630826324840
3.78808380000000	260.976348266602	238.630802786310
3.83399380000000	260.976348266602	238.630925514789
3.88249020000000	260.976348266602	238.630803789854
3.93437800000000	260.976348266602	238.630707011191
4.00928860000000	260.976348266602	238.630574541939
4.05950120000000	260.955252227783	238.630312602953
4.10919350000000	260.955252227783	238.630198106634
4.16018750000000	260.955252227783	238.630210050336
4.20908590000000	260.955252227783	238.630076479899
4.26475120000000	260.944703979492	238.619507118704
4.31328070000000	260.944703979492	238.619697528525
4.35703160000000	260.955252227783	238.619706667058
4.40988930000000	260.955252227783	238.619944324521
4.46627860000000	260.986896057129	238.630818161298
4.54180530000000	260.986896057129	238.630952343154
4.61515840000000	260.986896057129	238.630775172071
4.67064560000000	260.955252227783	238.630245256115
4.71308970000000	260.934156188965	238.629681085797
4.78620080000000	260.934156188965	238.629491333865
4.84098620000000	260.934156188965	238.629465181338
4.91114910000000	260.934156188965	238.629615799729
4.96432380000000	260.934156188965	238.629485963046
5.01258570000000	260.944703979492	238.629424514390
5.09341650000000	260.944703979492	238.629351923017
5.13396530000000	260.923607940674	238.628971274536
5.17799010000000	260.923607940674	238.628940288810
5.23025500000000	260.934156188965	238.628796367949
5.30420640000000	260.934156188965	238.628912004516
5.44382930000000	260.934156188965	238.628837721343
5.49061850000000	260.934156188965	238.628860279427
5.53903200000000	260.934156188965	238.628789113590
5.59404150000000	260.934156188965	238.618481672004
5.65108850000000	260.934156188965	238.607948833108
5.69758840000000	260.934156188965	238.608103590557
5.74587560000000	260.934156188965	238.608182528368
5.79680840000000	260.934156188965	238.608368824539
5.93445340000000	260.934156188965	238.608314759192
5.97571300000000	260.944703979492	238.608738348730
6.04791580000000	260.955252227783	238.608840437735
6.10008900000000	260.955252227783	238.608924337432
6.15062100000000	260.955252227783	238.608876152441
6.19296470000000	260.955252227783	238.608915095006
6.24081700000000	260.944703979492	238.608886017751
6.31437140000000	260.955252227783	238.608981718104
6.38634590000000	260.955252227783	238.609042627745
6.43551720000000	260.955252227783	238.609471416377
6.48984780000000	260.976348266602	238.609555000676
6.56251810000000	260.976348266602	238.609725486086
6.61531960000000	260.976348266602	238.609743285218
6.66567000000000	260.944703979492	238.609311984093
6.71182790000000	260.944703979492	238.609223319332
6.76399320000000	260.934156188965	238.609013490937
6.80845150000000	260.934156188965	238.609032022795
6.85541410000000	260.934156188965	238.609017080652
6.90914850000000	260.934156188965	238.609085805683
6.98126300000000	260.934156188965	238.609173112505
7.03251720000000	260.934156188965	238.609218238446
7.09396700000000	260.944703979492	238.609199761513
7.14209410000000	260.934156188965	238.609185691882
7.18983100000000	260.944703979492	238.609205622974
7.23722040000000	260.944703979492	238.609160570128
7.28343260000000	260.934156188965	238.608955610632
7.33214310000000	260.934156188965	238.608845914957
7.37240680000000	260.913060150146	238.608672884536
7.42941890000000	260.923607940674	238.608624818514
7.48433830000000	260.923607940674	238.608574326110
7.53294640000000	260.923607940674	238.608612405026
7.57727280000000	260.934156188965	238.608721611265
7.62455960000000	260.934156188965	238.608842133886
7.66831660000000	260.934156188965	238.608910210579
7.71711020000000	260.934156188965	238.608918007438
7.76220450000000	260.934156188965	238.608985139450
7.80934000000000	260.934156188965	238.608700345293
7.85646170000000	260.934156188965	238.608611386623
7.90180400000000	260.923607940674	238.608464791611
7.94990750000000	260.923607940674	238.608367196414
7.99197430000000	260.913060150146	238.608162648987
8.06551870000000	260.913060150146	238.607989308124
8.11589060000000	260.913060150146	238.607831067040
8.16591730000000	260.902512359619	238.607677307835
8.21303680000000	260.902512359619	238.607595978576
8.27356010000000	260.902512359619	238.607615730172
8.31711490000000	260.902512359619	238.607484379513
8.37461340000000	260.902512359619	238.607531325251
8.42495330000000	260.902512359619	238.607511268737
8.47407030000000	260.913060150146	238.607346638105
8.52391330000000	260.913060150146	238.607353514365
8.59665120000000	260.902512359619	238.607189727009
8.64879800000000	260.891964111328	238.607092818158
8.69799980000000	260.902512359619	238.607007598179
8.75080210000000	260.902512359619	238.607048012370
8.79785840000000	260.902512359619	238.607195473318
8.84201350000000	260.902512359619	238.607130919158
8.89596130000000	260.902512359619	238.607070983971
8.94100800000000	260.913060150146	238.607209963984
8.98643190000000	260.923607940674	238.607356536793
9.03577600000000	260.923607940674	238.607234091510
9.08357360000000	260.923607940674	238.607301244790
9.12978930000000	260.913060150146	238.607173346337
9.17441290000000	260.913060150146	238.607082511084
9.22127780000000	260.913060150146	238.607064332274
9.27003470000000	260.902512359619	238.606923849567
9.31675200000000	260.902512359619	238.606980834892
9.36169920000000	260.902512359619	238.607095632992
9.40813810000000	260.913060150146	238.607170093027
9.45630980000000	260.923607940674	238.607179382905
9.50201390000000	260.913060150146	238.607101776543
9.54780710000000	260.913060150146	238.607240744840
9.59484180000000	260.913060150146	238.607114775403
9.64328340000000	260.913060150146	238.607032473966
9.69035220000000	260.902512359619	238.606896277522
9.73619580000000	260.902512359619	238.606804535337
9.78097800000000	260.891964111328	238.606575296386
9.82941040000000	260.891964111328	238.606682411678
9.87721770000000	260.891964111328	238.606676466329
9.92084110000000	260.902512359619	238.606743737117
9.96939640000000	260.913060150146	238.606747212509
10.0157046000000	260.913060150146	238.606892827798
10.0627986000000	260.913060150146	238.606851508766
10.1071622000000	260.913060150146	238.606957300683
10.1575748000000	260.913060150146	238.606884829181
10.2134808000000	260.913060150146	238.606751419517
10.2823027000000	260.902512359619	238.606592094671
10.3302605000000	260.902512359619	238.606453783191
10.3752132000000	260.902512359619	238.606372109199
10.4256201000000	260.902512359619	238.606424094505
10.4724145000000	260.902512359619	238.606593820736
10.5283391000000	260.913060150146	238.606794707617
10.5762795000000	260.923607940674	238.607033772027
10.6286338000000	260.923607940674	238.607178784217
10.6863439000000	260.923607940674	238.607053559010
10.7601990000000	260.913060150146	238.606835624740
10.8061174000000	260.902512359619	238.606500676900
10.8810273000000	260.891964111328	238.606227949942

 };

 \addplot[color= black, line width = 1pt] table[x=time, y=u1]   {
time u1 u2
0	252	252
0.0702578000000000	252.559043884277	250.429203673205
0.140515600000000	253.487267761230	248.837817166271
0.187789100000000	253.613843994141	247.516600923868
0.262076000000000	255.797280578613	246.621364090645
0.308745600000000	257.210711975098	245.517184417158
0.364577100000000	258.877296295166	244.793930989897
0.437808400000000	261.092376251221	243.990692641822
0.487995200000000	262.590192260742	243.323683456387
0.544334000000000	264.235680084229	242.827346080557
0.615377000000000	266.260896148682	242.312814621207
0.668287200000000	267.326244163513	241.803320960772
0.741237400000000	266.324184150696	241.354667079589
0.793452400000000	265.079520149231	240.897490780601
0.863594800000000	265.090068168640	240.544352838395
0.912167200000000	265.628016128540	240.230098918069
0.984637700000000	267.220764141083	240.036943463590
1.03772920000000	267.695424156189	239.847882163939
1.08767120000000	267.115284175873	239.690717501798
1.16004530000000	265.849524135590	239.526059579792
1.23705000000000	264.193488235474	239.366255108551
1.29269880000000	263.613348083496	239.208289482963
1.36504500000000	263.434032211304	239.087684083753
1.43813100000000	263.792664184570	238.990005027103
1.49111870000000	263.887596130371	238.934003367133
1.53716770000000	263.834856262207	238.880712794271
1.61281120000000	263.318004226685	238.834292198060
1.66273970000000	263.001564102173	238.786764258546
1.73878690000000	262.315944213867	238.740929044328
1.79215060000000	262.189367980957	238.695075372950
1.84464480000000	262.052243957520	238.668987645664
1.91861270000000	261.946764221191	238.643235824316
1.96832290000000	261.862380065918	238.637156999946
2.03927980000000	262.294848175049	238.645846583354
2.08907450000000	262.347588043213	238.647773654781
2.13715090000000	261.936215972900	238.640808092039
2.19446000000000	261.672516174316	238.635069614298
2.26731980000000	261.577584228516	238.636166945623
2.31399870000000	261.472104034424	238.625939858111
2.36822950000000	261.366624298096	238.625733181575
2.44395090000000	261.250596313477	238.625944108115
2.51629250000000	261.187308197022	238.626318382483
2.56945110000000	261.113472290039	238.623496346141
2.64091180000000	261.060731964111	238.624670249907
2.68596540000000	261.081828002930	238.626291291609
2.73347740000000	261.092376251221	238.627519716405
2.78084500000000	261.208404235840	238.632096131293
2.82617970000000	261.261144104004	238.632556613483
2.87554220000000	261.229500274658	238.633121434582
2.92072350000000	261.187308197022	238.633029864007
2.96835740000000	261.124020080566	238.632567384439
3.01544280000000	261.081828002930	238.632128369964
3.06263540000000	261.050184173584	238.631930540900
3.10900470000000	261.007992095947	238.631297452502
3.15940700000000	260.976348266602	238.631016337806
3.20359640000000	260.986896057129	238.631315647058
3.25357260000000	261.007992095947	238.631480021378
3.29855520000000	260.997444305420	238.631431276645
3.34626250000000	260.997444305420	238.631498737483
3.40122650000000	260.986896057129	238.631195013338
3.44657730000000	260.986896057129	238.630966132216
3.49138760000000	260.955252227783	238.630705989080
3.57070660000000	260.955252227783	238.630605286082
3.64428250000000	260.955252227783	238.630625561162
3.69181170000000	260.976348266602	238.630756924020
3.73643390000000	260.976348266602	238.630826324840
3.78808380000000	260.976348266602	238.630802786310
3.83399380000000	260.976348266602	238.630925514789
3.88249020000000	260.976348266602	238.630803789854
3.93437800000000	260.976348266602	238.630707011191
4.00928860000000	260.976348266602	238.630574541939
4.05950120000000	260.955252227783	238.630312602953
4.10919350000000	260.955252227783	238.630198106634
4.16018750000000	260.955252227783	238.630210050336
4.20908590000000	260.955252227783	238.630076479899
4.26475120000000	260.944703979492	238.619507118704
4.31328070000000	260.944703979492	238.619697528525
4.35703160000000	260.955252227783	238.619706667058
4.40988930000000	260.955252227783	238.619944324521
4.46627860000000	260.986896057129	238.630818161298
4.54180530000000	260.986896057129	238.630952343154
4.61515840000000	260.986896057129	238.630775172071
4.67064560000000	260.955252227783	238.630245256115
4.71308970000000	260.934156188965	238.629681085797
4.78620080000000	260.934156188965	238.629491333865
4.84098620000000	260.934156188965	238.629465181338
4.91114910000000	260.934156188965	238.629615799729
4.96432380000000	260.934156188965	238.629485963046
5.01258570000000	260.944703979492	238.629424514390
5.09341650000000	260.944703979492	238.629351923017
5.13396530000000	260.923607940674	238.628971274536
5.17799010000000	260.923607940674	238.628940288810
5.23025500000000	260.934156188965	238.628796367949
5.30420640000000	260.934156188965	238.628912004516
5.44382930000000	260.934156188965	238.628837721343
5.49061850000000	260.934156188965	238.628860279427
5.53903200000000	260.934156188965	238.628789113590
5.59404150000000	260.934156188965	238.618481672004
5.65108850000000	260.934156188965	238.607948833108
5.69758840000000	260.934156188965	238.608103590557
5.74587560000000	260.934156188965	238.608182528368
5.79680840000000	260.934156188965	238.608368824539
5.93445340000000	260.934156188965	238.608314759192
5.97571300000000	260.944703979492	238.608738348730
6.04791580000000	260.955252227783	238.608840437735
6.10008900000000	260.955252227783	238.608924337432
6.15062100000000	260.955252227783	238.608876152441
6.19296470000000	260.955252227783	238.608915095006
6.24081700000000	260.944703979492	238.608886017751
6.31437140000000	260.955252227783	238.608981718104
6.38634590000000	260.955252227783	238.609042627745
6.43551720000000	260.955252227783	238.609471416377
6.48984780000000	260.976348266602	238.609555000676
6.56251810000000	260.976348266602	238.609725486086
6.61531960000000	260.976348266602	238.609743285218
6.66567000000000	260.944703979492	238.609311984093
6.71182790000000	260.944703979492	238.609223319332
6.76399320000000	260.934156188965	238.609013490937
6.80845150000000	260.934156188965	238.609032022795
6.85541410000000	260.934156188965	238.609017080652
6.90914850000000	260.934156188965	238.609085805683
6.98126300000000	260.934156188965	238.609173112505
7.03251720000000	260.934156188965	238.609218238446
7.09396700000000	260.944703979492	238.609199761513
7.14209410000000	260.934156188965	238.609185691882
7.18983100000000	260.944703979492	238.609205622974
7.23722040000000	260.944703979492	238.609160570128
7.28343260000000	260.934156188965	238.608955610632
7.33214310000000	260.934156188965	238.608845914957
7.37240680000000	260.913060150146	238.608672884536
7.42941890000000	260.923607940674	238.608624818514
7.48433830000000	260.923607940674	238.608574326110
7.53294640000000	260.923607940674	238.608612405026
7.57727280000000	260.934156188965	238.608721611265
7.62455960000000	260.934156188965	238.608842133886
7.66831660000000	260.934156188965	238.608910210579
7.71711020000000	260.934156188965	238.608918007438
7.76220450000000	260.934156188965	238.608985139450
7.80934000000000	260.934156188965	238.608700345293
7.85646170000000	260.934156188965	238.608611386623
7.90180400000000	260.923607940674	238.608464791611
7.94990750000000	260.923607940674	238.608367196414
7.99197430000000	260.913060150146	238.608162648987
8.06551870000000	260.913060150146	238.607989308124
8.11589060000000	260.913060150146	238.607831067040
8.16591730000000	260.902512359619	238.607677307835
8.21303680000000	260.902512359619	238.607595978576
8.27356010000000	260.902512359619	238.607615730172
8.31711490000000	260.902512359619	238.607484379513
8.37461340000000	260.902512359619	238.607531325251
8.42495330000000	260.902512359619	238.607511268737
8.47407030000000	260.913060150146	238.607346638105
8.52391330000000	260.913060150146	238.607353514365
8.59665120000000	260.902512359619	238.607189727009
8.64879800000000	260.891964111328	238.607092818158
8.69799980000000	260.902512359619	238.607007598179
8.75080210000000	260.902512359619	238.607048012370
8.79785840000000	260.902512359619	238.607195473318
8.84201350000000	260.902512359619	238.607130919158
8.89596130000000	260.902512359619	238.607070983971
8.94100800000000	260.913060150146	238.607209963984
8.98643190000000	260.923607940674	238.607356536793
9.03577600000000	260.923607940674	238.607234091510
9.08357360000000	260.923607940674	238.607301244790
9.12978930000000	260.913060150146	238.607173346337
9.17441290000000	260.913060150146	238.607082511084
9.22127780000000	260.913060150146	238.607064332274
9.27003470000000	260.902512359619	238.606923849567
9.31675200000000	260.902512359619	238.606980834892
9.36169920000000	260.902512359619	238.607095632992
9.40813810000000	260.913060150146	238.607170093027
9.45630980000000	260.923607940674	238.607179382905
9.50201390000000	260.913060150146	238.607101776543
9.54780710000000	260.913060150146	238.607240744840
9.59484180000000	260.913060150146	238.607114775403
9.64328340000000	260.913060150146	238.607032473966
9.69035220000000	260.902512359619	238.606896277522
9.73619580000000	260.902512359619	238.606804535337
9.78097800000000	260.891964111328	238.606575296386
9.82941040000000	260.891964111328	238.606682411678
9.87721770000000	260.891964111328	238.606676466329
9.92084110000000	260.902512359619	238.606743737117
9.96939640000000	260.913060150146	238.606747212509
10.0157046000000	260.913060150146	238.606892827798
10.0627986000000	260.913060150146	238.606851508766
10.1071622000000	260.913060150146	238.606957300683
10.1575748000000	260.913060150146	238.606884829181
10.2134808000000	260.913060150146	238.606751419517
10.2823027000000	260.902512359619	238.606592094671
10.3302605000000	260.902512359619	238.606453783191
10.3752132000000	260.902512359619	238.606372109199
10.4256201000000	260.902512359619	238.606424094505
10.4724145000000	260.902512359619	238.606593820736
10.5283391000000	260.913060150146	238.606794707617
10.5762795000000	260.923607940674	238.607033772027
10.6286338000000	260.923607940674	238.607178784217
10.6863439000000	260.923607940674	238.607053559010
10.7601990000000	260.913060150146	238.606835624740
10.8061174000000	260.902512359619	238.606500676900
10.8810273000000	260.891964111328	238.606227949942

 };

    \end{axis}

\end{tikzpicture}

%% file: FigPlot/Inputs_5_deg_c.tex
\begin{tikzpicture}[spy using outlines={rectangle, magnification=10,  connect spies}]
    \begin{axis}[
      xmin=0, xmax=3,
      ymin=235, ymax=275,
      ymajorgrids=true,
      grid style=dashed,
      legend pos=north west,
      width = 0.26\linewidth,
      height = 0.2\linewidth,
    label style={font=\footnotesize},
      tick label style={font=\footnotesize},
    xlabel={Time [s]},
    ]


    \addplot[color=red, line width = 1pt] table[x=time, y=u2] {
time u1 u2
0	252	252
0.0643425000000000	253.592748413086	250.429203673205
0.128685000000000	253.909187622070	248.836353598429
0.179990500000000	253.170827636719	247.559443002836
0.252648600000000	253.666583862305	246.527064564112
0.302044600000000	254.077955932617	245.373806367375
0.350144400000000	254.636999816895	244.597422272119
0.396763100000000	255.533580322266	243.830264609738
0.450177800000000	256.282488098145	243.196405285201
0.502453700000000	257.084135742188	242.662400427371
0.551189800000000	258.318251953125	242.227203016134
0.608195100000000	259.024968109131	241.833970811376
0.652493600000000	258.877295837402	241.439116410314
0.728552000000000	258.444828186035	241.089357775472
0.777774400000000	258.339348449707	240.720242486179
0.851622600000000	258.645239868164	240.444997594579
0.905830800000000	258.982776031494	240.196918562907
0.955523500000000	259.700039978027	240.047453161917
1.02737480000000	260.470043792725	239.908365971577
1.08015320000000	260.533331909180	239.785919513730
1.15260420000000	260.322371978760	239.651145242480
1.22881620000000	260.079768218994	239.501179168985
1.29924180000000	260.016480102539	239.372444178842
1.35085350000000	260.132508087158	239.280767695086
1.40091180000000	260.343468017578	239.193913136941
1.45428870000000	260.522784118652	239.138236492275
1.52652020000000	260.744291839600	239.102225550019
1.57532690000000	260.849772033691	239.060899324161
1.62653600000000	260.744291839600	239.005763842220
1.67408940000000	260.702100219727	238.963683861704
1.74732140000000	260.649359893799	238.930642214805
1.79952680000000	260.638812103272	238.890445215396
1.87166830000000	260.596620025635	238.858787284595
1.92292720000000	260.564976196289	238.817308848623
1.97431730000000	260.586072235107	238.778853350725
2.04630120000000	260.775936126709	238.764366238548
2.09779230000000	260.902511901855	238.765931308419
2.17052630000000	260.828675994873	238.756542780738
2.24498890000000	260.849772033691	238.738682413742
2.31860950000000	260.913060150146	238.743574873217
2.38932230000000	260.870868072510	238.741226288026
2.46233360000000	260.881415863037	238.742483024690
2.51546670000000	260.881415863037	238.741383298651
2.59004720000000	260.881415863037	238.741349100965
2.64201290000000	260.881415863037	238.741701613879
2.69392870000000	260.881415863037	238.742539456165
2.76697280000000	260.881415863037	238.742776678493
2.81353890000000	260.881415863037	238.742637988615
2.89141130000000	260.860319824219	238.731325247053
2.94285050000000	260.860319824219	238.730601254607
2.99279240000000	260.839223785400	238.719323834984
3.06632040000000	260.818128204346	238.708830500725
3.11758500000000	260.828675994873	238.709405314663
3.19272160000000	260.860319824219	238.711153761281
3.24291310000000	260.870868072510	238.712387502719
3.29436300000000	260.891964111328	238.713765385509
3.36604880000000	260.902511901855	238.714186649378
3.41098350000000	260.902511901855	238.714094680754
3.46188500000000	260.891964111328	238.713640862918
3.51265330000000	260.881415863037	238.713369889031
3.56054780000000	260.891964111328	238.713455792890
3.63277620000000	260.891964111328	238.713523861291
3.68521280000000	260.881415863037	238.703259841721
3.73704870000000	260.891964111328	238.703935475547
3.80829610000000	260.891964111328	238.704310711250
3.86010160000000	260.902511901855	238.704629765307
3.93342480000000	260.902511901855	238.704745385978
3.99064070000000	260.902511901855	238.704641793572
4.03295790000000	260.902511901855	238.704191249473
4.08445240000000	260.902511901855	238.704110648309
4.13528670000000	260.902511901855	238.704437840918
4.18031730000000	260.902511901855	238.714976401212
4.23250960000000	260.913060150146	238.715389064664
4.28595870000000	260.923607940674	238.715678545287
4.33425200000000	260.923607940674	238.715639096267
4.40759900000000	260.913060150146	238.715515864382
4.48068700000000	260.923607940674	238.715495012259
4.53229340000000	260.902511901855	238.714924041832
4.58475100000000	260.891964111328	238.704382927133
4.63096940000000	260.891964111328	238.704319295428
4.68064090000000	260.891964111328	238.704320905982
4.73215560000000	260.902511901855	238.704509342884
4.80514050000000	260.902511901855	238.704698335615
4.85505230000000	260.913060150146	238.704685054198
4.92631040000000	260.913060150146	238.704826722485
4.97549510000000	260.913060150146	238.704839602262
5.04962110000000	260.913060150146	238.704730589199
5.09581520000000	260.913060150146	238.704656529621
5.16866020000000	260.913060150146	238.704573806282
5.22395340000000	260.913060150146	238.704713029072
5.27664090000000	260.913060150146	238.704760238890
5.35232550000000	260.902511901855	238.704706501395
5.42946230000000	260.913060150146	238.704669533234
5.47994600000000	260.891964111328	238.704447514821
5.55592010000000	260.902511901855	238.704548515285
5.60624260000000	260.902511901855	238.694183584845
5.67921880000000	260.891964111328	238.694145207470
5.73095430000000	260.902511901855	238.694235867683
5.78207800000000	260.891964111328	238.694380208003
5.85512000000000	260.902511901855	238.694488312035
5.92532020000000	260.902511901855	238.694726746876
5.99801270000000	260.902511901855	238.694789274750
6.07309350000000	260.902511901855	238.694635232613
6.12482010000000	260.902511901855	238.694308583842
6.17628030000000	260.902511901855	238.694384789795
6.25184820000000	260.902511901855	238.694267190167
6.29781980000000	260.891964111328	238.694307208972
6.34405570000000	260.902511901855	238.694295491400
6.42389830000000	260.902511901855	238.694346022059
6.46959010000000	260.902511901855	238.694408466919
6.51471690000000	260.902511901855	238.694516604692
6.56369930000000	260.902511901855	238.694566033483
6.60987410000000	260.902511901855	238.694625767152
6.65659500000000	260.902511901855	238.694392217733
6.70683790000000	260.891964111328	238.694562390799
6.74919720000000	260.902511901855	238.694341965638
6.79594750000000	260.902511901855	238.694337788572
6.84371260000000	260.902511901855	238.694290084958
6.88870100000000	260.902511901855	238.694477076975
6.93238680000000	260.902511901855	238.694341932923
7.00794520000000	260.902511901855	238.694324962118
7.05609250000000	260.902511901855	238.694275750692
7.10247280000000	260.891964111328	238.694234212428
7.15088140000000	260.902511901855	238.694157302079
7.19550360000000	260.891964111328	238.694274353371
7.24637400000000	260.902511901855	238.694277500624
7.29426300000000	260.902511901855	238.694323349259
7.34042450000000	260.902511901855	238.694254471100
7.38668400000000	260.902511901855	238.694351380745
7.43277010000000	260.902511901855	238.694350343874
7.47850140000000	260.902511901855	238.694349702226
7.55670270000000	260.902511901855	238.694369972716
7.60610480000000	260.902511901855	238.694420698047
7.65153650000000	260.902511901855	238.694325362465
7.72340840000000	260.902511901855	238.694420572174
7.77281670000000	260.902511901855	238.694392545844
7.82997110000000	260.902511901855	238.694422122851
7.88450030000000	260.902511901855	238.694485823553
7.95760360000000	260.902511901855	238.694438263592
8.01047690000000	260.902511901855	238.694380252206
8.06586830000000	260.902511901855	238.694573237477
8.12036990000000	260.902511901855	238.694586661783
8.16678920000000	260.902511901855	238.694633315458
8.21360230000000	260.902511901855	238.694697376148
8.28674770000000	260.891964111328	238.694671938899
8.33698240000000	260.902511901855	238.694525689829
8.40945730000000	260.891964111328	238.694583364617
8.46729220000000	260.891964111328	238.694536892686
8.51267900000000	260.891964111328	238.694671790218
8.58536440000000	260.902511901855	238.694752423708
8.63650070000000	260.902511901855	238.694658411062
8.70908670000000	260.902511901855	238.684149734252
8.78112380000000	260.881415863037	238.683974242381
8.85341390000000	260.881415863037	238.683959686014
8.90297750000000	260.891964111328	238.684155250908
8.96614270000000	260.891964111328	238.684421462260
9.03744700000000	260.891964111328	238.684684549456
9.08835840000000	260.891964111328	238.684793557908
9.22819730000000	260.902511901855	238.684988128506
9.27726480000000	260.913060150146	238.685135558628
9.31996270000000	260.913060150146	238.685208176341
9.37377590000000	260.913060150146	238.685226165093
9.43374060000000	260.913060150146	238.685159802051
9.47259030000000	260.902511901855	238.674604103351
9.52630490000000	260.902511901855	238.674667755098
9.58112210000000	260.891964111328	238.674721611037
9.62543410000000	260.902511901855	238.675096726304
9.67754220000000	260.902511901855	238.675313611289
9.72372000000000	260.902511901855	238.675599398180
9.76403180000000	260.902511901855	238.675669030077
9.84415370000000	260.923607940674	238.675831475265
9.89161650000000	260.902511901855	238.675755629250
9.93858310000000	260.913060150146	238.675615134375
9.98914990000000	260.913060150146	238.675684971875
10.0465661000000	260.913060150146	238.675721565977
10.0908864000000	260.913060150146	238.675706626574
10.1349466000000	260.913060150146	238.675863292923
10.1833960000000	260.913060150146	238.676043310472
10.2304337000000	260.913060150146	238.676167181627
10.2834075000000	260.913060150146	238.676302556620
10.3432734000000	260.913060150146	238.676418621440
10.3857632000000	260.913060150146	238.676506905499
10.4305597000000	260.913060150146	238.676515570164
10.4768493000000	260.913060150146	238.676498236182
10.5242689000000	260.913060150146	238.676411275837
10.5717099000000	260.913060150146	238.676477593511
10.6223214000000	260.913060150146	238.676476386172
10.6846309000000	260.913060150146	238.676248297650
10.7448955000000	260.913060150146	238.676199311659
10.8004118000000	260.913060150146	238.676109890600
10.8474494000000	260.913060150146	238.676129648160
10.8954420000000	260.913060150146	238.676110078155
10.9468326000000	260.913060150146	238.676089085503
10.9976671000000	260.913060150146	238.676064650821
11.0437357000000	260.913060150146	238.676119998003
11.0897934000000	260.913060150146	238.676049753797
11.1385333000000	260.913060150146	238.675849764520
11.1808989000000	260.913060150146	238.675983002589
11.2546059000000	260.913060150146	238.675918612096
11.3054231000000	260.913060150146	238.676119200465

 };

 \addplot[color= black, line width = 1pt] table[x=time, y=u1] {
time u1 u2
0	252	252
0.0643425000000000	253.592748413086	250.429203673205
0.128685000000000	253.909187622070	248.836353598429
0.179990500000000	253.170827636719	247.559443002836
0.252648600000000	253.666583862305	246.527064564112
0.302044600000000	254.077955932617	245.373806367375
0.350144400000000	254.636999816895	244.597422272119
0.396763100000000	255.533580322266	243.830264609738
0.450177800000000	256.282488098145	243.196405285201
0.502453700000000	257.084135742188	242.662400427371
0.551189800000000	258.318251953125	242.227203016134
0.608195100000000	259.024968109131	241.833970811376
0.652493600000000	258.877295837402	241.439116410314
0.728552000000000	258.444828186035	241.089357775472
0.777774400000000	258.339348449707	240.720242486179
0.851622600000000	258.645239868164	240.444997594579
0.905830800000000	258.982776031494	240.196918562907
0.955523500000000	259.700039978027	240.047453161917
1.02737480000000	260.470043792725	239.908365971577
1.08015320000000	260.533331909180	239.785919513730
1.15260420000000	260.322371978760	239.651145242480
1.22881620000000	260.079768218994	239.501179168985
1.29924180000000	260.016480102539	239.372444178842
1.35085350000000	260.132508087158	239.280767695086
1.40091180000000	260.343468017578	239.193913136941
1.45428870000000	260.522784118652	239.138236492275
1.52652020000000	260.744291839600	239.102225550019
1.57532690000000	260.849772033691	239.060899324161
1.62653600000000	260.744291839600	239.005763842220
1.67408940000000	260.702100219727	238.963683861704
1.74732140000000	260.649359893799	238.930642214805
1.79952680000000	260.638812103272	238.890445215396
1.87166830000000	260.596620025635	238.858787284595
1.92292720000000	260.564976196289	238.817308848623
1.97431730000000	260.586072235107	238.778853350725
2.04630120000000	260.775936126709	238.764366238548
2.09779230000000	260.902511901855	238.765931308419
2.17052630000000	260.828675994873	238.756542780738
2.24498890000000	260.849772033691	238.738682413742
2.31860950000000	260.913060150146	238.743574873217
2.38932230000000	260.870868072510	238.741226288026
2.46233360000000	260.881415863037	238.742483024690
2.51546670000000	260.881415863037	238.741383298651
2.59004720000000	260.881415863037	238.741349100965
2.64201290000000	260.881415863037	238.741701613879
2.69392870000000	260.881415863037	238.742539456165
2.76697280000000	260.881415863037	238.742776678493
2.81353890000000	260.881415863037	238.742637988615
2.89141130000000	260.860319824219	238.731325247053
2.94285050000000	260.860319824219	238.730601254607
2.99279240000000	260.839223785400	238.719323834984
3.06632040000000	260.818128204346	238.708830500725
3.11758500000000	260.828675994873	238.709405314663
3.19272160000000	260.860319824219	238.711153761281
3.24291310000000	260.870868072510	238.712387502719
3.29436300000000	260.891964111328	238.713765385509
3.36604880000000	260.902511901855	238.714186649378
3.41098350000000	260.902511901855	238.714094680754
3.46188500000000	260.891964111328	238.713640862918
3.51265330000000	260.881415863037	238.713369889031
3.56054780000000	260.891964111328	238.713455792890
3.63277620000000	260.891964111328	238.713523861291
3.68521280000000	260.881415863037	238.703259841721
3.73704870000000	260.891964111328	238.703935475547
3.80829610000000	260.891964111328	238.704310711250
3.86010160000000	260.902511901855	238.704629765307
3.93342480000000	260.902511901855	238.704745385978
3.99064070000000	260.902511901855	238.704641793572
4.03295790000000	260.902511901855	238.704191249473
4.08445240000000	260.902511901855	238.704110648309
4.13528670000000	260.902511901855	238.704437840918
4.18031730000000	260.902511901855	238.714976401212
4.23250960000000	260.913060150146	238.715389064664
4.28595870000000	260.923607940674	238.715678545287
4.33425200000000	260.923607940674	238.715639096267
4.40759900000000	260.913060150146	238.715515864382
4.48068700000000	260.923607940674	238.715495012259
4.53229340000000	260.902511901855	238.714924041832
4.58475100000000	260.891964111328	238.704382927133
4.63096940000000	260.891964111328	238.704319295428
4.68064090000000	260.891964111328	238.704320905982
4.73215560000000	260.902511901855	238.704509342884
4.80514050000000	260.902511901855	238.704698335615
4.85505230000000	260.913060150146	238.704685054198
4.92631040000000	260.913060150146	238.704826722485
4.97549510000000	260.913060150146	238.704839602262
5.04962110000000	260.913060150146	238.704730589199
5.09581520000000	260.913060150146	238.704656529621
5.16866020000000	260.913060150146	238.704573806282
5.22395340000000	260.913060150146	238.704713029072
5.27664090000000	260.913060150146	238.704760238890
5.35232550000000	260.902511901855	238.704706501395
5.42946230000000	260.913060150146	238.704669533234
5.47994600000000	260.891964111328	238.704447514821
5.55592010000000	260.902511901855	238.704548515285
5.60624260000000	260.902511901855	238.694183584845
5.67921880000000	260.891964111328	238.694145207470
5.73095430000000	260.902511901855	238.694235867683
5.78207800000000	260.891964111328	238.694380208003
5.85512000000000	260.902511901855	238.694488312035
5.92532020000000	260.902511901855	238.694726746876
5.99801270000000	260.902511901855	238.694789274750
6.07309350000000	260.902511901855	238.694635232613
6.12482010000000	260.902511901855	238.694308583842
6.17628030000000	260.902511901855	238.694384789795
6.25184820000000	260.902511901855	238.694267190167
6.29781980000000	260.891964111328	238.694307208972
6.34405570000000	260.902511901855	238.694295491400
6.42389830000000	260.902511901855	238.694346022059
6.46959010000000	260.902511901855	238.694408466919
6.51471690000000	260.902511901855	238.694516604692
6.56369930000000	260.902511901855	238.694566033483
6.60987410000000	260.902511901855	238.694625767152
6.65659500000000	260.902511901855	238.694392217733
6.70683790000000	260.891964111328	238.694562390799
6.74919720000000	260.902511901855	238.694341965638
6.79594750000000	260.902511901855	238.694337788572
6.84371260000000	260.902511901855	238.694290084958
6.88870100000000	260.902511901855	238.694477076975
6.93238680000000	260.902511901855	238.694341932923
7.00794520000000	260.902511901855	238.694324962118
7.05609250000000	260.902511901855	238.694275750692
7.10247280000000	260.891964111328	238.694234212428
7.15088140000000	260.902511901855	238.694157302079
7.19550360000000	260.891964111328	238.694274353371
7.24637400000000	260.902511901855	238.694277500624
7.29426300000000	260.902511901855	238.694323349259
7.34042450000000	260.902511901855	238.694254471100
7.38668400000000	260.902511901855	238.694351380745
7.43277010000000	260.902511901855	238.694350343874
7.47850140000000	260.902511901855	238.694349702226
7.55670270000000	260.902511901855	238.694369972716
7.60610480000000	260.902511901855	238.694420698047
7.65153650000000	260.902511901855	238.694325362465
7.72340840000000	260.902511901855	238.694420572174
7.77281670000000	260.902511901855	238.694392545844
7.82997110000000	260.902511901855	238.694422122851
7.88450030000000	260.902511901855	238.694485823553
7.95760360000000	260.902511901855	238.694438263592
8.01047690000000	260.902511901855	238.694380252206
8.06586830000000	260.902511901855	238.694573237477
8.12036990000000	260.902511901855	238.694586661783
8.16678920000000	260.902511901855	238.694633315458
8.21360230000000	260.902511901855	238.694697376148
8.28674770000000	260.891964111328	238.694671938899
8.33698240000000	260.902511901855	238.694525689829
8.40945730000000	260.891964111328	238.694583364617
8.46729220000000	260.891964111328	238.694536892686
8.51267900000000	260.891964111328	238.694671790218
8.58536440000000	260.902511901855	238.694752423708
8.63650070000000	260.902511901855	238.694658411062
8.70908670000000	260.902511901855	238.684149734252
8.78112380000000	260.881415863037	238.683974242381
8.85341390000000	260.881415863037	238.683959686014
8.90297750000000	260.891964111328	238.684155250908
8.96614270000000	260.891964111328	238.684421462260
9.03744700000000	260.891964111328	238.684684549456
9.08835840000000	260.891964111328	238.684793557908
9.22819730000000	260.902511901855	238.684988128506
9.27726480000000	260.913060150146	238.685135558628
9.31996270000000	260.913060150146	238.685208176341
9.37377590000000	260.913060150146	238.685226165093
9.43374060000000	260.913060150146	238.685159802051
9.47259030000000	260.902511901855	238.674604103351
9.52630490000000	260.902511901855	238.674667755098
9.58112210000000	260.891964111328	238.674721611037
9.62543410000000	260.902511901855	238.675096726304
9.67754220000000	260.902511901855	238.675313611289
9.72372000000000	260.902511901855	238.675599398180
9.76403180000000	260.902511901855	238.675669030077
9.84415370000000	260.923607940674	238.675831475265
9.89161650000000	260.902511901855	238.675755629250
9.93858310000000	260.913060150146	238.675615134375
9.98914990000000	260.913060150146	238.675684971875
10.0465661000000	260.913060150146	238.675721565977
10.0908864000000	260.913060150146	238.675706626574
10.1349466000000	260.913060150146	238.675863292923
10.1833960000000	260.913060150146	238.676043310472
10.2304337000000	260.913060150146	238.676167181627
10.2834075000000	260.913060150146	238.676302556620
10.3432734000000	260.913060150146	238.676418621440
10.3857632000000	260.913060150146	238.676506905499
10.4305597000000	260.913060150146	238.676515570164
10.4768493000000	260.913060150146	238.676498236182
10.5242689000000	260.913060150146	238.676411275837
10.5717099000000	260.913060150146	238.676477593511
10.6223214000000	260.913060150146	238.676476386172
10.6846309000000	260.913060150146	238.676248297650
10.7448955000000	260.913060150146	238.676199311659
10.8004118000000	260.913060150146	238.676109890600
10.8474494000000	260.913060150146	238.676129648160
10.8954420000000	260.913060150146	238.676110078155
10.9468326000000	260.913060150146	238.676089085503
10.9976671000000	260.913060150146	238.676064650821
11.0437357000000	260.913060150146	238.676119998003
11.0897934000000	260.913060150146	238.676049753797
11.1385333000000	260.913060150146	238.675849764520
11.1808989000000	260.913060150146	238.675983002589
11.2546059000000	260.913060150146	238.675918612096
11.3054231000000	260.913060150146	238.676119200465

 };

    \end{axis}

\end{tikzpicture}

%% file: FigPlot/Inputs_5_deg_d.tex
\begin{tikzpicture}[spy using outlines={rectangle, magnification=10,  connect spies}]
    \begin{axis}[
      xmin=0, xmax=3,
      ymin=235, ymax=275,
      ymajorgrids=true,
      grid style=dashed,
      legend pos=north west,
      width = 0.26\linewidth,
      height = 0.2\linewidth,
    label style={font=\footnotesize},
      tick label style={font=\footnotesize},
    xlabel={Time [s]},
 legend columns=2 
    ]
        
 \legend{\footnotesize{$L_1$}, \footnotesize{$L_2$}}

    \addplot[color= red, line width = 1pt] table[x=time, y=u2] {
time u1 u2
0	252	252
0.0699137000000000	252.822744140625	250.429203673205
0.139827400000000	253.856448669434	249.069275329188
0.194102700000000	253.286856079102	247.541385427785
0.243916900000000	252.516852722168	246.599654218562
0.318287600000000	254.120148010254	245.590319497084
0.367103200000000	254.352203979492	244.547695505941
0.443274200000000	255.449196166992	243.986890428733
0.492855500000000	256.936463928223	243.356759442549
0.545328700000000	258.328800659180	242.987562972237
0.615871000000000	259.721136474609	242.599028887248
0.665425800000000	259.436340179443	242.114216668395
0.737518800000000	257.780304565430	241.639709695387
0.782812400000000	257.295096130371	241.141747320141
0.842091600000000	257.569344177246	240.822479972558
0.915809700000000	258.634692077637	240.546669019184
0.966468700000000	259.510176086426	240.406925989226
1.03915690000000	261.345528259277	240.325310514962
1.09039560000000	261.282240142822	240.181592703018
1.14207730000000	260.280180358887	239.994501233958
1.18624300000000	259.858260498047	239.832633325929
1.23742990000000	259.879356079102	239.674013628735
1.29333670000000	259.794972381592	239.516590638346
1.36637840000000	260.058672180176	239.413563182838
1.41663260000000	260.343468475342	239.357654401274
1.46470550000000	261.081828460693	239.326786008411
1.53808770000000	261.197856445313	239.312774782520
1.61079750000000	261.081828460693	239.294616924935
1.66405270000000	260.976348266602	239.253730997202
1.71628120000000	260.818128204346	239.204077528073
1.78681480000000	260.702100219727	239.160861190309
1.83741000000000	260.765388336182	239.117359444415
1.88954520000000	260.902512359619	239.087540164386
1.94469740000000	260.976348266602	239.072658267229
2.02077230000000	261.071280212402	239.049669725921
2.06874850000000	261.113472290039	239.033340366066
2.14151810000000	261.176760406494	239.026320815139
2.19046000000000	261.124020080566	239.014669702322
2.26387860000000	261.050184173584	238.989148450490
2.31193920000000	260.997444305420	238.985909970359
2.38382540000000	261.029088134766	238.977536500399
2.42911650000000	261.007992095947	238.956977726952
2.50460640000000	260.976348266602	238.935444536420
2.55386040000000	260.976348266602	238.914941943671
2.62781960000000	261.166212158203	238.921630490968
2.67542040000000	261.240048522949	238.923418807733
2.73018150000000	261.240048522949	238.924643967011
2.80759430000000	261.208404235840	238.922207099394
2.85714990000000	261.155664367676	238.916288807687
2.93292070000000	261.113472290039	238.915059997853
3.00625250000000	261.081828460693	238.913332664417
3.08050290000000	261.102924499512	238.914663902638
3.13096320000000	261.092376251221	238.903621481370
3.20954100000000	261.166212158203	238.908550891520
3.28319390000000	261.155664367676	238.908930293072
3.33546770000000	261.145116119385	238.907237736449
3.40944320000000	261.134568328857	238.907765910552
3.45846300000000	261.124020080566	238.897093763101
3.53127440000000	261.102924499512	238.895824702764
3.57886240000000	261.092376251221	238.884640134524
3.62521450000000	261.081828460693	238.884187258674
3.67860200000000	261.092376251221	238.884404788911
3.73051010000000	261.102924499512	238.885620962462
3.77912090000000	261.113472290039	238.886474582307
3.85302030000000	261.134568328857	238.887480605496
3.90209650000000	261.145116119385	238.888256410646
3.97713010000000	261.155664367676	238.888505869338
4.02614210000000	261.134568328857	238.888029036931
4.07917320000000	261.134568328857	238.887396015982
4.15102430000000	261.134568328857	238.897590395435
4.19735050000000	261.134568328857	238.897332262666
4.24341020000000	261.134568328857	238.897508809467
4.29729610000000	261.134568328857	238.897780863656
4.36837350000000	261.134568328857	238.898018393706
4.42018700000000	261.145116119385	238.898313696833
4.47241690000000	261.155664367676	238.898403197065
4.54443800000000	261.134568328857	238.898043303373
4.59277480000000	261.134568328857	238.897887531095
4.66432510000000	261.134568328857	238.897540588402
4.71592330000000	261.134568328857	238.897183707917
4.79045940000000	261.124020080566	238.886895161077
4.83558790000000	261.124020080566	238.886964685172
4.91186060000000	261.124020080566	238.887379380619
4.96080200000000	261.124020080566	238.877237562465
5.03008380000000	261.124020080566	238.877389880602
5.08021380000000	261.124020080566	238.877479873809
5.13469410000000	261.134568328857	238.877760389378
5.18084930000000	261.134568328857	238.877759598053
5.23046010000000	261.134568328857	238.877858174511
5.28126140000000	261.134568328857	238.867565122380
5.35332670000000	261.134568328857	238.867605021243
5.40415820000000	261.134568328857	238.867809544407
5.45042660000000	261.134568328857	238.868209804461
5.50249240000000	261.134568328857	238.868545309968
5.55464570000000	261.145116119385	238.868735716191
5.62790190000000	261.155664367676	238.869050862663
5.67754230000000	261.145116119385	238.869004627151
5.72970100000000	261.145116119385	238.868940041942
5.78081820000000	261.145116119385	238.868827687439
5.83111440000000	261.145116119385	238.868901242985
5.88087670000000	261.155664367676	238.869064890712
5.95273140000000	261.145116119385	238.869034358197
6.00030180000000	261.155664367676	238.869206409211
6.04896470000000	261.155664367676	238.869217597316
6.09961320000000	261.155664367676	238.869362963739
6.15256720000000	261.155664367676	238.869407956639
6.22651190000000	261.155664367676	238.869445440303
6.27682850000000	261.155664367676	238.869359645593
6.34885270000000	261.145116119385	238.869079046310
6.39690600000000	261.145116119385	238.869122529412
6.44872540000000	261.155664367676	238.869102595574
6.52333810000000	261.155664367676	238.869270760664
6.57160050000000	261.155664367676	238.869260661550
6.64542100000000	261.145116119385	238.869466183211
6.69113550000000	261.155664367676	238.869468786222
6.74693520000000	261.155664367676	238.869430042050
6.82225890000000	261.155664367676	238.869460916068
6.87245330000000	261.155664367676	238.869368898158
6.94514420000000	261.155664367676	238.869362481474
7.01563640000000	261.155664367676	238.869344395804
7.09011700000000	261.145116119385	238.869097930375
7.14254890000000	261.134568328857	238.858417482068
7.21533040000000	261.134568328857	238.858359864455
7.26448740000000	261.134568328857	238.858684884249
7.32049080000000	261.145116119385	238.859131684634
7.39445410000000	261.155664367676	238.859523761167
7.44117300000000	261.155664367676	238.860045307619
7.51739930000000	261.155664367676	238.860277012531
7.56822100000000	261.155664367676	238.860277063584
7.63955130000000	261.155664367676	238.860099932675
7.71922830000000	261.155664367676	238.860102017293
7.79069860000000	261.155664367676	238.859802114686
7.84410460000000	261.155664367676	238.859684601423
7.89107560000000	261.155664367676	238.859951284234
7.96379290000000	261.155664367676	238.859862147155
8.01673110000000	261.155664367676	238.860197749795
8.09096280000000	261.155664367676	238.860248369344
8.14213200000000	261.155664367676	238.860266989628
8.18757810000000	261.176760406494	238.860391538204
8.26129620000000	261.155664367676	238.860260695666
8.31260930000000	261.155664367676	238.849680067999
8.38675750000000	261.155664367676	238.849608937039
8.43807680000000	261.155664367676	238.849895903111
8.48338650000000	261.155664367676	238.850370461785
8.53913100000000	261.166212158203	238.850708674214
8.58833550000000	261.176760406494	238.851158564963
8.63805540000000	261.176760406494	238.851248745878
8.69103920000000	261.176760406494	238.851473405664
8.74064900000000	261.176760406494	238.851469687747
8.78978460000000	261.176760406494	238.851662525812
8.86457880000000	261.176760406494	238.851550084448
8.91526890000000	261.176760406494	238.851552897877
8.98770450000000	261.176760406494	238.851488403943
9.03914410000000	261.176760406494	238.851378804496
9.08520050000000	261.176760406494	238.840948736233
9.15839670000000	261.176760406494	238.841051515101
9.21065370000000	261.176760406494	238.841184631663
9.28312770000000	261.176760406494	238.841482384009
9.33757670000000	261.187308197022	238.852338765234
9.39246670000000	261.197856445313	238.852816373170
9.46612550000000	261.197856445313	238.853104868749
9.51787250000000	261.197856445313	238.853092516687
9.59048520000000	261.197856445313	238.852922401096
9.63931710000000	261.197856445313	238.852699963340
9.69090780000000	261.187308197022	238.852501330245
9.76187610000000	261.187308197022	238.852409743513
9.81277130000000	261.197856445313	238.852636407256
9.88740460000000	261.197856445313	238.852549349484
9.93451580000000	261.187308197022	238.852513879131
10.0084573000000	261.187308197022	238.852498800607
10.0592610000000	261.187308197022	238.852369131648
10.1117554000000	261.187308197022	238.852380671808
10.1858978000000	261.187308197022	238.852434607997
10.2367241000000	261.187308197022	238.852450971627
10.2923635000000	261.187308197022	238.841879650656
10.3373338000000	261.187308197022	238.841637223983
10.3872943000000	261.187308197022	238.841433534438
10.4652493000000	261.166212158203	238.841406190477
10.5216448000000	261.176760406494	238.841383128730
10.5694969000000	261.176760406494	238.841484585479
10.6202200000000	261.176760406494	238.841647125501
10.6663138000000	261.176760406494	238.841660607489
10.7123956000000	261.176760406494	238.841672270381
10.7615487000000	261.176760406494	238.841833836611
10.8071981000000	261.176760406494	238.841929825445
10.8557634000000	261.176760406494	238.841940652640
10.9017323000000	261.176760406494	238.841907481791
10.9468140000000	261.176760406494	238.841924373489
10.9931837000000	261.176760406494	238.841823195473
11.0414305000000	261.176760406494	238.841835879882
11.0861015000000	261.176760406494	238.841912330944
11.1329093000000	261.176760406494	238.841993799695
11.1846944000000	261.176760406494	238.831638712886
11.2273296000000	261.176760406494	238.831576532050
11.2721457000000	261.176760406494	238.831640247606
11.3150864000000	261.187308197022	238.842213669235
11.3618545000000	261.197856445313	238.842309942988
11.4100199000000	261.197856445313	238.842230669175
11.4560203000000	261.197856445313	238.842202371735
11.5007451000000	261.187308197022	238.842385606651

 };

 \addplot[color= black, line width = 1pt] table[x=time, y=u1] {
time u1 u2
0	252	252
0.0699137000000000	252.822744140625	250.429203673205
0.139827400000000	253.856448669434	249.069275329188
0.194102700000000	253.286856079102	247.541385427785
0.243916900000000	252.516852722168	246.599654218562
0.318287600000000	254.120148010254	245.590319497084
0.367103200000000	254.352203979492	244.547695505941
0.443274200000000	255.449196166992	243.986890428733
0.492855500000000	256.936463928223	243.356759442549
0.545328700000000	258.328800659180	242.987562972237
0.615871000000000	259.721136474609	242.599028887248
0.665425800000000	259.436340179443	242.114216668395
0.737518800000000	257.780304565430	241.639709695387
0.782812400000000	257.295096130371	241.141747320141
0.842091600000000	257.569344177246	240.822479972558
0.915809700000000	258.634692077637	240.546669019184
0.966468700000000	259.510176086426	240.406925989226
1.03915690000000	261.345528259277	240.325310514962
1.09039560000000	261.282240142822	240.181592703018
1.14207730000000	260.280180358887	239.994501233958
1.18624300000000	259.858260498047	239.832633325929
1.23742990000000	259.879356079102	239.674013628735
1.29333670000000	259.794972381592	239.516590638346
1.36637840000000	260.058672180176	239.413563182838
1.41663260000000	260.343468475342	239.357654401274
1.46470550000000	261.081828460693	239.326786008411
1.53808770000000	261.197856445313	239.312774782520
1.61079750000000	261.081828460693	239.294616924935
1.66405270000000	260.976348266602	239.253730997202
1.71628120000000	260.818128204346	239.204077528073
1.78681480000000	260.702100219727	239.160861190309
1.83741000000000	260.765388336182	239.117359444415
1.88954520000000	260.902512359619	239.087540164386
1.94469740000000	260.976348266602	239.072658267229
2.02077230000000	261.071280212402	239.049669725921
2.06874850000000	261.113472290039	239.033340366066
2.14151810000000	261.176760406494	239.026320815139
2.19046000000000	261.124020080566	239.014669702322
2.26387860000000	261.050184173584	238.989148450490
2.31193920000000	260.997444305420	238.985909970359
2.38382540000000	261.029088134766	238.977536500399
2.42911650000000	261.007992095947	238.956977726952
2.50460640000000	260.976348266602	238.935444536420
2.55386040000000	260.976348266602	238.914941943671
2.62781960000000	261.166212158203	238.921630490968
2.67542040000000	261.240048522949	238.923418807733
2.73018150000000	261.240048522949	238.924643967011
2.80759430000000	261.208404235840	238.922207099394
2.85714990000000	261.155664367676	238.916288807687
2.93292070000000	261.113472290039	238.915059997853
3.00625250000000	261.081828460693	238.913332664417
3.08050290000000	261.102924499512	238.914663902638
3.13096320000000	261.092376251221	238.903621481370
3.20954100000000	261.166212158203	238.908550891520
3.28319390000000	261.155664367676	238.908930293072
3.33546770000000	261.145116119385	238.907237736449
3.40944320000000	261.134568328857	238.907765910552
3.45846300000000	261.124020080566	238.897093763101
3.53127440000000	261.102924499512	238.895824702764
3.57886240000000	261.092376251221	238.884640134524
3.62521450000000	261.081828460693	238.884187258674
3.67860200000000	261.092376251221	238.884404788911
3.73051010000000	261.102924499512	238.885620962462
3.77912090000000	261.113472290039	238.886474582307
3.85302030000000	261.134568328857	238.887480605496
3.90209650000000	261.145116119385	238.888256410646
3.97713010000000	261.155664367676	238.888505869338
4.02614210000000	261.134568328857	238.888029036931
4.07917320000000	261.134568328857	238.887396015982
4.15102430000000	261.134568328857	238.897590395435
4.19735050000000	261.134568328857	238.897332262666
4.24341020000000	261.134568328857	238.897508809467
4.29729610000000	261.134568328857	238.897780863656
4.36837350000000	261.134568328857	238.898018393706
4.42018700000000	261.145116119385	238.898313696833
4.47241690000000	261.155664367676	238.898403197065
4.54443800000000	261.134568328857	238.898043303373
4.59277480000000	261.134568328857	238.897887531095
4.66432510000000	261.134568328857	238.897540588402
4.71592330000000	261.134568328857	238.897183707917
4.79045940000000	261.124020080566	238.886895161077
4.83558790000000	261.124020080566	238.886964685172
4.91186060000000	261.124020080566	238.887379380619
4.96080200000000	261.124020080566	238.877237562465
5.03008380000000	261.124020080566	238.877389880602
5.08021380000000	261.124020080566	238.877479873809
5.13469410000000	261.134568328857	238.877760389378
5.18084930000000	261.134568328857	238.877759598053
5.23046010000000	261.134568328857	238.877858174511
5.28126140000000	261.134568328857	238.867565122380
5.35332670000000	261.134568328857	238.867605021243
5.40415820000000	261.134568328857	238.867809544407
5.45042660000000	261.134568328857	238.868209804461
5.50249240000000	261.134568328857	238.868545309968
5.55464570000000	261.145116119385	238.868735716191
5.62790190000000	261.155664367676	238.869050862663
5.67754230000000	261.145116119385	238.869004627151
5.72970100000000	261.145116119385	238.868940041942
5.78081820000000	261.145116119385	238.868827687439
5.83111440000000	261.145116119385	238.868901242985
5.88087670000000	261.155664367676	238.869064890712
5.95273140000000	261.145116119385	238.869034358197
6.00030180000000	261.155664367676	238.869206409211
6.04896470000000	261.155664367676	238.869217597316
6.09961320000000	261.155664367676	238.869362963739
6.15256720000000	261.155664367676	238.869407956639
6.22651190000000	261.155664367676	238.869445440303
6.27682850000000	261.155664367676	238.869359645593
6.34885270000000	261.145116119385	238.869079046310
6.39690600000000	261.145116119385	238.869122529412
6.44872540000000	261.155664367676	238.869102595574
6.52333810000000	261.155664367676	238.869270760664
6.57160050000000	261.155664367676	238.869260661550
6.64542100000000	261.145116119385	238.869466183211
6.69113550000000	261.155664367676	238.869468786222
6.74693520000000	261.155664367676	238.869430042050
6.82225890000000	261.155664367676	238.869460916068
6.87245330000000	261.155664367676	238.869368898158
6.94514420000000	261.155664367676	238.869362481474
7.01563640000000	261.155664367676	238.869344395804
7.09011700000000	261.145116119385	238.869097930375
7.14254890000000	261.134568328857	238.858417482068
7.21533040000000	261.134568328857	238.858359864455
7.26448740000000	261.134568328857	238.858684884249
7.32049080000000	261.145116119385	238.859131684634
7.39445410000000	261.155664367676	238.859523761167
7.44117300000000	261.155664367676	238.860045307619
7.51739930000000	261.155664367676	238.860277012531
7.56822100000000	261.155664367676	238.860277063584
7.63955130000000	261.155664367676	238.860099932675
7.71922830000000	261.155664367676	238.860102017293
7.79069860000000	261.155664367676	238.859802114686
7.84410460000000	261.155664367676	238.859684601423
7.89107560000000	261.155664367676	238.859951284234
7.96379290000000	261.155664367676	238.859862147155
8.01673110000000	261.155664367676	238.860197749795
8.09096280000000	261.155664367676	238.860248369344
8.14213200000000	261.155664367676	238.860266989628
8.18757810000000	261.176760406494	238.860391538204
8.26129620000000	261.155664367676	238.860260695666
8.31260930000000	261.155664367676	238.849680067999
8.38675750000000	261.155664367676	238.849608937039
8.43807680000000	261.155664367676	238.849895903111
8.48338650000000	261.155664367676	238.850370461785
8.53913100000000	261.166212158203	238.850708674214
8.58833550000000	261.176760406494	238.851158564963
8.63805540000000	261.176760406494	238.851248745878
8.69103920000000	261.176760406494	238.851473405664
8.74064900000000	261.176760406494	238.851469687747
8.78978460000000	261.176760406494	238.851662525812
8.86457880000000	261.176760406494	238.851550084448
8.91526890000000	261.176760406494	238.851552897877
8.98770450000000	261.176760406494	238.851488403943
9.03914410000000	261.176760406494	238.851378804496
9.08520050000000	261.176760406494	238.840948736233
9.15839670000000	261.176760406494	238.841051515101
9.21065370000000	261.176760406494	238.841184631663
9.28312770000000	261.176760406494	238.841482384009
9.33757670000000	261.187308197022	238.852338765234
9.39246670000000	261.197856445313	238.852816373170
9.46612550000000	261.197856445313	238.853104868749
9.51787250000000	261.197856445313	238.853092516687
9.59048520000000	261.197856445313	238.852922401096
9.63931710000000	261.197856445313	238.852699963340
9.69090780000000	261.187308197022	238.852501330245
9.76187610000000	261.187308197022	238.852409743513
9.81277130000000	261.197856445313	238.852636407256
9.88740460000000	261.197856445313	238.852549349484
9.93451580000000	261.187308197022	238.852513879131
10.0084573000000	261.187308197022	238.852498800607
10.0592610000000	261.187308197022	238.852369131648
10.1117554000000	261.187308197022	238.852380671808
10.1858978000000	261.187308197022	238.852434607997
10.2367241000000	261.187308197022	238.852450971627
10.2923635000000	261.187308197022	238.841879650656
10.3373338000000	261.187308197022	238.841637223983
10.3872943000000	261.187308197022	238.841433534438
10.4652493000000	261.166212158203	238.841406190477
10.5216448000000	261.176760406494	238.841383128730
10.5694969000000	261.176760406494	238.841484585479
10.6202200000000	261.176760406494	238.841647125501
10.6663138000000	261.176760406494	238.841660607489
10.7123956000000	261.176760406494	238.841672270381
10.7615487000000	261.176760406494	238.841833836611
10.8071981000000	261.176760406494	238.841929825445
10.8557634000000	261.176760406494	238.841940652640
10.9017323000000	261.176760406494	238.841907481791
10.9468140000000	261.176760406494	238.841924373489
10.9931837000000	261.176760406494	238.841823195473
11.0414305000000	261.176760406494	238.841835879882
11.0861015000000	261.176760406494	238.841912330944
11.1329093000000	261.176760406494	238.841993799695
11.1846944000000	261.176760406494	238.831638712886
11.2273296000000	261.176760406494	238.831576532050
11.2721457000000	261.176760406494	238.831640247606
11.3150864000000	261.187308197022	238.842213669235
11.3618545000000	261.197856445313	238.842309942988
11.4100199000000	261.197856445313	238.842230669175
11.4560203000000	261.197856445313	238.842202371735
11.5007451000000	261.187308197022	238.842385606651

 };

    \end{axis}

\end{tikzpicture}

%% file: FigPlot/CtrlAgl_5deg_i.tex
\begin{tikzpicture}
    \begin{axis}[
      label style={font=\footnotesize},
      tick label style={font=\footnotesize},
      xmin=0, xmax=3,
      ymin=-0.2, ymax=6,
      ymajorgrids=true,
      grid style=dashed,
      legend pos=north west,
      width = 0.26\linewidth,
      height = 0.20\linewidth,
 y label style={at={(-0.23,.35)}},
    ylabel={Position $q_i$ [deg]}
    ]


    \addplot[color=blue, line width = 1pt] table {
0	0
0.0833301000000000	0.0375134662747445
0.166660200000000	0.148665434508786
0.211318000000000	1.16586776092331
0.282248500000000	1.47483136223987
0.352285200000000	2.18891968301663
0.399046300000000	2.82390854196966
0.452256300000000	2.90362646237545
0.501360200000000	2.58511045399307
0.553929200000000	2.92893798058751
0.607818200000000	3.62734649972249
0.658328600000000	3.65634355320437
0.709592300000000	3.66135255714096
0.759063700000000	3.76604455790549
0.803066100000000	3.92825294216250
0.848140100000000	4.09418886567219
0.922882200000000	4.39449000137036
0.971851200000000	4.50885723258533
1.02888970000000	4.54252505033123
1.07405850000000	4.53965602599220
1.12745240000000	4.48495789868208
1.18123060000000	4.48813759576480
1.25626780000000	4.56773047849954
1.30611060000000	4.64303127276827
1.35802290000000	4.74071831089123
1.40417920000000	4.77281627621415
1.44921470000000	4.80298265939470
1.52575350000000	4.82757346471125
1.57646670000000	4.80250989202517
1.63315720000000	4.82903489102602
1.68973810000000	4.84185116347218
1.76045670000000	4.85622409619880
1.80843790000000	4.86729307387526
1.86244250000000	4.88293547569850
1.92161830000000	4.89729591439944
1.99283720000000	4.90682022582355
2.04022980000000	4.91898251838657
2.08861790000000	4.92057846508507
2.13662970000000	4.93638437701291
2.19325590000000	4.94188214274294
2.25078930000000	4.94298796385870
2.31671690000000	4.94163772798701
2.36272120000000	4.93955835284296
2.41604910000000	4.93688217648242
2.46438970000000	4.93409648806893
2.50983200000000	4.93346539268773
2.55698240000000	4.93287628002508
2.60539580000000	4.93369563729429
2.65863300000000	4.93439973520838
2.71362010000000	4.93541091776214
2.76654080000000	4.93589234206075
2.82151360000000	4.93574630800815
2.89958490000000	4.93655994538534
2.95295490000000	4.93736204488268
3.02505640000000	4.93960489079410
3.07114240000000	4.94215532363596
3.11908960000000	4.94522351596811
3.16481720000000	4.94760639039738
3.21357320000000	4.94915740160100
3.25961480000000	4.95160777886128
3.31416200000000	4.95466034035228
3.36645200000000	4.95567156709396
3.42144450000000	4.95685141147331
3.49896250000000	4.95739404759263
3.55090310000000	4.95677283737413
3.62507800000000	4.95577070594958
3.68313420000000	4.95425205005312
3.73618780000000	4.95407846813768
3.79067750000000	4.95416025801624
3.83649170000000	4.95400922478782
3.88957190000000	4.95385340593010
3.94221460000000	4.95340659047839
4.01628730000000	4.95310052491224
4.06607190000000	4.95264636685955
4.11373440000000	4.95257342955629
4.16849380000000	4.95236414272953
4.21830790000000	4.95275629327021
4.26332260000000	4.95312536289917
4.31098230000000	4.95317005097638
4.36503000000000	4.95219674438910
4.41217070000000	4.95194181103488
4.45795670000000	4.95115686995092
4.50796140000000	4.95029495659626
4.55449210000000	4.94957637060783
4.60284320000000	4.95013929202289
4.67593310000000	4.94991506121163
4.72646680000000	4.95025450661589
4.77331730000000	4.95037849327073
4.84801160000000	4.95095569718522
4.89773920000000	4.95117491691907
4.97065330000000	4.95089383671097
5.02030260000000	4.95094108449835
5.07064080000000	4.95040695283163
5.12790470000000	4.95021115430730
5.17428430000000	4.94991202864765
5.23125240000000	4.94997930589140
5.29053750000000	4.95014453600226
5.34645290000000	4.95027042636570
5.41998510000000	4.95039890402609
5.47011250000000	4.95032519279741
5.52837630000000	4.95024402068336
5.60156950000000	4.95033343772392
5.64903530000000	4.95043724926297
5.70461970000000	4.94990953166113
5.77806180000000	4.94974708316482
5.83189990000000	4.94942666801135
5.87855080000000	4.94915656858885
5.92372320000000	4.94987275799812
5.97151780000000	4.95011354622310
6.01852520000000	4.95010772642916
6.07384330000000	4.95002907412075
6.11856380000000	4.94966383724325
6.16388860000000	4.94935357227348
6.21481440000000	4.94977050334216
6.26116280000000	4.95035309888113
6.30761450000000	4.95050601500340
6.35596450000000	4.95089029234097
6.40085030000000	4.95120028952685
6.44867970000000	4.95082155765070
6.49582840000000	4.95079926140087
6.54112860000000	4.95006304798797
6.58887200000000	4.94999045070298
6.63603760000000	4.94980912611800
6.68237380000000	4.94953936553679
6.72978850000000	4.94926246496984
6.81995840000000	4.94961205176963
6.87254930000000	4.94998743611786
6.91909200000000	4.95066965813537
6.96541440000000	4.95108089087565
7.01148190000000	4.95147150645240
7.06104690000000	4.95172344271024
7.10730770000000	4.95185424790206
7.15341420000000	4.95152088525826
7.20008600000000	4.95139263283074
7.24719490000000	4.95127346686815
7.29637380000000	4.95070997200606
7.35157210000000	4.95100845470878
7.39605240000000	4.95023380800688
7.43897680000000	4.95074497923299
7.48887700000000	4.95062627985163
7.53488450000000	4.95076115121274
7.57927940000000	4.95023865691843
7.62983070000000	4.95047744043366
7.70555680000000	4.95055106704825
7.75305120000000	4.95027002012880
7.82900700000000	4.95101062096469
7.88718080000000	4.95067764183626
7.93831550000000	4.95025019696112
7.99140750000000	4.95036606885995
8.03747250000000	4.94983746434728
8.09177080000000	4.94942107291745
8.14376920000000	4.94933990837231
8.21581540000000	4.94900921002962
8.26017970000000	4.94904765314969
8.30660620000000	4.94928225526117
8.36066980000000	4.95099101016845
8.41408620000000	4.95238111846985
8.46012840000000	4.95355464562146
8.53376750000000	4.95459615385727
8.60897580000000	4.95600689869342
8.66196170000000	4.95507234267829
8.71365120000000	4.95514312997374
8.78402480000000	4.95454093337585
8.83751320000000	4.95315389097402
8.89510740000000	4.95330096768582
8.94903600000000	4.95325725815860
8.99156420000000	4.95284638859169
9.03707810000000	4.95278271927174
9.07999170000000	4.95295709491149
9.13320450000000	4.95263045509900
9.20858240000000	4.95241933196577
9.26260230000000	4.95267668254018
9.31175670000000	4.95281369480104
9.38366510000000	4.95270090393955
9.43521330000000	4.95259556753396
9.51003450000000	4.95295866562615
9.58196170000000	4.95301264966598
9.62971470000000	4.95316990510082
9.67616200000000	4.95332885522458
9.72924950000000	4.95267209086648
9.78214960000000	4.95242256639996
9.85448190000000	4.95209476719260
9.91177250000000	4.95224575137385
9.96919160000000	4.95268531522798
10.0201195000000	4.95251433605952
10.0533409000000	4.95239317907768
10.1112530000000	4.95244994785783
10.1713297000000	4.95200571076056
10.2204696000000	4.95172330710084
10.3911564000000	4.95200186557942
10.4353331000000	4.95205720093600
10.4842983000000	4.95223400079046
10.5482505000000	4.95205040054262
10.5941404000000	4.95198654073277
10.6351714000000	4.95199197593915
10.6833146000000	4.95170545895545
10.7327105000000	4.95212068872261
10.7766383000000	4.95238638646065
10.8235484000000	4.95261095346203
10.8666660000000	4.95288138992196

    };

    \end{axis}
\end{tikzpicture}

%% file: FigPlot/CtrlAgl_5deg_j.tex
\begin{tikzpicture}
    \begin{axis}[ 
      label style={font=\footnotesize},
      tick label style={font=\footnotesize},
      xmin=0, xmax=3,
       ymin=-0.2, ymax=6,
      ymajorgrids=true,
      grid style=dashed,
      legend pos=north west,
      width = 0.26\linewidth,
      height = 0.2\linewidth,
    ]


    \addplot[color=blue, line width = 1pt] table {
0	0
0.0351693000000000	-7.59492801925909e-05
0.0703386000000000	0.00423897012562919
0.148496900000000	0.496098364074429
0.201737700000000	1.08913772358282
0.247780300000000	1.32905317332077
0.294785000000000	1.83246747479590
0.342281700000000	2.46988296633699
0.409490200000000	2.94680556770660
0.458248600000000	3.17712125093851
0.506489200000000	3.36083869191664
0.557434400000000	3.52903792992480
0.606649000000000	3.53650668757157
0.681280700000000	3.58514658769639
0.731638100000000	3.70835819167960
0.801053000000000	3.94078478733944
0.854182600000000	4.20675418806602
0.907190300000000	4.46424342119921
0.980507700000000	4.55339304090235
1.02732550000000	4.50863524722119
1.09169730000000	4.45135095957918
1.14825280000000	4.48316487576592
1.22400040000000	4.52496471019774
1.27507570000000	4.61633544285397
1.34881600000000	4.71393243232496
1.39758300000000	4.78489449318803
1.44794900000000	4.80068520403251
1.52158790000000	4.81510697859291
1.57097370000000	4.80595415810547
1.64676630000000	4.79338491873828
1.69823030000000	4.81301413072067
1.77289050000000	4.82460087907702
1.82633950000000	4.82040404823638
1.88192120000000	4.88304196488685
1.93575550000000	4.91666746174004
2.00822640000000	4.92397135290786
2.05801240000000	4.89821918800266
2.11165880000000	4.90081445703976
2.18381150000000	4.89951012554891
2.23172010000000	4.89275275213964
2.27975380000000	4.90577727399335
2.32959890000000	4.90313270639256
2.38223710000000	4.90429439233598
2.45523820000000	4.90856380950358
2.50471720000000	4.90456690470637
2.57734070000000	4.92264242365155
2.62834530000000	4.92704326647100
2.70262810000000	4.93123036675418
2.75031250000000	4.93221945996833
2.79906550000000	4.93194331431686
2.85712100000000	4.93163511052194
2.92952880000000	4.93070623010858
2.99998960000000	4.92998345671931
3.05092360000000	4.93050326818544
3.09950470000000	4.93076328955472
3.17321990000000	4.93144401195913
3.22975650000000	4.93203562494262
3.27754030000000	4.93216349347696
3.35231980000000	4.93201773259443
3.40167140000000	4.93202802032837
3.47674440000000	4.93254953123532
3.52853590000000	4.93256688807268
3.57844300000000	4.93289199947753
3.64998970000000	4.93391113786617
3.70076380000000	4.93459522250513
3.77539740000000	4.93565864859111
3.82764150000000	4.93655831291540
3.90210060000000	4.93723630950505
3.95397540000000	4.93805722699847
4.00262030000000	4.93872080602835
4.07635370000000	4.94021059011391
4.12253370000000	4.94028982724557
4.19719100000000	4.94060478058103
4.24970940000000	4.94099440661588
4.29975670000000	4.94097728234013
4.37335240000000	4.94215702395173
4.44758310000000	4.94317336499221
4.50437690000000	4.94500624582901
4.57821880000000	4.94655962603626
4.65250370000000	4.94844819174414
4.69823190000000	4.94896705876375
4.77374880000000	4.94904396050846
4.84703600000000	4.94927646673777
4.92142350000000	4.94986532470152
4.97689370000000	4.95050956392779
5.04802270000000	4.95100932595803
5.12169520000000	4.95016531824356
5.17404890000000	4.94994115687038
5.24906120000000	4.95029723234705
5.30285310000000	4.94997037061462
5.37191050000000	4.95056401695517
5.41742470000000	4.95045322609126
5.49481290000000	4.94948216300043
5.54683880000000	4.94911260434155
5.59895360000000	4.94840137258990
5.64972910000000	4.94853301861681
5.69788680000000	4.94856613391699
5.74592350000000	4.94899760730370
5.81900420000000	4.94903652959982
5.86975430000000	4.94903909712297
5.91796750000000	4.94894092804396
5.99513290000000	4.95068880459189
6.04629930000000	4.95323471938784
6.11944820000000	4.95456286000066
6.17254580000000	4.95558789769815
6.21796840000000	4.95532471431470
6.29270900000000	4.95489905565191
6.33983420000000	4.95457603256716
6.41559140000000	4.95442287304999
6.46380240000000	4.95432016065409
6.53576670000000	4.95389502022137
6.59170980000001	4.95395139428953
6.63939380000001	4.95363952577785
6.71584840000001	4.95353197595850
6.76636760000001	4.95352588998153
6.83706250000001	4.95360175056582
6.88657700000001	4.95414619940958
6.96071860000001	4.95377591411018
7.01241260000001	4.95369099993986
7.06527690000001	4.95335498402778
7.13721860000001	4.95297691893017
7.18156550000001	4.95331304528310
7.25645890000001	4.95411168771753
7.30501440000001	4.95355292147784
7.37946840000001	4.95333534526578
7.45451670000001	4.95372530609386
7.52935260000001	4.95337180262386
7.57802370000001	4.95408553436803
7.65496680000001	4.95420106445656
7.73130480000001	4.95329204585079
7.77725650000001	4.95357462115954
7.82140090000001	4.95374096135972
7.87272670000001	4.95343640192454
7.92498390000001	4.95322890489828
7.99749630000001	4.95310316041729
8.06792770000001	4.95301141132480
8.11759760000001	4.95292314065144
8.17554920000001	4.95369261389580
8.22246360000001	4.95323374887358
8.29796780000001	4.95337567208847
8.35615830000001	4.95405700248239
8.40288140000001	4.95367687725957
8.45295430000001	4.95432278161103
8.50175240000001	4.95402155603379
8.57358570000001	4.95355002049223
8.62492530000001	4.95265409526737
8.69602020000001	4.95196405091400
8.74511670000001	4.95180915210497
8.81738310000001	4.95259388176409
8.86704660000001	4.95289381666596
8.94394290000001	4.95267919296428
8.98776170000001	4.95300451346213
9.06086210000001	4.95251005465582
9.11479590000000	4.95225305221795
9.16692200000001	4.95216127745335
9.22248820000000	4.95246188192308
9.29255310000001	4.95198977875523
9.34484670000001	4.95202558520867
9.41636100000001	4.95209848348112
9.46676060000001	4.95222333584569
9.53819570000001	4.95275413714427
9.58819150000001	4.95255447812651
9.65934480000001	4.95175747381183
9.70865200000001	4.95211367712532
9.75298840000001	4.95172602529667
9.80339250000001	4.95230857551752
9.85313390000001	4.95238335555308
9.92673050000001	4.95244036977057
9.97590090000001	4.95151777977853
10.0276919000000	4.95173960044250
10.0746680000000	4.95140918926624
10.1257859000000	4.95168799172317
10.1980052000000	4.95280804886373
10.2490589000000	4.95317981004375
10.3005217000000	4.95289415751817
10.3563421000000	4.95287863912797
10.4286921000000	4.95281846860774
10.4797823000000	4.95292607672222
10.5271726000000	4.95229951828420
10.5775426000000	4.95258437745245
10.6275654000000	4.95221824805224
10.6833161000000	4.95231527718211
10.7588773000000	4.95222591218849
10.8076871000000	4.95189805822599
10.8772527000000	4.95219418096975
10.9274159000000	4.95211574726889
11.0026443000000	4.95234458476059
11.0537744000000	4.95289770379712
11.1300654000000	4.95276827943929
11.1804039000000	4.95271265071069
11.2312700000000	4.95248875703017
11.3032016000000	4.95199316027926
11.3517778000000	4.95146145652416
11.4264786000000	4.95123837732846
11.4802431000000	4.95126739980290
11.5331921000000	4.95166958874925
11.5901475000000	4.95199372714420
11.6646416000000	4.95221444823706
11.7152861000000	4.95229804804064
11.7911238000000	4.95231572051005

    };

    \end{axis}
\end{tikzpicture}

%% file: FigPlot/CtrlAgl_5deg_k.tex
\begin{tikzpicture}
    \begin{axis}[
      label style={font=\footnotesize},
      tick label style={font=\footnotesize},
      xmin=0, xmax=3,
      ymin=-0.2, ymax=6,
      ymajorgrids=true,
      grid style=dashed,
      legend pos=north west,
      width = 0.26\linewidth,
      height = 0.2\linewidth,
    ]


    \addplot[color=blue, line width = 1pt] table {
0	0
0.0487482000000000	0.0371809696293590
0.0974964000000000	0.00411316746316198
0.138490000000000	0.00797415678532390
0.213782700000000	0.997067976751473
0.284278800000000	1.68336206746104
0.334459900000000	2.10265922388865
0.407996000000000	2.99852122259204
0.458020800000000	3.05571860361015
0.531451000000000	2.93085699002624
0.582233000000000	3.46643650443403
0.635163200000000	3.42922626123641
0.683605600000000	3.40786499972211
0.733938400000000	3.57882286214449
0.808890800000000	3.88893554254278
0.853224500000000	4.20203477682975
0.924172700000000	4.33362356188808
0.973229600000000	4.47883951758430
1.04499130000000	4.43381782244975
1.09412660000000	4.44390931936916
1.16885160000000	4.41745315978192
1.24278080000000	4.50046986372525
1.32041540000000	4.60527572109939
1.37030290000000	4.71587775992228
1.42479590000000	4.76151506037634
1.47081030000000	4.79344525250199
1.52673890000000	4.79010204338515
1.57676330000000	4.79589954682485
1.62724360000000	4.78284021360463
1.69792160000000	4.79203627510068
1.76928790000000	4.82003655207891
1.82979160000000	4.80351166457213
1.87813500000000	4.81903767381832
1.92948860000000	4.87479214286440
1.98485660000000	4.90441106522328
2.05683790000000	4.91457502429706
2.10722040000000	4.90752738093629
2.17975670000000	4.89094322323165
2.22945410000000	4.89390216444596
2.30054200000000	4.89070276652894
2.37195240000000	4.90557980884717
2.44510370000000	4.88968822019804
2.49463330000000	4.89095736769983
2.54204380000000	4.90726749813917
2.59537220000000	4.91354774471631
2.67301260000000	4.91220699739737
2.74664260000000	4.90963605555578
2.79651280000000	4.91063405252287
2.87324300000000	4.90748634253992
2.92387430000000	4.90275169641615
2.97975740000000	4.90028841401679
3.05179200000000	4.89922461823965
3.10253940000000	4.89957624976537
3.17560650000000	4.90180083894985
3.22750890000000	4.90683793460813
3.27940900000000	4.90934523944179
3.35557130000000	4.91173921320030
3.41035090000000	4.91139924104103
3.46210260000000	4.91019347441300
3.53525050000000	4.90926951374444
3.58407130000000	4.90933886772092
3.65722070000000	4.91158003008293
3.70737310000000	4.91389365027575
3.76098710000000	4.91646842803943
3.83436950000000	4.91947137825738
3.88691320000000	4.92097169486204
3.95766180000000	4.92109181444810
4.00939950000000	4.91972967754318
4.06022410000000	4.91844139162352
4.13380330000000	4.91600106236271
4.18530650000000	4.91480469066022
4.23028330000000	4.91541651367886
4.30424990000000	4.91821401912024
4.37591330000000	4.92098097039967
4.42770990000000	4.92465756513121
4.48048200000000	4.92618177533556
4.52651010000000	4.92648386861516
4.57659800000000	4.92491868429144
4.62805580000000	4.92299709885898
4.70023480000000	4.92170150728928
4.77867510000000	4.92044967535215
4.84827160000000	4.92129600472420
4.90154380000000	4.92369020100550
4.94651170000000	4.92467225209952
4.99775950000000	4.92536996878881
5.04591190000000	4.92562623625529
5.12076930000000	4.92561890620409
5.17113470000000	4.92424395307668
5.24317480000000	4.92333290480210
5.29195280000000	4.92219485849944
5.36305920000000	4.92197827038820
5.41787570000000	4.92297292093405
5.47277950000000	4.92335937310637
5.53081010000000	4.92432331003751
5.58554190000000	4.92470220397977
5.64014100000000	4.92400076252478
5.71232210000000	4.92330680862389
5.78786040000000	4.92202200137331
5.83772220000000	4.92137946846536
5.88781990000000	4.92152316877423
5.96017680000000	4.92196995166558
6.01309240000000	4.92302874017193
6.08582000000000	4.92266837992307
6.13735290000000	4.92326378066978
6.19202550000000	4.92341314582010
6.26556470000000	4.92265534635156
6.32301330000000	4.92191017761974
6.37312840000000	4.92157953663228
6.44801980000000	4.92130952265856
6.49939580000000	4.92180739913101
6.54839500000000	4.92247898668020
6.62165970000000	4.92271622605994
6.67514920000000	4.92375061668660
6.72553850000000	4.92386632642619
6.79780440000000	4.92363096581591
6.85430880000000	4.92319787891160
6.92641730000000	4.92239028317170
6.97241650000000	4.92157893957948
7.04374850000000	4.92159217242918
7.09540550000000	4.92139134062778
7.14638830000000	4.92167132803726
7.21848990000000	4.92159944860406
7.27117530000000	4.92131624873393
7.34240760000000	4.92122871961644
7.39138590000000	4.92035466273879
7.46360980000000	4.92009526143331
7.51977050000000	4.92013161369276
7.56387530000000	4.91999432696945
7.63816660000000	4.92041292403010
7.68848970000000	4.91948033959594
7.74662500000000	4.91907525883275
7.82006730000000	4.91901464845267
7.87107610000000	4.91843244686746
7.94422270000000	4.91813595529661
7.99409940000000	4.91849248439224
8.06771130000000	4.91837062980238
8.11480590000000	4.91789134675566
8.16425000000000	4.91843669359373
8.23691800000000	4.91827326405190
8.28651080000000	4.91879268166993
8.34288010000000	4.91918269178271
8.41449800000000	4.91938257602411
8.49070440000000	4.91925918529872
8.54289500000000	4.91877843256908
8.59453580000000	4.91846012663052
8.66798550000000	4.91900100620502
8.71511970000000	4.91906059336449
8.76378030000000	4.91910071110503
8.83594520000000	4.91913624278049
8.88684500000000	4.91888705189081
8.95977010000000	4.91901705107566
9.01165020000000	4.91886457666293
9.05697230000000	4.91891223196122
9.10893500000000	4.91888259161989
9.15547630000000	4.91847804743545
9.22675990000000	4.91958991100066
9.27887360000000	4.91908661631883
9.35236770000000	4.91871180504786
9.40137280000000	4.91858187456525
9.47305790000000	4.91766241410386
9.52669080000000	4.91698160187636
9.57519330000000	4.91626136714890
9.63228530000000	4.91630852742364
9.70575090000000	4.91614540353183
9.75827440000000	4.91639867800887
9.83233890000000	4.91635708679043
9.88513170000000	4.91592271161337
9.93398260000000	4.91565026434086
9.98161320000000	4.91550577188865
10.0341627000000	4.91475436430133
10.0846703000000	4.91381651680003
10.1573654000000	4.91348504734400
10.2079093000000	4.91302991600665
10.2595960000000	4.91316137402768
10.3313935000000	4.91322333135148
10.4068454000000	4.91310829712651
10.4617324000000	4.91276666700346
10.5236031000000	4.91230119540137
10.5719989000000	4.91162116383227
10.6187639000000	4.91120002196568
10.6935565000000	4.91090838680005
10.7453525000000	4.91110094960246
10.7928314000000	4.91102609493393
10.8433441000000	4.91098727222264
10.8957672000000	4.91171510656838
10.9439643000000	4.91205979180914
10.9936400000000	4.91448965834058
11.0463997000000	4.91950152099275
11.1002506000000	4.92497030571557
11.1539174000000	4.92948872314036
11.1993113000000	4.93415176175890
11.2487345000000	4.93701506137423
11.3010574000000	4.93750663098283
11.3532291000000	4.93780961051867
11.4091438000000	4.93915229950872
11.4630337000000	4.94053937598934
11.5131554000000	4.94205328830727
11.5655087000000	4.94346599191371
11.6189510000000	4.94446765807208
11.6631047000000	4.94536159934400

    };

    \end{axis}
\end{tikzpicture}

%% file: FigPlot/CtrlAgl_5deg_l.tex
\begin{tikzpicture}
    \begin{axis}[
      label style={font=\footnotesize},
      tick label style={font=\footnotesize},
      xmin=0, xmax=3,
      ymin=-0.2, ymax=6,
      ymajorgrids=true,
      grid style=dashed,
      legend pos=north west,
      width = 0.26\linewidth,
      height = 0.2\linewidth,
    ]


    \addplot[color=blue, line width = 1pt] table {
0	0
0.0705605000000000	0.0622426788991338
0.141121000000000	0.0121046357787623
0.188711500000000	0.870380970305976
0.258844700000000	1.46446900852654
0.310719800000000	1.85601855727137
0.358813800000000	2.23760492625311
0.432870700000000	2.86269019980580
0.486241300000000	3.20869884424361
0.560112600000000	3.34028653278566
0.608670500000000	3.41884900816758
0.686909200000000	3.42870554089716
0.757143500000000	3.63688112323551
0.829473600000000	3.89374679550751
0.885210100000000	4.10595348276876
0.929167700000000	4.23703235083169
1.00369430000000	4.35778108678732
1.05305600000000	4.45479286647336
1.12539510000000	4.47125871379023
1.17520410000000	4.45434683253162
1.24742540000000	4.49192883099746
1.29806280000000	4.65424128941363
1.35084540000000	4.69065365408870
1.42306060000000	4.75598114497649
1.47314870000000	4.79618589637035
1.53159770000000	4.81397589581136
1.60898140000000	4.79562779077052
1.65383540000000	4.81087914846787
1.72627660000000	4.84727435650897
1.77728170000000	4.86994732792892
1.85085310000000	4.84448212529106
1.90111280000000	4.86102329528288
1.95314600000000	4.87502616363269
2.02504370000000	4.92482870066864
2.07580820000000	4.94889152308267
2.14775530000000	4.95059604182233
2.22354890000000	4.92844592087919
2.27895960000000	4.94538530978151
2.35237390000000	4.92504180218490
2.40275330000000	4.93557302832850
2.45554700000000	4.93218997279170
2.52950280000000	4.93019074177515
2.58335600000000	4.93146461740187
2.62886440000000	4.93336161185102
2.67352690000000	4.93632886430284
2.72303120000000	4.93911778139961
2.76765620000000	4.94057270828604
2.81372940000000	4.94055502282907
2.86343150000000	4.94053394894553
2.90969950000000	4.93826216813791
2.95568160000000	4.93562964043100
3.00357710000000	4.93265038669031
3.05113870000000	4.93061538872072
3.10131400000000	4.92947542459508
3.15241880000000	4.92969395108996
3.19668980000000	4.93076136534546
3.24333310000000	4.93255460196677
3.29264200000000	4.93409830028939
3.34070160000000	4.93489030983787
3.38542100000000	4.93495970092441
3.43412570000000	4.93438821492431
3.47891860000000	4.93359857770515
3.52477840000000	4.93230246712053
3.57354620000000	4.93097089612412
3.62068810000000	4.93083088621229
3.67881990000000	4.93173222971071
3.75666400000000	4.93373386123352
3.80348610000000	4.93511123695513
3.85274600000000	4.93652726299298
3.89582470000000	4.93695163579650
3.94367400000000	4.93646582523102
3.99160100000000	4.93526923353503
4.03808490000000	4.93515433713468
4.10450510000000	4.93386788693850
4.15324540000000	4.93355288030032
4.20094150000000	4.93402578114680
4.24677040000000	4.93460116316337
4.29533600000000	4.93476907388545
4.36995310000000	4.93563898299194
4.41432340000000	4.93552473762427
4.46545710000000	4.93498188903360
4.52012140000000	4.93475227109866
4.56883230000000	4.93409519802965
4.61665320000000	4.93383162716745
4.66002990000000	4.93400456515252
4.71101060000000	4.93388723021649
4.75931970000000	4.93394351276248
4.81200530000000	4.93476445577854
4.86748720000000	4.93461491778635
4.91500030000000	4.93456208324421
4.96508950000000	4.93482289468206
5.02114030000000	4.93423158974481
5.09387890000000	4.93329447262448
5.16673300000000	4.93332742030823
5.24056180000000	4.93389261312143
5.29102600000000	4.93522790102754
5.33786410000000	4.93552066408126
5.41024300000000	4.93555721338882
5.46253340000000	4.93539444666981
5.51701050000000	4.93467079015209
5.59443050000000	4.93431745019632
5.64772310000000	4.93372213388893
5.71707670000000	4.93369572804339
5.79313590000000	4.93429221278325
5.93561610000000	4.93892990045842
5.98076840000000	4.93959527832240
6.03277930000000	4.94020885896571
6.07913600000000	4.94019675481055
6.13080370000000	4.94057064725143
6.17615940000000	4.94011638903339
6.22243500000000	4.93963691673902
6.27047480000000	4.93976286936157
6.34393780000000	4.93971680282611
6.39088780000000	4.94002124655063
6.43782980000000	4.94019507141497
6.48642430000000	4.94048696326509
6.53427880000000	4.94056785669988
6.57946870000000	4.94081922335165
6.62856790000000	4.94068297672042
6.67507400000000	4.94142081059126
6.72479990000000	4.94184260198681
6.78627540000000	4.94208114954576
6.83104520000000	4.94175064739085
6.87535740000000	4.94159061697761
6.92201500000000	4.94135784573868
6.97447420000000	4.94107053700475
7.02028400000000	4.94130661740880
7.06746490000000	4.94193607144755
7.11613720000000	4.94180676742022
7.17409540000000	4.94191806673670
7.21490210000000	4.94232579159927
7.26765580000000	4.94196176079699
7.30866970000000	4.94166777530187
7.35893000000000	4.94131587118812
7.40433870000000	4.94135218520175
7.45143870000000	4.94159497969696
7.49225660000000	4.94139641090323
7.54049930000000	4.94125835202558
7.58930310000000	4.94175445566675
7.63815120000000	4.94191875871355
7.68482410000000	4.94215528866806
7.73526950000000	4.94192241002649
7.77721120000000	4.94242777095177
7.81892170000000	4.94245777454789
7.86803120000000	4.94216599805908
7.91852010000000	4.94174974944333
7.96834050000000	4.94152535346000
8.01274680000000	4.94149489511538
8.06270340000000	4.94140019151526
8.10501760000000	4.94156352672331
8.15062070000000	4.94129240091465
8.20077210000000	4.94098442852785
8.24897150000000	4.94077776998368
8.29346200000000	4.94028200284801
8.34232140000000	4.94001896242999
8.41245180000000	4.93988801827258
8.46287310000000	4.94088877654622
8.53593800000000	4.94135479643357
8.58656890000000	4.94173357646248
8.64346490000000	4.94177566880301
8.69053640000000	4.94116052471860
8.73896140000000	4.94097545289796
8.78658200000000	4.93995195057886
8.83214410000000	4.94027637016034
8.87739020000000	4.94002482283886
8.92565640000000	4.94033349210366
8.97316800000000	4.94047875354881
9.02045950000000	4.94059732150515
9.06602040000000	4.94044909860443
9.11537580000000	4.94044058287332
9.15853320000000	4.94029120478991
9.20275380000000	4.94008193901442
9.25258190000000	4.94012538693105
9.29754920000000	4.94079521282241
9.34408080000000	4.94086139025706
9.39278730000000	4.94081251200683
9.43556990000000	4.94107306268991
9.48425780000000	4.94097524029678
9.53013040000000	4.94058977087038
9.57791320000000	4.94081638088327
9.65489160000000	4.94074739982900
9.70084610000000	4.94073048579213
9.75187310000000	4.94040312977859
9.79583310000000	4.94035430369170
9.84346370000000	4.94042917808441
9.90734860000000	4.94093697788238
9.95188700000000	4.94057422720584
10.0008081000000	4.94128551068539
10.0443153000000	4.94144733018958
10.0913515000000	4.94163593258518
10.1426129000000	4.94150541090628
10.2140365000000	4.94211127410959
10.2635272000000	4.94141501600140
10.3095437000000	4.94116213302525
10.3583561000000	4.94106395118006
10.4027392000000	4.94098228331898
10.4500986000000	4.94071363196648
10.4977717000000	4.94106594583323
10.5438957000000	4.94101443882732
10.5902857000000	4.94112995121491

    };

    \end{axis}
\end{tikzpicture}

%% file: FigPlot/5deg_ss-5.tex
\begin{tikzpicture}
    \begin{axis}[
      label style={font=\footnotesize},
      tick label style={font=\footnotesize},
     xmin=2, xmax=8,
      ymin=4.85, ymax=5,
      ymajorgrids=true,
      grid style=dashed,
      legend pos=north west,
      width = 0.26\linewidth,
      height = 0.2\linewidth,
    ylabel={Steady-state $q_i$ [deg]}
    ]


    \addplot[color=blue, line width = 1pt] table {
0	0
0.0833301000000000	0.0375134662747445
0.166660200000000	0.148665434508786
0.211318000000000	1.16586776092331
0.282248500000000	1.47483136223987
0.352285200000000	2.18891968301663
0.399046300000000	2.82390854196966
0.452256300000000	2.90362646237545
0.501360200000000	2.58511045399307
0.553929200000000	2.92893798058751
0.607818200000000	3.62734649972249
0.658328600000000	3.65634355320437
0.709592300000000	3.66135255714096
0.759063700000000	3.76604455790549
0.803066100000000	3.92825294216250
0.848140100000000	4.09418886567219
0.922882200000000	4.39449000137036
0.971851200000000	4.50885723258533
1.02888970000000	4.54252505033123
1.07405850000000	4.53965602599220
1.12745240000000	4.48495789868208
1.18123060000000	4.48813759576480
1.25626780000000	4.56773047849954
1.30611060000000	4.64303127276827
1.35802290000000	4.74071831089123
1.40417920000000	4.77281627621415
1.44921470000000	4.80298265939470
1.52575350000000	4.82757346471125
1.57646670000000	4.80250989202517
1.63315720000000	4.82903489102602
1.68973810000000	4.84185116347218
1.76045670000000	4.85622409619880
1.80843790000000	4.86729307387526
1.86244250000000	4.88293547569850
1.92161830000000	4.89729591439944
1.99283720000000	4.90682022582355
2.04022980000000	4.91898251838657
2.08861790000000	4.92057846508507
2.13662970000000	4.93638437701291
2.19325590000000	4.94188214274294
2.25078930000000	4.94298796385870
2.31671690000000	4.94163772798701
2.36272120000000	4.93955835284296
2.41604910000000	4.93688217648242
2.46438970000000	4.93409648806893
2.50983200000000	4.93346539268773
2.55698240000000	4.93287628002508
2.60539580000000	4.93369563729429
2.65863300000000	4.93439973520838
2.71362010000000	4.93541091776214
2.76654080000000	4.93589234206075
2.82151360000000	4.93574630800815
2.89958490000000	4.93655994538534
2.95295490000000	4.93736204488268
3.02505640000000	4.93960489079410
3.07114240000000	4.94215532363596
3.11908960000000	4.94522351596811
3.16481720000000	4.94760639039738
3.21357320000000	4.94915740160100
3.25961480000000	4.95160777886128
3.31416200000000	4.95466034035228
3.36645200000000	4.95567156709396
3.42144450000000	4.95685141147331
3.49896250000000	4.95739404759263
3.55090310000000	4.95677283737413
3.62507800000000	4.95577070594958
3.68313420000000	4.95425205005312
3.73618780000000	4.95407846813768
3.79067750000000	4.95416025801624
3.83649170000000	4.95400922478782
3.88957190000000	4.95385340593010
3.94221460000000	4.95340659047839
4.01628730000000	4.95310052491224
4.06607190000000	4.95264636685955
4.11373440000000	4.95257342955629
4.16849380000000	4.95236414272953
4.21830790000000	4.95275629327021
4.26332260000000	4.95312536289917
4.31098230000000	4.95317005097638
4.36503000000000	4.95219674438910
4.41217070000000	4.95194181103488
4.45795670000000	4.95115686995092
4.50796140000000	4.95029495659626
4.55449210000000	4.94957637060783
4.60284320000000	4.95013929202289
4.67593310000000	4.94991506121163
4.72646680000000	4.95025450661589
4.77331730000000	4.95037849327073
4.84801160000000	4.95095569718522
4.89773920000000	4.95117491691907
4.97065330000000	4.95089383671097
5.02030260000000	4.95094108449835
5.07064080000000	4.95040695283163
5.12790470000000	4.95021115430730
5.17428430000000	4.94991202864765
5.23125240000000	4.94997930589140
5.29053750000000	4.95014453600226
5.34645290000000	4.95027042636570
5.41998510000000	4.95039890402609
5.47011250000000	4.95032519279741
5.52837630000000	4.95024402068336
5.60156950000000	4.95033343772392
5.64903530000000	4.95043724926297
5.70461970000000	4.94990953166113
5.77806180000000	4.94974708316482
5.83189990000000	4.94942666801135
5.87855080000000	4.94915656858885
5.92372320000000	4.94987275799812
5.97151780000000	4.95011354622310
6.01852520000000	4.95010772642916
6.07384330000000	4.95002907412075
6.11856380000000	4.94966383724325
6.16388860000000	4.94935357227348
6.21481440000000	4.94977050334216
6.26116280000000	4.95035309888113
6.30761450000000	4.95050601500340
6.35596450000000	4.95089029234097
6.40085030000000	4.95120028952685
6.44867970000000	4.95082155765070
6.49582840000000	4.95079926140087
6.54112860000000	4.95006304798797
6.58887200000000	4.94999045070298
6.63603760000000	4.94980912611800
6.68237380000000	4.94953936553679
6.72978850000000	4.94926246496984
6.81995840000000	4.94961205176963
6.87254930000000	4.94998743611786
6.91909200000000	4.95066965813537
6.96541440000000	4.95108089087565
7.01148190000000	4.95147150645240
7.06104690000000	4.95172344271024
7.10730770000000	4.95185424790206
7.15341420000000	4.95152088525826
7.20008600000000	4.95139263283074
7.24719490000000	4.95127346686815
7.29637380000000	4.95070997200606
7.35157210000000	4.95100845470878
7.39605240000000	4.95023380800688
7.43897680000000	4.95074497923299
7.48887700000000	4.95062627985163
7.53488450000000	4.95076115121274
7.57927940000000	4.95023865691843
7.62983070000000	4.95047744043366
7.70555680000000	4.95055106704825
7.75305120000000	4.95027002012880
7.82900700000000	4.95101062096469
7.88718080000000	4.95067764183626
7.93831550000000	4.95025019696112
7.99140750000000	4.95036606885995
8.03747250000000	4.94983746434728
8.09177080000000	4.94942107291745
8.14376920000000	4.94933990837231
8.21581540000000	4.94900921002962
8.26017970000000	4.94904765314969
8.30660620000000	4.94928225526117
8.36066980000000	4.95099101016845
8.41408620000000	4.95238111846985
8.46012840000000	4.95355464562146
8.53376750000000	4.95459615385727
8.60897580000000	4.95600689869342
8.66196170000000	4.95507234267829
8.71365120000000	4.95514312997374
8.78402480000000	4.95454093337585
8.83751320000000	4.95315389097402
8.89510740000000	4.95330096768582
8.94903600000000	4.95325725815860
8.99156420000000	4.95284638859169
9.03707810000000	4.95278271927174
9.07999170000000	4.95295709491149
9.13320450000000	4.95263045509900
9.20858240000000	4.95241933196577
9.26260230000000	4.95267668254018
9.31175670000000	4.95281369480104
9.38366510000000	4.95270090393955
9.43521330000000	4.95259556753396
9.51003450000000	4.95295866562615
9.58196170000000	4.95301264966598
9.62971470000000	4.95316990510082
9.67616200000000	4.95332885522458
9.72924950000000	4.95267209086648
9.78214960000000	4.95242256639996
9.85448190000000	4.95209476719260
9.91177250000000	4.95224575137385
9.96919160000000	4.95268531522798
10.0201195000000	4.95251433605952
10.0533409000000	4.95239317907768
10.1112530000000	4.95244994785783
10.1713297000000	4.95200571076056
10.2204696000000	4.95172330710084
10.3911564000000	4.95200186557942
10.4353331000000	4.95205720093600
10.4842983000000	4.95223400079046
10.5482505000000	4.95205040054262
10.5941404000000	4.95198654073277
10.6351714000000	4.95199197593915
10.6833146000000	4.95170545895545
10.7327105000000	4.95212068872261
10.7766383000000	4.95238638646065
10.8235484000000	4.95261095346203
10.8666660000000	4.95288138992196

    };

    \end{axis}
\end{tikzpicture}

%% file: FigPlot/5deg_ss-6.tex
\begin{tikzpicture}
    \begin{axis}[ 
      label style={font=\footnotesize},
      tick label style={font=\footnotesize},
     xmin=2, xmax=8,
      ymin=4.85, ymax=5,
      ymajorgrids=true,
      grid style=dashed,
      legend pos=north west,
      width = 0.26\linewidth,
      height = 0.2\linewidth,
    ]


    \addplot[color=blue, line width = 1pt] table {
0	0
0.0351693000000000	-7.59492801925909e-05
0.0703386000000000	0.00423897012562919
0.148496900000000	0.496098364074429
0.201737700000000	1.08913772358282
0.247780300000000	1.32905317332077
0.294785000000000	1.83246747479590
0.342281700000000	2.46988296633699
0.409490200000000	2.94680556770660
0.458248600000000	3.17712125093851
0.506489200000000	3.36083869191664
0.557434400000000	3.52903792992480
0.606649000000000	3.53650668757157
0.681280700000000	3.58514658769639
0.731638100000000	3.70835819167960
0.801053000000000	3.94078478733944
0.854182600000000	4.20675418806602
0.907190300000000	4.46424342119921
0.980507700000000	4.55339304090235
1.02732550000000	4.50863524722119
1.09169730000000	4.45135095957918
1.14825280000000	4.48316487576592
1.22400040000000	4.52496471019774
1.27507570000000	4.61633544285397
1.34881600000000	4.71393243232496
1.39758300000000	4.78489449318803
1.44794900000000	4.80068520403251
1.52158790000000	4.81510697859291
1.57097370000000	4.80595415810547
1.64676630000000	4.79338491873828
1.69823030000000	4.81301413072067
1.77289050000000	4.82460087907702
1.82633950000000	4.82040404823638
1.88192120000000	4.88304196488685
1.93575550000000	4.91666746174004
2.00822640000000	4.92397135290786
2.05801240000000	4.89821918800266
2.11165880000000	4.90081445703976
2.18381150000000	4.89951012554891
2.23172010000000	4.89275275213964
2.27975380000000	4.90577727399335
2.32959890000000	4.90313270639256
2.38223710000000	4.90429439233598
2.45523820000000	4.90856380950358
2.50471720000000	4.90456690470637
2.57734070000000	4.92264242365155
2.62834530000000	4.92704326647100
2.70262810000000	4.93123036675418
2.75031250000000	4.93221945996833
2.79906550000000	4.93194331431686
2.85712100000000	4.93163511052194
2.92952880000000	4.93070623010858
2.99998960000000	4.92998345671931
3.05092360000000	4.93050326818544
3.09950470000000	4.93076328955472
3.17321990000000	4.93144401195913
3.22975650000000	4.93203562494262
3.27754030000000	4.93216349347696
3.35231980000000	4.93201773259443
3.40167140000000	4.93202802032837
3.47674440000000	4.93254953123532
3.52853590000000	4.93256688807268
3.57844300000000	4.93289199947753
3.64998970000000	4.93391113786617
3.70076380000000	4.93459522250513
3.77539740000000	4.93565864859111
3.82764150000000	4.93655831291540
3.90210060000000	4.93723630950505
3.95397540000000	4.93805722699847
4.00262030000000	4.93872080602835
4.07635370000000	4.94021059011391
4.12253370000000	4.94028982724557
4.19719100000000	4.94060478058103
4.24970940000000	4.94099440661588
4.29975670000000	4.94097728234013
4.37335240000000	4.94215702395173
4.44758310000000	4.94317336499221
4.50437690000000	4.94500624582901
4.57821880000000	4.94655962603626
4.65250370000000	4.94844819174414
4.69823190000000	4.94896705876375
4.77374880000000	4.94904396050846
4.84703600000000	4.94927646673777
4.92142350000000	4.94986532470152
4.97689370000000	4.95050956392779
5.04802270000000	4.95100932595803
5.12169520000000	4.95016531824356
5.17404890000000	4.94994115687038
5.24906120000000	4.95029723234705
5.30285310000000	4.94997037061462
5.37191050000000	4.95056401695517
5.41742470000000	4.95045322609126
5.49481290000000	4.94948216300043
5.54683880000000	4.94911260434155
5.59895360000000	4.94840137258990
5.64972910000000	4.94853301861681
5.69788680000000	4.94856613391699
5.74592350000000	4.94899760730370
5.81900420000000	4.94903652959982
5.86975430000000	4.94903909712297
5.91796750000000	4.94894092804396
5.99513290000000	4.95068880459189
6.04629930000000	4.95323471938784
6.11944820000000	4.95456286000066
6.17254580000000	4.95558789769815
6.21796840000000	4.95532471431470
6.29270900000000	4.95489905565191
6.33983420000000	4.95457603256716
6.41559140000000	4.95442287304999
6.46380240000000	4.95432016065409
6.53576670000000	4.95389502022137
6.59170980000001	4.95395139428953
6.63939380000001	4.95363952577785
6.71584840000001	4.95353197595850
6.76636760000001	4.95352588998153
6.83706250000001	4.95360175056582
6.88657700000001	4.95414619940958
6.96071860000001	4.95377591411018
7.01241260000001	4.95369099993986
7.06527690000001	4.95335498402778
7.13721860000001	4.95297691893017
7.18156550000001	4.95331304528310
7.25645890000001	4.95411168771753
7.30501440000001	4.95355292147784
7.37946840000001	4.95333534526578
7.45451670000001	4.95372530609386
7.52935260000001	4.95337180262386
7.57802370000001	4.95408553436803
7.65496680000001	4.95420106445656
7.73130480000001	4.95329204585079
7.77725650000001	4.95357462115954
7.82140090000001	4.95374096135972
7.87272670000001	4.95343640192454
7.92498390000001	4.95322890489828
7.99749630000001	4.95310316041729
8.06792770000001	4.95301141132480
8.11759760000001	4.95292314065144
8.17554920000001	4.95369261389580
8.22246360000001	4.95323374887358
8.29796780000001	4.95337567208847
8.35615830000001	4.95405700248239
8.40288140000001	4.95367687725957
8.45295430000001	4.95432278161103
8.50175240000001	4.95402155603379
8.57358570000001	4.95355002049223
8.62492530000001	4.95265409526737
8.69602020000001	4.95196405091400
8.74511670000001	4.95180915210497
8.81738310000001	4.95259388176409
8.86704660000001	4.95289381666596
8.94394290000001	4.95267919296428
8.98776170000001	4.95300451346213
9.06086210000001	4.95251005465582
9.11479590000000	4.95225305221795
9.16692200000001	4.95216127745335
9.22248820000000	4.95246188192308
9.29255310000001	4.95198977875523
9.34484670000001	4.95202558520867
9.41636100000001	4.95209848348112
9.46676060000001	4.95222333584569
9.53819570000001	4.95275413714427
9.58819150000001	4.95255447812651
9.65934480000001	4.95175747381183
9.70865200000001	4.95211367712532
9.75298840000001	4.95172602529667
9.80339250000001	4.95230857551752
9.85313390000001	4.95238335555308
9.92673050000001	4.95244036977057
9.97590090000001	4.95151777977853
10.0276919000000	4.95173960044250
10.0746680000000	4.95140918926624
10.1257859000000	4.95168799172317
10.1980052000000	4.95280804886373
10.2490589000000	4.95317981004375
10.3005217000000	4.95289415751817
10.3563421000000	4.95287863912797
10.4286921000000	4.95281846860774
10.4797823000000	4.95292607672222
10.5271726000000	4.95229951828420
10.5775426000000	4.95258437745245
10.6275654000000	4.95221824805224
10.6833161000000	4.95231527718211
10.7588773000000	4.95222591218849
10.8076871000000	4.95189805822599
10.8772527000000	4.95219418096975
10.9274159000000	4.95211574726889
11.0026443000000	4.95234458476059
11.0537744000000	4.95289770379712
11.1300654000000	4.95276827943929
11.1804039000000	4.95271265071069
11.2312700000000	4.95248875703017
11.3032016000000	4.95199316027926
11.3517778000000	4.95146145652416
11.4264786000000	4.95123837732846
11.4802431000000	4.95126739980290
11.5331921000000	4.95166958874925
11.5901475000000	4.95199372714420
11.6646416000000	4.95221444823706
11.7152861000000	4.95229804804064
11.7911238000000	4.95231572051005

    };

    \end{axis}
\end{tikzpicture}

%% file: FigPlot/5deg_ss-7.tex
\begin{tikzpicture}
    \begin{axis}[
      label style={font=\footnotesize},
      tick label style={font=\footnotesize},
      xmin=2, xmax=8,
      ymin=4.85, ymax=5,
      ymajorgrids=true,
      grid style=dashed,
      legend pos=north west,
      width = 0.26\linewidth,
      height = 0.2\linewidth,
    ]


    \addplot[color=blue, line width = 1pt] table {
0	0
0.0487482000000000	0.0371809696293590
0.0974964000000000	0.00411316746316198
0.138490000000000	0.00797415678532390
0.213782700000000	0.997067976751473
0.284278800000000	1.68336206746104
0.334459900000000	2.10265922388865
0.407996000000000	2.99852122259204
0.458020800000000	3.05571860361015
0.531451000000000	2.93085699002624
0.582233000000000	3.46643650443403
0.635163200000000	3.42922626123641
0.683605600000000	3.40786499972211
0.733938400000000	3.57882286214449
0.808890800000000	3.88893554254278
0.853224500000000	4.20203477682975
0.924172700000000	4.33362356188808
0.973229600000000	4.47883951758430
1.04499130000000	4.43381782244975
1.09412660000000	4.44390931936916
1.16885160000000	4.41745315978192
1.24278080000000	4.50046986372525
1.32041540000000	4.60527572109939
1.37030290000000	4.71587775992228
1.42479590000000	4.76151506037634
1.47081030000000	4.79344525250199
1.52673890000000	4.79010204338515
1.57676330000000	4.79589954682485
1.62724360000000	4.78284021360463
1.69792160000000	4.79203627510068
1.76928790000000	4.82003655207891
1.82979160000000	4.80351166457213
1.87813500000000	4.81903767381832
1.92948860000000	4.87479214286440
1.98485660000000	4.90441106522328
2.05683790000000	4.91457502429706
2.10722040000000	4.90752738093629
2.17975670000000	4.89094322323165
2.22945410000000	4.89390216444596
2.30054200000000	4.89070276652894
2.37195240000000	4.90557980884717
2.44510370000000	4.88968822019804
2.49463330000000	4.89095736769983
2.54204380000000	4.90726749813917
2.59537220000000	4.91354774471631
2.67301260000000	4.91220699739737
2.74664260000000	4.90963605555578
2.79651280000000	4.91063405252287
2.87324300000000	4.90748634253992
2.92387430000000	4.90275169641615
2.97975740000000	4.90028841401679
3.05179200000000	4.89922461823965
3.10253940000000	4.89957624976537
3.17560650000000	4.90180083894985
3.22750890000000	4.90683793460813
3.27940900000000	4.90934523944179
3.35557130000000	4.91173921320030
3.41035090000000	4.91139924104103
3.46210260000000	4.91019347441300
3.53525050000000	4.90926951374444
3.58407130000000	4.90933886772092
3.65722070000000	4.91158003008293
3.70737310000000	4.91389365027575
3.76098710000000	4.91646842803943
3.83436950000000	4.91947137825738
3.88691320000000	4.92097169486204
3.95766180000000	4.92109181444810
4.00939950000000	4.91972967754318
4.06022410000000	4.91844139162352
4.13380330000000	4.91600106236271
4.18530650000000	4.91480469066022
4.23028330000000	4.91541651367886
4.30424990000000	4.91821401912024
4.37591330000000	4.92098097039967
4.42770990000000	4.92465756513121
4.48048200000000	4.92618177533556
4.52651010000000	4.92648386861516
4.57659800000000	4.92491868429144
4.62805580000000	4.92299709885898
4.70023480000000	4.92170150728928
4.77867510000000	4.92044967535215
4.84827160000000	4.92129600472420
4.90154380000000	4.92369020100550
4.94651170000000	4.92467225209952
4.99775950000000	4.92536996878881
5.04591190000000	4.92562623625529
5.12076930000000	4.92561890620409
5.17113470000000	4.92424395307668
5.24317480000000	4.92333290480210
5.29195280000000	4.92219485849944
5.36305920000000	4.92197827038820
5.41787570000000	4.92297292093405
5.47277950000000	4.92335937310637
5.53081010000000	4.92432331003751
5.58554190000000	4.92470220397977
5.64014100000000	4.92400076252478
5.71232210000000	4.92330680862389
5.78786040000000	4.92202200137331
5.83772220000000	4.92137946846536
5.88781990000000	4.92152316877423
5.96017680000000	4.92196995166558
6.01309240000000	4.92302874017193
6.08582000000000	4.92266837992307
6.13735290000000	4.92326378066978
6.19202550000000	4.92341314582010
6.26556470000000	4.92265534635156
6.32301330000000	4.92191017761974
6.37312840000000	4.92157953663228
6.44801980000000	4.92130952265856
6.49939580000000	4.92180739913101
6.54839500000000	4.92247898668020
6.62165970000000	4.92271622605994
6.67514920000000	4.92375061668660
6.72553850000000	4.92386632642619
6.79780440000000	4.92363096581591
6.85430880000000	4.92319787891160
6.92641730000000	4.92239028317170
6.97241650000000	4.92157893957948
7.04374850000000	4.92159217242918
7.09540550000000	4.92139134062778
7.14638830000000	4.92167132803726
7.21848990000000	4.92159944860406
7.27117530000000	4.92131624873393
7.34240760000000	4.92122871961644
7.39138590000000	4.92035466273879
7.46360980000000	4.92009526143331
7.51977050000000	4.92013161369276
7.56387530000000	4.91999432696945
7.63816660000000	4.92041292403010
7.68848970000000	4.91948033959594
7.74662500000000	4.91907525883275
7.82006730000000	4.91901464845267
7.87107610000000	4.91843244686746
7.94422270000000	4.91813595529661
7.99409940000000	4.91849248439224
8.06771130000000	4.91837062980238
8.11480590000000	4.91789134675566
8.16425000000000	4.91843669359373
8.23691800000000	4.91827326405190
8.28651080000000	4.91879268166993
8.34288010000000	4.91918269178271
8.41449800000000	4.91938257602411
8.49070440000000	4.91925918529872
8.54289500000000	4.91877843256908
8.59453580000000	4.91846012663052
8.66798550000000	4.91900100620502
8.71511970000000	4.91906059336449
8.76378030000000	4.91910071110503
8.83594520000000	4.91913624278049
8.88684500000000	4.91888705189081
8.95977010000000	4.91901705107566
9.01165020000000	4.91886457666293
9.05697230000000	4.91891223196122
9.10893500000000	4.91888259161989
9.15547630000000	4.91847804743545
9.22675990000000	4.91958991100066
9.27887360000000	4.91908661631883
9.35236770000000	4.91871180504786
9.40137280000000	4.91858187456525
9.47305790000000	4.91766241410386
9.52669080000000	4.91698160187636
9.57519330000000	4.91626136714890
9.63228530000000	4.91630852742364
9.70575090000000	4.91614540353183
9.75827440000000	4.91639867800887
9.83233890000000	4.91635708679043
9.88513170000000	4.91592271161337
9.93398260000000	4.91565026434086
9.98161320000000	4.91550577188865
10.0341627000000	4.91475436430133
10.0846703000000	4.91381651680003
10.1573654000000	4.91348504734400
10.2079093000000	4.91302991600665
10.2595960000000	4.91316137402768
10.3313935000000	4.91322333135148
10.4068454000000	4.91310829712651
10.4617324000000	4.91276666700346
10.5236031000000	4.91230119540137
10.5719989000000	4.91162116383227
10.6187639000000	4.91120002196568
10.6935565000000	4.91090838680005
10.7453525000000	4.91110094960246
10.7928314000000	4.91102609493393
10.8433441000000	4.91098727222264
10.8957672000000	4.91171510656838
10.9439643000000	4.91205979180914
10.9936400000000	4.91448965834058
11.0463997000000	4.91950152099275
11.1002506000000	4.92497030571557
11.1539174000000	4.92948872314036
11.1993113000000	4.93415176175890
11.2487345000000	4.93701506137423
11.3010574000000	4.93750663098283
11.3532291000000	4.93780961051867
11.4091438000000	4.93915229950872
11.4630337000000	4.94053937598934
11.5131554000000	4.94205328830727
11.5655087000000	4.94346599191371
11.6189510000000	4.94446765807208
11.6631047000000	4.94536159934400

    };

    \end{axis}
\end{tikzpicture}

%% file: FigPlot/5deg_ss-8.tex
\begin{tikzpicture}
    \begin{axis}[
      label style={font=\footnotesize},
      tick label style={font=\footnotesize},
     xmin=2, xmax=8,
      ymin=4.85, ymax=5,
      ymajorgrids=true,
      grid style=dashed,
      legend pos=north west,
      width = 0.26\linewidth,
      height = 0.2\linewidth,
    ]


    \addplot[color=blue, line width = 1pt] table {
0	0
0.0705605000000000	0.0622426788991338
0.141121000000000	0.0121046357787623
0.188711500000000	0.870380970305976
0.258844700000000	1.46446900852654
0.310719800000000	1.85601855727137
0.358813800000000	2.23760492625311
0.432870700000000	2.86269019980580
0.486241300000000	3.20869884424361
0.560112600000000	3.34028653278566
0.608670500000000	3.41884900816758
0.686909200000000	3.42870554089716
0.757143500000000	3.63688112323551
0.829473600000000	3.89374679550751
0.885210100000000	4.10595348276876
0.929167700000000	4.23703235083169
1.00369430000000	4.35778108678732
1.05305600000000	4.45479286647336
1.12539510000000	4.47125871379023
1.17520410000000	4.45434683253162
1.24742540000000	4.49192883099746
1.29806280000000	4.65424128941363
1.35084540000000	4.69065365408870
1.42306060000000	4.75598114497649
1.47314870000000	4.79618589637035
1.53159770000000	4.81397589581136
1.60898140000000	4.79562779077052
1.65383540000000	4.81087914846787
1.72627660000000	4.84727435650897
1.77728170000000	4.86994732792892
1.85085310000000	4.84448212529106
1.90111280000000	4.86102329528288
1.95314600000000	4.87502616363269
2.02504370000000	4.92482870066864
2.07580820000000	4.94889152308267
2.14775530000000	4.95059604182233
2.22354890000000	4.92844592087919
2.27895960000000	4.94538530978151
2.35237390000000	4.92504180218490
2.40275330000000	4.93557302832850
2.45554700000000	4.93218997279170
2.52950280000000	4.93019074177515
2.58335600000000	4.93146461740187
2.62886440000000	4.93336161185102
2.67352690000000	4.93632886430284
2.72303120000000	4.93911778139961
2.76765620000000	4.94057270828604
2.81372940000000	4.94055502282907
2.86343150000000	4.94053394894553
2.90969950000000	4.93826216813791
2.95568160000000	4.93562964043100
3.00357710000000	4.93265038669031
3.05113870000000	4.93061538872072
3.10131400000000	4.92947542459508
3.15241880000000	4.92969395108996
3.19668980000000	4.93076136534546
3.24333310000000	4.93255460196677
3.29264200000000	4.93409830028939
3.34070160000000	4.93489030983787
3.38542100000000	4.93495970092441
3.43412570000000	4.93438821492431
3.47891860000000	4.93359857770515
3.52477840000000	4.93230246712053
3.57354620000000	4.93097089612412
3.62068810000000	4.93083088621229
3.67881990000000	4.93173222971071
3.75666400000000	4.93373386123352
3.80348610000000	4.93511123695513
3.85274600000000	4.93652726299298
3.89582470000000	4.93695163579650
3.94367400000000	4.93646582523102
3.99160100000000	4.93526923353503
4.03808490000000	4.93515433713468
4.10450510000000	4.93386788693850
4.15324540000000	4.93355288030032
4.20094150000000	4.93402578114680
4.24677040000000	4.93460116316337
4.29533600000000	4.93476907388545
4.36995310000000	4.93563898299194
4.41432340000000	4.93552473762427
4.46545710000000	4.93498188903360
4.52012140000000	4.93475227109866
4.56883230000000	4.93409519802965
4.61665320000000	4.93383162716745
4.66002990000000	4.93400456515252
4.71101060000000	4.93388723021649
4.75931970000000	4.93394351276248
4.81200530000000	4.93476445577854
4.86748720000000	4.93461491778635
4.91500030000000	4.93456208324421
4.96508950000000	4.93482289468206
5.02114030000000	4.93423158974481
5.09387890000000	4.93329447262448
5.16673300000000	4.93332742030823
5.24056180000000	4.93389261312143
5.29102600000000	4.93522790102754
5.33786410000000	4.93552066408126
5.41024300000000	4.93555721338882
5.46253340000000	4.93539444666981
5.51701050000000	4.93467079015209
5.59443050000000	4.93431745019632
5.64772310000000	4.93372213388893
5.71707670000000	4.93369572804339
5.79313590000000	4.93429221278325
5.93561610000000	4.93892990045842
5.98076840000000	4.93959527832240
6.03277930000000	4.94020885896571
6.07913600000000	4.94019675481055
6.13080370000000	4.94057064725143
6.17615940000000	4.94011638903339
6.22243500000000	4.93963691673902
6.27047480000000	4.93976286936157
6.34393780000000	4.93971680282611
6.39088780000000	4.94002124655063
6.43782980000000	4.94019507141497
6.48642430000000	4.94048696326509
6.53427880000000	4.94056785669988
6.57946870000000	4.94081922335165
6.62856790000000	4.94068297672042
6.67507400000000	4.94142081059126
6.72479990000000	4.94184260198681
6.78627540000000	4.94208114954576
6.83104520000000	4.94175064739085
6.87535740000000	4.94159061697761
6.92201500000000	4.94135784573868
6.97447420000000	4.94107053700475
7.02028400000000	4.94130661740880
7.06746490000000	4.94193607144755
7.11613720000000	4.94180676742022
7.17409540000000	4.94191806673670
7.21490210000000	4.94232579159927
7.26765580000000	4.94196176079699
7.30866970000000	4.94166777530187
7.35893000000000	4.94131587118812
7.40433870000000	4.94135218520175
7.45143870000000	4.94159497969696
7.49225660000000	4.94139641090323
7.54049930000000	4.94125835202558
7.58930310000000	4.94175445566675
7.63815120000000	4.94191875871355
7.68482410000000	4.94215528866806
7.73526950000000	4.94192241002649
7.77721120000000	4.94242777095177
7.81892170000000	4.94245777454789
7.86803120000000	4.94216599805908
7.91852010000000	4.94174974944333
7.96834050000000	4.94152535346000
8.01274680000000	4.94149489511538
8.06270340000000	4.94140019151526
8.10501760000000	4.94156352672331
8.15062070000000	4.94129240091465
8.20077210000000	4.94098442852785
8.24897150000000	4.94077776998368
8.29346200000000	4.94028200284801
8.34232140000000	4.94001896242999
8.41245180000000	4.93988801827258
8.46287310000000	4.94088877654622
8.53593800000000	4.94135479643357
8.58656890000000	4.94173357646248
8.64346490000000	4.94177566880301
8.69053640000000	4.94116052471860
8.73896140000000	4.94097545289796
8.78658200000000	4.93995195057886
8.83214410000000	4.94027637016034
8.87739020000000	4.94002482283886
8.92565640000000	4.94033349210366
8.97316800000000	4.94047875354881
9.02045950000000	4.94059732150515
9.06602040000000	4.94044909860443
9.11537580000000	4.94044058287332
9.15853320000000	4.94029120478991
9.20275380000000	4.94008193901442
9.25258190000000	4.94012538693105
9.29754920000000	4.94079521282241
9.34408080000000	4.94086139025706
9.39278730000000	4.94081251200683
9.43556990000000	4.94107306268991
9.48425780000000	4.94097524029678
9.53013040000000	4.94058977087038
9.57791320000000	4.94081638088327
9.65489160000000	4.94074739982900
9.70084610000000	4.94073048579213
9.75187310000000	4.94040312977859
9.79583310000000	4.94035430369170
9.84346370000000	4.94042917808441
9.90734860000000	4.94093697788238
9.95188700000000	4.94057422720584
10.0008081000000	4.94128551068539
10.0443153000000	4.94144733018958
10.0913515000000	4.94163593258518
10.1426129000000	4.94150541090628
10.2140365000000	4.94211127410959
10.2635272000000	4.94141501600140
10.3095437000000	4.94116213302525
10.3583561000000	4.94106395118006
10.4027392000000	4.94098228331898
10.4500986000000	4.94071363196648
10.4977717000000	4.94106594583323
10.5438957000000	4.94101443882732
10.5902857000000	4.94112995121491

    };

    \end{axis}
\end{tikzpicture}

%% file: FigPlot/Inputs_5_deg_i.tex
\begin{tikzpicture}[spy using outlines={rectangle, magnification=10,  connect spies}]
    \begin{axis}[
      xmin=0, xmax=3,
      ymin=235, ymax=275,
      ymajorgrids=true,
      grid style=dashed,
      legend pos=north west,
      width = 0.26\linewidth,
      height = 0.2\linewidth,
    label style={font=\footnotesize},
      tick label style={font=\footnotesize},
    xlabel={Time [s]},
    ylabel={Lengths $L_i$ [mm]}
    ]

    \addplot[color=red, line width = 1pt] table[x=time, y=u2] {
time u1 u2
0	252	252
0.0833301000000000	252.822744140625	250.429203673205
0.166660200000000	252.021095581055	249.111941177236
0.211318000000000	251.504243774414	247.564660890889
0.282248500000000	251.535888061523	246.966548340051
0.352285200000000	253.624391784668	245.861140194235
0.399046300000000	255.027276306152	244.988485468992
0.452256300000000	256.620023803711	244.428516620498
0.501360200000000	257.980716247559	243.799584913730
0.553929200000000	258.845651550293	243.108831987145
0.607818200000000	257.780303649902	242.510131976541
0.658328600000000	258.233867797852	242.117760646919
0.709592300000000	259.605108032227	241.736593850680
0.759063700000000	259.953191986084	241.347890986419
0.803066100000000	259.636751861572	241.011601871320
0.848140100000000	259.900452117920	240.725024517060
0.922882200000000	260.491139831543	240.492359903011
0.971851200000000	261.155663909912	240.312454240313
1.02888970000000	261.988955841064	240.179615455106
1.07405850000000	261.946763763428	240.053068488558
1.12745240000000	261.683063964844	239.925590009542
1.18123060000000	261.145116119385	239.771282062612
1.25626780000000	260.470043792725	239.624610095474
1.30611060000000	260.554427947998	239.501942207550
1.35802290000000	260.807579956055	239.409572038430
1.40417920000000	261.145116119385	239.345328465234
1.44921470000000	261.419363708496	239.292123764220
1.52575350000000	261.440459747314	239.248860944835
1.57646670000000	261.429911956787	239.203846506004
1.63315720000000	261.271691894531	239.153781390312
1.68973810000000	261.166211700440	239.119921471347
1.76045670000000	261.166211700440	239.092303534974
1.80843790000000	261.176759948731	239.065176468906
1.86244250000000	261.187307739258	239.047556478692
1.92161830000000	261.229499816895	239.031375103102
1.99283720000000	261.282240142822	239.025337861920
2.04022980000000	261.282240142822	239.017783137600
2.08861790000000	261.334980010986	239.011055328442
2.13662970000000	261.303335723877	239.011556709885
2.19325590000000	261.377172088623	239.016522283564
2.25078930000000	261.398268127441	239.018249457607
2.31671690000000	261.398268127441	239.018596861557
2.36272120000000	261.356076049805	239.018172672447
2.41604910000000	261.324431762695	239.017519417479
2.46438970000000	261.303335723877	239.016678671880
2.50983200000000	261.271691894531	239.015803522054
2.55698240000000	261.261144104004	239.015605257593
2.60539580000000	261.261144104004	239.015420182392
2.65863300000000	261.261144104004	239.015677591070
2.71362010000000	261.250595855713	239.005351914953
2.76654080000000	261.261144104004	239.005669587321
2.82151360000000	261.250595855713	238.995272125171
2.89958490000000	261.250595855713	238.995226247220
2.95295490000000	261.240048065186	238.984933152886
3.02505640000000	261.229499816895	238.974638264875
3.07114240000000	261.250595855713	238.975342875699
3.11908960000000	261.261144104004	238.976144117806
3.16481720000000	261.292787933350	238.977108018855
3.21357320000000	261.292787933350	238.967307914881
3.25961480000000	261.292787933350	238.957248304421
3.31416200000000	261.292787933350	238.958018113141
3.36645200000000	261.313883972168	238.969523978617
3.42144450000000	261.324431762695	238.969841664867
3.49896250000000	261.324431762695	238.970212323910
3.55090310000000	261.324431762695	238.970382798075
3.62507800000000	261.324431762695	238.970187639129
3.68313420000000	261.324431762695	238.969872810257
3.73618780000000	261.313883972168	238.969395710436
3.79067750000000	261.292787933350	238.969341178069
3.83649170000000	261.303335723877	238.969366873117
3.88957190000000	261.303335723877	238.969319424629
3.94221460000000	261.303335723877	238.969270472691
4.01628730000000	261.303335723877	238.969130101477
4.06607190000000	261.292787933350	238.958487073144
4.11373440000000	261.271691894531	238.958344395184
4.16849380000000	261.271691894531	238.958321481254
4.21830790000000	261.282240142822	238.968802606858
4.26332260000000	261.303335723877	238.968925804584
4.31098230000000	261.292787933350	238.969041751228
4.36503000000000	261.303335723877	238.969055790401
4.41217070000000	261.282240142822	238.968750017119
4.45795670000000	261.282240142822	238.968669927443
4.50796140000000	261.282240142822	238.968423330929
4.55449210000000	261.282240142822	238.968152552863
4.60284320000000	261.282240142822	238.967926802417
4.67593310000000	261.271691894531	238.957556774395
4.72646680000000	261.261144104004	238.957486330208
4.77331730000000	261.261144104004	238.957592970127
4.84801160000000	261.261144104004	238.957631921683
4.89773920000000	261.271691894531	238.957813255641
4.97065330000000	261.271691894531	238.957882125551
5.02030260000000	261.282240142822	238.957793821600
5.07064080000000	261.271691894531	238.957808664930
5.12790470000000	261.271691894531	238.957640862518
5.17428430000000	261.261144104004	238.957579350597
5.23125240000000	261.261144104004	238.957485377500
5.29053750000000	261.271691894531	238.957506513269
5.34645290000000	261.271691894531	238.957558421840
5.41998510000000	261.261144104004	238.957597971464
5.47011250000000	261.271691894531	238.957638333911
5.52837630000000	261.261144104004	238.957615176846
5.60156950000000	261.271691894531	238.957589675874
5.64903530000000	261.271691894531	238.957617767066
5.70461970000000	261.271691894531	238.957650380422
5.77806180000000	261.271691894531	238.957484593048
5.83189990000000	261.271691894531	238.957433558348
5.87855080000000	261.271691894531	238.957332896959
5.92372320000000	261.271691894531	238.957248042723
5.97151780000000	261.261144104004	238.957473040261
6.01852520000000	261.271691894531	238.957548686113
6.07384330000000	261.271691894531	238.957546857771
6.11856380000000	261.271691894531	238.957522148420
6.16388860000000	261.271691894531	238.957407405871
6.21481440000000	261.261144104004	238.957309933256
6.26116280000000	261.250595855713	238.946892209959
6.30761450000000	261.282240142822	238.957623943800
6.35596450000000	261.271691894531	238.957671983817
6.40085030000000	261.282240142822	238.957792708103
6.44867970000000	261.271691894531	238.957890096591
6.49582840000000	261.282240142822	238.957771114463
6.54112860000000	261.271691894531	238.957764109890
6.58887200000000	261.261144104004	238.957532821625
6.63603760000000	261.271691894531	238.957510014515
6.68237380000000	261.271691894531	238.957453049717
6.72978850000000	261.261144104004	238.957368301931
6.81995840000000	261.261144104004	238.946732604997
6.87254930000000	261.250595855713	238.946842430930
6.91909200000000	261.250595855713	238.946960361401
6.96541440000000	261.261144104004	238.947174687768
7.01148190000000	261.271691894531	238.947303880344
7.06104690000000	261.271691894531	238.947426595847
7.10730770000000	261.271691894531	238.947505743956
7.15341420000000	261.271691894531	238.947546837619
7.20008600000000	261.271691894531	238.947442108656
7.24719490000000	261.271691894531	238.947401816968
7.29637380000000	261.271691894531	238.947364379876
7.35157210000000	261.271691894531	238.947187352744
7.39605240000000	261.271691894531	238.947281123851
7.43897680000000	261.261144104004	238.947037761412
7.48887700000000	261.271691894531	238.947198350589
7.53488450000000	261.271691894531	238.947161060079
7.57927940000000	261.261144104004	238.947203431166
7.62983070000000	261.261144104004	238.947039284743
7.70555680000000	261.261144104004	238.947114300796
7.75305120000000	261.261144104004	238.947137431280
7.82900700000000	261.261144104004	238.947049137786
7.88718080000000	261.261144104004	238.947281804400
7.93831550000000	261.261144104004	238.947177195922
7.99140750000000	261.261144104004	238.947042910154
8.03747250000000	261.261144104004	238.947079312385
8.09177080000000	261.261144104004	238.946913246379
8.14376920000000	261.261144104004	238.946782433154
8.21581540000000	261.261144104004	238.946756934560
8.26017970000000	261.250595855713	238.936104336557
8.30660620000000	261.240048065186	238.936116413819
8.36066980000000	261.240048065186	238.936190116246
8.41408620000000	261.250595855713	238.936726937432
8.46012840000000	261.261144104004	238.937163652835
8.53376750000000	261.282240142822	238.937532327263
8.60897580000000	261.282240142822	238.937859526725
8.66196170000000	261.282240142822	238.938302725287
8.71365120000000	261.282240142822	238.938009125855
8.78402480000000	261.282240142822	238.938031364340
8.83751320000000	261.282240142822	238.937842178699
8.89510740000000	261.261144104004	238.937406426477
8.94903600000000	261.261144104004	238.937452631989
8.99156420000000	261.261144104004	238.937438900236
9.03707810000000	261.271691894531	238.937309821755
9.07999170000000	261.261144104004	238.937289819448
9.13320450000000	261.261144104004	238.937344601171
9.20858240000000	261.261144104004	238.937241984247
9.26260230000000	261.261144104004	238.937175657959
9.31175670000000	261.261144104004	238.937256507026
9.38366510000000	261.261144104004	238.937299550698
9.43521330000000	261.261144104004	238.937264116403
9.51003450000000	261.261144104004	238.937231023996
9.58196170000000	261.250595855713	238.937345094625
9.62971470000000	261.261144104004	238.937362054212
9.67616200000000	261.261144104004	238.937411457464
9.72924950000000	261.261144104004	238.937461393118
9.78214960000000	261.261144104004	238.947803770564
9.85448190000000	261.261144104004	238.937176674086
9.91177250000000	261.261144104004	238.937073692928
9.96919160000000	261.261144104004	238.947669832062
10.0201195000000	261.282240142822	238.947807925120
10.0533409000000	261.261144104004	238.937205504375
10.1112530000000	261.250595855713	238.937167441787
10.1713297000000	261.250595855713	238.937185276225
10.2204696000000	261.250595855713	238.937045715025
10.3911564000000	261.250595855713	238.936956995299
10.4353331000000	261.250595855713	238.937044507026
10.4842983000000	261.250595855713	238.937061891141
10.5482505000000	261.250595855713	238.937117434453
10.5941404000000	261.250595855713	238.937059754734
10.6351714000000	261.250595855713	238.937039692583
10.6833146000000	261.250595855713	238.937041400104
10.7327105000000	261.250595855713	238.936951388138
10.7766383000000	261.250595855713	238.937081836417
10.8235484000000	261.250595855713	238.937165307823
10.8666660000000	261.261144104004	238.937235857627

 };

 \addplot[color= black, line width = 1pt] table[x=time, y=u1] {
time u1 u2
0	252	252
0.0833301000000000	252.822744140625	250.429203673205
0.166660200000000	252.021095581055	249.111941177236
0.211318000000000	251.504243774414	247.564660890889
0.282248500000000	251.535888061523	246.966548340051
0.352285200000000	253.624391784668	245.861140194235
0.399046300000000	255.027276306152	244.988485468992
0.452256300000000	256.620023803711	244.428516620498
0.501360200000000	257.980716247559	243.799584913730
0.553929200000000	258.845651550293	243.108831987145
0.607818200000000	257.780303649902	242.510131976541
0.658328600000000	258.233867797852	242.117760646919
0.709592300000000	259.605108032227	241.736593850680
0.759063700000000	259.953191986084	241.347890986419
0.803066100000000	259.636751861572	241.011601871320
0.848140100000000	259.900452117920	240.725024517060
0.922882200000000	260.491139831543	240.492359903011
0.971851200000000	261.155663909912	240.312454240313
1.02888970000000	261.988955841064	240.179615455106
1.07405850000000	261.946763763428	240.053068488558
1.12745240000000	261.683063964844	239.925590009542
1.18123060000000	261.145116119385	239.771282062612
1.25626780000000	260.470043792725	239.624610095474
1.30611060000000	260.554427947998	239.501942207550
1.35802290000000	260.807579956055	239.409572038430
1.40417920000000	261.145116119385	239.345328465234
1.44921470000000	261.419363708496	239.292123764220
1.52575350000000	261.440459747314	239.248860944835
1.57646670000000	261.429911956787	239.203846506004
1.63315720000000	261.271691894531	239.153781390312
1.68973810000000	261.166211700440	239.119921471347
1.76045670000000	261.166211700440	239.092303534974
1.80843790000000	261.176759948731	239.065176468906
1.86244250000000	261.187307739258	239.047556478692
1.92161830000000	261.229499816895	239.031375103102
1.99283720000000	261.282240142822	239.025337861920
2.04022980000000	261.282240142822	239.017783137600
2.08861790000000	261.334980010986	239.011055328442
2.13662970000000	261.303335723877	239.011556709885
2.19325590000000	261.377172088623	239.016522283564
2.25078930000000	261.398268127441	239.018249457607
2.31671690000000	261.398268127441	239.018596861557
2.36272120000000	261.356076049805	239.018172672447
2.41604910000000	261.324431762695	239.017519417479
2.46438970000000	261.303335723877	239.016678671880
2.50983200000000	261.271691894531	239.015803522054
2.55698240000000	261.261144104004	239.015605257593
2.60539580000000	261.261144104004	239.015420182392
2.65863300000000	261.261144104004	239.015677591070
2.71362010000000	261.250595855713	239.005351914953
2.76654080000000	261.261144104004	239.005669587321
2.82151360000000	261.250595855713	238.995272125171
2.89958490000000	261.250595855713	238.995226247220
2.95295490000000	261.240048065186	238.984933152886
3.02505640000000	261.229499816895	238.974638264875
3.07114240000000	261.250595855713	238.975342875699
3.11908960000000	261.261144104004	238.976144117806
3.16481720000000	261.292787933350	238.977108018855
3.21357320000000	261.292787933350	238.967307914881
3.25961480000000	261.292787933350	238.957248304421
3.31416200000000	261.292787933350	238.958018113141
3.36645200000000	261.313883972168	238.969523978617
3.42144450000000	261.324431762695	238.969841664867
3.49896250000000	261.324431762695	238.970212323910
3.55090310000000	261.324431762695	238.970382798075
3.62507800000000	261.324431762695	238.970187639129
3.68313420000000	261.324431762695	238.969872810257
3.73618780000000	261.313883972168	238.969395710436
3.79067750000000	261.292787933350	238.969341178069
3.83649170000000	261.303335723877	238.969366873117
3.88957190000000	261.303335723877	238.969319424629
3.94221460000000	261.303335723877	238.969270472691
4.01628730000000	261.303335723877	238.969130101477
4.06607190000000	261.292787933350	238.958487073144
4.11373440000000	261.271691894531	238.958344395184
4.16849380000000	261.271691894531	238.958321481254
4.21830790000000	261.282240142822	238.968802606858
4.26332260000000	261.303335723877	238.968925804584
4.31098230000000	261.292787933350	238.969041751228
4.36503000000000	261.303335723877	238.969055790401
4.41217070000000	261.282240142822	238.968750017119
4.45795670000000	261.282240142822	238.968669927443
4.50796140000000	261.282240142822	238.968423330929
4.55449210000000	261.282240142822	238.968152552863
4.60284320000000	261.282240142822	238.967926802417
4.67593310000000	261.271691894531	238.957556774395
4.72646680000000	261.261144104004	238.957486330208
4.77331730000000	261.261144104004	238.957592970127
4.84801160000000	261.261144104004	238.957631921683
4.89773920000000	261.271691894531	238.957813255641
4.97065330000000	261.271691894531	238.957882125551
5.02030260000000	261.282240142822	238.957793821600
5.07064080000000	261.271691894531	238.957808664930
5.12790470000000	261.271691894531	238.957640862518
5.17428430000000	261.261144104004	238.957579350597
5.23125240000000	261.261144104004	238.957485377500
5.29053750000000	261.271691894531	238.957506513269
5.34645290000000	261.271691894531	238.957558421840
5.41998510000000	261.261144104004	238.957597971464
5.47011250000000	261.271691894531	238.957638333911
5.52837630000000	261.261144104004	238.957615176846
5.60156950000000	261.271691894531	238.957589675874
5.64903530000000	261.271691894531	238.957617767066
5.70461970000000	261.271691894531	238.957650380422
5.77806180000000	261.271691894531	238.957484593048
5.83189990000000	261.271691894531	238.957433558348
5.87855080000000	261.271691894531	238.957332896959
5.92372320000000	261.271691894531	238.957248042723
5.97151780000000	261.261144104004	238.957473040261
6.01852520000000	261.271691894531	238.957548686113
6.07384330000000	261.271691894531	238.957546857771
6.11856380000000	261.271691894531	238.957522148420
6.16388860000000	261.271691894531	238.957407405871
6.21481440000000	261.261144104004	238.957309933256
6.26116280000000	261.250595855713	238.946892209959
6.30761450000000	261.282240142822	238.957623943800
6.35596450000000	261.271691894531	238.957671983817
6.40085030000000	261.282240142822	238.957792708103
6.44867970000000	261.271691894531	238.957890096591
6.49582840000000	261.282240142822	238.957771114463
6.54112860000000	261.271691894531	238.957764109890
6.58887200000000	261.261144104004	238.957532821625
6.63603760000000	261.271691894531	238.957510014515
6.68237380000000	261.271691894531	238.957453049717
6.72978850000000	261.261144104004	238.957368301931
6.81995840000000	261.261144104004	238.946732604997
6.87254930000000	261.250595855713	238.946842430930
6.91909200000000	261.250595855713	238.946960361401
6.96541440000000	261.261144104004	238.947174687768
7.01148190000000	261.271691894531	238.947303880344
7.06104690000000	261.271691894531	238.947426595847
7.10730770000000	261.271691894531	238.947505743956
7.15341420000000	261.271691894531	238.947546837619
7.20008600000000	261.271691894531	238.947442108656
7.24719490000000	261.271691894531	238.947401816968
7.29637380000000	261.271691894531	238.947364379876
7.35157210000000	261.271691894531	238.947187352744
7.39605240000000	261.271691894531	238.947281123851
7.43897680000000	261.261144104004	238.947037761412
7.48887700000000	261.271691894531	238.947198350589
7.53488450000000	261.271691894531	238.947161060079
7.57927940000000	261.261144104004	238.947203431166
7.62983070000000	261.261144104004	238.947039284743
7.70555680000000	261.261144104004	238.947114300796
7.75305120000000	261.261144104004	238.947137431280
7.82900700000000	261.261144104004	238.947049137786
7.88718080000000	261.261144104004	238.947281804400
7.93831550000000	261.261144104004	238.947177195922
7.99140750000000	261.261144104004	238.947042910154
8.03747250000000	261.261144104004	238.947079312385
8.09177080000000	261.261144104004	238.946913246379
8.14376920000000	261.261144104004	238.946782433154
8.21581540000000	261.261144104004	238.946756934560
8.26017970000000	261.250595855713	238.936104336557
8.30660620000000	261.240048065186	238.936116413819
8.36066980000000	261.240048065186	238.936190116246
8.41408620000000	261.250595855713	238.936726937432
8.46012840000000	261.261144104004	238.937163652835
8.53376750000000	261.282240142822	238.937532327263
8.60897580000000	261.282240142822	238.937859526725
8.66196170000000	261.282240142822	238.938302725287
8.71365120000000	261.282240142822	238.938009125855
8.78402480000000	261.282240142822	238.938031364340
8.83751320000000	261.282240142822	238.937842178699
8.89510740000000	261.261144104004	238.937406426477
8.94903600000000	261.261144104004	238.937452631989
8.99156420000000	261.261144104004	238.937438900236
9.03707810000000	261.271691894531	238.937309821755
9.07999170000000	261.261144104004	238.937289819448
9.13320450000000	261.261144104004	238.937344601171
9.20858240000000	261.261144104004	238.937241984247
9.26260230000000	261.261144104004	238.937175657959
9.31175670000000	261.261144104004	238.937256507026
9.38366510000000	261.261144104004	238.937299550698
9.43521330000000	261.261144104004	238.937264116403
9.51003450000000	261.261144104004	238.937231023996
9.58196170000000	261.250595855713	238.937345094625
9.62971470000000	261.261144104004	238.937362054212
9.67616200000000	261.261144104004	238.937411457464
9.72924950000000	261.261144104004	238.937461393118
9.78214960000000	261.261144104004	238.947803770564
9.85448190000000	261.261144104004	238.937176674086
9.91177250000000	261.261144104004	238.937073692928
9.96919160000000	261.261144104004	238.947669832062
10.0201195000000	261.282240142822	238.947807925120
10.0533409000000	261.261144104004	238.937205504375
10.1112530000000	261.250595855713	238.937167441787
10.1713297000000	261.250595855713	238.937185276225
10.2204696000000	261.250595855713	238.937045715025
10.3911564000000	261.250595855713	238.936956995299
10.4353331000000	261.250595855713	238.937044507026
10.4842983000000	261.250595855713	238.937061891141
10.5482505000000	261.250595855713	238.937117434453
10.5941404000000	261.250595855713	238.937059754734
10.6351714000000	261.250595855713	238.937039692583
10.6833146000000	261.250595855713	238.937041400104
10.7327105000000	261.250595855713	238.936951388138
10.7766383000000	261.250595855713	238.937081836417
10.8235484000000	261.250595855713	238.937165307823
10.8666660000000	261.261144104004	238.937235857627

 };

    \end{axis}

\end{tikzpicture}

%% file: FigPlot/Inputs_5_deg_j.tex
\begin{tikzpicture}[spy using outlines={rectangle, magnification=10,  connect spies}]
    \begin{axis}[
      xmin=0, xmax=3,
      ymin=235, ymax=275,
      ymajorgrids=true,
      grid style=dashed,
      legend pos=north west,
      width = 0.26\linewidth,
      height = 0.2\linewidth,
    label style={font=\footnotesize},
      tick label style={font=\footnotesize},
    xlabel={Time [s]},
    ]


    \addplot[color= red, line width = 1pt] table[x=time, y=u2] {
time u1 u2
0	252	252
0.0351693000000000	252.305892333984	250.429203673205
0.0703386000000000	252.875484008789	249.669723788865
0.148496900000000	251.715204162598	248.647923538022
0.201737700000000	251.694107666016	247.209697310799
0.247780300000000	252.506304016113	246.288466142284
0.294785000000000	253.814255676270	245.340681980968
0.342281700000000	254.658096313477	244.528418293500
0.409490200000000	256.503996276855	243.853184267262
0.458248600000000	257.938524169922	243.211914067826
0.506489200000000	259.383600311279	242.725225073866
0.557434400000000	260.206343994141	242.255542928517
0.606649000000000	260.280179901123	241.833723633032
0.681280700000000	259.657847900391	241.414149236080
0.731638100000000	259.151543884277	240.965317972894
0.801053000000000	258.940583953857	240.624298256682
0.854182600000000	259.225380249023	240.307040735967
0.907190300000000	259.984836273193	240.084706984577
0.980507700000000	261.514296112061	239.933542728131
1.02732550000000	261.946764221191	239.803330302671
1.09169730000000	261.303336181641	239.673240784707
1.14825280000000	260.891964111328	239.518120371547
1.22400040000000	260.311824188232	239.380444109647
1.27507570000000	260.290728149414	239.256451891487
1.34881600000000	260.564976196289	239.158579705296
1.39758300000000	260.828675994873	239.094309673536
1.44794900000000	261.155663909912	239.042765782171
1.52158790000000	261.218952026367	239.005535418180
1.57097370000000	261.155663909912	238.967873159117
1.64676630000000	260.965800018311	238.912257847593
1.69823030000000	260.828675994873	238.866117942478
1.77289050000000	260.797032165527	238.830091648130
1.82633950000000	260.744292297363	238.791540570372
1.88192120000000	260.681004180908	238.758577809969
1.93575550000000	260.818128204346	238.746611804738
2.00822640000000	261.166212158203	238.746628691127
2.05801240000000	261.166212158203	238.748923276210
2.11165880000000	261.039635925293	238.740832995002
2.18381150000000	260.923607940674	238.741648322816
2.23172010000000	260.891964111328	238.730689848939
2.27975380000000	260.839224243164	238.718018251418
2.32959890000000	260.870868072510	238.711563150635
2.38223710000000	260.849772033691	238.700183629166
2.45523820000000	260.839224243164	238.690001708568
2.50471720000000	260.849772033691	238.691342985529
2.57734070000000	260.807579956055	238.679538614800
2.62834530000000	260.881416320801	238.685217206553
2.70262810000000	260.934156188965	238.686599772100
2.75031250000000	260.944703979492	238.687915188449
2.79906550000000	260.955252227783	238.688225921246
2.85712100000000	260.955252227783	238.688139167531
2.92952880000000	260.944703979492	238.688042342453
2.99998960000000	260.923607940674	238.687750526065
3.05092360000000	260.923607940674	238.687523460108
3.09950470000000	260.923607940674	238.687686763697
3.17321990000000	260.913060150146	238.677219745764
3.22975650000000	260.913060150146	238.677433601015
3.27754030000000	260.913060150146	238.677619461715
3.35231980000000	260.913060150146	238.677659632800
3.40167140000000	260.913060150146	238.677613840668
3.47674440000000	260.913060150146	238.677617072655
3.52853590000000	260.913060150146	238.667234035138
3.57844300000000	260.913060150146	238.667239487950
3.64998970000000	260.913060150146	238.667341624710
3.70076380000000	260.902511901855	238.667661796477
3.77539740000000	260.913060150146	238.667876708005
3.82764150000000	260.913060150146	238.657662087108
3.90210060000000	260.913060150146	238.657944724991
3.95397540000000	260.913060150146	238.658157723902
4.00262030000000	260.913060150146	238.658415622738
4.07635370000000	260.923607940674	238.658624092239
4.12253370000000	260.934156188965	238.659092121713
4.19719100000000	260.944703979492	238.659117014792
4.24970940000000	260.944703979492	238.659215960300
4.29975670000000	260.944703979492	238.659338364929
4.37335240000000	260.944703979492	238.659332985179
4.44758310000000	260.923607940674	238.649156736937
4.50437690000000	260.934156188965	238.649476029892
4.57821880000000	260.944703979492	238.650051846389
4.65250370000000	260.965800018311	238.650539855174
4.69823190000000	260.965800018311	238.651133165589
4.77374880000000	260.965800018311	238.651296172471
4.84703600000000	260.965800018311	238.651320331867
4.92142350000000	260.965800018311	238.651393375853
4.97689370000000	260.965800018311	238.651578371038
5.04802270000000	260.965800018311	238.651780764760
5.12169520000000	260.965800018311	238.651937769632
5.17404890000000	260.965800018311	238.651672616789
5.24906120000000	260.965800018311	238.651602194417
5.30285310000000	260.965800018311	238.651714058827
5.37191050000000	260.965800018311	238.651611372185
5.41742470000000	260.965800018311	238.651797871683
5.49481290000000	260.965800018311	238.651763065707
5.54683880000000	260.965800018311	238.651457997240
5.59895360000000	260.965800018311	238.651341896963
5.64972910000000	260.965800018311	238.651118456918
5.69788680000000	260.965800018311	238.651159814737
5.74592350000000	260.965800018311	238.651170218216
5.81900420000000	260.965800018311	238.651305769578
5.86975430000000	260.944703979492	238.640769291323
5.91796750000000	260.944703979492	238.640770097934
5.99513290000000	260.955252227783	238.640739257209
6.04629930000000	260.944703979492	238.630739662766
6.11944820000000	260.986896057129	238.642088191543
6.17254580000000	260.986896057129	238.642505439222
6.21796840000000	260.986896057129	238.642827464312
6.29270900000000	260.986896057129	238.642744782814
6.33983420000000	260.986896057129	238.642611058201
6.41559140000000	260.976348266602	238.642509577506
6.46380240000000	260.986896057129	238.642461461024
6.53576670000000	260.986896057129	238.652977899028
6.59170980000001	260.976348266602	238.652844337222
6.63939380000001	260.986896057129	238.652862047658
6.71584840000001	260.986896057129	238.652764071276
6.76636760000001	260.986896057129	238.652730283503
6.83706250000001	260.986896057129	238.652728371537
6.88657700000001	260.986896057129	238.652752203843
6.96071860000001	260.986896057129	238.652923247491
7.01241260000001	260.986896057129	238.652806918934
7.06527690000001	260.986896057129	238.652780242360
7.13721860000001	260.986896057129	238.652674679848
7.18156550000001	260.986896057129	238.652555907195
7.25645890000001	260.976348266602	238.652661504403
7.30501440000001	260.986896057129	238.652912405324
7.37946840000001	260.986896057129	238.652736863732
7.45451670000001	260.986896057129	238.652668510149
7.52935260000001	260.986896057129	238.652791019957
7.57802370000001	260.976348266602	238.652679963566
7.65496680000001	260.986896057129	238.652904189007
7.73130480000001	260.986896057129	238.652940483854
7.77725650000001	260.986896057129	238.652654907237
7.82140090000001	260.986896057129	238.652743680888
7.87272670000001	260.976348266602	238.652795938203
7.92498390000001	260.976348266602	238.652700258035
7.99749630000001	260.986896057129	238.652635070922
8.06792770000001	260.986896057129	238.652595567128
8.11759760000001	260.986896057129	238.652566743300
8.17554920000001	260.986896057129	238.652539012251
8.22246360000001	260.986896057129	238.652780749400
8.29796780000001	260.986896057129	238.652636592701
8.35615830000001	260.986896057129	238.652681179194
8.40288140000001	260.986896057129	238.652895225450
8.45295430000001	260.986896057129	238.652775805590
8.50175240000001	260.986896057129	238.652978722426
8.57358570000001	260.986896057129	238.652884089620
8.62492530000001	260.986896057129	238.652735952361
8.69602020000001	260.986896057129	238.652454489150
8.74511670000001	260.976348266602	238.652237705323
8.81738310000001	260.965800018311	238.652189042427
8.86704660000001	260.976348266602	238.652435572520
8.94394290000001	260.976348266602	238.652529799849
8.98776170000001	260.976348266602	238.652462373824
9.06086210000001	260.976348266602	238.652564576273
9.11479590000000	260.976348266602	238.652409237458
9.16692200000001	260.976348266602	238.652328497761
9.22248820000000	260.976348266602	238.652299665868
9.29255310000001	260.976348266602	238.652394103547
9.34484670000001	260.976348266602	238.652245787963
9.41636100000001	260.965800018311	238.652257036892
9.46676060000001	260.976348266602	238.652279938560
9.53819570000001	260.976348266602	238.652319162087
9.58819150000001	260.976348266602	238.652485918233
9.65934480000001	260.976348266602	238.652423193503
9.70865200000001	260.976348266602	238.652172807213
9.75298840000001	260.976348266602	238.652284711784
9.80339250000001	260.976348266602	238.652162927370
9.85313390000001	260.976348266602	238.652345940920
9.92673050000001	260.976348266602	238.652369433761
9.97590090000001	260.976348266602	238.652387345305
10.0276919000000	260.976348266602	238.652097505111
10.0746680000000	260.976348266602	238.652167192128
10.1257859000000	260.976348266602	238.652063390396
10.1980052000000	260.976348266602	238.652150978771
10.2490589000000	260.976348266602	238.652502855099
10.3005217000000	260.976348266602	238.652619647318
10.3563421000000	260.976348266602	238.652529906931
10.4286921000000	260.976348266602	238.652525031685
10.4797823000000	260.976348266602	238.652506128558
10.5271726000000	260.976348266602	238.652539934644
10.5775426000000	260.976348266602	238.652343095506
10.6275654000000	260.976348266602	238.652432586653
10.6833161000000	260.976348266602	238.652317563709
10.7588773000000	260.976348266602	238.652348046310
10.8076871000000	260.976348266602	238.652319971469
10.8772527000000	260.976348266602	238.652216973109
10.9274159000000	260.976348266602	238.652310002812
11.0026443000000	260.976348266602	238.652285362139
11.0537744000000	260.976348266602	238.652357253557
11.1300654000000	260.976348266602	238.652531021027
11.1804039000000	260.976348266602	238.652490361166
11.2312700000000	260.976348266602	238.652472884885
11.3032016000000	260.976348266602	238.652402546611
11.3517778000000	260.976348266602	238.652246850300
11.4264786000000	260.976348266602	238.652079810639
11.4802431000000	260.976348266602	238.652009728243
11.5331921000000	260.976348266602	238.652018845922
11.5901475000000	260.976348266602	238.652145197306
11.6646416000000	260.976348266602	238.652247028386
11.7152861000000	260.976348266602	238.652316369962
11.7911238000000	260.976348266602	238.652342633615
 };

\addplot[color= black, line width = 1pt] table[x=time, y=u1] {
time u1 u2
0	252	252
0.0351693000000000	252.305892333984	250.429203673205
0.0703386000000000	252.875484008789	249.669723788865
0.148496900000000	251.715204162598	248.647923538022
0.201737700000000	251.694107666016	247.209697310799
0.247780300000000	252.506304016113	246.288466142284
0.294785000000000	253.814255676270	245.340681980968
0.342281700000000	254.658096313477	244.528418293500
0.409490200000000	256.503996276855	243.853184267262
0.458248600000000	257.938524169922	243.211914067826
0.506489200000000	259.383600311279	242.725225073866
0.557434400000000	260.206343994141	242.255542928517
0.606649000000000	260.280179901123	241.833723633032
0.681280700000000	259.657847900391	241.414149236080
0.731638100000000	259.151543884277	240.965317972894
0.801053000000000	258.940583953857	240.624298256682
0.854182600000000	259.225380249023	240.307040735967
0.907190300000000	259.984836273193	240.084706984577
0.980507700000000	261.514296112061	239.933542728131
1.02732550000000	261.946764221191	239.803330302671
1.09169730000000	261.303336181641	239.673240784707
1.14825280000000	260.891964111328	239.518120371547
1.22400040000000	260.311824188232	239.380444109647
1.27507570000000	260.290728149414	239.256451891487
1.34881600000000	260.564976196289	239.158579705296
1.39758300000000	260.828675994873	239.094309673536
1.44794900000000	261.155663909912	239.042765782171
1.52158790000000	261.218952026367	239.005535418180
1.57097370000000	261.155663909912	238.967873159117
1.64676630000000	260.965800018311	238.912257847593
1.69823030000000	260.828675994873	238.866117942478
1.77289050000000	260.797032165527	238.830091648130
1.82633950000000	260.744292297363	238.791540570372
1.88192120000000	260.681004180908	238.758577809969
1.93575550000000	260.818128204346	238.746611804738
2.00822640000000	261.166212158203	238.746628691127
2.05801240000000	261.166212158203	238.748923276210
2.11165880000000	261.039635925293	238.740832995002
2.18381150000000	260.923607940674	238.741648322816
2.23172010000000	260.891964111328	238.730689848939
2.27975380000000	260.839224243164	238.718018251418
2.32959890000000	260.870868072510	238.711563150635
2.38223710000000	260.849772033691	238.700183629166
2.45523820000000	260.839224243164	238.690001708568
2.50471720000000	260.849772033691	238.691342985529
2.57734070000000	260.807579956055	238.679538614800
2.62834530000000	260.881416320801	238.685217206553
2.70262810000000	260.934156188965	238.686599772100
2.75031250000000	260.944703979492	238.687915188449
2.79906550000000	260.955252227783	238.688225921246
2.85712100000000	260.955252227783	238.688139167531
2.92952880000000	260.944703979492	238.688042342453
2.99998960000000	260.923607940674	238.687750526065
3.05092360000000	260.923607940674	238.687523460108
3.09950470000000	260.923607940674	238.687686763697
3.17321990000000	260.913060150146	238.677219745764
3.22975650000000	260.913060150146	238.677433601015
3.27754030000000	260.913060150146	238.677619461715
3.35231980000000	260.913060150146	238.677659632800
3.40167140000000	260.913060150146	238.677613840668
3.47674440000000	260.913060150146	238.677617072655
3.52853590000000	260.913060150146	238.667234035138
3.57844300000000	260.913060150146	238.667239487950
3.64998970000000	260.913060150146	238.667341624710
3.70076380000000	260.902511901855	238.667661796477
3.77539740000000	260.913060150146	238.667876708005
3.82764150000000	260.913060150146	238.657662087108
3.90210060000000	260.913060150146	238.657944724991
3.95397540000000	260.913060150146	238.658157723902
4.00262030000000	260.913060150146	238.658415622738
4.07635370000000	260.923607940674	238.658624092239
4.12253370000000	260.934156188965	238.659092121713
4.19719100000000	260.944703979492	238.659117014792
4.24970940000000	260.944703979492	238.659215960300
4.29975670000000	260.944703979492	238.659338364929
4.37335240000000	260.944703979492	238.659332985179
4.44758310000000	260.923607940674	238.649156736937
4.50437690000000	260.934156188965	238.649476029892
4.57821880000000	260.944703979492	238.650051846389
4.65250370000000	260.965800018311	238.650539855174
4.69823190000000	260.965800018311	238.651133165589
4.77374880000000	260.965800018311	238.651296172471
4.84703600000000	260.965800018311	238.651320331867
4.92142350000000	260.965800018311	238.651393375853
4.97689370000000	260.965800018311	238.651578371038
5.04802270000000	260.965800018311	238.651780764760
5.12169520000000	260.965800018311	238.651937769632
5.17404890000000	260.965800018311	238.651672616789
5.24906120000000	260.965800018311	238.651602194417
5.30285310000000	260.965800018311	238.651714058827
5.37191050000000	260.965800018311	238.651611372185
5.41742470000000	260.965800018311	238.651797871683
5.49481290000000	260.965800018311	238.651763065707
5.54683880000000	260.965800018311	238.651457997240
5.59895360000000	260.965800018311	238.651341896963
5.64972910000000	260.965800018311	238.651118456918
5.69788680000000	260.965800018311	238.651159814737
5.74592350000000	260.965800018311	238.651170218216
5.81900420000000	260.965800018311	238.651305769578
5.86975430000000	260.944703979492	238.640769291323
5.91796750000000	260.944703979492	238.640770097934
5.99513290000000	260.955252227783	238.640739257209
6.04629930000000	260.944703979492	238.630739662766
6.11944820000000	260.986896057129	238.642088191543
6.17254580000000	260.986896057129	238.642505439222
6.21796840000000	260.986896057129	238.642827464312
6.29270900000000	260.986896057129	238.642744782814
6.33983420000000	260.986896057129	238.642611058201
6.41559140000000	260.976348266602	238.642509577506
6.46380240000000	260.986896057129	238.642461461024
6.53576670000000	260.986896057129	238.652977899028
6.59170980000001	260.976348266602	238.652844337222
6.63939380000001	260.986896057129	238.652862047658
6.71584840000001	260.986896057129	238.652764071276
6.76636760000001	260.986896057129	238.652730283503
6.83706250000001	260.986896057129	238.652728371537
6.88657700000001	260.986896057129	238.652752203843
6.96071860000001	260.986896057129	238.652923247491
7.01241260000001	260.986896057129	238.652806918934
7.06527690000001	260.986896057129	238.652780242360
7.13721860000001	260.986896057129	238.652674679848
7.18156550000001	260.986896057129	238.652555907195
7.25645890000001	260.976348266602	238.652661504403
7.30501440000001	260.986896057129	238.652912405324
7.37946840000001	260.986896057129	238.652736863732
7.45451670000001	260.986896057129	238.652668510149
7.52935260000001	260.986896057129	238.652791019957
7.57802370000001	260.976348266602	238.652679963566
7.65496680000001	260.986896057129	238.652904189007
7.73130480000001	260.986896057129	238.652940483854
7.77725650000001	260.986896057129	238.652654907237
7.82140090000001	260.986896057129	238.652743680888
7.87272670000001	260.976348266602	238.652795938203
7.92498390000001	260.976348266602	238.652700258035
7.99749630000001	260.986896057129	238.652635070922
8.06792770000001	260.986896057129	238.652595567128
8.11759760000001	260.986896057129	238.652566743300
8.17554920000001	260.986896057129	238.652539012251
8.22246360000001	260.986896057129	238.652780749400
8.29796780000001	260.986896057129	238.652636592701
8.35615830000001	260.986896057129	238.652681179194
8.40288140000001	260.986896057129	238.652895225450
8.45295430000001	260.986896057129	238.652775805590
8.50175240000001	260.986896057129	238.652978722426
8.57358570000001	260.986896057129	238.652884089620
8.62492530000001	260.986896057129	238.652735952361
8.69602020000001	260.986896057129	238.652454489150
8.74511670000001	260.976348266602	238.652237705323
8.81738310000001	260.965800018311	238.652189042427
8.86704660000001	260.976348266602	238.652435572520
8.94394290000001	260.976348266602	238.652529799849
8.98776170000001	260.976348266602	238.652462373824
9.06086210000001	260.976348266602	238.652564576273
9.11479590000000	260.976348266602	238.652409237458
9.16692200000001	260.976348266602	238.652328497761
9.22248820000000	260.976348266602	238.652299665868
9.29255310000001	260.976348266602	238.652394103547
9.34484670000001	260.976348266602	238.652245787963
9.41636100000001	260.965800018311	238.652257036892
9.46676060000001	260.976348266602	238.652279938560
9.53819570000001	260.976348266602	238.652319162087
9.58819150000001	260.976348266602	238.652485918233
9.65934480000001	260.976348266602	238.652423193503
9.70865200000001	260.976348266602	238.652172807213
9.75298840000001	260.976348266602	238.652284711784
9.80339250000001	260.976348266602	238.652162927370
9.85313390000001	260.976348266602	238.652345940920
9.92673050000001	260.976348266602	238.652369433761
9.97590090000001	260.976348266602	238.652387345305
10.0276919000000	260.976348266602	238.652097505111
10.0746680000000	260.976348266602	238.652167192128
10.1257859000000	260.976348266602	238.652063390396
10.1980052000000	260.976348266602	238.652150978771
10.2490589000000	260.976348266602	238.652502855099
10.3005217000000	260.976348266602	238.652619647318
10.3563421000000	260.976348266602	238.652529906931
10.4286921000000	260.976348266602	238.652525031685
10.4797823000000	260.976348266602	238.652506128558
10.5271726000000	260.976348266602	238.652539934644
10.5775426000000	260.976348266602	238.652343095506
10.6275654000000	260.976348266602	238.652432586653
10.6833161000000	260.976348266602	238.652317563709
10.7588773000000	260.976348266602	238.652348046310
10.8076871000000	260.976348266602	238.652319971469
10.8772527000000	260.976348266602	238.652216973109
10.9274159000000	260.976348266602	238.652310002812
11.0026443000000	260.976348266602	238.652285362139
11.0537744000000	260.976348266602	238.652357253557
11.1300654000000	260.976348266602	238.652531021027
11.1804039000000	260.976348266602	238.652490361166
11.2312700000000	260.976348266602	238.652472884885
11.3032016000000	260.976348266602	238.652402546611
11.3517778000000	260.976348266602	238.652246850300
11.4264786000000	260.976348266602	238.652079810639
11.4802431000000	260.976348266602	238.652009728243
11.5331921000000	260.976348266602	238.652018845922
11.5901475000000	260.976348266602	238.652145197306
11.6646416000000	260.976348266602	238.652247028386
11.7152861000000	260.976348266602	238.652316369962
11.7911238000000	260.976348266602	238.652342633615
 };

    \end{axis}

\end{tikzpicture}

%% file: FigPlot/Inputs_5_deg_k.tex
\begin{tikzpicture}[spy using outlines={rectangle, magnification=10,  connect spies}]
    \begin{axis}[
      xmin=0, xmax=3,
      ymin=235, ymax=275,
      ymajorgrids=true,
      grid style=dashed,
      legend pos=north west,
      width = 0.26\linewidth,
      height = 0.2\linewidth,
    label style={font=\footnotesize},
      tick label style={font=\footnotesize},
    xlabel={Time [s]},
    ]


    \addplot[color=red, line width = 1pt] table[x=time, y=u2] {
time u1 u2
0	252	252
0.0487482000000000	252.358632202148	250.429203673205
0.0974964000000000	253.645488281250	249.533756581174
0.138490000000000	253.793160095215	248.204867658528
0.213782700000000	251.989452209473	247.204020382398
0.284278800000000	252.326987915039	245.932553663218
0.334459900000000	253.814256591797	244.924590890166
0.407996000000000	256.029336547852	244.138641805898
0.458020800000000	257.516604309082	243.491860360793
0.531451000000000	259.531272125244	242.950785563716
0.582233000000000	259.837164001465	242.299776294007
0.635163200000000	258.919488372803	241.887893179936
0.683605600000000	259.594560241699	241.433186879847
0.733938400000000	259.204284210205	240.994008390255
0.808890800000000	258.792912139893	240.604700029299
0.853224500000000	259.204284210205	240.269657149785
0.924172700000000	260.575523986816	240.083223422283
0.973229600000000	261.387720336914	239.881959498561
1.04499130000000	261.714708251953	239.737716544919
1.09412660000000	261.577584228516	239.575899832758
1.16885160000000	260.754840087891	239.420850565525
1.24278080000000	260.322371978760	239.254319513372
1.32041540000000	260.343468017578	239.122180375603
1.37030290000000	260.501688079834	239.017982083324
1.42479590000000	260.860320281982	238.947249002260
1.47081030000000	261.176759948731	238.898297808825
1.52673890000000	261.166212158203	238.845040400307
1.57676330000000	261.071280212402	238.791250232023
1.62724360000000	260.881416320801	238.740331703280
1.69792160000000	260.681004180908	238.694037830640
1.76928790000000	260.586072235107	238.654733865400
1.82979160000000	260.659908142090	238.631886124736
1.87813500000000	260.586072235107	238.605599097162
1.92948860000000	260.533332366943	238.568285574711
1.98485660000000	260.670456390381	238.564703945638
2.05683790000000	260.955252227783	238.563462129527
2.10722040000000	260.955252227783	238.566655231443
2.17975670000000	260.860320281982	238.553892442927
2.22945410000000	260.786484375000	238.548682376126
2.30054200000000	260.733744049072	238.549611954924
2.37195240000000	260.702100219727	238.548606834425
2.44510370000000	260.786484375000	238.553280595110
2.49463330000000	260.702100219727	238.537741230295
2.54204380000000	260.638812103272	238.527591238687
2.59537220000000	260.723196258545	238.532715217284
2.67301260000000	260.786484375000	238.534688214935
2.74664260000000	260.754840087891	238.534267006742
2.79651280000000	260.733744049072	238.533459321542
2.87324300000000	260.744292297363	238.533772851536
2.92387430000000	260.712648010254	238.532783969280
2.97975740000000	260.681004180908	238.531296536332
3.05179200000000	260.649360351563	238.519973967288
3.10253940000000	260.649360351563	238.519639765989
3.17560650000000	260.638812103272	238.519750234290
3.22750890000000	260.649360351563	238.520449109594
3.27940900000000	260.681004180908	238.511484684866
3.35557130000000	260.712648010254	238.512272377910
3.41035090000000	260.702100219727	238.502475760893
3.46210260000000	260.691551971436	238.502368955489
3.53525050000000	260.681004180908	238.501990152731
3.58407130000000	260.659908142090	238.491151175872
3.65722070000000	260.659908142090	238.491172964066
3.70737310000000	260.670456390381	238.491877045987
3.76098710000000	260.681004180908	238.492603891207
3.83436950000000	260.691551971436	238.482865906498
3.88691320000000	260.723196258545	238.494356186132
3.95766180000000	260.733744049072	238.494827524495
4.00939950000000	260.712648010254	238.494865261176
4.06022410000000	260.702100219727	238.494437333246
4.13380330000000	260.691551971436	238.494032606288
4.18530650000000	260.681004180908	238.493265954240
4.23028330000000	260.659908142090	238.482343227985
4.30424990000000	260.670456390381	238.482535437855
4.37591330000000	260.691551971436	238.483414300109
4.42770990000000	260.712648010254	238.484283563491
4.48048200000000	260.733744049072	238.485438599791
4.52651010000000	260.733744049072	238.485917444549
4.57659800000000	260.733744049072	238.486012349951
4.62805580000000	260.712648010254	238.485520632794
4.70023480000000	260.712648010254	238.484916948926
4.77867510000000	260.702100219727	238.484509926831
4.84827160000000	260.702100219727	238.484116652229
4.90154380000000	260.702100219727	238.484382534443
4.94651170000000	260.712648010254	238.485134693388
4.99775950000000	260.723196258545	238.485443213838
5.04591190000000	260.723196258545	238.485662408000
5.12076930000000	260.723196258545	238.485742916799
5.17113470000000	260.723196258545	238.485740613996
5.24317480000000	260.712648010254	238.485308659731
5.29195280000000	260.702100219727	238.485022445475
5.36305920000000	260.702100219727	238.484664917684
5.41787570000000	260.702100219727	238.484596874522
5.47277950000000	260.702100219727	238.484909353207
5.53081010000000	260.702100219727	238.485030760738
5.58554190000000	260.723196258545	238.485333590456
5.64014100000000	260.723196258545	238.485452623498
5.71232210000000	260.723196258545	238.485232259166
5.78786040000000	260.702100219727	238.485014247119
5.83772220000000	260.691551971436	238.484610613017
5.88781990000000	260.702100219727	238.484408755350
5.96017680000000	260.702100219727	238.484453900134
6.01309240000000	260.702100219727	238.484594261119
6.08582000000000	260.702100219727	238.484926889338
6.13735290000000	260.702100219727	238.484813678827
6.19202550000000	260.702100219727	238.485000729488
6.26556470000000	260.702100219727	238.485047653934
6.32301330000000	260.712648010254	238.484809584210
6.37312840000000	260.702100219727	238.484575482548
6.44801980000000	260.702100219727	238.484471608619
6.49939580000000	260.702100219727	238.484386781227
6.54839500000000	260.702100219727	238.484543193734
6.62165970000000	260.702100219727	238.484754179185
6.67514920000000	260.702100219727	238.484828710134
6.72553850000000	260.712648010254	238.485153673534
6.79780440000000	260.712648010254	238.485190024820
6.85430880000000	260.723196258545	238.485116084104
6.92641730000000	260.712648010254	238.484980025840
6.97241650000000	260.712648010254	238.484726312156
7.04374850000000	260.691551971436	238.484471421049
7.09540550000000	260.702100219727	238.484475578271
7.14638830000000	260.691551971436	238.484412485100
7.21848990000000	260.702100219727	238.484500445739
7.27117530000000	260.702100219727	238.484477864149
7.34240760000000	260.702100219727	238.484388894286
7.39138590000000	260.691551971436	238.484361396203
7.46360980000000	260.702100219727	238.484086803136
7.51977050000000	260.702100219727	238.484005309813
7.56387530000000	260.691551971436	238.484016730212
7.63816660000000	260.702100219727	238.483973600315
7.68848970000000	260.702100219727	238.484105106461
7.74662500000000	260.691551971436	238.483812126420
7.82006730000000	260.681004180908	238.483684866545
7.87107610000000	260.691551971436	238.483665825232
7.94422270000000	260.691551971436	238.483482921210
7.99409940000000	260.691551971436	238.483389775636
8.06771130000000	260.691551971436	238.494048657555
8.11480590000000	260.691551971436	238.494010375806
8.16425000000000	260.691551971436	238.483312929596
8.23691800000000	260.691551971436	238.483484255358
8.28651080000000	260.691551971436	238.483432912454
8.34288010000000	260.691551971436	238.483596092311
8.41449800000000	260.691551971436	238.483718617601
8.49070440000000	260.691551971436	238.483781413088
8.54289500000000	260.691551971436	238.483742648748
8.59453580000000	260.691551971436	238.483591615824
8.66798550000000	260.691551971436	238.483491617064
8.71511970000000	260.691551971436	238.483661539394
8.76378030000000	260.691551971436	238.483680259252
8.83594520000000	260.691551971436	238.483692862612
8.88684500000000	260.691551971436	238.483704025217
8.95977010000000	260.691551971436	238.483625739590
9.01165020000000	260.691551971436	238.483666580039
9.05697230000000	260.681004180908	238.483618678789
9.10893500000000	260.691551971436	238.483633650143
9.15547630000000	260.691551971436	238.483624338355
9.22675990000000	260.691551971436	238.483497247051
9.27887360000000	260.691551971436	238.483846549292
9.35236770000000	260.691551971436	238.483688434604
9.40137280000000	260.691551971436	238.483570684171
9.47305790000000	260.691551971436	238.483529865306
9.52669080000000	260.681004180908	238.483241008283
9.57519330000000	260.670456390381	238.483027124814
9.63228530000000	260.681004180908	238.482800856401
9.70575090000000	260.670456390381	238.482815672238
9.75827440000000	260.681004180908	238.482764425356
9.83233890000000	260.670456390381	238.482843993880
9.88513170000000	260.681004180908	238.482830927613
9.93398260000000	260.681004180908	238.482694464626
9.98161320000000	260.670456390381	238.482608872791
10.0341627000000	260.670456390381	238.482563479149
10.0846703000000	260.670456390381	238.482327417493
10.1573654000000	260.659908142090	238.482032784011
10.2079093000000	260.670456390381	238.481928649810
10.2595960000000	260.670456390381	238.481785666084
10.3313935000000	260.670456390381	238.481826964839
10.4068454000000	260.659908142090	238.481846429306
10.4617324000000	260.670456390381	238.481810290239
10.5236031000000	260.670456390381	238.481702963970
10.5719989000000	260.659908142090	238.481556731754
10.6187639000000	260.659908142090	238.481343093536
10.6935565000000	260.659908142090	238.481210787916
10.7453525000000	260.659908142090	238.481119168027
10.7928314000000	260.649360351563	238.481179663415
10.8433441000000	260.649360351563	238.470607441073
10.8957672000000	260.649360351563	238.470595244558
10.9439643000000	260.617716064453	238.449728319407
10.9936400000000	260.617716064453	238.439287899414
11.0463997000000	260.617716064453	238.429504389499
11.1002506000000	260.649360351563	238.431078912588
11.1539174000000	260.681004180908	238.432796981979
11.1993113000000	260.702100219727	238.434216484678
11.2487345000000	260.744292297363	238.435681421464
11.3010574000000	260.744292297363	238.436580953568
11.3532291000000	260.744292297363	238.436735384715
11.4091438000000	260.744292297363	238.436830568543
11.4630337000000	260.744292297363	238.437252386730
11.5131554000000	260.744292297363	238.437688149658
11.5655087000000	260.765388336182	238.438163759240
11.6189510000000	260.765388336182	238.438607573167
11.6631047000000	260.765388336182	238.438922255872
 };

\addplot[color= black, line width = 1pt] table[x=time, y=u1] {
time u1 u2
0	252	252
0.0487482000000000	252.358632202148	250.429203673205
0.0974964000000000	253.645488281250	249.533756581174
0.138490000000000	253.793160095215	248.204867658528
0.213782700000000	251.989452209473	247.204020382398
0.284278800000000	252.326987915039	245.932553663218
0.334459900000000	253.814256591797	244.924590890166
0.407996000000000	256.029336547852	244.138641805898
0.458020800000000	257.516604309082	243.491860360793
0.531451000000000	259.531272125244	242.950785563716
0.582233000000000	259.837164001465	242.299776294007
0.635163200000000	258.919488372803	241.887893179936
0.683605600000000	259.594560241699	241.433186879847
0.733938400000000	259.204284210205	240.994008390255
0.808890800000000	258.792912139893	240.604700029299
0.853224500000000	259.204284210205	240.269657149785
0.924172700000000	260.575523986816	240.083223422283
0.973229600000000	261.387720336914	239.881959498561
1.04499130000000	261.714708251953	239.737716544919
1.09412660000000	261.577584228516	239.575899832758
1.16885160000000	260.754840087891	239.420850565525
1.24278080000000	260.322371978760	239.254319513372
1.32041540000000	260.343468017578	239.122180375603
1.37030290000000	260.501688079834	239.017982083324
1.42479590000000	260.860320281982	238.947249002260
1.47081030000000	261.176759948731	238.898297808825
1.52673890000000	261.166212158203	238.845040400307
1.57676330000000	261.071280212402	238.791250232023
1.62724360000000	260.881416320801	238.740331703280
1.69792160000000	260.681004180908	238.694037830640
1.76928790000000	260.586072235107	238.654733865400
1.82979160000000	260.659908142090	238.631886124736
1.87813500000000	260.586072235107	238.605599097162
1.92948860000000	260.533332366943	238.568285574711
1.98485660000000	260.670456390381	238.564703945638
2.05683790000000	260.955252227783	238.563462129527
2.10722040000000	260.955252227783	238.566655231443
2.17975670000000	260.860320281982	238.553892442927
2.22945410000000	260.786484375000	238.548682376126
2.30054200000000	260.733744049072	238.549611954924
2.37195240000000	260.702100219727	238.548606834425
2.44510370000000	260.786484375000	238.553280595110
2.49463330000000	260.702100219727	238.537741230295
2.54204380000000	260.638812103272	238.527591238687
2.59537220000000	260.723196258545	238.532715217284
2.67301260000000	260.786484375000	238.534688214935
2.74664260000000	260.754840087891	238.534267006742
2.79651280000000	260.733744049072	238.533459321542
2.87324300000000	260.744292297363	238.533772851536
2.92387430000000	260.712648010254	238.532783969280
2.97975740000000	260.681004180908	238.531296536332
3.05179200000000	260.649360351563	238.519973967288
3.10253940000000	260.649360351563	238.519639765989
3.17560650000000	260.638812103272	238.519750234290
3.22750890000000	260.649360351563	238.520449109594
3.27940900000000	260.681004180908	238.511484684866
3.35557130000000	260.712648010254	238.512272377910
3.41035090000000	260.702100219727	238.502475760893
3.46210260000000	260.691551971436	238.502368955489
3.53525050000000	260.681004180908	238.501990152731
3.58407130000000	260.659908142090	238.491151175872
3.65722070000000	260.659908142090	238.491172964066
3.70737310000000	260.670456390381	238.491877045987
3.76098710000000	260.681004180908	238.492603891207
3.83436950000000	260.691551971436	238.482865906498
3.88691320000000	260.723196258545	238.494356186132
3.95766180000000	260.733744049072	238.494827524495
4.00939950000000	260.712648010254	238.494865261176
4.06022410000000	260.702100219727	238.494437333246
4.13380330000000	260.691551971436	238.494032606288
4.18530650000000	260.681004180908	238.493265954240
4.23028330000000	260.659908142090	238.482343227985
4.30424990000000	260.670456390381	238.482535437855
4.37591330000000	260.691551971436	238.483414300109
4.42770990000000	260.712648010254	238.484283563491
4.48048200000000	260.733744049072	238.485438599791
4.52651010000000	260.733744049072	238.485917444549
4.57659800000000	260.733744049072	238.486012349951
4.62805580000000	260.712648010254	238.485520632794
4.70023480000000	260.712648010254	238.484916948926
4.77867510000000	260.702100219727	238.484509926831
4.84827160000000	260.702100219727	238.484116652229
4.90154380000000	260.702100219727	238.484382534443
4.94651170000000	260.712648010254	238.485134693388
4.99775950000000	260.723196258545	238.485443213838
5.04591190000000	260.723196258545	238.485662408000
5.12076930000000	260.723196258545	238.485742916799
5.17113470000000	260.723196258545	238.485740613996
5.24317480000000	260.712648010254	238.485308659731
5.29195280000000	260.702100219727	238.485022445475
5.36305920000000	260.702100219727	238.484664917684
5.41787570000000	260.702100219727	238.484596874522
5.47277950000000	260.702100219727	238.484909353207
5.53081010000000	260.702100219727	238.485030760738
5.58554190000000	260.723196258545	238.485333590456
5.64014100000000	260.723196258545	238.485452623498
5.71232210000000	260.723196258545	238.485232259166
5.78786040000000	260.702100219727	238.485014247119
5.83772220000000	260.691551971436	238.484610613017
5.88781990000000	260.702100219727	238.484408755350
5.96017680000000	260.702100219727	238.484453900134
6.01309240000000	260.702100219727	238.484594261119
6.08582000000000	260.702100219727	238.484926889338
6.13735290000000	260.702100219727	238.484813678827
6.19202550000000	260.702100219727	238.485000729488
6.26556470000000	260.702100219727	238.485047653934
6.32301330000000	260.712648010254	238.484809584210
6.37312840000000	260.702100219727	238.484575482548
6.44801980000000	260.702100219727	238.484471608619
6.49939580000000	260.702100219727	238.484386781227
6.54839500000000	260.702100219727	238.484543193734
6.62165970000000	260.702100219727	238.484754179185
6.67514920000000	260.702100219727	238.484828710134
6.72553850000000	260.712648010254	238.485153673534
6.79780440000000	260.712648010254	238.485190024820
6.85430880000000	260.723196258545	238.485116084104
6.92641730000000	260.712648010254	238.484980025840
6.97241650000000	260.712648010254	238.484726312156
7.04374850000000	260.691551971436	238.484471421049
7.09540550000000	260.702100219727	238.484475578271
7.14638830000000	260.691551971436	238.484412485100
7.21848990000000	260.702100219727	238.484500445739
7.27117530000000	260.702100219727	238.484477864149
7.34240760000000	260.702100219727	238.484388894286
7.39138590000000	260.691551971436	238.484361396203
7.46360980000000	260.702100219727	238.484086803136
7.51977050000000	260.702100219727	238.484005309813
7.56387530000000	260.691551971436	238.484016730212
7.63816660000000	260.702100219727	238.483973600315
7.68848970000000	260.702100219727	238.484105106461
7.74662500000000	260.691551971436	238.483812126420
7.82006730000000	260.681004180908	238.483684866545
7.87107610000000	260.691551971436	238.483665825232
7.94422270000000	260.691551971436	238.483482921210
7.99409940000000	260.691551971436	238.483389775636
8.06771130000000	260.691551971436	238.494048657555
8.11480590000000	260.691551971436	238.494010375806
8.16425000000000	260.691551971436	238.483312929596
8.23691800000000	260.691551971436	238.483484255358
8.28651080000000	260.691551971436	238.483432912454
8.34288010000000	260.691551971436	238.483596092311
8.41449800000000	260.691551971436	238.483718617601
8.49070440000000	260.691551971436	238.483781413088
8.54289500000000	260.691551971436	238.483742648748
8.59453580000000	260.691551971436	238.483591615824
8.66798550000000	260.691551971436	238.483491617064
8.71511970000000	260.691551971436	238.483661539394
8.76378030000000	260.691551971436	238.483680259252
8.83594520000000	260.691551971436	238.483692862612
8.88684500000000	260.691551971436	238.483704025217
8.95977010000000	260.691551971436	238.483625739590
9.01165020000000	260.691551971436	238.483666580039
9.05697230000000	260.681004180908	238.483618678789
9.10893500000000	260.691551971436	238.483633650143
9.15547630000000	260.691551971436	238.483624338355
9.22675990000000	260.691551971436	238.483497247051
9.27887360000000	260.691551971436	238.483846549292
9.35236770000000	260.691551971436	238.483688434604
9.40137280000000	260.691551971436	238.483570684171
9.47305790000000	260.691551971436	238.483529865306
9.52669080000000	260.681004180908	238.483241008283
9.57519330000000	260.670456390381	238.483027124814
9.63228530000000	260.681004180908	238.482800856401
9.70575090000000	260.670456390381	238.482815672238
9.75827440000000	260.681004180908	238.482764425356
9.83233890000000	260.670456390381	238.482843993880
9.88513170000000	260.681004180908	238.482830927613
9.93398260000000	260.681004180908	238.482694464626
9.98161320000000	260.670456390381	238.482608872791
10.0341627000000	260.670456390381	238.482563479149
10.0846703000000	260.670456390381	238.482327417493
10.1573654000000	260.659908142090	238.482032784011
10.2079093000000	260.670456390381	238.481928649810
10.2595960000000	260.670456390381	238.481785666084
10.3313935000000	260.670456390381	238.481826964839
10.4068454000000	260.659908142090	238.481846429306
10.4617324000000	260.670456390381	238.481810290239
10.5236031000000	260.670456390381	238.481702963970
10.5719989000000	260.659908142090	238.481556731754
10.6187639000000	260.659908142090	238.481343093536
10.6935565000000	260.659908142090	238.481210787916
10.7453525000000	260.659908142090	238.481119168027
10.7928314000000	260.649360351563	238.481179663415
10.8433441000000	260.649360351563	238.470607441073
10.8957672000000	260.649360351563	238.470595244558
10.9439643000000	260.617716064453	238.449728319407
10.9936400000000	260.617716064453	238.439287899414
11.0463997000000	260.617716064453	238.429504389499
11.1002506000000	260.649360351563	238.431078912588
11.1539174000000	260.681004180908	238.432796981979
11.1993113000000	260.702100219727	238.434216484678
11.2487345000000	260.744292297363	238.435681421464
11.3010574000000	260.744292297363	238.436580953568
11.3532291000000	260.744292297363	238.436735384715
11.4091438000000	260.744292297363	238.436830568543
11.4630337000000	260.744292297363	238.437252386730
11.5131554000000	260.744292297363	238.437688149658
11.5655087000000	260.765388336182	238.438163759240
11.6189510000000	260.765388336182	238.438607573167
11.6631047000000	260.765388336182	238.438922255872
 };

    \end{axis}

\end{tikzpicture}

%% file: FigPlot/Inputs_5_deg_l.tex
\begin{tikzpicture}[spy using outlines={rectangle, magnification=10,  connect spies}]
    \begin{axis}[
      xmin=0, xmax=3,
      ymin=235, ymax=275,
      ymajorgrids=true,
      grid style=dashed,
      legend pos=north west,
      width = 0.26\linewidth,
      height = 0.2\linewidth,
    label style={font=\footnotesize},
      tick label style={font=\footnotesize},
    xlabel={Time [s]},
 legend columns=2 
    ]
        
 \legend{\footnotesize{$L_1$}, \footnotesize{$L_2$}}

    \addplot[color= red, line width = 1pt] table[x=time, y=u2] {
time u1 u2
0	252	252
0.0705605000000000	252.822744140625	250.429203673205
0.141121000000000	254.858507995605	248.993134313361
0.188711500000000	256.124267578125	247.479566515283
0.258844700000000	258.286607666016	246.694401868149
0.310719800000000	259.805520172119	245.551991973057
0.358813800000000	261.271691894531	244.715133643107
0.432870700000000	262.769507904053	243.938432031086
0.486241300000000	263.149235916138	243.248776562096
0.560112600000000	264.066912002563	242.735146840672
0.608670500000000	263.877048110962	242.217441532430
0.686909200000000	263.001564102173	241.799106304559
0.757143500000000	262.273751907349	241.316995306109
0.829473600000000	262.031147918701	240.897188074590
0.885210100000000	262.178820190430	240.545417154120
0.929167700000000	262.389780120850	240.295642811055
1.00369430000000	262.590192031860	240.104767398239
1.05305600000000	262.505808105469	239.910644847654
1.12539510000000	262.453068008423	239.761806811545
1.17520410000000	262.294847946167	239.619306980549
1.24742540000000	261.746352081299	239.476869932920
1.29806280000000	261.556488189697	239.341003936456
1.35084540000000	261.672516174316	239.254873706768
1.42306060000000	261.841284027100	239.181928833231
1.47314870000000	261.830736236572	239.118067914503
1.53159770000000	261.809640197754	239.077958741500
1.60898140000000	261.704160003662	239.030807766491
1.65383540000000	261.598679809570	238.982850546127
1.72627660000000	261.419364166260	238.945450739348
1.77728170000000	261.451007995606	238.914691638004
1.85085310000000	261.461555786133	238.890170274940
1.90111280000000	261.345527801514	238.861074564532
1.95314600000000	261.271691894531	238.834626839235
2.02504370000000	261.282240142822	238.828479095069
2.07580820000000	261.366623840332	238.823027611408
2.14775530000000	261.503747863770	238.830587170020
2.22354890000000	261.419364166260	238.820575785375
2.27895960000000	261.366623840332	238.824163994652
2.35237390000000	261.366623840332	238.818938785625
2.40275330000000	261.292787933350	238.812547684224
2.45554700000000	261.292787933350	238.815856166492
2.52950280000000	261.261144104004	238.814793348250
2.58335600000000	261.240048065186	238.814165271303
2.62886440000000	261.240048065186	238.814565471134
2.67352690000000	261.240048065186	238.815161429516
2.72303120000000	261.250595855713	238.816093619367
2.76765620000000	261.250595855713	238.816969783513
2.81372940000000	261.250595855713	238.817426862275
2.86343150000000	261.250595855713	238.817421306225
2.90969950000000	261.250595855713	238.817414685669
2.95568160000000	261.229499816895	238.816700984679
3.00357710000000	261.208404235840	238.815873951709
3.05113870000000	261.208404235840	238.814937991542
3.10131400000000	261.208404235840	238.814298678075
3.15241880000000	261.208404235840	238.813940547783
3.19668980000000	261.208404235840	238.814009199906
3.24333310000000	261.197855987549	238.814344537985
3.29264200000000	261.208404235840	238.814907899884
3.34070160000000	261.208404235840	238.815392867015
3.38542100000000	261.218952026367	238.815641684153
3.43412570000000	261.229499816895	238.815663484006
3.47891860000000	261.208404235840	238.815483946384
3.52477840000000	261.218952026367	238.815235874535
3.57354620000000	261.208404235840	238.804279983331
3.62068810000000	261.208404235840	238.803861657965
3.67881990000000	261.197855987549	238.803817672554
3.75666400000000	261.208404235840	238.804100837966
3.80348610000000	261.208404235840	238.804729669054
3.85274600000000	261.208404235840	238.805162384399
3.89582470000000	261.197855987549	238.805607242099
3.94367400000000	261.208404235840	238.805740562747
3.99160100000000	261.208404235840	238.805587940857
4.03808490000000	261.208404235840	238.805212020489
4.10450510000000	261.208404235840	238.805175924720
4.15324540000000	261.208404235840	238.804771774471
4.20094150000000	261.208404235840	238.804672812217
4.24677040000000	261.208404235840	238.804821378400
4.29533600000000	261.208404235840	238.805002139991
4.36995310000000	261.208404235840	238.805054890701
4.41432340000000	261.208404235840	238.805328180706
4.46545710000000	261.208404235840	238.805292289466
4.52012140000000	261.208404235840	238.805121748551
4.56883230000000	261.208404235840	238.805049611949
4.61665320000000	261.208404235840	238.804843186357
4.66002990000000	261.208404235840	238.804760383128
4.71101060000000	261.208404235840	238.804814713199
4.75931970000000	261.208404235840	238.804777851341
4.81200530000000	261.208404235840	238.804795533025
4.86748720000000	261.208404235840	238.805053439879
4.91500030000000	261.208404235840	238.805006461134
4.96508950000000	261.197855987549	238.804989862673
5.02114030000000	261.208404235840	238.805071799002
5.09387890000000	261.208404235840	238.804886035078
5.16673300000000	261.208404235840	238.794044756052
5.24056180000000	261.208404235840	238.804601981872
5.29102600000000	261.208404235840	238.804779542431
5.33786410000000	261.208404235840	238.805199035498
5.41024300000000	261.208404235840	238.805291009724
5.46253340000000	261.208404235840	238.805302492028
5.51701050000000	261.208404235840	238.805251357355
5.59443050000000	261.208404235840	238.805024013955
5.64772310000000	261.208404235840	238.794366133934
5.71707670000000	261.208404235840	238.794179109800
5.79313590000000	261.197855987549	238.783622108105
5.93561610000000	261.197855987549	238.783809499312
5.98076840000000	261.197855987549	238.785266471865
6.03277930000000	261.197855987549	238.785475506486
6.07913600000000	261.197855987549	238.785668268530
6.13080370000000	261.197855987549	238.785664465898
6.17615940000000	261.197855987549	238.785781927672
6.22243500000000	261.197855987549	238.785639218244
6.27047480000000	261.197855987549	238.785488587581
6.34393780000000	261.197855987549	238.785528156764
6.39088780000000	261.197855987549	238.785513684535
6.43782980000000	261.197855987549	238.785609328352
6.48642430000000	261.197855987549	238.775115230989
6.53427880000000	261.197855987549	238.775206931518
6.57946870000000	261.197855987549	238.775232344940
6.62856790000000	261.197855987549	238.775311314103
6.67507400000000	261.197855987549	238.775268510961
6.72479990000000	261.197855987549	238.775500308308
6.78627540000000	261.197855987549	238.775632817983
6.83104520000000	261.197855987549	238.775707759909
6.87535740000000	261.197855987549	238.775603929595
6.92201500000000	261.197855987549	238.775553654558
6.97447420000000	261.197855987549	238.775480527316
7.02028400000000	261.197855987549	238.775390266615
7.06746490000000	261.197855987549	238.775464433462
7.11613720000000	261.197855987549	238.775662182280
7.17409540000000	261.197855987549	238.775621560222
7.21490210000000	261.197855987549	238.775656525933
7.26765580000000	261.197855987549	238.775784616477
7.30866970000000	261.197855987549	238.775670252827
7.35893000000000	261.197855987549	238.775577894560
7.40433870000000	261.197855987549	238.775467340622
7.45143870000000	261.197855987549	238.775478749006
7.49225660000000	261.197855987549	238.775555025146
7.54049930000000	261.197855987549	238.775492642920
7.58930310000000	261.197855987549	238.775449270444
7.63815120000000	261.197855987549	238.775605126000
7.68482410000000	261.197855987549	238.775656743324
7.73526950000000	261.197855987549	238.775731051401
7.77721120000000	261.197855987549	238.775657890418
7.81892170000000	261.197855987549	238.775816654235
7.86803120000000	261.197855987549	238.775826080143
7.91852010000000	261.197855987549	238.775734415855
7.96834050000000	261.197855987549	238.775603647496
8.01274680000000	261.197855987549	238.775533151419
8.06270340000000	261.197855987549	238.775523582648
8.10501760000000	261.197855987549	238.775493830634
8.15062070000000	261.197855987549	238.775545143903
8.20077210000000	261.197855987549	238.775459967218
8.24897150000000	261.197855987549	238.775363214840
8.29346200000000	261.197855987549	238.775298291143
8.34232140000000	261.197855987549	238.775142541304
8.41245180000000	261.197855987549	238.775059904720
8.46287310000000	261.197855987549	238.775018767399
8.53593800000000	261.197855987549	238.775333164883
8.58656890000000	261.197855987549	238.775479569349
8.64346490000000	261.197855987549	238.775598566604
8.69053640000000	261.197855987549	238.775611790303
8.73896140000000	261.197855987549	238.775418537089
8.78658200000000	261.197855987549	238.775360395062
8.83214410000000	261.197855987549	238.775038852326
8.87739020000000	261.197855987549	238.775140771743
8.92565640000000	261.197855987549	238.775061745821
8.97316800000000	261.197855987549	238.775158717131
9.02045950000000	261.176759948731	238.775204352360
9.06602040000000	261.197855987549	238.775241601582
9.11537580000000	261.197855987549	238.775195035984
9.15853320000000	261.197855987549	238.775192360688
9.20275380000000	261.197855987549	238.775145432179
9.25258190000000	261.197855987549	238.775079689397
9.29754920000000	261.197855987549	238.775093338963
9.34408080000000	261.197855987549	238.775303770973
9.39278730000000	261.197855987549	238.775324561227
9.43556990000000	261.197855987549	238.775309205672
9.48425780000000	261.197855987549	238.775391060083
9.53013040000000	261.197855987549	238.775360328272
9.57791320000000	261.197855987549	238.775239229480
9.65489160000000	261.176759948731	238.775310421115
9.70084610000000	261.197855987549	238.775288750078
9.75187310000000	261.197855987549	238.775283436376
9.79583310000000	261.197855987549	238.775180594452
9.84346370000000	261.197855987549	238.775165255284
9.90734860000000	261.197855987549	238.775188777768
9.95188700000000	261.197855987549	238.775348307780
10.0008081000000	261.197855987549	238.775234346294
10.0443153000000	261.197855987549	238.775457802589
10.0913515000000	261.197855987549	238.775508639686
10.1426129000000	261.197855987549	238.775567890876
10.2140365000000	261.197855987549	238.775526886281
10.2635272000000	261.197855987549	238.775717223820
10.3095437000000	261.197855987549	238.775498487884
10.3583561000000	261.197855987549	238.775419042354
10.4027392000000	261.197855987549	238.775388197618
10.4500986000000	261.197855987549	238.775362540902
10.4977717000000	261.197855987549	238.775278141591
10.5438957000000	261.197855987549	238.775388824256
10.5902857000000	261.197855987549	238.775372642853
 };
  \addplot[color= black, line width = 1pt] table[x=time, y=u1] {
time u1 u2
0	252	252
0.0705605000000000	252.822744140625	250.429203673205
0.141121000000000	254.858507995605	248.993134313361
0.188711500000000	256.124267578125	247.479566515283
0.258844700000000	258.286607666016	246.694401868149
0.310719800000000	259.805520172119	245.551991973057
0.358813800000000	261.271691894531	244.715133643107
0.432870700000000	262.769507904053	243.938432031086
0.486241300000000	263.149235916138	243.248776562096
0.560112600000000	264.066912002563	242.735146840672
0.608670500000000	263.877048110962	242.217441532430
0.686909200000000	263.001564102173	241.799106304559
0.757143500000000	262.273751907349	241.316995306109
0.829473600000000	262.031147918701	240.897188074590
0.885210100000000	262.178820190430	240.545417154120
0.929167700000000	262.389780120850	240.295642811055
1.00369430000000	262.590192031860	240.104767398239
1.05305600000000	262.505808105469	239.910644847654
1.12539510000000	262.453068008423	239.761806811545
1.17520410000000	262.294847946167	239.619306980549
1.24742540000000	261.746352081299	239.476869932920
1.29806280000000	261.556488189697	239.341003936456
1.35084540000000	261.672516174316	239.254873706768
1.42306060000000	261.841284027100	239.181928833231
1.47314870000000	261.830736236572	239.118067914503
1.53159770000000	261.809640197754	239.077958741500
1.60898140000000	261.704160003662	239.030807766491
1.65383540000000	261.598679809570	238.982850546127
1.72627660000000	261.419364166260	238.945450739348
1.77728170000000	261.451007995606	238.914691638004
1.85085310000000	261.461555786133	238.890170274940
1.90111280000000	261.345527801514	238.861074564532
1.95314600000000	261.271691894531	238.834626839235
2.02504370000000	261.282240142822	238.828479095069
2.07580820000000	261.366623840332	238.823027611408
2.14775530000000	261.503747863770	238.830587170020
2.22354890000000	261.419364166260	238.820575785375
2.27895960000000	261.366623840332	238.824163994652
2.35237390000000	261.366623840332	238.818938785625
2.40275330000000	261.292787933350	238.812547684224
2.45554700000000	261.292787933350	238.815856166492
2.52950280000000	261.261144104004	238.814793348250
2.58335600000000	261.240048065186	238.814165271303
2.62886440000000	261.240048065186	238.814565471134
2.67352690000000	261.240048065186	238.815161429516
2.72303120000000	261.250595855713	238.816093619367
2.76765620000000	261.250595855713	238.816969783513
2.81372940000000	261.250595855713	238.817426862275
2.86343150000000	261.250595855713	238.817421306225
2.90969950000000	261.250595855713	238.817414685669
2.95568160000000	261.229499816895	238.816700984679
3.00357710000000	261.208404235840	238.815873951709
3.05113870000000	261.208404235840	238.814937991542
3.10131400000000	261.208404235840	238.814298678075
3.15241880000000	261.208404235840	238.813940547783
3.19668980000000	261.208404235840	238.814009199906
3.24333310000000	261.197855987549	238.814344537985
3.29264200000000	261.208404235840	238.814907899884
3.34070160000000	261.208404235840	238.815392867015
3.38542100000000	261.218952026367	238.815641684153
3.43412570000000	261.229499816895	238.815663484006
3.47891860000000	261.208404235840	238.815483946384
3.52477840000000	261.218952026367	238.815235874535
3.57354620000000	261.208404235840	238.804279983331
3.62068810000000	261.208404235840	238.803861657965
3.67881990000000	261.197855987549	238.803817672554
3.75666400000000	261.208404235840	238.804100837966
3.80348610000000	261.208404235840	238.804729669054
3.85274600000000	261.208404235840	238.805162384399
3.89582470000000	261.197855987549	238.805607242099
3.94367400000000	261.208404235840	238.805740562747
3.99160100000000	261.208404235840	238.805587940857
4.03808490000000	261.208404235840	238.805212020489
4.10450510000000	261.208404235840	238.805175924720
4.15324540000000	261.208404235840	238.804771774471
4.20094150000000	261.208404235840	238.804672812217
4.24677040000000	261.208404235840	238.804821378400
4.29533600000000	261.208404235840	238.805002139991
4.36995310000000	261.208404235840	238.805054890701
4.41432340000000	261.208404235840	238.805328180706
4.46545710000000	261.208404235840	238.805292289466
4.52012140000000	261.208404235840	238.805121748551
4.56883230000000	261.208404235840	238.805049611949
4.61665320000000	261.208404235840	238.804843186357
4.66002990000000	261.208404235840	238.804760383128
4.71101060000000	261.208404235840	238.804814713199
4.75931970000000	261.208404235840	238.804777851341
4.81200530000000	261.208404235840	238.804795533025
4.86748720000000	261.208404235840	238.805053439879
4.91500030000000	261.208404235840	238.805006461134
4.96508950000000	261.197855987549	238.804989862673
5.02114030000000	261.208404235840	238.805071799002
5.09387890000000	261.208404235840	238.804886035078
5.16673300000000	261.208404235840	238.794044756052
5.24056180000000	261.208404235840	238.804601981872
5.29102600000000	261.208404235840	238.804779542431
5.33786410000000	261.208404235840	238.805199035498
5.41024300000000	261.208404235840	238.805291009724
5.46253340000000	261.208404235840	238.805302492028
5.51701050000000	261.208404235840	238.805251357355
5.59443050000000	261.208404235840	238.805024013955
5.64772310000000	261.208404235840	238.794366133934
5.71707670000000	261.208404235840	238.794179109800
5.79313590000000	261.197855987549	238.783622108105
5.93561610000000	261.197855987549	238.783809499312
5.98076840000000	261.197855987549	238.785266471865
6.03277930000000	261.197855987549	238.785475506486
6.07913600000000	261.197855987549	238.785668268530
6.13080370000000	261.197855987549	238.785664465898
6.17615940000000	261.197855987549	238.785781927672
6.22243500000000	261.197855987549	238.785639218244
6.27047480000000	261.197855987549	238.785488587581
6.34393780000000	261.197855987549	238.785528156764
6.39088780000000	261.197855987549	238.785513684535
6.43782980000000	261.197855987549	238.785609328352
6.48642430000000	261.197855987549	238.775115230989
6.53427880000000	261.197855987549	238.775206931518
6.57946870000000	261.197855987549	238.775232344940
6.62856790000000	261.197855987549	238.775311314103
6.67507400000000	261.197855987549	238.775268510961
6.72479990000000	261.197855987549	238.775500308308
6.78627540000000	261.197855987549	238.775632817983
6.83104520000000	261.197855987549	238.775707759909
6.87535740000000	261.197855987549	238.775603929595
6.92201500000000	261.197855987549	238.775553654558
6.97447420000000	261.197855987549	238.775480527316
7.02028400000000	261.197855987549	238.775390266615
7.06746490000000	261.197855987549	238.775464433462
7.11613720000000	261.197855987549	238.775662182280
7.17409540000000	261.197855987549	238.775621560222
7.21490210000000	261.197855987549	238.775656525933
7.26765580000000	261.197855987549	238.775784616477
7.30866970000000	261.197855987549	238.775670252827
7.35893000000000	261.197855987549	238.775577894560
7.40433870000000	261.197855987549	238.775467340622
7.45143870000000	261.197855987549	238.775478749006
7.49225660000000	261.197855987549	238.775555025146
7.54049930000000	261.197855987549	238.775492642920
7.58930310000000	261.197855987549	238.775449270444
7.63815120000000	261.197855987549	238.775605126000
7.68482410000000	261.197855987549	238.775656743324
7.73526950000000	261.197855987549	238.775731051401
7.77721120000000	261.197855987549	238.775657890418
7.81892170000000	261.197855987549	238.775816654235
7.86803120000000	261.197855987549	238.775826080143
7.91852010000000	261.197855987549	238.775734415855
7.96834050000000	261.197855987549	238.775603647496
8.01274680000000	261.197855987549	238.775533151419
8.06270340000000	261.197855987549	238.775523582648
8.10501760000000	261.197855987549	238.775493830634
8.15062070000000	261.197855987549	238.775545143903
8.20077210000000	261.197855987549	238.775459967218
8.24897150000000	261.197855987549	238.775363214840
8.29346200000000	261.197855987549	238.775298291143
8.34232140000000	261.197855987549	238.775142541304
8.41245180000000	261.197855987549	238.775059904720
8.46287310000000	261.197855987549	238.775018767399
8.53593800000000	261.197855987549	238.775333164883
8.58656890000000	261.197855987549	238.775479569349
8.64346490000000	261.197855987549	238.775598566604
8.69053640000000	261.197855987549	238.775611790303
8.73896140000000	261.197855987549	238.775418537089
8.78658200000000	261.197855987549	238.775360395062
8.83214410000000	261.197855987549	238.775038852326
8.87739020000000	261.197855987549	238.775140771743
8.92565640000000	261.197855987549	238.775061745821
8.97316800000000	261.197855987549	238.775158717131
9.02045950000000	261.176759948731	238.775204352360
9.06602040000000	261.197855987549	238.775241601582
9.11537580000000	261.197855987549	238.775195035984
9.15853320000000	261.197855987549	238.775192360688
9.20275380000000	261.197855987549	238.775145432179
9.25258190000000	261.197855987549	238.775079689397
9.29754920000000	261.197855987549	238.775093338963
9.34408080000000	261.197855987549	238.775303770973
9.39278730000000	261.197855987549	238.775324561227
9.43556990000000	261.197855987549	238.775309205672
9.48425780000000	261.197855987549	238.775391060083
9.53013040000000	261.197855987549	238.775360328272
9.57791320000000	261.197855987549	238.775239229480
9.65489160000000	261.176759948731	238.775310421115
9.70084610000000	261.197855987549	238.775288750078
9.75187310000000	261.197855987549	238.775283436376
9.79583310000000	261.197855987549	238.775180594452
9.84346370000000	261.197855987549	238.775165255284
9.90734860000000	261.197855987549	238.775188777768
9.95188700000000	261.197855987549	238.775348307780
10.0008081000000	261.197855987549	238.775234346294
10.0443153000000	261.197855987549	238.775457802589
10.0913515000000	261.197855987549	238.775508639686
10.1426129000000	261.197855987549	238.775567890876
10.2140365000000	261.197855987549	238.775526886281
10.2635272000000	261.197855987549	238.775717223820
10.3095437000000	261.197855987549	238.775498487884
10.3583561000000	261.197855987549	238.775419042354
10.4027392000000	261.197855987549	238.775388197618
10.4500986000000	261.197855987549	238.775362540902
10.4977717000000	261.197855987549	238.775278141591
10.5438957000000	261.197855987549	238.775388824256
10.5902857000000	261.197855987549	238.775372642853
 };

    \end{axis}

\end{tikzpicture}

%% file: FigPlot/CtrlAgl_10deg.tex
\begin{tikzpicture}[spy using outlines={rectangle, magnification=10,  connect spies}]
    \begin{axis}[
      xmin=0, xmax=3,
      ymin=-0.2, ymax=11,
      ymajorgrids=true,
      grid style=dashed,
      legend pos=north west,
      width = 0.53\linewidth,
      height = 0.4\linewidth,
      label style={font=\footnotesize},
      tick label style={font=\footnotesize},
      ylabel={Position $q_i$ [deg]}
    ]


    \addplot[color= blue, line width = 1pt] table {
0	0
0.0568485000000000	0.00209126714784225
0.106636900000000	0.00743696399724719
0.161638500000000	0.388400158181980
0.237777400000000	1.37202235246266
0.290532700000000	2.02273944766424
0.365799200000000	3.22815339818914
0.438718600000000	4.60052548843545
0.487416100000000	5.17725573730891
0.535712000000000	5.61782213236252
0.586536000000000	5.97051030356368
0.635139300000000	6.14879899598734
0.706669000000000	6.26649709193724
0.757562000000000	6.50951727742588
0.830890900000000	6.70214474858734
0.881397700000000	7.30846043668172
0.950236100000000	7.58912643159160
1.00552990000000	8.10584513287051
1.07888820000000	8.24492781090661
1.12974440000000	8.36961111682759
1.17986130000000	8.38441201358182
1.25112780000000	8.46849376478978
1.32074750000000	8.64060081673111
1.39544250000000	8.83564928736188
1.47076090000000	8.96275436150149
1.52213480000000	9.05662710183979
1.57252550000000	9.07959334408097
1.64488990000000	9.11793513663683
1.69843880000000	9.23856670700627
1.75164320000000	9.30478327632848
1.82328790000000	9.33985652284568
1.86865240000000	9.42846269347486
1.91507870000000	9.44163541286136
1.96804670000000	9.46817283163400
2.02119760000000	9.52070343346727
2.06936960000000	9.56877028469705
2.11511750000000	9.58461903016030
2.17005090000000	9.61688893954348
2.24004790000000	9.65568146614171
2.29150660000000	9.66835414186884
2.36735690000000	9.71428891522744
2.41366020000000	9.71740982670630
2.46251270000000	9.72210335205928
2.53840660000000	9.73765872437619
2.58636250000000	9.75576345698985
2.65704820000000	9.71844188844325
2.72977950000000	9.79414747674658
2.78028050000000	9.81593301475552
2.83110010000000	9.82342477961994
2.90208190000000	9.77148798252749
2.95259440000000	9.77680489499167
3.00009410000000	9.77728861993960
3.05300540000000	9.83286153319490
3.09702660000000	9.83427654216623
3.17005090000000	9.79589718966077
3.21888970000000	9.81177971970041
3.26531320000000	9.81671538836751
3.31608710000000	9.83019048092551
3.36906870000000	9.83470148602969
3.44207910000000	9.83763212107118
3.49354800000000	9.84124350065025
3.56546340000000	9.84246508278187
3.61569280000000	9.84216777171383
3.68714190000000	9.84500255522507
3.73316150000000	9.84849175365776
3.80794530000000	9.85382895433311
3.85869200000000	9.85714792036309
3.91004490000000	9.85947408110933
3.98400180000000	9.86179485952621
4.03540950000000	9.86342681665611
4.10589140000000	9.86437634478516
4.15844440000000	9.86552491050980
4.20995290000000	9.86617856445928
4.28641990000000	9.86791054483868
4.33626890000000	9.86929665692849
4.38868640000000	9.87010192582634
4.43748430000000	9.87129100892221
4.48551760000000	9.87276727328427
4.53661490000000	9.87206652419971
4.61049910000000	9.87227698433492
4.66217680000000	9.87399375172196
4.73612090000000	9.87447018043791
4.78777920000000	9.87689834586926
4.83594500000000	9.87939665832561
4.90997300000000	9.88157788543642
4.95878910000000	9.88353463397251
5.03354700000000	9.88646833514496
5.08270270000000	9.88758444455975
5.13924270000000	9.88820900602906
5.20868620000000	9.88868262426106
5.26043240000000	9.88804367951682
5.32965990000000	9.88691688194161
5.38064340000000	9.88665084085617
5.45411100000000	9.88580633164090
5.50786260000000	9.88493523542328
5.55265470000000	9.88522649081629
5.62539540000000	9.88523177553109
5.67641810000000	9.88516663256903
5.75213750000000	9.88493202289858
5.80259620000000	9.88366436492078
5.87822660000000	9.88334760965168
5.92918960000000	9.88225446230326
5.97974560000000	9.88188547453757
6.05367320000000	9.88151202883586
6.10597540000000	9.88114807478523
6.15179810000000	9.88147000842142
6.22501070000000	9.88211426277518
6.29700710000000	9.88341768098753
6.37330250000000	9.88314194222977
6.44670880000000	9.88417435863539
6.49777170000000	9.88393454056586
6.56981430000000	9.88396977876387
6.64270070000000	9.88421875565212
6.71886240000000	9.88454616638480
6.76212200000000	9.88350216628520
6.81847900000000	9.88302318063962
6.89376010000000	9.88290351587376
6.94049880000000	9.88241380209900
7.01414720000000	9.88271309365023
7.06521790000000	9.88314262744850
7.12291350000000	9.88303113464811
7.19486690000000	9.88294353622950
7.27239570000000	9.88366721045329
7.34507170000000	9.88339106767175
7.41703630000000	9.88267728394947
7.49161380000000	9.88297039131218
7.54033520000000	9.88328591239364
7.61201630000000	9.88383056311586
7.66180400000000	9.88395418472313
7.73544550000000	9.88450896858125
7.78449240000000	9.88492094027926
7.85764700000000	9.88482509883562
7.91052170000000	9.88554747154348
7.95964560000000	9.88526967269044
8.03396260000000	9.88396986012288
8.08631400000000	9.88454466657089
8.13487540000000	9.88440706870236
8.20533190000000	9.88410726162167
8.27716440000000	9.88537623136692
8.32722930000000	9.88526177013436
8.37575200000000	9.88513488511455
8.44697930000000	9.88455948871221
8.49834420000000	9.88500001446127
8.54782680000000	9.88581688204021
8.59135000000000	9.88547962280208
8.64606410000000	9.88638480730877
8.71896220000000	9.88671844261789
8.76760220000000	9.88672189151196
8.84232420000000	9.88564232658142
8.89403000000000	9.88636944629376
8.96208750000000	9.88652332250006
9.01618560000000	9.88701438498379
9.08740180000000	9.88794904582407
9.14031210000000	9.88800463450190
9.18666600000000	9.88802570953548
9.25899220000000	9.88693921207781
9.30817080000000	9.88751945944912
9.38079650000000	9.88774650026722
9.43231620000000	9.88676642529981
9.48105900000000	9.88709747441328
9.55284690000000	9.88718176026377
9.62523490000000	9.88653238971820
9.67708060000000	9.88713549862664
9.72594439999999	9.88694380225373
9.80065199999999	9.88600366458789
9.84820459999999	9.88699250007494
9.92301879999999	9.88763961523242
9.97284370000000	9.88833666118645
10.0468063000000	9.88911305274714
10.0951053000000	9.88909481562979
10.1477399000000	9.88906469102792
10.2215896000000	9.88824428856417
10.2715403000000	9.88844100313354
10.3455588000000	9.88782238495115
10.3944275000000	9.88764719095041
10.4497091000000	9.88803710372762
10.5221349000000	9.88809520262285
10.5737620000000	9.88806480718701
10.6445610000000	9.88835471435620
10.6958525000000	9.88832188881768
10.7506991000000	9.88903273471112
10.8201071000000	9.88998935450628
10.8719025000000	9.89015668299742
10.9467901000000	9.89013789279075
10.9998535000000	9.88939868128284
11.0501985000000	9.88868671601433
11.1236846000000	9.88849868444975
11.1730506000000	9.88881137988122
11.2475298000000	9.88959085396270
11.2973975000000	9.88921361694387
11.3474696000000	9.88778999808915
11.4205668000000	9.88786346528368
11.4715709000000	9.88725976310050
11.5233796000000	9.88808137530112
11.5947958000000	9.88856681023927
11.6405942000000	9.88906198033087
11.7140368000000	9.88941570018407
11.7895098000000	9.88921181509629
11.8415404000000	9.88997529509602
11.8942799000000	9.88990380836490

 };

    \end{axis}

\end{tikzpicture}

%% file: FigPlot/CtrlAgl_15deg.tex
\begin{tikzpicture}[spy using outlines={rectangle, magnification=10,  connect spies}]
    \begin{axis}[
      xmin=0, xmax=3,
      ymin=-0.2, ymax=16,
      ymajorgrids=true,
      grid style=dashed,
      legend pos=north west,
      width = 0.53\linewidth,
      height = 0.4\linewidth,
    label style={font=\footnotesize},
      tick label style={font=\footnotesize},
    ]


    \addplot[color= blue, line width = 1pt] table {
0	0
0.0487526000000000	0.0624999898014306
0.0975052000000000	0.00697976272778576
0.143032500000000	0.620926587154643
0.188358900000000	1.18334881136406
0.264057500000000	1.82035964947818
0.312974000000000	2.87383839448342
0.367888900000000	3.65514035450863
0.413679000000000	4.73429452728438
0.487660900000000	5.56336430270483
0.559690000000000	6.65713334074370
0.613076500000000	7.63615363089761
0.667355400000000	8.16457638096568
0.714430100000000	8.55064039156516
0.760709700000000	8.75961455316978
0.812520500000000	8.92075995415537
0.885118100000000	9.12884862886772
0.933742300000000	9.62915523259447
1.00839010000000	9.98127130599690
1.05910620000000	10.5699171723658
1.11058830000000	10.7758219279761
1.18158730000000	10.9769418318110
1.22824010000000	11.1165381256361
1.29992460000000	11.2382954002159
1.34634940000000	11.5008673467401
1.41911340000000	11.9046715702167
1.47011770000000	12.2823429544541
1.51635420000000	12.5866709148795
1.61034270000000	13.0730359589099
1.65907990000000	13.2863539684912
1.70368670000000	13.2632492844417
1.77540000000000	13.0804043140855
1.82406180000000	13.0171318787452
1.90190150000000	13.1842989495800
1.94391300000000	13.4018906688975
1.99151400000000	13.6174483478891
2.04267460000000	13.6719317836834
2.11490280000000	13.8495163141163
2.16477620000000	14.1347129117762
2.21734950000000	14.2300662883895
2.28994270000000	14.1987467118617
2.34186410000000	14.2922383587639
2.41357060000000	14.3299323317035
2.46674570000000	14.2856022291409
2.51838540000000	14.3183109296176
2.59361540000000	14.3703687187940
2.64591030000000	14.4081879399383
2.71487140000000	14.4595972838647
2.76636870000000	14.5029498809852
2.81317750000000	14.5406651124787
2.88653810000000	14.5682881573280
2.95995780000000	14.6181050365077
3.01380470000000	14.6575656021366
3.06218510000000	14.6746931017090
3.12118210000000	14.6831364434522
3.16304890000000	14.6921234170735
3.21672770000000	14.7013553244975
3.29024150000000	14.7139321786167
3.34057050000000	14.7211600607198
3.38549450000000	14.7313813373828
3.45827470000000	14.7390581507227
3.50818850000000	14.7489850421057
3.58103730000000	14.7492386150037
3.63466050000000	14.7513797877887
3.70775960000000	14.7501898967421
3.75541010000000	14.7472267293976
3.82690840000000	14.7473613853777
3.87527990000000	14.7457359129246
3.94908590000000	14.7453923830002
4.00047170000000	14.7499029989120
4.05332180000000	14.7519803735987
4.12546960000000	14.7532864752476
4.17671490000000	14.7540731950236
4.25337940000000	14.7523067347690
4.30546980000000	14.7513721459186
4.35945060000000	14.7524839277700
4.40591480000000	14.7523951405784
4.47815420000000	14.7543374694659
4.55166240000000	14.7524364691204
4.62537950000000	14.7515302549628
4.69738750000000	14.7522795660466
4.74782250000000	14.7511123439756
4.79535100000000	14.7517643526320
4.84407110000000	14.7523457006414
4.89330690000000	14.7512343783820
4.93965880000000	14.7504193673728
4.98888100000000	14.7511912347835
5.04065200000000	14.7522881759060
5.11135020000000	14.7522222062398
5.16411360000000	14.7549551141871
5.23627660000000	14.7590310022495
5.28702340000000	14.7598466437506
5.36120890000000	14.7635072775787
5.41198520000000	14.7626185278457
5.46119770000000	14.7609225362756
5.53831220000000	14.7593397712032
5.58368630000000	14.7604199371599
5.65964570000000	14.7599517284325
5.70965880000000	14.7599370171347
5.75896700000000	14.7624214684433
5.83435670000000	14.7615148594201
5.88301900000000	14.7629122715827
5.95903220000000	14.7649198406987
6.00637610000000	14.7640066787758
6.06152310000000	14.7653387525803
6.13513020000000	14.7642347132645
6.18828360000000	14.7642406797971
6.23334970000000	14.7669940041717
6.30637400000000	14.7666914477580
6.37818300000000	14.7693372336368
6.45516590000000	14.7695713648976
6.52829560000000	14.7664630210842
6.57808080000000	14.7680170180383
6.65571870000000	14.7685040332878
6.72894010000000	14.7687913189880
6.78162450000000	14.7700568319193
6.82602720000000	14.7689467109131
6.87734800000000	14.7685221085540
6.92550690000000	14.7678677570986
6.99556950000000	14.7683357164855
7.07189430000000	14.7709702404609
7.12186130000000	14.7713336403592
7.17538820000000	14.7721140516537
7.22711230000000	14.7729007046866
7.28465490000000	14.7721585921086
7.35896510000000	14.7734660362297
7.41158300000000	14.7724493477180
7.48742610000000	14.7722770089093
7.55942780000000	14.7752456647420
7.60446490000000	14.7775174161533
7.65252850000000	14.7777775863369
7.72739220000000	14.7794258750911
7.77954060000000	14.7772229422692
7.82696940000000	14.7771394788862
7.90069550000000	14.7755837592043
7.95319590000000	14.7746374201491
8.02588460000000	14.7778730333527
8.07135520000000	14.7782850739757
8.11947290000000	14.7797759416631
8.17336240000000	14.7818346830164
8.22483720000000	14.7820356298750
8.29777350000000	14.7814486246965
8.34937560000000	14.7811632880778
8.40069140000000	14.7815567218672
8.47230060000000	14.7811226255208
8.52251060000000	14.7849535850096
8.56865020000000	14.7862926962484
8.62218070000000	14.7889707422270
8.66651990000000	14.7899339095782
8.73878860000000	14.7932954351708
8.79110310000000	14.7923052540799
8.84775960000000	14.7899583092380
8.91904310000000	14.7893863473223
8.97088090000000	14.7885665227339
9.04236380000000	14.7875429475257
9.09296950000000	14.7888269305372
9.16417790000000	14.7888461878756
9.23687210000000	14.7869986879144
9.28315240000000	14.7875054984446
9.32998230000000	14.7871500166642
9.40452220000000	14.7865017252413
9.45409540000000	14.7870464302259
9.52731380000000	14.7877837385770
9.57594700000000	14.7855966590233
9.63020150000000	14.7859507204082
9.70472110000000	14.7850397619815
9.75310700000000	14.7842958695599
9.82772210000000	14.7846037693823
9.87288530000000	14.7841266209126
9.93082120000000	14.7843282363180
9.97674460000000	14.7854263444129
10.0284970000000	14.7832958524710
10.0828089000000	14.7837949708839
10.1328649000000	14.7839680417287
10.2090181000000	14.7831441447983
10.2526942000000	14.7832197276869
10.3260745000000	14.7836737084765
10.3978214000000	14.7835603510440
10.4516214000000	14.7844909769654
10.5038901000000	14.7842194023633
10.5758648000000	14.7844090515178
10.6264299000000	14.7839562681512
10.6977790000000	14.7845264766347
10.7543630000000	14.7838481292738
10.8034700000000	14.7827280843269
10.8776999000000	14.7837172327639
10.9273997000000	14.7806528260874
10.9761203000000	14.7814564708778
11.0516203000000	14.7814710796747
11.0956500000000	14.7808414486395
11.1696673000000	14.7830297073708
11.2235787000000	14.7814069905103
11.2784846000000	14.7839556456388
11.3194659000000	14.7840810857150
11.3709830000000	14.7839849385035
11.4426842000000	14.7846198689237
11.4959322000000	14.7843154904973
11.5713826000000	14.7849051859739
11.6232485000000	14.7849624022110
11.6737394000000	14.7846936201251

 };

    \end{axis}

\end{tikzpicture}

%% file: FigPlot/10deg_ss.tex
\begin{tikzpicture}[spy using outlines={rectangle, magnification=10,  connect spies}]
    \begin{axis}[
      xmin=4, xmax=8,
      ymin=9.8, ymax=10,
      ymajorgrids=true,
      grid style=dashed,
      legend pos=north west,
      width = 0.53\linewidth,
      height = 0.4\linewidth,
      label style={font=\footnotesize},
      tick label style={font=\footnotesize},
     ylabel={Steady-state $q_i$ [deg]}
    ]


    \addplot[color= blue, line width = 1pt] table {

0	0
0.0568485000000000	0.00209126714784225
0.106636900000000	0.00743696399724719
0.161638500000000	0.388400158181980
0.237777400000000	1.37202235246266
0.290532700000000	2.02273944766424
0.365799200000000	3.22815339818914
0.438718600000000	4.60052548843545
0.487416100000000	5.17725573730891
0.535712000000000	5.61782213236252
0.586536000000000	5.97051030356368
0.635139300000000	6.14879899598734
0.706669000000000	6.26649709193724
0.757562000000000	6.50951727742588
0.830890900000000	6.70214474858734
0.881397700000000	7.30846043668172
0.950236100000000	7.58912643159160
1.00552990000000	8.10584513287051
1.07888820000000	8.24492781090661
1.12974440000000	8.36961111682759
1.17986130000000	8.38441201358182
1.25112780000000	8.46849376478978
1.32074750000000	8.64060081673111
1.39544250000000	8.83564928736188
1.47076090000000	8.96275436150149
1.52213480000000	9.05662710183979
1.57252550000000	9.07959334408097
1.64488990000000	9.11793513663683
1.69843880000000	9.23856670700627
1.75164320000000	9.30478327632848
1.82328790000000	9.33985652284568
1.86865240000000	9.42846269347486
1.91507870000000	9.44163541286136
1.96804670000000	9.46817283163400
2.02119760000000	9.52070343346727
2.06936960000000	9.56877028469705
2.11511750000000	9.58461903016030
2.17005090000000	9.61688893954348
2.24004790000000	9.65568146614171
2.29150660000000	9.66835414186884
2.36735690000000	9.71428891522744
2.41366020000000	9.71740982670630
2.46251270000000	9.72210335205928
2.53840660000000	9.73765872437619
2.58636250000000	9.75576345698985
2.65704820000000	9.71844188844325
2.72977950000000	9.79414747674658
2.78028050000000	9.81593301475552
2.83110010000000	9.82342477961994
2.90208190000000	9.77148798252749
2.95259440000000	9.77680489499167
3.00009410000000	9.77728861993960
3.05300540000000	9.83286153319490
3.09702660000000	9.83427654216623
3.17005090000000	9.79589718966077
3.21888970000000	9.81177971970041
3.26531320000000	9.81671538836751
3.31608710000000	9.83019048092551
3.36906870000000	9.83470148602969
3.44207910000000	9.83763212107118
3.49354800000000	9.84124350065025
3.56546340000000	9.84246508278187
3.61569280000000	9.84216777171383
3.68714190000000	9.84500255522507
3.73316150000000	9.84849175365776
3.80794530000000	9.85382895433311
3.85869200000000	9.85714792036309
3.91004490000000	9.85947408110933
3.98400180000000	9.86179485952621
4.03540950000000	9.86342681665611
4.10589140000000	9.86437634478516
4.15844440000000	9.86552491050980
4.20995290000000	9.86617856445928
4.28641990000000	9.86791054483868
4.33626890000000	9.86929665692849
4.38868640000000	9.87010192582634
4.43748430000000	9.87129100892221
4.48551760000000	9.87276727328427
4.53661490000000	9.87206652419971
4.61049910000000	9.87227698433492
4.66217680000000	9.87399375172196
4.73612090000000	9.87447018043791
4.78777920000000	9.87689834586926
4.83594500000000	9.87939665832561
4.90997300000000	9.88157788543642
4.95878910000000	9.88353463397251
5.03354700000000	9.88646833514496
5.08270270000000	9.88758444455975
5.13924270000000	9.88820900602906
5.20868620000000	9.88868262426106
5.26043240000000	9.88804367951682
5.32965990000000	9.88691688194161
5.38064340000000	9.88665084085617
5.45411100000000	9.88580633164090
5.50786260000000	9.88493523542328
5.55265470000000	9.88522649081629
5.62539540000000	9.88523177553109
5.67641810000000	9.88516663256903
5.75213750000000	9.88493202289858
5.80259620000000	9.88366436492078
5.87822660000000	9.88334760965168
5.92918960000000	9.88225446230326
5.97974560000000	9.88188547453757
6.05367320000000	9.88151202883586
6.10597540000000	9.88114807478523
6.15179810000000	9.88147000842142
6.22501070000000	9.88211426277518
6.29700710000000	9.88341768098753
6.37330250000000	9.88314194222977
6.44670880000000	9.88417435863539
6.49777170000000	9.88393454056586
6.56981430000000	9.88396977876387
6.64270070000000	9.88421875565212
6.71886240000000	9.88454616638480
6.76212200000000	9.88350216628520
6.81847900000000	9.88302318063962
6.89376010000000	9.88290351587376
6.94049880000000	9.88241380209900
7.01414720000000	9.88271309365023
7.06521790000000	9.88314262744850
7.12291350000000	9.88303113464811
7.19486690000000	9.88294353622950
7.27239570000000	9.88366721045329
7.34507170000000	9.88339106767175
7.41703630000000	9.88267728394947
7.49161380000000	9.88297039131218
7.54033520000000	9.88328591239364
7.61201630000000	9.88383056311586
7.66180400000000	9.88395418472313
7.73544550000000	9.88450896858125
7.78449240000000	9.88492094027926
7.85764700000000	9.88482509883562
7.91052170000000	9.88554747154348
7.95964560000000	9.88526967269044
8.03396260000000	9.88396986012288
8.08631400000000	9.88454466657089
8.13487540000000	9.88440706870236
8.20533190000000	9.88410726162167
8.27716440000000	9.88537623136692
8.32722930000000	9.88526177013436
8.37575200000000	9.88513488511455
8.44697930000000	9.88455948871221
8.49834420000000	9.88500001446127
8.54782680000000	9.88581688204021
8.59135000000000	9.88547962280208
8.64606410000000	9.88638480730877
8.71896220000000	9.88671844261789
8.76760220000000	9.88672189151196
8.84232420000000	9.88564232658142
8.89403000000000	9.88636944629376
8.96208750000000	9.88652332250006
9.01618560000000	9.88701438498379
9.08740180000000	9.88794904582407
9.14031210000000	9.88800463450190
9.18666600000000	9.88802570953548
9.25899220000000	9.88693921207781
9.30817080000000	9.88751945944912
9.38079650000000	9.88774650026722
9.43231620000000	9.88676642529981
9.48105900000000	9.88709747441328
9.55284690000000	9.88718176026377
9.62523490000000	9.88653238971820
9.67708060000000	9.88713549862664
9.72594439999999	9.88694380225373
9.80065199999999	9.88600366458789
9.84820459999999	9.88699250007494
9.92301879999999	9.88763961523242
9.97284370000000	9.88833666118645
10.0468063000000	9.88911305274714
10.0951053000000	9.88909481562979
10.1477399000000	9.88906469102792
10.2215896000000	9.88824428856417
10.2715403000000	9.88844100313354
10.3455588000000	9.88782238495115
10.3944275000000	9.88764719095041
10.4497091000000	9.88803710372762
10.5221349000000	9.88809520262285
10.5737620000000	9.88806480718701
10.6445610000000	9.88835471435620
10.6958525000000	9.88832188881768
10.7506991000000	9.88903273471112
10.8201071000000	9.88998935450628
10.8719025000000	9.89015668299742
10.9467901000000	9.89013789279075
10.9998535000000	9.88939868128284
11.0501985000000	9.88868671601433
11.1236846000000	9.88849868444975
11.1730506000000	9.88881137988122
11.2475298000000	9.88959085396270
11.2973975000000	9.88921361694387
11.3474696000000	9.88778999808915
11.4205668000000	9.88786346528368
11.4715709000000	9.88725976310050
11.5233796000000	9.88808137530112
11.5947958000000	9.88856681023927
11.6405942000000	9.88906198033087
11.7140368000000	9.88941570018407
11.7895098000000	9.88921181509629
11.8415404000000	9.88997529509602
11.8942799000000	9.88990380836490

 };

    \end{axis}

\end{tikzpicture}

%% file: FigPlot/15deg_ss.tex
\begin{tikzpicture}[spy using outlines={rectangle, magnification=10,  connect spies}]
    \begin{axis}[
      xmin=4, xmax=8,
      ymin=14.6, ymax=15,
      ymajorgrids=true,
      grid style=dashed,
      legend pos=north west,
      width = 0.53\linewidth,
      height = 0.4\linewidth,
    label style={font=\footnotesize},
      tick label style={font=\footnotesize},
    ]


    \addplot[color= blue, line width = 1pt] table {
0	0
0.0487526000000000	0.0624999898014306
0.0975052000000000	0.00697976272778576
0.143032500000000	0.620926587154643
0.188358900000000	1.18334881136406
0.264057500000000	1.82035964947818
0.312974000000000	2.87383839448342
0.367888900000000	3.65514035450863
0.413679000000000	4.73429452728438
0.487660900000000	5.56336430270483
0.559690000000000	6.65713334074370
0.613076500000000	7.63615363089761
0.667355400000000	8.16457638096568
0.714430100000000	8.55064039156516
0.760709700000000	8.75961455316978
0.812520500000000	8.92075995415537
0.885118100000000	9.12884862886772
0.933742300000000	9.62915523259447
1.00839010000000	9.98127130599690
1.05910620000000	10.5699171723658
1.11058830000000	10.7758219279761
1.18158730000000	10.9769418318110
1.22824010000000	11.1165381256361
1.29992460000000	11.2382954002159
1.34634940000000	11.5008673467401
1.41911340000000	11.9046715702167
1.47011770000000	12.2823429544541
1.51635420000000	12.5866709148795
1.61034270000000	13.0730359589099
1.65907990000000	13.2863539684912
1.70368670000000	13.2632492844417
1.77540000000000	13.0804043140855
1.82406180000000	13.0171318787452
1.90190150000000	13.1842989495800
1.94391300000000	13.4018906688975
1.99151400000000	13.6174483478891
2.04267460000000	13.6719317836834
2.11490280000000	13.8495163141163
2.16477620000000	14.1347129117762
2.21734950000000	14.2300662883895
2.28994270000000	14.1987467118617
2.34186410000000	14.2922383587639
2.41357060000000	14.3299323317035
2.46674570000000	14.2856022291409
2.51838540000000	14.3183109296176
2.59361540000000	14.3703687187940
2.64591030000000	14.4081879399383
2.71487140000000	14.4595972838647
2.76636870000000	14.5029498809852
2.81317750000000	14.5406651124787
2.88653810000000	14.5682881573280
2.95995780000000	14.6181050365077
3.01380470000000	14.6575656021366
3.06218510000000	14.6746931017090
3.12118210000000	14.6831364434522
3.16304890000000	14.6921234170735
3.21672770000000	14.7013553244975
3.29024150000000	14.7139321786167
3.34057050000000	14.7211600607198
3.38549450000000	14.7313813373828
3.45827470000000	14.7390581507227
3.50818850000000	14.7489850421057
3.58103730000000	14.7492386150037
3.63466050000000	14.7513797877887
3.70775960000000	14.7501898967421
3.75541010000000	14.7472267293976
3.82690840000000	14.7473613853777
3.87527990000000	14.7457359129246
3.94908590000000	14.7453923830002
4.00047170000000	14.7499029989120
4.05332180000000	14.7519803735987
4.12546960000000	14.7532864752476
4.17671490000000	14.7540731950236
4.25337940000000	14.7523067347690
4.30546980000000	14.7513721459186
4.35945060000000	14.7524839277700
4.40591480000000	14.7523951405784
4.47815420000000	14.7543374694659
4.55166240000000	14.7524364691204
4.62537950000000	14.7515302549628
4.69738750000000	14.7522795660466
4.74782250000000	14.7511123439756
4.79535100000000	14.7517643526320
4.84407110000000	14.7523457006414
4.89330690000000	14.7512343783820
4.93965880000000	14.7504193673728
4.98888100000000	14.7511912347835
5.04065200000000	14.7522881759060
5.11135020000000	14.7522222062398
5.16411360000000	14.7549551141871
5.23627660000000	14.7590310022495
5.28702340000000	14.7598466437506
5.36120890000000	14.7635072775787
5.41198520000000	14.7626185278457
5.46119770000000	14.7609225362756
5.53831220000000	14.7593397712032
5.58368630000000	14.7604199371599
5.65964570000000	14.7599517284325
5.70965880000000	14.7599370171347
5.75896700000000	14.7624214684433
5.83435670000000	14.7615148594201
5.88301900000000	14.7629122715827
5.95903220000000	14.7649198406987
6.00637610000000	14.7640066787758
6.06152310000000	14.7653387525803
6.13513020000000	14.7642347132645
6.18828360000000	14.7642406797971
6.23334970000000	14.7669940041717
6.30637400000000	14.7666914477580
6.37818300000000	14.7693372336368
6.45516590000000	14.7695713648976
6.52829560000000	14.7664630210842
6.57808080000000	14.7680170180383
6.65571870000000	14.7685040332878
6.72894010000000	14.7687913189880
6.78162450000000	14.7700568319193
6.82602720000000	14.7689467109131
6.87734800000000	14.7685221085540
6.92550690000000	14.7678677570986
6.99556950000000	14.7683357164855
7.07189430000000	14.7709702404609
7.12186130000000	14.7713336403592
7.17538820000000	14.7721140516537
7.22711230000000	14.7729007046866
7.28465490000000	14.7721585921086
7.35896510000000	14.7734660362297
7.41158300000000	14.7724493477180
7.48742610000000	14.7722770089093
7.55942780000000	14.7752456647420
7.60446490000000	14.7775174161533
7.65252850000000	14.7777775863369
7.72739220000000	14.7794258750911
7.77954060000000	14.7772229422692
7.82696940000000	14.7771394788862
7.90069550000000	14.7755837592043
7.95319590000000	14.7746374201491
8.02588460000000	14.7778730333527
8.07135520000000	14.7782850739757
8.11947290000000	14.7797759416631
8.17336240000000	14.7818346830164
8.22483720000000	14.7820356298750
8.29777350000000	14.7814486246965
8.34937560000000	14.7811632880778
8.40069140000000	14.7815567218672
8.47230060000000	14.7811226255208
8.52251060000000	14.7849535850096
8.56865020000000	14.7862926962484
8.62218070000000	14.7889707422270
8.66651990000000	14.7899339095782
8.73878860000000	14.7932954351708
8.79110310000000	14.7923052540799
8.84775960000000	14.7899583092380
8.91904310000000	14.7893863473223
8.97088090000000	14.7885665227339
9.04236380000000	14.7875429475257
9.09296950000000	14.7888269305372
9.16417790000000	14.7888461878756
9.23687210000000	14.7869986879144
9.28315240000000	14.7875054984446
9.32998230000000	14.7871500166642
9.40452220000000	14.7865017252413
9.45409540000000	14.7870464302259
9.52731380000000	14.7877837385770
9.57594700000000	14.7855966590233
9.63020150000000	14.7859507204082
9.70472110000000	14.7850397619815
9.75310700000000	14.7842958695599
9.82772210000000	14.7846037693823
9.87288530000000	14.7841266209126
9.93082120000000	14.7843282363180
9.97674460000000	14.7854263444129
10.0284970000000	14.7832958524710
10.0828089000000	14.7837949708839
10.1328649000000	14.7839680417287
10.2090181000000	14.7831441447983
10.2526942000000	14.7832197276869
10.3260745000000	14.7836737084765
10.3978214000000	14.7835603510440
10.4516214000000	14.7844909769654
10.5038901000000	14.7842194023633
10.5758648000000	14.7844090515178
10.6264299000000	14.7839562681512
10.6977790000000	14.7845264766347
10.7543630000000	14.7838481292738
10.8034700000000	14.7827280843269
10.8776999000000	14.7837172327639
10.9273997000000	14.7806528260874
10.9761203000000	14.7814564708778
11.0516203000000	14.7814710796747
11.0956500000000	14.7808414486395
11.1696673000000	14.7830297073708
11.2235787000000	14.7814069905103
11.2784846000000	14.7839556456388
11.3194659000000	14.7840810857150
11.3709830000000	14.7839849385035
11.4426842000000	14.7846198689237
11.4959322000000	14.7843154904973
11.5713826000000	14.7849051859739
11.6232485000000	14.7849624022110
11.6737394000000	14.7846936201251

 };

    \end{axis}

\end{tikzpicture}

%% file: FigPlot/Inputs_10_deg.tex
\begin{tikzpicture}[spy using outlines={rectangle, magnification=10,  connect spies}]
    \begin{axis}[
      xmin=0, xmax=3,
      ymin=200, ymax=310,
      ymajorgrids=true,
      grid style=dashed,
      legend pos=north west,
      width = 0.53\linewidth,
      height = 0.4\linewidth,
      label style={font=\footnotesize},
      tick label style={font=\footnotesize},
      xlabel={Time [s]},
    ylabel={Lengths $L_i$ [mm]},
    legend columns=2 
    ]
        
 \legend{\footnotesize{$L_1$}, \footnotesize{$L_2$}}

    \addplot[color= red, line width = 1pt] table[x=time, y=u2] {time u1 u2
0	252	252
0.0568485000000000	251.852327728272	248.858407346410
0.106636900000000	251.303831634522	247.508920599568
0.161638500000000	250.101359710693	245.896756005622
0.237777400000000	249.004367523193	244.297115392368
0.290532700000000	249.436836090088	242.201184655130
0.365799200000000	251.725751953125	240.844509333739
0.438718600000000	253.835351715088	238.955381470670
0.487416100000000	255.333167953491	237.508981543380
0.535712000000000	256.841531982422	236.445501777684
0.586536000000000	258.360443916321	235.339246706759
0.635139300000000	259.847711906433	234.289766201805
0.706669000000000	261.936215972900	233.269882471048
0.757562000000000	263.339100036621	232.093837788523
0.830890900000000	264.710339813232	231.210316667401
0.881397700000000	264.509928131104	230.163293309729
0.950236100000000	265.005683898926	229.488835867088
1.00552990000000	266.640623931885	228.733169968180
1.07888820000000	268.644743499756	228.209881818584
1.12974440000000	270.037080230713	227.683982424680
1.17986130000000	270.342971649170	227.248495858049
1.25112780000000	269.604611663818	226.799578802369
1.32074750000000	269.035019989014	226.340786344027
1.39544250000000	269.319815826416	225.941292136606
1.47076090000000	269.963243865967	225.612293762646
1.52213480000000	270.279683990479	225.325235253267
1.57252550000000	270.501192169189	225.091025072534
1.64488990000000	270.279683990479	224.855636370559
1.69843880000000	270.047628021240	224.625078040175
1.75164320000000	270.216396331787	224.430922343045
1.82328790000000	270.648863983154	224.240761757053
1.86865240000000	270.553932037354	224.083015693975
1.91507870000000	270.785988006592	223.952630707899
1.96804670000000	270.849275665283	223.830191891307
2.02119760000000	270.669959564209	223.701404843855
2.06936960000000	270.764891510010	223.591334332808
2.11511750000000	270.933659820557	223.500955243150
2.17005090000000	270.933659820557	223.411001412054
2.24004790000000	270.965304107666	223.326206441751
2.29150660000000	271.060236053467	223.264558024190
2.36735690000000	271.039139556885	223.215797563467
2.41366020000000	271.155167999268	223.166939823901
2.46251270000000	271.165715789795	223.115182250049
2.53840660000000	271.070783843994	223.074465602417
2.58636250000000	271.091880340576	223.037157642538
2.65704820000000	271.165715789795	223.000654249926
2.72977950000000	270.838727874756	222.957285046260
2.78028050000000	270.996947479248	222.938877496155
2.83110010000000	271.418868255615	222.924624212662
2.90208190000000	271.450511627197	222.916430945008
2.95259440000000	271.239552154541	222.889567643989
3.00009410000000	270.965304107666	222.870140589193
3.05300540000000	270.933659820557	222.849198805867
3.09702660000000	271.186812286377	222.835013264360
3.17005090000000	271.345031890869	222.824907265430
3.21888970000000	271.197360076904	222.802303161241
3.26531320000000	271.091880340576	222.786195393101
3.31608710000000	271.102428131104	222.766652229144
3.36906870000000	271.176263580322	222.770885554323
3.44207910000000	271.229004364014	222.772302728372
3.49354800000000	271.229004364014	222.762676539524
3.56546340000000	271.239552154541	222.753260550770
3.61569280000000	271.218455657959	222.743097447115
3.68714190000000	271.186812286377	222.732457169088
3.73316150000000	271.197360076904	222.733347742594
3.80794530000000	271.229004364014	222.723893369500
3.85869200000000	271.229004364014	222.704476350544
3.91004490000000	271.250099945068	222.705519034473
3.98400180000000	271.260647735596	222.706249819425
4.03540950000000	271.271195526123	222.706978913467
4.10589140000000	271.281744232178	222.707491607920
4.15844440000000	271.271195526123	222.697239373870
4.20995290000000	271.271195526123	222.697600206435
4.28641990000000	271.271195526123	222.697805557879
4.33626890000000	271.281744232178	222.698349675563
4.38868640000000	271.281744232178	222.688238260519
4.43748430000000	271.281744232178	222.688491243204
4.48551760000000	271.292292022705	222.688864804676
4.53661490000000	271.292292022705	222.678781711803
4.61049910000000	271.281744232178	222.678561564986
4.66217680000000	271.292292022705	222.678627682987
4.73612090000000	271.292292022705	222.668616484259
4.78777920000000	271.271195526123	222.658219283754
4.83594500000000	271.292292022705	222.658982114422
4.90997300000000	271.302839813232	222.649220107428
4.95878910000000	271.313387603760	222.649905360135
5.03354700000000	271.323936309814	222.650520090817
5.08270270000000	271.345031890869	222.651441740223
5.13924270000000	271.355579681397	222.651792376336
5.20868620000000	271.355579681397	222.651988588109
5.26043240000000	271.355579681397	222.652137379665
5.32965990000000	271.345031890869	222.651936649253
5.38064340000000	271.345031890869	222.651582655355
5.45411100000000	271.334484100342	222.651499076083
5.50786260000000	271.334484100342	222.651233765688
5.55265470000000	271.323936309814	222.650960102740
5.62539540000000	271.323936309814	222.651051603321
5.67641810000000	271.323936309814	222.651053263563
5.75213750000000	271.313387603760	222.651032798298
5.80259620000000	271.313387603760	222.650959093496
5.87822660000000	271.302839813232	222.650560846997
5.92918960000000	271.313387603760	222.650461335394
5.97974560000000	271.292292022705	222.650117913026
6.05367320000000	271.281744232178	222.639455117101
6.10597540000000	271.292292022705	222.639337795674
6.15179810000000	271.281744232178	222.639223456136
6.22501070000000	271.281744232178	222.639324594571
6.29700710000000	271.271195526123	222.628976455936
6.37330250000000	271.292292022705	222.629385936844
6.44670880000000	271.281744232178	222.629299310959
6.49777170000000	271.292292022705	222.629623654138
6.56981430000000	271.292292022705	222.629548313070
6.64270070000000	271.292292022705	222.629559383476
6.71886240000000	271.302839813232	222.629637601872
6.76212200000000	271.302839813232	222.629740460988
6.81847900000000	271.292292022705	222.629412478683
6.89376010000000	271.281744232178	222.629262000905
6.94049880000000	271.281744232178	222.629224407110
7.01414720000000	271.271195526123	222.629070558990
7.06521790000000	271.271195526123	222.618617709204
7.12291350000000	271.271195526123	222.618752651227
7.19486690000000	271.292292022705	222.629264499730
7.27239570000000	271.271195526123	222.618690104875
7.34507170000000	271.281744232178	222.618917453838
7.41703630000000	271.271195526123	222.618830701025
7.49161380000000	271.260647735596	222.608059584255
7.54033520000000	271.260647735596	222.608151666648
7.61201630000000	271.260647735596	222.608250790520
7.66180400000000	271.271195526123	222.608421897590
7.73544550000000	271.271195526123	222.608460734464
7.78449240000000	271.271195526123	222.608635024953
7.85764700000000	271.292292022705	222.608764449679
7.91052170000000	271.271195526123	222.608734340201
7.95964560000000	271.281744232178	222.608961280281
8.03396260000000	271.281744232178	222.608874007197
8.08631400000000	271.271195526123	222.608465659036
8.13487540000000	271.271195526123	222.608646239807
8.20533190000000	271.271195526123	222.608603012162
8.27716440000000	271.271195526123	222.608508824990
8.32722930000000	271.281744232178	222.608907483593
8.37575200000000	271.281744232178	222.598320987427
8.44697930000000	271.260647735596	222.598281125322
8.49834420000000	271.271195526123	222.598100359211
8.54782680000000	271.260647735596	222.598238754457
8.59135000000000	271.271195526123	222.598495380975
8.64606410000000	271.271195526123	222.598389427861
8.71896220000000	271.271195526123	222.598673799960
8.76760220000000	271.281744232178	222.598778614584
8.84232420000000	271.281744232178	222.598779698086
8.89403000000000	271.281744232178	222.598440542760
8.96208750000000	271.271195526123	222.598668974155
9.01618560000000	271.271195526123	222.598717315791
9.08740180000000	271.271195526123	222.588324712620
9.14031210000000	271.271195526123	222.588618344983
9.18666600000000	271.281744232178	222.599182683681
9.25899220000000	271.281744232178	222.599189304598
9.30817080000000	271.281744232178	222.598847971355
9.38079650000000	271.292292022705	222.599030261443
9.43231620000000	271.281744232178	222.599101588420
9.48105900000000	271.271195526123	222.598793688788
9.55284690000000	271.292292022705	222.598897690934
9.62523490000000	271.281744232178	222.598924170115
9.67708060000000	271.271195526123	222.588173289342
9.72594439999999	271.281744232178	222.588362761593
9.80065199999999	271.239552154541	222.577755663401
9.84820459999999	271.271195526123	222.588007185443
9.92301879999999	271.281744232178	222.588317837273
9.97284370000000	271.271195526123	222.588521134496
10.0468063000000	271.292292022705	222.588740117941
10.0951053000000	271.281744232178	222.588984028543
10.1477399000000	271.281744232178	222.588978299183
10.2215896000000	271.281744232178	222.588968835261
10.2715403000000	271.281744232178	222.588711098225
10.3455588000000	271.260647735596	222.578226022930
10.3944275000000	271.260647735596	222.578031678296
10.4497091000000	271.260647735596	222.577976639478
10.5221349000000	271.260647735596	222.578099134189
10.5737620000000	271.260647735596	222.578117386496
10.6445610000000	271.260647735596	222.578107837488
10.6958525000000	271.260647735596	222.578198914511
10.7506991000000	271.260647735596	222.578188602064
10.8201071000000	271.281744232178	222.578411920888
10.8719025000000	271.271195526123	222.578712451860
10.9467901000000	271.281744232178	222.578765019656
10.9998535000000	271.281744232178	222.578759116538
11.0501985000000	271.281744232178	222.578526886394
11.1236846000000	271.271195526123	222.578303215908
11.1730506000000	271.271195526123	222.578244144050
11.2475298000000	271.271195526123	222.578342380217
11.2973975000000	271.271195526123	222.578587259222
11.3474696000000	271.271195526123	222.578468746717
11.4205668000000	271.260647735596	222.578021503663
11.4715709000000	271.271195526123	222.578044584063
11.5233796000000	271.260647735596	222.577854925429
11.5947958000000	271.260647735596	222.578113042514
11.6405942000000	271.260647735596	222.578265546398
11.7140368000000	271.281744232178	222.578421108670
11.7895098000000	271.271195526123	222.578532233039
11.8415404000000	271.271195526123	222.578468180650
11.8942799000000	271.281744232178	222.578708034966

};

  \addplot[color= black, line width = 1pt] table[x=time, y=u1] {time u1 u2
0	252	252
0.0568485000000000	251.852327728272	248.858407346410
0.106636900000000	251.303831634522	247.508920599568
0.161638500000000	250.101359710693	245.896756005622
0.237777400000000	249.004367523193	244.297115392368
0.290532700000000	249.436836090088	242.201184655130
0.365799200000000	251.725751953125	240.844509333739
0.438718600000000	253.835351715088	238.955381470670
0.487416100000000	255.333167953491	237.508981543380
0.535712000000000	256.841531982422	236.445501777684
0.586536000000000	258.360443916321	235.339246706759
0.635139300000000	259.847711906433	234.289766201805
0.706669000000000	261.936215972900	233.269882471048
0.757562000000000	263.339100036621	232.093837788523
0.830890900000000	264.710339813232	231.210316667401
0.881397700000000	264.509928131104	230.163293309729
0.950236100000000	265.005683898926	229.488835867088
1.00552990000000	266.640623931885	228.733169968180
1.07888820000000	268.644743499756	228.209881818584
1.12974440000000	270.037080230713	227.683982424680
1.17986130000000	270.342971649170	227.248495858049
1.25112780000000	269.604611663818	226.799578802369
1.32074750000000	269.035019989014	226.340786344027
1.39544250000000	269.319815826416	225.941292136606
1.47076090000000	269.963243865967	225.612293762646
1.52213480000000	270.279683990479	225.325235253267
1.57252550000000	270.501192169189	225.091025072534
1.64488990000000	270.279683990479	224.855636370559
1.69843880000000	270.047628021240	224.625078040175
1.75164320000000	270.216396331787	224.430922343045
1.82328790000000	270.648863983154	224.240761757053
1.86865240000000	270.553932037354	224.083015693975
1.91507870000000	270.785988006592	223.952630707899
1.96804670000000	270.849275665283	223.830191891307
2.02119760000000	270.669959564209	223.701404843855
2.06936960000000	270.764891510010	223.591334332808
2.11511750000000	270.933659820557	223.500955243150
2.17005090000000	270.933659820557	223.411001412054
2.24004790000000	270.965304107666	223.326206441751
2.29150660000000	271.060236053467	223.264558024190
2.36735690000000	271.039139556885	223.215797563467
2.41366020000000	271.155167999268	223.166939823901
2.46251270000000	271.165715789795	223.115182250049
2.53840660000000	271.070783843994	223.074465602417
2.58636250000000	271.091880340576	223.037157642538
2.65704820000000	271.165715789795	223.000654249926
2.72977950000000	270.838727874756	222.957285046260
2.78028050000000	270.996947479248	222.938877496155
2.83110010000000	271.418868255615	222.924624212662
2.90208190000000	271.450511627197	222.916430945008
2.95259440000000	271.239552154541	222.889567643989
3.00009410000000	270.965304107666	222.870140589193
3.05300540000000	270.933659820557	222.849198805867
3.09702660000000	271.186812286377	222.835013264360
3.17005090000000	271.345031890869	222.824907265430
3.21888970000000	271.197360076904	222.802303161241
3.26531320000000	271.091880340576	222.786195393101
3.31608710000000	271.102428131104	222.766652229144
3.36906870000000	271.176263580322	222.770885554323
3.44207910000000	271.229004364014	222.772302728372
3.49354800000000	271.229004364014	222.762676539524
3.56546340000000	271.239552154541	222.753260550770
3.61569280000000	271.218455657959	222.743097447115
3.68714190000000	271.186812286377	222.732457169088
3.73316150000000	271.197360076904	222.733347742594
3.80794530000000	271.229004364014	222.723893369500
3.85869200000000	271.229004364014	222.704476350544
3.91004490000000	271.250099945068	222.705519034473
3.98400180000000	271.260647735596	222.706249819425
4.03540950000000	271.271195526123	222.706978913467
4.10589140000000	271.281744232178	222.707491607920
4.15844440000000	271.271195526123	222.697239373870
4.20995290000000	271.271195526123	222.697600206435
4.28641990000000	271.271195526123	222.697805557879
4.33626890000000	271.281744232178	222.698349675563
4.38868640000000	271.281744232178	222.688238260519
4.43748430000000	271.281744232178	222.688491243204
4.48551760000000	271.292292022705	222.688864804676
4.53661490000000	271.292292022705	222.678781711803
4.61049910000000	271.281744232178	222.678561564986
4.66217680000000	271.292292022705	222.678627682987
4.73612090000000	271.292292022705	222.668616484259
4.78777920000000	271.271195526123	222.658219283754
4.83594500000000	271.292292022705	222.658982114422
4.90997300000000	271.302839813232	222.649220107428
4.95878910000000	271.313387603760	222.649905360135
5.03354700000000	271.323936309814	222.650520090817
5.08270270000000	271.345031890869	222.651441740223
5.13924270000000	271.355579681397	222.651792376336
5.20868620000000	271.355579681397	222.651988588109
5.26043240000000	271.355579681397	222.652137379665
5.32965990000000	271.345031890869	222.651936649253
5.38064340000000	271.345031890869	222.651582655355
5.45411100000000	271.334484100342	222.651499076083
5.50786260000000	271.334484100342	222.651233765688
5.55265470000000	271.323936309814	222.650960102740
5.62539540000000	271.323936309814	222.651051603321
5.67641810000000	271.323936309814	222.651053263563
5.75213750000000	271.313387603760	222.651032798298
5.80259620000000	271.313387603760	222.650959093496
5.87822660000000	271.302839813232	222.650560846997
5.92918960000000	271.313387603760	222.650461335394
5.97974560000000	271.292292022705	222.650117913026
6.05367320000000	271.281744232178	222.639455117101
6.10597540000000	271.292292022705	222.639337795674
6.15179810000000	271.281744232178	222.639223456136
6.22501070000000	271.281744232178	222.639324594571
6.29700710000000	271.271195526123	222.628976455936
6.37330250000000	271.292292022705	222.629385936844
6.44670880000000	271.281744232178	222.629299310959
6.49777170000000	271.292292022705	222.629623654138
6.56981430000000	271.292292022705	222.629548313070
6.64270070000000	271.292292022705	222.629559383476
6.71886240000000	271.302839813232	222.629637601872
6.76212200000000	271.302839813232	222.629740460988
6.81847900000000	271.292292022705	222.629412478683
6.89376010000000	271.281744232178	222.629262000905
6.94049880000000	271.281744232178	222.629224407110
7.01414720000000	271.271195526123	222.629070558990
7.06521790000000	271.271195526123	222.618617709204
7.12291350000000	271.271195526123	222.618752651227
7.19486690000000	271.292292022705	222.629264499730
7.27239570000000	271.271195526123	222.618690104875
7.34507170000000	271.281744232178	222.618917453838
7.41703630000000	271.271195526123	222.618830701025
7.49161380000000	271.260647735596	222.608059584255
7.54033520000000	271.260647735596	222.608151666648
7.61201630000000	271.260647735596	222.608250790520
7.66180400000000	271.271195526123	222.608421897590
7.73544550000000	271.271195526123	222.608460734464
7.78449240000000	271.271195526123	222.608635024953
7.85764700000000	271.292292022705	222.608764449679
7.91052170000000	271.271195526123	222.608734340201
7.95964560000000	271.281744232178	222.608961280281
8.03396260000000	271.281744232178	222.608874007197
8.08631400000000	271.271195526123	222.608465659036
8.13487540000000	271.271195526123	222.608646239807
8.20533190000000	271.271195526123	222.608603012162
8.27716440000000	271.271195526123	222.608508824990
8.32722930000000	271.281744232178	222.608907483593
8.37575200000000	271.281744232178	222.598320987427
8.44697930000000	271.260647735596	222.598281125322
8.49834420000000	271.271195526123	222.598100359211
8.54782680000000	271.260647735596	222.598238754457
8.59135000000000	271.271195526123	222.598495380975
8.64606410000000	271.271195526123	222.598389427861
8.71896220000000	271.271195526123	222.598673799960
8.76760220000000	271.281744232178	222.598778614584
8.84232420000000	271.281744232178	222.598779698086
8.89403000000000	271.281744232178	222.598440542760
8.96208750000000	271.271195526123	222.598668974155
9.01618560000000	271.271195526123	222.598717315791
9.08740180000000	271.271195526123	222.588324712620
9.14031210000000	271.271195526123	222.588618344983
9.18666600000000	271.281744232178	222.599182683681
9.25899220000000	271.281744232178	222.599189304598
9.30817080000000	271.281744232178	222.598847971355
9.38079650000000	271.292292022705	222.599030261443
9.43231620000000	271.281744232178	222.599101588420
9.48105900000000	271.271195526123	222.598793688788
9.55284690000000	271.292292022705	222.598897690934
9.62523490000000	271.281744232178	222.598924170115
9.67708060000000	271.271195526123	222.588173289342
9.72594439999999	271.281744232178	222.588362761593
9.80065199999999	271.239552154541	222.577755663401
9.84820459999999	271.271195526123	222.588007185443
9.92301879999999	271.281744232178	222.588317837273
9.97284370000000	271.271195526123	222.588521134496
10.0468063000000	271.292292022705	222.588740117941
10.0951053000000	271.281744232178	222.588984028543
10.1477399000000	271.281744232178	222.588978299183
10.2215896000000	271.281744232178	222.588968835261
10.2715403000000	271.281744232178	222.588711098225
10.3455588000000	271.260647735596	222.578226022930
10.3944275000000	271.260647735596	222.578031678296
10.4497091000000	271.260647735596	222.577976639478
10.5221349000000	271.260647735596	222.578099134189
10.5737620000000	271.260647735596	222.578117386496
10.6445610000000	271.260647735596	222.578107837488
10.6958525000000	271.260647735596	222.578198914511
10.7506991000000	271.260647735596	222.578188602064
10.8201071000000	271.281744232178	222.578411920888
10.8719025000000	271.271195526123	222.578712451860
10.9467901000000	271.281744232178	222.578765019656
10.9998535000000	271.281744232178	222.578759116538
11.0501985000000	271.281744232178	222.578526886394
11.1236846000000	271.271195526123	222.578303215908
11.1730506000000	271.271195526123	222.578244144050
11.2475298000000	271.271195526123	222.578342380217
11.2973975000000	271.271195526123	222.578587259222
11.3474696000000	271.271195526123	222.578468746717
11.4205668000000	271.260647735596	222.578021503663
11.4715709000000	271.271195526123	222.578044584063
11.5233796000000	271.260647735596	222.577854925429
11.5947958000000	271.260647735596	222.578113042514
11.6405942000000	271.260647735596	222.578265546398
11.7140368000000	271.281744232178	222.578421108670
11.7895098000000	271.271195526123	222.578532233039
11.8415404000000	271.271195526123	222.578468180650
11.8942799000000	271.281744232178	222.578708034966

};

   \end{axis}

\end{tikzpicture}

%% file: FigPlot/Inputs_15_deg.tex
\begin{tikzpicture}[spy using outlines={rectangle, magnification=10,  connect spies}]
    \begin{axis}[
      xmin=0, xmax=3,
      ymin=200, ymax=330,
      ymajorgrids=true,
      grid style=dashed,
      legend pos=north west,
      width = 0.53\linewidth,
      height = 0.4\linewidth,
    label style={font=\footnotesize},
      tick label style={font=\footnotesize},
    xlabel={Time [s]},
 legend columns=2 
    ]
        
 \legend{\footnotesize{$L_1$}, \footnotesize{$L_2$}}

    \addplot[color= red, line width = 1pt] table[x=time, y=u2] {
time u1 u2
0	252	252
0.0487526000000000	252.801647644043	247.287611019615
0.0975052000000000	253.160279846191	245.303125944862
0.143032500000000	251.873423767090	243.735127415662
0.188358900000000	251.092871704102	242.461833234345
0.264057500000000	252.801647644043	241.267283152749
0.312974000000000	253.961927490234	239.115202030185
0.367888900000000	254.774123840332	237.895606718865
0.413679000000000	256.187555236816	236.485022981346
0.487660900000000	258.529211425781	235.410617866989
0.559690000000000	260.807579498291	233.392709288293
0.613076500000000	262.453067779541	231.542342590756
0.667355400000000	264.172391738892	230.278259884906
0.714430100000000	265.606919631958	228.788231800562
0.760709700000000	267.083639688492	227.538277152138
0.812520500000000	268.623647689819	226.243236792906
0.885118100000000	270.722699890137	224.827689017993
0.933742300000000	272.083391876221	223.110449698483
1.00839010000000	274.351211242676	222.107348553954
1.05910620000000	275.859575500488	220.561930262536
1.11058830000000	277.420679626465	219.639319752971
1.18158730000000	279.488087768555	218.480440965899
1.22824010000000	280.363571777344	217.256766736981
1.29992460000000	279.772883605957	216.372398329143
1.34634940000000	279.329867248535	215.208177581249
1.41911340000000	279.298223876953	214.415184813338
1.47011770000000	280.004939575195	213.476692089035
1.51635420000000	281.333987731934	212.793694325138
1.61034270000000	284.203044433594	212.161490494317
1.65907990000000	285.405516357422	211.586475299949
1.70368670000000	286.702919311523	211.157733072486
1.77540000000000	287.557307739258	210.728555576606
1.82406180000000	286.639632568359	210.164808203404
1.90190150000000	284.255784301758	209.659723062069
1.94391300000000	283.317011718750	209.132099765375
1.99151400000000	283.464684448242	208.789086149751
2.04267460000000	284.899211425781	208.455984058623
2.11490280000000	286.186067504883	208.114466501583
2.16477620000000	286.386480102539	207.801075419000
2.21734950000000	287.377992553711	207.616424525728
2.28994270000000	289.487590942383	207.435421199819
2.34186410000000	289.603619384766	207.246266679118
2.41357060000000	289.097316284180	207.096322760679
2.46674570000000	289.139507446289	206.949943235980
2.51838540000000	288.770328369141	206.788345625088
2.59361540000000	288.000324096680	206.640399930854
2.64591030000000	288.126899414063	206.509083469220
2.71487140000000	288.285119018555	206.383840704514
2.76636870000000	288.411696166992	206.294511689906
2.81317750000000	288.654299926758	206.213198448641
2.88653810000000	288.917999267578	206.140664713841
2.95995780000000	288.970739135742	206.064956762989
3.01380470000000	289.107863159180	206.006771747936
3.06218510000000	289.392659912109	205.966430613136
3.12118210000000	289.529783935547	205.929616551600
3.16304890000000	289.455948486328	205.890077943530
3.21672770000000	289.403206787109	205.861256997451
3.29024150000000	289.371564331055	205.821966124596
3.34057050000000	289.403206787109	205.794272972737
3.38549450000000	289.371564331055	205.775446266759
3.45827470000000	289.434851074219	205.768110500527
3.50818850000000	289.434851074219	205.759975367566
3.58103730000000	289.519235229492	205.763093992470
3.63466050000000	289.445399780273	205.752623117636
3.70775960000000	289.424304199219	205.753295786905
3.75541010000000	289.382111206055	205.752921971608
3.82690840000000	289.318824462891	205.751991065132
3.87527990000000	289.276631469727	205.730939618556
3.94908590000000	289.244987182617	205.730428961324
4.00047170000000	289.213342895508	205.719774163216
4.05332180000000	289.276631469727	205.721191214997
4.12546960000000	289.276631469727	205.721843841502
4.17671490000000	289.276631469727	205.711703628327
4.25337940000000	289.266084594727	205.711950783634
4.30546980000000	289.244987182617	205.711395833778
4.35945060000000	289.223891601563	205.711102224032
4.40591480000000	289.234440307617	205.711451500601
4.47815420000000	289.223891601563	205.700876732282
4.55166240000000	289.255535888672	205.701486932899
4.62537950000000	289.192247314453	205.690342841027
4.69738750000000	289.192247314453	205.690058145453
4.74782250000000	289.202796020508	205.690293548472
4.79535100000000	289.192247314453	205.689926854844
4.84407110000000	289.192247314453	205.690131689405
4.89330690000000	289.192247314453	205.679763788159
4.93965880000000	289.192247314453	205.679414655974
4.98888100000000	289.171151733398	205.679158612714
5.04065200000000	289.171151733398	205.668854227013
5.11135020000000	289.171151733398	205.658651966230
5.16411360000000	289.171151733398	205.658631241248
5.23627660000000	289.213342895508	205.659489809601
5.28702340000000	289.244987182617	205.660770287601
5.36120890000000	289.244987182617	205.661026528935
5.41198520000000	289.266084594727	205.651626013860
5.46119770000000	289.255535888672	205.651346804897
5.53831220000000	289.234440307617	205.650813993431
5.58368630000000	289.202796020508	205.639769878119
5.65964570000000	289.223891601563	205.640109222262
5.70965880000000	289.213342895508	205.639962130153
5.75896700000000	289.202796020508	205.639957508462
5.83435670000000	289.234440307617	205.640738021860
5.88301900000000	289.223891601563	205.640453202235
5.95903220000000	289.244987182617	205.630345337214
6.00637610000000	289.244987182617	205.630976033652
6.06152310000000	289.244987182617	205.630689155373
6.13513020000000	289.255535888672	205.631107638701
6.18828360000000	289.223891601563	205.620213919521
6.23334970000000	289.223891601563	205.620215793962
6.30637400000000	289.266084594727	205.621080776325
6.37818300000000	289.244987182617	205.620985725425
6.45516590000000	289.287180175781	205.621816923572
6.52829560000000	289.266084594727	205.611339940968
6.57808080000000	289.234440307617	205.610363425959
6.65571870000000	289.255535888672	205.610851628501
6.72894010000000	289.244987182617	205.611004628854
6.78162450000000	289.244987182617	205.600548007318
6.82602720000000	289.255535888672	205.600945579931
6.87734800000000	289.234440307617	205.600596825131
6.92550690000000	289.234440307617	205.600463432366
6.99556950000000	289.213342895508	205.589710986793
7.07189430000000	289.213342895508	205.589858000571
7.12186130000000	289.244987182617	205.590685660687
7.17538820000000	289.255535888672	205.590799826132
7.22711230000000	289.255535888672	205.591044999571
7.28465490000000	289.255535888672	205.580741596801
7.35896510000000	289.244987182617	205.580508455259
7.41158300000000	289.255535888672	205.580919200943
7.48742610000000	289.234440307617	205.580599798827
7.55942780000000	289.244987182617	205.580545656994
7.60446490000000	289.266084594727	205.570931412729
7.65252850000000	289.287180175781	205.571645104484
7.72739220000000	289.287180175781	205.571726839357
7.77954060000000	289.297727050781	205.572244664542
7.82696940000000	289.276631469727	205.571552592785
7.90069550000000	289.266084594727	205.560979496989
7.95319590000000	289.244987182617	205.560490753237
8.02588460000000	289.234440307617	205.560193452055
8.07135520000000	289.266084594727	205.561209949922
8.11947290000000	289.276631469727	205.561339396301
8.17336240000000	289.287180175781	205.561807766199
8.22483720000000	289.308275756836	205.562454538870
8.29777350000000	289.308275756836	205.562517668187
8.34937560000000	289.297727050781	205.562333255072
8.40069140000000	289.297727050781	205.551693076820
8.47230060000000	289.287180175781	205.541269802690
8.52251060000000	289.276631469727	205.541133427301
8.56865020000000	289.329371337891	205.542336958719
8.62218070000000	289.329371337891	205.542757652922
8.66651990000000	289.361015625000	205.543598985880
8.73878860000000	289.361015625000	205.543901573827
8.79110310000000	289.392659912109	205.544957628238
8.84775960000000	289.382111206055	205.544646553674
8.91904310000000	289.339920043945	205.543909239206
8.97088090000000	289.339920043945	205.543729552071
9.04236380000000	289.329371337891	205.543471996580
9.09296950000000	289.329371337891	205.543150430945
9.16417790000000	289.339920043945	205.543553806105
9.23687210000000	289.339920043945	205.543559855976
9.28315240000000	289.308275756836	205.542979446745
9.32998230000000	289.318824462891	205.543138665969
9.40452220000000	289.308275756836	205.543026988074
9.45409540000000	289.308275756836	205.542823321317
9.52731380000000	289.308275756836	205.542994445435
9.57594700000000	289.318824462891	205.543226077685
9.63020150000000	289.297727050781	205.542538986379
9.70472110000000	289.297727050781	205.542650218044
9.75310700000000	289.287180175781	205.542364032013
9.82772210000000	289.287180175781	205.542130331317
9.87288530000000	289.297727050781	205.542227060899
9.93082120000000	289.297727050781	205.542077160286
9.97674460000000	289.297727050781	205.542140499634
10.0284970000000	289.297727050781	205.542485480466
10.0828089000000	289.276631469727	205.541816166683
10.1328649000000	289.287180175781	205.541972969357
10.2090181000000	289.297727050781	205.542027341166
10.2526942000000	289.287180175781	205.541768506312
10.3260745000000	289.266084594727	205.531245376377
10.3978214000000	289.287180175781	205.531387998648
10.4516214000000	289.276631469727	205.531352386360
10.5038901000000	289.287180175781	205.531644751116
10.5758648000000	289.276631469727	205.531559433438
10.6264299000000	289.287180175781	205.531619013478
10.6977790000000	289.276631469727	205.531476767388
10.7543630000000	289.287180175781	205.531655903666
10.8034700000000	289.276631469727	205.531442794557
10.8776999000000	289.266084594727	205.531090922060
10.9273997000000	289.276631469727	205.531401672206
10.9761203000000	289.234440307617	205.530438960456
11.0516203000000	289.255535888672	205.520140895803
11.0956500000000	289.244987182617	205.520145485292
11.1696673000000	289.234440307617	205.519947680869
11.2235787000000	289.255535888672	205.520635142624
11.2784846000000	289.244987182617	205.520125351087
11.3194659000000	289.276631469727	205.520926034710
11.3709830000000	289.266084594727	205.510418567872
11.4426842000000	289.255535888672	205.510388362335
11.4959322000000	289.266084594727	205.510587831609
11.5713826000000	289.266084594727	205.510492208307
11.6232485000000	289.266084594727	205.510677466604
11.6737394000000	289.266084594727	205.510695441615

};

  \addplot[color= black, line width = 1pt] table[x=time, y=u1] {
time u1 u2
0	252	252
0.0487526000000000	252.801647644043	247.287611019615
0.0975052000000000	253.160279846191	245.303125944862
0.143032500000000	251.873423767090	243.735127415662
0.188358900000000	251.092871704102	242.461833234345
0.264057500000000	252.801647644043	241.267283152749
0.312974000000000	253.961927490234	239.115202030185
0.367888900000000	254.774123840332	237.895606718865
0.413679000000000	256.187555236816	236.485022981346
0.487660900000000	258.529211425781	235.410617866989
0.559690000000000	260.807579498291	233.392709288293
0.613076500000000	262.453067779541	231.542342590756
0.667355400000000	264.172391738892	230.278259884906
0.714430100000000	265.606919631958	228.788231800562
0.760709700000000	267.083639688492	227.538277152138
0.812520500000000	268.623647689819	226.243236792906
0.885118100000000	270.722699890137	224.827689017993
0.933742300000000	272.083391876221	223.110449698483
1.00839010000000	274.351211242676	222.107348553954
1.05910620000000	275.859575500488	220.561930262536
1.11058830000000	277.420679626465	219.639319752971
1.18158730000000	279.488087768555	218.480440965899
1.22824010000000	280.363571777344	217.256766736981
1.29992460000000	279.772883605957	216.372398329143
1.34634940000000	279.329867248535	215.208177581249
1.41911340000000	279.298223876953	214.415184813338
1.47011770000000	280.004939575195	213.476692089035
1.51635420000000	281.333987731934	212.793694325138
1.61034270000000	284.203044433594	212.161490494317
1.65907990000000	285.405516357422	211.586475299949
1.70368670000000	286.702919311523	211.157733072486
1.77540000000000	287.557307739258	210.728555576606
1.82406180000000	286.639632568359	210.164808203404
1.90190150000000	284.255784301758	209.659723062069
1.94391300000000	283.317011718750	209.132099765375
1.99151400000000	283.464684448242	208.789086149751
2.04267460000000	284.899211425781	208.455984058623
2.11490280000000	286.186067504883	208.114466501583
2.16477620000000	286.386480102539	207.801075419000
2.21734950000000	287.377992553711	207.616424525728
2.28994270000000	289.487590942383	207.435421199819
2.34186410000000	289.603619384766	207.246266679118
2.41357060000000	289.097316284180	207.096322760679
2.46674570000000	289.139507446289	206.949943235980
2.51838540000000	288.770328369141	206.788345625088
2.59361540000000	288.000324096680	206.640399930854
2.64591030000000	288.126899414063	206.509083469220
2.71487140000000	288.285119018555	206.383840704514
2.76636870000000	288.411696166992	206.294511689906
2.81317750000000	288.654299926758	206.213198448641
2.88653810000000	288.917999267578	206.140664713841
2.95995780000000	288.970739135742	206.064956762989
3.01380470000000	289.107863159180	206.006771747936
3.06218510000000	289.392659912109	205.966430613136
3.12118210000000	289.529783935547	205.929616551600
3.16304890000000	289.455948486328	205.890077943530
3.21672770000000	289.403206787109	205.861256997451
3.29024150000000	289.371564331055	205.821966124596
3.34057050000000	289.403206787109	205.794272972737
3.38549450000000	289.371564331055	205.775446266759
3.45827470000000	289.434851074219	205.768110500527
3.50818850000000	289.434851074219	205.759975367566
3.58103730000000	289.519235229492	205.763093992470
3.63466050000000	289.445399780273	205.752623117636
3.70775960000000	289.424304199219	205.753295786905
3.75541010000000	289.382111206055	205.752921971608
3.82690840000000	289.318824462891	205.751991065132
3.87527990000000	289.276631469727	205.730939618556
3.94908590000000	289.244987182617	205.730428961324
4.00047170000000	289.213342895508	205.719774163216
4.05332180000000	289.276631469727	205.721191214997
4.12546960000000	289.276631469727	205.721843841502
4.17671490000000	289.276631469727	205.711703628327
4.25337940000000	289.266084594727	205.711950783634
4.30546980000000	289.244987182617	205.711395833778
4.35945060000000	289.223891601563	205.711102224032
4.40591480000000	289.234440307617	205.711451500601
4.47815420000000	289.223891601563	205.700876732282
4.55166240000000	289.255535888672	205.701486932899
4.62537950000000	289.192247314453	205.690342841027
4.69738750000000	289.192247314453	205.690058145453
4.74782250000000	289.202796020508	205.690293548472
4.79535100000000	289.192247314453	205.689926854844
4.84407110000000	289.192247314453	205.690131689405
4.89330690000000	289.192247314453	205.679763788159
4.93965880000000	289.192247314453	205.679414655974
4.98888100000000	289.171151733398	205.679158612714
5.04065200000000	289.171151733398	205.668854227013
5.11135020000000	289.171151733398	205.658651966230
5.16411360000000	289.171151733398	205.658631241248
5.23627660000000	289.213342895508	205.659489809601
5.28702340000000	289.244987182617	205.660770287601
5.36120890000000	289.244987182617	205.661026528935
5.41198520000000	289.266084594727	205.651626013860
5.46119770000000	289.255535888672	205.651346804897
5.53831220000000	289.234440307617	205.650813993431
5.58368630000000	289.202796020508	205.639769878119
5.65964570000000	289.223891601563	205.640109222262
5.70965880000000	289.213342895508	205.639962130153
5.75896700000000	289.202796020508	205.639957508462
5.83435670000000	289.234440307617	205.640738021860
5.88301900000000	289.223891601563	205.640453202235
5.95903220000000	289.244987182617	205.630345337214
6.00637610000000	289.244987182617	205.630976033652
6.06152310000000	289.244987182617	205.630689155373
6.13513020000000	289.255535888672	205.631107638701
6.18828360000000	289.223891601563	205.620213919521
6.23334970000000	289.223891601563	205.620215793962
6.30637400000000	289.266084594727	205.621080776325
6.37818300000000	289.244987182617	205.620985725425
6.45516590000000	289.287180175781	205.621816923572
6.52829560000000	289.266084594727	205.611339940968
6.57808080000000	289.234440307617	205.610363425959
6.65571870000000	289.255535888672	205.610851628501
6.72894010000000	289.244987182617	205.611004628854
6.78162450000000	289.244987182617	205.600548007318
6.82602720000000	289.255535888672	205.600945579931
6.87734800000000	289.234440307617	205.600596825131
6.92550690000000	289.234440307617	205.600463432366
6.99556950000000	289.213342895508	205.589710986793
7.07189430000000	289.213342895508	205.589858000571
7.12186130000000	289.244987182617	205.590685660687
7.17538820000000	289.255535888672	205.590799826132
7.22711230000000	289.255535888672	205.591044999571
7.28465490000000	289.255535888672	205.580741596801
7.35896510000000	289.244987182617	205.580508455259
7.41158300000000	289.255535888672	205.580919200943
7.48742610000000	289.234440307617	205.580599798827
7.55942780000000	289.244987182617	205.580545656994
7.60446490000000	289.266084594727	205.570931412729
7.65252850000000	289.287180175781	205.571645104484
7.72739220000000	289.287180175781	205.571726839357
7.77954060000000	289.297727050781	205.572244664542
7.82696940000000	289.276631469727	205.571552592785
7.90069550000000	289.266084594727	205.560979496989
7.95319590000000	289.244987182617	205.560490753237
8.02588460000000	289.234440307617	205.560193452055
8.07135520000000	289.266084594727	205.561209949922
8.11947290000000	289.276631469727	205.561339396301
8.17336240000000	289.287180175781	205.561807766199
8.22483720000000	289.308275756836	205.562454538870
8.29777350000000	289.308275756836	205.562517668187
8.34937560000000	289.297727050781	205.562333255072
8.40069140000000	289.297727050781	205.551693076820
8.47230060000000	289.287180175781	205.541269802690
8.52251060000000	289.276631469727	205.541133427301
8.56865020000000	289.329371337891	205.542336958719
8.62218070000000	289.329371337891	205.542757652922
8.66651990000000	289.361015625000	205.543598985880
8.73878860000000	289.361015625000	205.543901573827
8.79110310000000	289.392659912109	205.544957628238
8.84775960000000	289.382111206055	205.544646553674
8.91904310000000	289.339920043945	205.543909239206
8.97088090000000	289.339920043945	205.543729552071
9.04236380000000	289.329371337891	205.543471996580
9.09296950000000	289.329371337891	205.543150430945
9.16417790000000	289.339920043945	205.543553806105
9.23687210000000	289.339920043945	205.543559855976
9.28315240000000	289.308275756836	205.542979446745
9.32998230000000	289.318824462891	205.543138665969
9.40452220000000	289.308275756836	205.543026988074
9.45409540000000	289.308275756836	205.542823321317
9.52731380000000	289.308275756836	205.542994445435
9.57594700000000	289.318824462891	205.543226077685
9.63020150000000	289.297727050781	205.542538986379
9.70472110000000	289.297727050781	205.542650218044
9.75310700000000	289.287180175781	205.542364032013
9.82772210000000	289.287180175781	205.542130331317
9.87288530000000	289.297727050781	205.542227060899
9.93082120000000	289.297727050781	205.542077160286
9.97674460000000	289.297727050781	205.542140499634
10.0284970000000	289.297727050781	205.542485480466
10.0828089000000	289.276631469727	205.541816166683
10.1328649000000	289.287180175781	205.541972969357
10.2090181000000	289.297727050781	205.542027341166
10.2526942000000	289.287180175781	205.541768506312
10.3260745000000	289.266084594727	205.531245376377
10.3978214000000	289.287180175781	205.531387998648
10.4516214000000	289.276631469727	205.531352386360
10.5038901000000	289.287180175781	205.531644751116
10.5758648000000	289.276631469727	205.531559433438
10.6264299000000	289.287180175781	205.531619013478
10.6977790000000	289.276631469727	205.531476767388
10.7543630000000	289.287180175781	205.531655903666
10.8034700000000	289.276631469727	205.531442794557
10.8776999000000	289.266084594727	205.531090922060
10.9273997000000	289.276631469727	205.531401672206
10.9761203000000	289.234440307617	205.530438960456
11.0516203000000	289.255535888672	205.520140895803
11.0956500000000	289.244987182617	205.520145485292
11.1696673000000	289.234440307617	205.519947680869
11.2235787000000	289.255535888672	205.520635142624
11.2784846000000	289.244987182617	205.520125351087
11.3194659000000	289.276631469727	205.520926034710
11.3709830000000	289.266084594727	205.510418567872
11.4426842000000	289.255535888672	205.510388362335
11.4959322000000	289.266084594727	205.510587831609
11.5713826000000	289.266084594727	205.510492208307
11.6232485000000	289.266084594727	205.510677466604
11.6737394000000	289.266084594727	205.510695441615

};

   \end{axis}

\end{tikzpicture}

%% file: FigPlot/LargeGamma.tex
\begin{tikzpicture}

\begin{axis}[%
      width  = .8\linewidth,
      height = 0.4\linewidth,
at={(0.758in,0.481in)},
ymajorgrids=true,
grid style=dashed,
legend pos=north west,
scale only axis,
label style={font=\footnotesize},
tick label style={font=\footnotesize},
      xmin=0, xmax=29,
      ymin=-0.2, ymax=15,
xlabel={Time [s]},
    ylabel={Position $q_i$ [deg]}
]



    \addplot[color=blue, line width = 1pt] table {
0	0
0.0555397000000000	0.0596396341594503
0.102904000000000	0.00473421421469450
0.180742100000000	0.927941028989043
0.233388500000000	1.97609139516601
0.276293100000000	2.51973226404204
0.323599000000000	2.85306551986256
0.395068300000000	4.11454314577149
0.443143800000000	4.85268897165228
0.495190100000000	5.35965814938512
0.568049700000000	5.72111083958185
0.617929400000000	6.03776011861104
0.662125000000000	6.03776011861104
0.714243100000000	6.32756305955879
0.763629300000000	6.65328364989582
0.835647800000000	7.05391474134098
0.881463100000000	7.79354237664507
0.955296600000000	8.09744580697805
1.00601250000000	8.39763081755575
1.07940570000000	8.43038868691166
1.12881300000000	8.45693179307860
1.17451200000000	8.45146709588377
1.22505960000000	8.52225374676143
1.29603220000000	8.73093049169799
1.37120120000000	8.99096863708300
1.44388450000000	9.12436860181958
1.48954150000000	9.17110287510132
1.57162260000000	9.19756686135352
1.62064160000000	9.25327860894162
1.69712510000000	9.28636294296781
1.76880630000000	9.35085191449812
1.81930460000000	9.38907051852968
1.86678160000000	9.42139710622380
1.94031170000000	9.47555665450937
2.01073110000000	9.57613752711101
2.06470100000000	9.61763934883147
2.10763320000000	9.63207686092398
2.18070870000000	9.65912995792558
2.25610220000000	9.69625623092104
2.30502400000000	9.69775379891354
2.35372840000000	9.68842374210736
2.42715790000000	9.70288065621120
2.49650160000000	9.70693610411137
2.57048350000000	9.78811300031926
2.64283180000000	9.77872706715604
2.69032100000000	9.79194667642119
2.76264830000000	9.80179874971319
2.80832920000000	9.81362091045408
2.88215950000000	9.81776378848936
2.92952710000000	9.82210934653544
3.00249410000000	9.82271151239813
3.05512200000000	9.82270795202050
3.10563870000000	9.82315859851107
3.17878980000000	9.82209347930111
3.22748340000000	9.82095352829409
3.30119140000000	9.81918461912200
3.34960410000000	9.86261640956514
3.40231410000000	9.86487762884445
3.47492350000000	9.85997223367029
3.52775880000000	9.81564509129014
3.60040850000000	9.86063613342201
3.64930730000000	9.85581024295214
3.72130690000000	9.84729659843200
3.76779830000000	9.83099959271254
3.84081170000000	9.85950032058785
3.88921100000000	9.85131395180158
3.93594850000000	9.88053811411292
3.98425140000000	9.87703828618525
4.03266820000000	9.89752916371833
4.10647680000000	9.89635075053810
4.15646530000000	9.89833172926145
4.22816560000000	9.90687734372402
4.30026390000000	9.86111446896441
4.37528520000000	9.86155769873808
4.42010340000000	9.86207741309596
4.46968290000000	9.91568083298889
4.52035540000000	9.90103056321522
4.56945620000000	9.88858294227920
4.64212090000000	9.89726698258988
4.69402980000000	9.88895315915348
4.76868260000000	9.88837998760716
4.81514590000000	9.84821068850767
4.88515490000000	9.84629543843412
4.93392830000000	9.88976199749621
5.00802690000000	9.85309672428998
5.05820620000000	9.83880615986321
5.13032910000000	9.80959051274362
5.17798920000000	9.88224561174455
5.25105040000000	9.85752030081176
5.30087100000000	9.84791055196699
5.37407890000000	9.83027482194571
5.42297030000000	9.85425797896774
5.50285700000000	9.85556160143857
5.57432630000000	9.89977232940786
5.62228660000000	9.90372648380607
5.67098230000000	9.89879548791501
5.72528650000000	9.88222955762519
5.76635160000000	9.87614830237605
5.83979380000000	9.87634610845058
5.91104390000000	9.88277581446370
5.96462930000000	9.88428296007641
6.02020180000000	9.89350970160856
6.06971190000000	9.90155228280331
6.14527150000000	9.89886233324412
6.19382210000000	9.84256846490583
6.26466390000000	9.88314564942116
6.31380980000000	9.87695435336295
6.35999420000000	9.82128526658689
6.41422290000000	9.84723261383689
6.46901160000000	9.75748791447291
6.51555520000000	9.76236056381631
6.56555860000000	9.85806228233593
6.61579680000000	9.71559792482196
6.68759860000000	9.70150821228351
6.73586690000000	9.86377957850112
6.80977680000000	9.85275822872229
6.85647860000000	10.0340948761739
6.92746760000000	10.0705905456386
6.97246370000000	9.43624292156630
7.04648270000000	9.75894639973283
7.09738980000000	9.92881354550433
7.16987500000000	9.74879635137820
7.22187260000000	9.81532263274122
7.27158310000000	9.89993694221543
7.34479420000000	9.90736014974118
7.39547440000000	9.94429955163700
7.46930150000000	9.77008151205649
7.51784450000000	9.85414210961724
7.58869210000000	9.95336716448833
7.63830930000000	9.88689184591385
7.71202530000000	9.94568201805142
7.76249510000000	9.92908930566814
7.81183130000000	9.90902240883975
7.88370000000000	9.90448253732326
7.93317090000000	9.90418483644768
8.00745500000000	9.97100884782783
8.05426030000000	9.89591775868965
8.12804530000000	9.88958330514835
8.20727970000000	9.91124630991075
8.27711040000000	9.88192740972375
8.32603280000000	9.94871194754281
8.37310450000000	9.93554675059414
8.42553700000000	9.97641034978574
8.47310380000000	9.94235427657237
8.54922070000000	9.83319278781223
8.59863860000000	9.91044845490549
8.66948600000000	9.92204421687343
8.71991170000000	9.89836274402624
8.77213980000000	9.85923964215815
8.81782060000000	9.76643865156634
8.86898980000000	9.74255793372772
8.94010980000000	9.99467514738514
8.98877270000000	9.83480123791456
9.06430340000000	9.53618028584102
9.11364590000000	9.70761112749637
9.16320010000000	9.63502358741705
9.23702690000000	9.46366232334595
9.28884130000000	9.71634331696367
9.36080290000000	9.24629934340557
9.41068780000000	8.96107100058842
9.45928400000000	9.42837620487017
9.53417480000000	9.25411502886531
9.60708970000000	8.96319478271591
9.67950500000000	9.70913697012112
9.75298590000000	10.0629249993537
9.80480410000000	9.84758384740488
9.85766600000000	10.2024324722406
9.92908770000000	10.1199337759222
9.97661010000000	9.72643755484233
10.0488700000000	10.5173704447512
10.0990576000000	10.5689080296298
10.1725686000000	9.99397692782889
10.2235899000000	9.83330119341546
10.2739975000000	9.91473292903166
10.3455451000000	9.90584042117152
10.3966758000000	9.87775039980931
10.4685492000000	9.89161528198142
10.5174724000000	9.93297909359391
10.5894606000000	9.99491282948788
10.6408534000000	9.95233071722261
10.6852618000000	10.0525070218599
10.7355436000000	9.95661292508461
10.8118265000000	10.3371479713520
10.8824184000000	10.1382476673573
10.9307985000000	9.94155597757678
10.9798518000000	9.96213705982609
11.0284633000000	9.95423733600934
11.0793241000000	9.93844694859809
11.1517323000000	9.87946268395173
11.2245768000000	9.77772355017068
11.2693057000000	9.54196060743443
11.3425139000000	9.47128904830973
11.3956146000000	9.03049754319787
11.4461519000000	8.69123303717719
11.5202162000000	8.41366032646031
11.5727552000000	8.06104045109562
11.6253701000000	7.66987919183581
11.6776416000000	7.19922240530410
11.7503839000000	6.73890430137310
11.8232565000000	6.10437210270254
11.8736529000000	5.97259552832143
11.9191064000000	5.99230193497441
11.9912558000000	6.07292039254377
12.0663176000000	6.10685657571505
12.1139914000000	6.13150522388628
12.1642576000000	6.16258843902577
12.2396985000000	6.22788243550164
12.3092499000000	6.49957778085422
12.3851240000000	7.24449410392652
12.4341423000000	8.47140495025740
12.5092922000000	9.54557885915778
12.5579538000000	10.1910117679459
12.6366681000000	10.5443179145172
12.6804257000000	10.8099436479998
12.7302637000000	10.8950507548658
12.7811211000000	11.0157922868830
12.8305282000000	11.6428412120793
12.8819407000000	12.4426541939256
12.9522426000000	12.5327152802011
13.0018332000000	11.3129910352079
13.0762519000000	10.9746949243104
13.1207982000000	10.3159499688766
13.1939148000000	10.0048810992381
13.2438692000000	9.90479461148487
13.3161855000000	9.90200346133468
13.3647672000000	9.93307134330085
13.4137509000000	9.97032569490957
13.4629977000000	9.97032569490957
13.5110061000000	9.31231894711270
13.5840267000000	8.50454560813652
13.6339202000000	8.20274265051019
13.6809203000000	8.11127365097278
13.7273726000000	8.16682847011763
13.7826616000000	8.28425315788604
13.8386346000000	8.10730293348537
13.9156732000000	7.50902559131543
13.9675892000000	6.96456156327682
14.0385953000000	6.86195358272566
14.1158112000000	6.82236626308968
14.1888222000000	6.78003042830741
14.2397903000000	6.77732293763645
14.3147725000000	6.78539442809060
14.3888147000000	6.81471287462433
14.4373855000000	6.82563847725140
14.4837887000000	6.84500310759201
14.5287647000000	6.94926379278244
14.5781546000000	7.50456669651905
14.6243404000000	8.39047017736164
14.6732432000000	9.15583041083316
14.7182117000000	9.87965658181628
14.7706865000000	10.3776665587668
14.8143653000000	10.6337837698416
14.8611401000000	10.7285407022584
14.9192140000000	10.8026696922497
14.9988037000000	10.8951321653174
15.0453135000000	11.4858836901156
15.1192050000000	12.5972917553367
15.1704881000000	11.9592221179795
15.2423455000000	11.5226518434977
15.3171760000000	11.5226518434977
15.3905699000000	10.2632749177098
15.4382707000000	9.97081796363611
15.4854469000000	9.91227235311003
15.5352121000000	9.84579491462945
15.5862818000000	9.90489479315874
15.6580518000000	10.0596933760984
15.7083011000000	10.0080441835288
15.7733788000000	9.82404113451045
15.8427603000000	9.67658128509873
15.9212019000000	9.24878067358426
15.9935173000000	8.88208197816396
16.0436278000000	8.60618522591617
16.1213698000000	8.37844312262084
16.1654369000000	7.88574440211815
16.2184236000000	7.08561458632948
16.2902557000000	6.74419000086295
16.3415119000000	6.59196203153606
16.4213430000000	6.62917710584521
16.4765667000000	6.58076759808695
16.5315193000000	6.58860709972168
16.6539406000000	6.62855193338433
16.6881274000000	6.69170499591295
16.7628204000000	6.69484297394973
16.8137449000000	6.73245648738262
16.8877637000000	6.75873415212222
16.9399485000000	7.38588208210679
16.9922635000000	8.22562685489238
17.0642914000000	8.78513337365999
17.1120334000000	9.72951828060945
17.1852708000000	10.3151403988485
17.2348601000000	10.6001024624995
17.3066397000000	10.8347095903653
17.3566833000000	10.9028457623855
17.4296894000000	11.4960039141129
17.4785060000000	12.4208001196759
17.5520578000000	12.7718075964219
17.6246242000000	11.6602147058607
17.6975858000000	11.0448867066874
17.7465549000000	10.4949249265850
17.8226297000000	10.2872603693619
17.8724886000000	10.0806818584943
17.9237697000000	10.0456619967809
17.9676071000000	10.0756602533568
18.0184958000000	10.0986134297373
18.0906244000000	10.0368112898776
18.1392310000000	9.91057244193411
18.2081877000000	9.75074168094729
18.2594358000000	9.67884709909026
18.3143986000000	9.27495529866598
18.3891413000000	9.00386875716151
18.4385343000000	8.60102306335781
18.5086077000000	8.35559891266195
18.5572878000000	8.35559891266195
18.6301260000000	8.35559891266195
18.6799158000000	6.88685810056573
18.7518065000000	6.63948172974562
18.8252664000000	6.60662240164104
18.9013500000000	6.52567177678233
18.9537347000000	6.53445575250223
19.0265872000000	6.54777494112724
19.0984249000000	6.62840602806132
19.1723380000000	6.65090462231112
19.2191806000000	6.67765737805102
19.2918739000000	6.69189838447333
19.3398940000000	6.76615309101468
19.4117414000000	6.90570817624347
19.4598103000000	7.52774808314454
19.5322536000000	8.71799184152586
19.5833316000000	9.53818856086696
19.6352548000000	10.1481403408822
19.7086574000000	10.6974636557654
19.7581125000000	10.8821983639621
19.8321750000000	11.0119597597474
19.8810181000000	11.2302595082280
19.9300720000000	12.0230399979047
20.0041025000000	12.8683262419440
20.0539577000000	12.1313073577545
20.1268643000000	11.3402864904097
20.1761897000000	11.3402864904097
20.2514499000000	10.3950262276759
20.3280711000000	10.0955906370966
20.4013797000000	9.97324024664468
20.4492299000000	9.95806729442256
20.4965804000000	9.96830686465018
20.5669984000000	9.80241704762449
20.6411343000000	9.63519793159380
20.6880572000000	9.18867051380909
20.7611997000000	8.90856930704734
20.8125313000000	8.57749735578285
20.8603399000000	8.45862614877716
20.9338856000000	8.27903570690657
20.9827433000000	8.08726328949684
21.0567138000000	7.14223949706776
21.1280256000000	6.77983068611343
21.2007441000000	6.76621486842832
21.2516033000000	6.69646379876544
21.3287213000000	6.67857689803983
21.3768858000000	6.69846942140660
21.4510197000000	6.71314592223895
21.4966758000000	6.74306813523663
21.5706869000000	6.75479864652072
21.6189058000000	6.79583357411496
21.6918767000000	6.80588068718443
21.7427108000000	7.45218166490610
21.7921897000000	8.16931825192675
21.8630179000000	8.67589933153221
21.9160461000000	9.58924802833816
21.9709839000000	10.1998650751399
22.0489021000000	10.5537022123350
22.1004571000000	10.8389034942270
22.1455370000000	10.9089656521109
22.2150226000000	11.1045056291725
22.2869866000000	12.1524319598213
22.3384349000000	12.5945723229688
22.4103739000000	12.0446210405470
22.4625062000000	11.3006854186060
22.5130445000000	11.0651622619998
22.5864657000000	10.5448677886142
22.6369089000000	10.2410151171690
22.7077183000000	10.1459716636426
22.7569174000000	10.0346328236858
22.8317962000000	9.97583628476950
22.8785909000000	9.84712361835861
22.9522531000000	9.65350003316751
23.0007134000000	9.51551173452103
23.0737736000000	9.01290150563964
23.1223053000000	8.74309637462364
23.1721414000000	8.49390810591548
23.2451449000000	8.28601896839373
23.2922486000000	8.28601896839373
23.3661285000000	7.62628010181356
23.4157980000000	6.87631332009307
23.4877540000000	6.74395040958561
23.5381112000000	6.71333722836434
23.6120373000000	6.66355500883792
23.6623231000000	6.66571276019993
23.7300954000000	6.68561838765259
23.7817006000000	6.74069418285496
23.8305285000000	6.76470222444293
23.9048700000000	6.79614149444477
23.9535742000000	6.81131176615832
24.0324211000000	6.86897285579859
24.0792817000000	7.49182976335428
24.1529571000000	8.29590330166181
24.2314075000000	9.10923236870912
24.2819634000000	10.0685921960723
24.3549781000000	10.4901832638286
24.3997860000000	10.7253796617783
24.4738783000000	10.9030106939059
24.5260579000000	11.0439668140508
24.5984252000000	12.0171683558516
24.6461746000000	12.7032266801388
24.7167084000000	12.5091205534018
24.7682564000000	11.5041121582346
24.8406927000000	11.0779092831440
24.8898956000000	10.4792708349432
24.9617541000000	10.3875666243298
25.0335438000000	10.1734041730105
25.1048285000000	10.0499985595679
25.1564085000000	9.95452839665288
25.2017772000000	9.72461630148184
25.2506030000000	9.58720200347237
25.3016729000000	9.40428504057675
25.3756502000000	8.87900835276827
25.4247848000000	8.68424436013772
25.4978325000000	8.37671112965145
25.5491161000000	8.16462623026226
25.6186063000000	7.78504838324725
25.6662512000000	7.04868215034563
25.7392689000000	6.79693000419630
25.7894905000000	6.73086545587310
25.8619196000000	6.66737765074792
25.9344943000000	6.65374308750153
26.0076223000000	6.68180940007590
26.0580347000000	6.72231336881357
26.1346985000000	6.73587854085728
26.1843409000000	6.77601674415770
26.2314883000000	6.78543505965124
26.3062915000000	6.81495674743550
26.3539948000000	6.98078339069032
26.4035762000000	7.48876818916720
26.4564041000000	8.12109157750945
26.5057485000000	8.75360722573759
26.5773378000000	9.89726496411403
26.6253461000000	10.4395413995514
26.6990216000000	10.8578041820545
26.7480684000000	10.9597719700268
26.8206949000000	11.0648239103174
26.8684385000000	11.6513575550440
26.9426310000000	12.4704806527838
26.9934221000000	12.4153083038558
27.0421588000000	11.9705523919672
27.1135767000000	11.4683262870629
27.1616271000000	10.7227530612430
27.2373161000000	10.5048271676023
27.2881376000000	10.2125745833133
27.3607924000000	10.1329254228885
27.4352240000000	10.0098179502595
27.5071810000000	9.92561964213691
27.5587924000000	9.81031158120490
27.6051835000000	9.70688229084658
27.6529642000000	9.61378387997351
27.7031239000000	9.49100953138426
27.7778728000000	9.00734329856077
27.8301534000000	8.55667838869343
27.9022951000000	8.20176958369982
27.9530830000000	7.88095102526263
28.0032345000000	7.59306405823188
28.0750698000000	6.89681086212808
28.1239638000000	6.66712376254732
28.1959486000000	6.64440800210773
28.2495467000000	6.57687496024398
28.3240132000000	6.57989525261999
28.3809626000000	6.60476062280727
28.4338036000000	6.62790031418317
28.5052096000000	6.65589020375498
28.5498992000000	6.67412749466902
28.6216239000000	6.68863338811231
28.6699532000000	6.71115543113247
28.7438165000000	6.97410943219742
28.7976275000000	7.66328360736209
28.8509623000000	8.22207604613068
28.9212302000000	9.16717705843448
28.9718219000000	9.90350471456212
29.0427470000000	10.5321006775410
29.0904587000000	10.7779211636819
29.1597590000000	10.8986928905042
29.2090234000000	11.1095885089001
29.2664799000000	11.7472691772939
29.3424148000000	12.5248628191434
29.3925520000000	12.3214313229638
29.4419423000000	11.8799947439698
29.5126754000000	11.3811902417289
29.5618204000000	10.6715893385131
29.6180463000000	10.4808869538250
29.6864523000000	10.2884893132107
29.7590348000000	10.0866835581776
29.8085487000000	9.95934053330564
29.8816808000000	9.88301599849782
29.9288810000000	9.84644718232242

    };

    \end{axis}
\end{tikzpicture}

%% file: FigPlot/8degstiffness.tex
\begin{tikzpicture}

\begin{axis}[%
      width  = 0.8 \linewidth,
      height = 0.4 \linewidth,
at={(0.758in,0.481in)},
ymajorgrids=true,
grid style=dashed,
legend pos=north west,
scale only axis,
label style={font=\footnotesize},
tick label style={font=\footnotesize},
xmin=0,
xmax=10,
xlabel={The gain $\gamma$},
ymin=0.1,
ymax=0.6,
ylabel={stiffness [N/mm]},
axis background/.style={fill=white},
]

\addplot[domain = 0:10, color =  darkgray!50,fill opacity=0.5, line width = 2pt] {0.0321*x  +0.2099 };

\addplot[only marks, mark=x, mark options={}, mark size=2pt, color=color2, fill=red, forget plot] table[row sep=crcr]{%
x	y\\
0.1	0.19777071055\\
1	0.27026931267\\
2	0.25920545014\\
3	0.31217312812\\
4	0.40808804911\\
5	0.3938939039\\
6	0.47122718447\\
7	0.37837193828\\
8	0.48497414644\\
9	0.5274997035\\
10	0.44892686237\\
};
\end{axis}


\end{tikzpicture}%

%% file: FigPlot/10degstiffness.tex
\begin{tikzpicture}
\begin{axis}[%
      ymajorgrids=true,
      grid style=dashed,
      width  = 0.8 \linewidth,
      height = 0.4 \linewidth,
at={(0.758in,0.481in)},
ymajorgrids=true,
grid style=dashed,
legend pos=north west,
scale only axis,
label style={font=\footnotesize},
tick label style={font=\footnotesize},
xmin=0,
xmax=10,
xlabel={The gain $\gamma$},
ymin=0.1,
ymax=1.5,
ylabel={stiffness [N/mm]},
axis background/.style={fill=white},
]

\addplot[domain = 0:10, color =  darkgray!50,fill opacity=0.5, line width = 2pt] {0.0805*x  + 0.4305 };

\addplot[only marks, mark=x, mark options={}, mark size=2pt, color=color2, fill=red, forget plot] table[row sep=crcr]{%
x	y\\
10	1.42454339895000\\
9	1.31032142374000\\
8	1.03157524281000\\
7	0.880266478630000\\
6	0.847478925020000\\
5	0.624648301060000\\
4	0.606672711610000\\
3	0.650014181950000\\
2	0.478064380850000\\
1	0.648747162770000\\
0.1	0.667930931500000\\
};
\end{axis}


\end{tikzpicture}%

%% file: FigPlot/dist_1.tex
\begin{tikzpicture}[spy using outlines={rectangle, magnification=2, width=2.5cm, height=1.5cm, connect spies}]
    \begin{axis}[
      legend style={
        at={(.95,0.05)}, 
        anchor=south east, 
    },
      xmin=0, xmax=15,
      ymin=0, ymax=12,
      ymajorgrids=true,
      grid style=dashed,
      width = 0.48\linewidth,
     height = 0.28\textwidth,
    label style={font=\footnotesize},
      tick label style={font=\footnotesize},
    xlabel={Time [s]},
    ylabel={Position $q_i$ [deg]},
    ]

 \legend{\footnotesize $K_{\tt d} = 0.1 \mbox{,}~ \gamma = 5$, \footnotesize $K_{\tt d}  = 0.1 \mbox{,}~ \gamma = 1$}

    \addplot[color=blue, line width = 1pt] table {
0	0
0.0673295000000000	0.0126986555557583
0.136273700000000	0.00783830783995683
0.209220600000000	1.74560950322669
0.261866500000000	2.48520375911399
0.323604400000000	3.78530590938622
0.379829400000000	4.65724131800020
0.436161300000000	5.42513013771750
0.496914200000000	6.20193536848880
0.555601900000000	6.50324664802290
0.609536400000000	6.70307943674372
0.686994100000000	7.04585035037342
0.741234600000000	7.54781856866663
0.800284400000000	7.82452349184735
0.855544400000000	8.16786888073854
0.912273200000000	8.39925789121412
0.972693100000000	8.65201914169165
1.03098060000000	8.73390654109164
1.08874520000000	8.82010253707219
1.14260790000000	8.94473850348917
1.19988450000000	9.04748251339065
1.25521340000000	9.15495839369617
1.30301560000000	9.25600379875933
1.35867010000000	9.34037522050398
1.41799020000000	9.42040475600629
1.47434390000000	9.48102401514984
1.53237300000000	9.51431942114207
1.58666240000000	9.53900866700228
1.63947800000000	9.59167219115028
1.69031530000000	9.60883506266684
1.74697270000000	9.66895178659389
1.79642920000000	9.69742903394125
1.84020560000000	9.70767491511994
1.88893690000000	9.71963865667151
1.94357510000000	9.74728686174114
1.99823940000000	9.76501014087551
2.07063010000000	9.78228329342212
2.12040650000000	9.78831341510972
2.19244560000000	9.82927463065215
2.25284520000000	9.83547444714693
2.29857730000000	9.84000650609855
2.35479230000000	9.79225506054188
2.41623820000000	9.79539191299836
2.47394310000000	9.80030535192186
2.53123810000000	9.80344352729673
2.58686840000000	9.80653117912265
2.64326000000000	9.81282969786723
2.69207600000000	9.87649317854372
2.74359730000000	9.88297115270735
2.81629650000000	9.88410874243902
2.87293050000000	9.83962929338178
2.92382070000000	9.84218473344206
2.99715660000000	9.84240044722664
3.04705300000000	9.84254461979935
3.09274180000000	9.84297389935606
3.14565290000000	9.84283287194041
3.19991870000000	9.84343224819951
3.27082140000000	9.84479680582309
3.32273000000000	9.84579155246608
3.37867860000000	9.84824652935530
3.43469810000000	9.85264214054762
3.50598190000000	9.85993803488747
3.55700930000000	9.86501899794650
3.62957670000000	9.86995921045625
3.68009280000000	9.87093267244067
3.72644850000000	9.87112169760177
3.79749350000000	9.87150240762885
3.84491120000000	9.87289795636824
3.88888370000000	9.87362119739666
3.94182050000000	9.87500650388757
4.01488840000000	9.87602127543665
4.09166260000000	9.87651013583189
4.16228590000000	9.87892663673807
4.21283530000000	9.88188252803277
4.26059790000000	9.88334054218938
4.31281950000000	9.88482197467612
4.35686160000000	9.88510826953024
4.40291930000000	9.88487725093954
4.45272610000000	9.88461216080960
4.52578520000000	9.88431237148800
4.57629860000000	9.88441350086934
4.65042310000000	9.88558100349985
4.70044390000000	9.88609055761115
4.75268140000000	9.88658671186691
4.81927510000000	9.88836246998564
4.89474240000000	9.89060747580574
4.94042450000000	9.89249203178425
4.99398010000000	9.89342954107824
5.05168140000000	9.89429976620119
5.10766000000000	9.89624490815408
5.17866480000000	9.89808056968374
5.24998820000000	9.89971093814891
5.29847290000000	9.89970593619394
5.37152880000000	9.89926169890370
5.42683250000000	9.89878579802405
5.47287340000000	9.89957783250403
5.52885650000000	9.90020372122420
5.57591380000000	9.90056515874485
5.62569350000000	9.90145552762189
5.69806500000000	9.90171352486573
5.74723600000000	9.90264055522273
5.82060920000000	9.90269274134513
5.89738680000000	9.90281298193587
5.96875310000000	9.90226699291730
6.02504280000000	9.90104205113940
6.09674830000000	9.90104410246508
6.15468120000000	9.90153745886762
6.22634080000000	9.90184787735109
6.30047130000000	9.90266781368035
6.35304440000000	9.90257482843416
6.40353320000000	9.90230274861263
6.47511040000000	9.90246029800978
6.52331530000000	9.90229801083208
6.56831070000000	9.90206703337191
6.61812130000000	9.90185886234683
6.66530110000000	9.90131226440228
6.73930610000000	9.90006178589728
6.78616780000000	9.89930271587657
6.85808710000000	9.89854828215874
6.93251790000000	9.89803092371538
6.98331880000000	9.89913868005829
7.03250960000000	9.89924530114944
7.07749230000000	9.89963328464525
7.12807300000000	9.89870935284479
7.20144350000000	9.89898128254307
7.25124620000000	9.89815712790071
7.30072090000000	9.89780444258915
7.37519850000000	9.89812594353037
7.42492910000000	9.89785254590018
7.49804220000000	9.89756920395352
7.54892050000000	9.89807086769926
7.59525650000000	9.89852818563790
7.66666160000000	9.89767623847670
7.74032770000000	9.89795776092953
7.79124190000000	9.89713474486855
7.84473080000000	9.89622219299700
7.91611340000000	9.89621802028867
7.96997970000000	9.89556301332473
8.02462980000000	9.89556917406105
8.09929440000000	9.89565975438692
8.14953340000000	9.89583003436515
8.22254320000000	9.89614279139640
8.26902760000000	9.89642647864590
8.31982160000000	9.89666238454621
8.36858150000000	9.89678702185680
8.42122440000000	9.89696621453326
8.46865720000000	9.89676028688862
8.54283870000000	9.89694434117406
8.58914800000000	9.89694248827055
8.63692170000000	9.89741633352405
8.69015860000000	9.89865394603428
8.73891050000000	9.89847284574118
8.81198340000000	9.86953552037476
8.86215370000000	9.86988048617409
8.93587760000000	9.86988883154583
8.98264590000000	9.87066606894366
9.03793210000000	9.87475081712637
9.11337580000000	9.87715870701881
9.16010880000000	9.88029447164274
9.20650720000000	9.88425088143291
9.25600020000000	9.88571042343959
9.32696030000000	9.88644769344712
9.40169420000000	9.88548610132742
9.45874080000000	9.79537392441490
9.51464230000000	9.78170173365239
9.57023900000000	9.76043012697975
9.62627860000000	9.67901921497593
9.70171830000000	9.53899989342443
9.77544270000000	9.39938663776400
9.82252740000000	9.32582408835566
9.87133820000000	9.23396295067374
9.92421730000000	9.14929436380758
9.97302420000000	9.08496871660360
10.0283535000000	9.05877131891392
10.1031242000000	9.01266778069126
10.1611313000000	9.01557885339152
10.2200942000000	9.10352167863518
10.2743320000000	9.20134366051750
10.3357712000000	9.39151416495985
10.3890332000000	10.0390160163303
10.4443682000000	10.2075032435971
10.5191269000000	10.1510914159408
10.5723032000000	10.0511811179768
10.6286811000000	10.0700503046364
10.7021832000000	10.0607267687212
10.7491576000000	10.0152813925963
10.8064511000000	10.0148955280279
10.8627972000000	9.98691474234182
10.9180290000000	9.94425703128333
10.9801629000000	9.92689244177547
11.0361066000000	9.91168477647539
11.0886707000000	9.90892486398726
11.1477069000000	9.92440559963201
11.2020588000000	9.94507761368701
11.2584094000000	9.95798933517214
11.3054243000000	9.96282855922796
11.3600640000000	9.96324308412865
11.4156699000000	9.94543956471676
11.4729756000000	9.92214546645923
11.5203977000000	9.92092264689206
11.5707352000000	9.91259612371738
11.6170507000000	9.91477991460676
11.6876883000000	9.92149949168352
11.7620315000000	9.94397710485551
11.8111134000000	9.95910035801045
11.8664402000000	9.95568910100517
11.9234559000000	9.94350696693225
11.9735845000000	9.91376682046449
12.0281879000000	9.90149795101485
12.0735548000000	9.90279811914689
12.1466965000000	9.90347015926151
12.2022697000000	9.95467464340867
12.2456980000000	9.97985031981068
12.3046607000000	9.93171222283288
12.3510960000000	9.93306778515776
12.4026597000000	9.99179539554081
12.4739996000000	9.96925138739522
12.5245277000000	9.95486870219301
12.5736784000000	9.93891771141508
12.6303114000000	9.93964194942585
12.7049647000000	9.95233934137535
12.7603150000000	9.97337925034541
12.8333163000000	9.98288804168315
12.8798249000000	9.97881011101843
12.9524247000000	9.96707592768632
13.0043591000000	9.94720961336907
13.0759616000000	9.94020818737835
13.1282771000000	9.94649703386085
13.1749873000000	9.95016462522583
13.2483893000000	9.96713568626074
13.3203769000000	9.97366106899409
13.3730881000000	9.96922694469827
13.4200575000000	9.96152830366894
13.4660864000000	9.94908205301375
13.5183507000000	9.93920117264601
13.5906811000000	9.93650041301680
13.6415118000000	9.89233374038045
13.6954153000000	9.94645160806963
13.7677969000000	9.95513183893167
13.8211138000000	9.95285420227830
13.8688042000000	9.89847056304886
13.9428571000000	9.93998943354952
13.9947655000000	9.93641340768681
14.0387771000000	9.93597110695653
14.0978010000000	9.88900943650737
14.1479356000000	9.88987959028870
14.2022195000000	9.89245266893691
14.2557674000000	9.89459271329604
14.3148480000000	9.89518463138457
14.3674009000000	9.89439379156297
14.4222850000000	9.89271758744500
14.4681146000000	9.89340767317416
14.5230401000000	9.89371621656581
14.5742830000000	9.89619947718470
14.6261991000000	9.90051049177225
14.6725098000000	9.90224062142681
14.7179337000000	9.90400294808630
14.7664057000000	9.90569093258295
};


    \addplot[color=red, line width = 1pt] table {
0	0
0.0813863000000000	0.000174627943495904
0.130206300000000	0.913666007660531
0.184899100000000	1.67571271757900
0.258899800000000	2.32762486303668
0.334580000000000	3.57413549357085
0.409067200000000	4.75767087563897
0.470847500000000	5.88065371570678
0.522030200000000	6.30199952697062
0.575130000000000	6.67094303967127
0.629711700000000	7.03000216431743
0.704988500000000	7.05623645594160
0.779386100000000	7.00962382653013
0.829842000000000	7.18765002257013
0.881026000000000	7.35470288962969
0.952991700000000	7.72272643388349
1.00112630000000	8.26104209250415
1.09029680000000	8.49985736569667
1.13570280000000	8.69235336140031
1.18164000000000	8.67988198614093
1.24058070000000	8.73268831403134
1.30079530000000	8.73334171866389
1.35304360000000	8.86966064870188
1.40658620000000	8.98188340370530
1.46357350000000	9.09394659764533
1.51305570000000	9.18458768013640
1.56543360000000	9.24758984418177
1.65083840000000	9.30525487544986
1.69315490000000	9.35913094664439
1.78030960000000	9.38066508664409
1.86575620000000	9.43241398196459
1.91409350000000	9.45620630381825
1.96524570000000	9.46497388131297
2.03759190000000	9.51443868961337
2.08605310000000	9.60067405874245
2.14007210000000	9.63277328046597
2.22262790000000	9.67021218602611
2.26973280000000	9.69939622743264
2.32738740000000	9.69975623694297
2.38127770000000	9.71592732894084
2.44625540000000	9.72999453863402
2.49526240000000	9.74576299741750
2.56860930000000	9.77209521774847
2.61711500000000	9.78640042177058
2.67039970000000	9.79444872876377
2.71909350000000	9.79565760526534
2.76238770000000	9.78951842547428
2.81181520000000	9.79603683034088
2.86587430000000	9.81728615232796
2.91231290000000	9.82370211843554
2.97098860000000	9.82659145901094
3.04837700000000	9.83310592519899
3.10347390000000	9.81826747926362
3.15835800000000	9.82097354489901
3.21505170000000	9.84441075999257
3.27063160000000	9.82577799431773
3.32063820000000	9.82794456259957
3.37409670000000	9.83059882483671
3.42186820000000	9.85359662682788
3.51759050000000	9.86062901823421
3.56929880000000	9.84381501651128
3.62715140000000	9.84641942805585
3.67917900000000	9.84739148495353
3.75257630000000	9.84627306127778
3.81058880000000	9.84361311008952
3.85665360000000	9.84148838678299
3.89914910000000	9.84022257314325
3.95708140000000	9.83928050421919
4.00819930000000	9.83859426762961
4.09302530000000	9.83902445379790
4.14163220000000	9.84126801785573
4.20954480000000	9.84451937326947
4.25860380000000	9.84787642717137
4.31263000000000	9.85136534546098
4.35826730000000	9.85542435684314
4.41465450000000	9.86063258523535
4.48735690000000	9.86532169574553
4.53790700000000	9.86690538041318
4.59308450000000	9.86798015957331
4.64946230000000	9.86779347118943
4.69464650000000	9.86791472263947
4.75036910000000	9.86986166015487
4.83748850000000	9.87074355246664
4.88111680000000	9.87224994586806
4.95654950000000	9.87219989031342
5.01286030000000	9.87122836280814
5.06718320000000	9.87067845983760
5.11827680000000	9.86979101468211
5.16237290000000	9.86860630508058
5.23180070000000	9.86939580903648
5.31120190000000	9.87060319545224
5.36000530000000	9.87232092753021
5.41331530000000	9.87370173463822
5.46251540000000	9.87462510080399
5.53577430000000	9.87485556367410
5.59168970000000	9.87566583079643
5.64051060000000	9.87633016445525
5.69217840000000	9.87751914780007
5.73724360000000	9.87899716969642
5.78618550000000	9.88069765587924
5.84426550000000	9.88267865370779
5.90300600000000	9.88374878968085
5.97507790000000	9.88406149318145
6.02783470000000	9.88340305532876
6.08144310000000	9.88279335825002
7.17413570000000	9.89307258752361
7.22570730000000	9.89326204725547
7.27418670000000	9.89336078889178
7.32468540000000	9.89372284970960
7.37900760000000	9.89380603666520
7.43934800000000	9.89462724027703
7.49195940000000	9.89552959382857
7.55137450000000	9.89545831272000
7.60959240000000	9.89501449760016
7.66120130000000	9.89432723327570
7.71542670000000	9.89379103253569
7.76737000000000	9.89379279947703
7.81608770000000	9.89422161440056
7.87102180000000	9.89461906640957
7.94286120000000	9.89496744705537
8.02974520000000	9.89523288159660
8.08611690000000	9.89564508813311
8.14240670000000	9.89573291158752
8.20042850000000	9.89579477865310
8.27440800000000	9.89626604805695
8.32764630000000	9.89617223190576
8.37854570000000	9.89681644793402
8.43173700000000	9.91660091004025
8.47925610000000	9.91690127166858
8.55706160000000	9.89860873603467
8.61321410000000	9.91032592676591
8.66933440000000	9.90886566303462
8.75682070000000	9.91364395309134
8.80552250000000	9.91972526176489
8.85801400000000	9.92158787495208
8.90749390000000	9.92183701805957
8.95903390000000	9.91812961343078
9.01434930000000	9.91493779487087
9.06587400000000	9.91069295166266
9.11646760000000	9.91268880286641
9.17342140000000	9.91312029611549
9.25361950000000	9.91746939098987
9.30858110000000	9.90780506189052
9.37099490000000	9.85096606754033
9.41415120000000	9.82961931597391
9.46325670000000	9.82882637608623
9.53445590000000	9.83190218152557
9.58770810000000	9.76893748620688
9.63926580000000	9.70398053221851
9.68991990000000	9.63610138520582
9.74341990000000	9.57863110019628
9.82810180000000	9.53969284163433
9.87584910000000	9.47557070121337
9.92186820000000	9.42171852407774
9.97568210000000	9.35668024776221
10.0569898000000	9.31840756832832
10.1092690000000	9.27421821898820
10.1679439000000	9.22761992814112
10.2189389000000	9.18110651701195
10.3020508000000	9.17136264581248
10.3519763000000	9.20244202906439
10.4069085000000	9.30529797996204
10.4648274000000	9.39493828128360
10.5118991000000	9.47254031370806
10.5665926000000	9.62927395813452
10.6204205000000	9.75246075687785
10.6651827000000	9.73985018302870
10.7221177000000	9.87641310129171
10.7791669000000	10.0335213205177
10.8370916000000	10.0672472136209
10.8898651000000	10.1206662064723
10.9419997000000	10.1645934580365
10.9912860000000	10.2382247665709
11.0453725000000	10.2651904616164
11.0995222000000	10.1537055766252
11.1538651000000	10.0129105720702
11.2246492000000	9.93999478978170
11.2787154000000	9.90768942291170
11.3323186000000	9.87492208649632
11.3913257000000	9.86918951821454
11.4448435000000	9.85097413182541
11.5152213000000	9.84617461968246
11.5650639000000	9.85861571428526
11.6504330000000	9.87374627797144
11.7341139000000	9.90763840837945
11.7894550000000	9.91559271907894
11.8452816000000	9.91779165317971
11.8896338000000	9.93558327818237
11.9619014000000	9.92569051509612
12.0366483000000	9.92115432935694
12.0814652000000	9.93816755512663
12.1565420000000	9.94641169771620
12.2059729000000	9.94273074413197
12.2642789000000	9.94489847487970
12.3190260000000	9.94477121306433
12.3656998000000	9.96330841001575
12.4294726000000	9.95597860928985
12.5072961000000	9.92398811906447
12.5616151000000	9.92521110862959
12.6093213000000	9.93828150319339
12.6796046000000	9.93878707723691
12.7570789000000	9.94102270256813
12.8137659000000	9.94607816434513
12.8647579000000	9.94343230933745
12.9236657000000	9.94093964788113
12.9759238000000	9.93772028495929
13.0598553000000	9.93253884506748
13.1110529000000	9.93381805277286
13.1554014000000	9.93577387440304
13.2041880000000	9.93741552124195
13.2530920000000	9.94282012410856
13.3096881000000	9.94423489722441
13.3670358000000	9.94194289688747
13.4205469000000	9.92448225041348
13.5059495000000	9.93675307556747
13.5540076000000	9.93593909050676
13.6403112000000	9.91793961335849
13.6912816000000	9.93751945176448
13.7395734000000	9.93831777018804
13.7922026000000	9.93639901413276
13.8453646000000	9.92030141486022
13.8966400000000	9.92017858767694
13.9522235000000	9.91932802445374
14.0395550000000	9.93106676182794
14.0902952000000	9.91471329349900
14.1714062000000	9.91371703731619
14.2194578000000	9.91250289130252
14.2696402000000	9.91273289890191
14.3256146000000	9.91298494113603
14.3764113000000	9.91189630808084
14.4279352000000	9.91139993210441
14.4846309000000	9.91032808514011
14.5560168000000	9.90911778207812
14.6143081000000	9.90946814399611
14.6600816000000	9.90909385810307
14.7108944000000	9.91003817404886
14.7849907000000	9.91062423745788
14.8429751000000	9.91088280542162
14.8951473000000	9.91076285070095
14.9556159000000	9.90996687666245
15.0023434000000	9.90964267383126
15.0552345000000	9.90952157535203
15.1249972000000	9.90972212866292
15.1733571000000	9.91053824042267
15.2300044000000	9.91064864015933
15.2839158000000	9.91053989574604
15.3361784000000	9.91041471299289
15.3881125000000	9.91016367838924
15.4385219000000	9.91004979010868
15.5125561000000	9.90981543212625
15.5634821000000	9.91010439297201
15.6057447000000	9.90979599884025
15.6588439000000	9.91013078765850
15.7166741000000	9.91043966721322
15.7637717000000	9.91017268312955
15.8266027000000	9.90976585352392
15.8795055000000	9.90996898660630
15.9399525000000	9.90931338687238
};
\end{axis}
  \spy [black] on (5,2.92) in node [right, fill=darkgreen, fill opacity=0.1, text opacity=1] at (1,2);
\end{tikzpicture}
\begin{tikzpicture}
    \begin{axis}[
      legend style={
        at={(.95,0.05)}, 
        anchor=south east, 
    },
      xmin=0, xmax=1.2,
      ymin=0, ymax=10,
      ymajorgrids=true,
      grid style=dashed,
      width = 0.48\linewidth,
     height = 0.28\textwidth,
    label style={font=\footnotesize},
      tick label style={font=\footnotesize},
    xlabel={Time [s]},
    ylabel={Position $q_i$ [deg]},
    ]

 \legend{\footnotesize $K_{\tt d}  = 0.1 \mbox{,}~ \gamma = 0.1$, \footnotesize $K_{\tt d} = 0.1 \mbox{,}~ \gamma = 5$}

    \addplot[color=blue, line width = 1pt] table {
0	0
0.0693885000000000	0.109240468314845
0.122497300000000	0.658412365112443
0.198155400000000	1.43027787785727
0.250228200000000	2.22150846303955
0.327088600000000	2.50383784170712
0.376834900000000	4.21007791950355
0.426985700000000	5.31695372439002
0.499697600000000	6.16262269717174
0.546779000000000	7.12005202561288
0.593915000000000	7.85864900670623
0.667654900000000	8.59982642319644
0.740667000000000	8.72504648315072
0.790164100000000	8.09170690873985
0.862515900000000	7.21344675718541
0.916427500000000	6.81981283032429
0.968644700000000	7.17668925521300
1.04061820000000	7.87220204107362
1.08974340000000	8.92766293309179
1.13864280000000	9.19173040069490
1.18702780000000	9.09933567289280
1.23652980000000	8.95762668022661
1.31095920000000	8.91524534424238
1.38401230000000	9.11701911773409
1.45600430000000	9.14784792410412
1.50406770000000	9.16724399187681
1.55391600000000	9.25522890574226
1.60031070000000	9.26801439850978
1.65290480000000	9.28084236384503
1.72351040000000	9.34263787699840
1.77607740000000	9.37303091802694
1.84975970000000	9.39640351526368
1.89996990000000	9.42201480046752
1.96826940000000	9.43716821968769
2.01636100000000	9.47289941222153
2.08909300000000	9.48035550470976
2.16623180000000	9.50430602140996
2.23823040000000	9.52482615186395
2.29178930000000	9.53877934359532
2.34023630000000	9.59190060113967
2.41436650000000	9.57023638376716
2.46857310000000	9.58818705596218
2.54398980000000	9.60368995693969
2.59555190000000	9.64849172759737
2.64094630000000	9.66255413977053
2.69824780000000	9.66541092474847
2.76959480000000	9.67291448447276
2.84298750000000	9.68141811702656
2.89657770000000	9.68972878170973
2.94741910000000	9.70091776645849
3.00335530000000	9.70679733325227
3.05351540000000	9.71440484346116
3.12890720000000	9.72146296626594
3.17647860000000	9.72703515065965
3.23273970000000	9.73423449121033
3.29044450000000	9.68847451611622
3.34325310000000	9.69396260305281
3.39583500000000	9.69830083379056
3.44978520000000	9.75775073008878
3.52314970000000	9.76632704998052
3.57683850000000	9.77284582637660
3.62382260000000	9.77550530607615
3.67435210000000	9.77940879974756
3.71969280000000	9.78355003371998
3.79310170000000	9.73986914209341
3.84596240000000	9.74611393715758
3.91869830000000	9.74979409800284
3.97121970000000	9.75647079758699
4.02041110000000	9.76079656254923
4.07538600000000	9.76419100029334
4.14940640000000	9.77002848381311
4.22140480000000	9.77299806709166
4.29751650000000	9.77666989306932
4.36918340000000	9.77911281768722
4.41840210000000	9.78278572118367
4.47395890000000	9.78500099186056
4.54735950000000	9.78897398479023
4.60023670000000	9.79107816695948
4.65750640000000	9.79367770149387
4.71418380000000	9.79612949975662
4.76057990000000	9.79750432430745
4.80878580000000	9.79847430748920
4.85619780000000	9.79858253410108
4.92993910000000	9.80007362051136
4.98560800000000	9.80109431465961
5.05724310000000	9.80161236315587
};



    \addplot[color=red, line width = 1pt] table {
0	0
0.0741747000000000	0.0122468207092195
0.147141200000000	0.887343542813589
0.193181700000000	2.20734334316893
0.267529200000000	2.99945044672740
0.317248000000000	4.40609058868594
0.392388700000000	5.16252710053975
0.465810800000000	6.15733394702869
0.537465900000000	6.44630700829765
0.586915300000000	6.58555709669073
0.635386000000000	6.72221720403734
0.708779900000000	7.25944719275328
0.759911600000000	7.71798311668990
0.815225600000000	8.30613495834736
0.887606500000000	8.54284435331146
0.934010700000000	8.69353679027896
1.00636590000000	8.74217866942862
1.05519400000000	8.76852559860376
1.12817890000000	8.79867358652195
1.18003670000000	8.94052262414792
1.25109110000000	8.95584041304672
1.30195170000000	9.03196757439526
1.37670390000000	9.06002205071216
1.42997070000000	9.13194688703554
1.47870930000000	9.14452342310275
1.52683040000000	9.19019286821544
1.57591100000000	9.22710634826593
1.64731070000000	9.26349263786202
1.69646040000000	9.30365517431978
1.76744040000000	9.32135909236110
1.81659860000000	9.35519070756618
1.88960930000000	9.37585419570425
1.93572110000000	9.40345062192143
1.98241600000000	9.41140620502165
2.03435300000000	9.43677198179696
2.08173240000000	9.45476508061074
2.12991360000000	9.47548930128937
2.18005060000000	9.49201024695179
2.22982840000000	9.50740098974555
2.30223840000000	9.53857211943114
2.35448380000000	9.55048077748533
2.42612890000000	9.56144894924570
2.47598380000000	9.57938164053713
2.55038590000000	9.58981078775735
2.59541860000000	9.61023530892994
2.67016780000000	9.61579505001340
2.71657280000000	9.63676840142620
2.76437800000000	9.67784483675581
2.83663340000000	9.68876856409737
2.88757350000000	9.69820931998343
2.93565250000000	9.70465870505806
2.98343080000000	9.70773560602583
3.05385290000000	9.71710946107104
3.10322710000000	9.72585563024005
3.17870950000000	9.73061381623793
3.22746680000000	9.69109809540213
3.28016830000000	9.69641454884859
3.32711900000000	9.70104452893226
3.37421940000000	9.70608246474770
3.44694820000000	9.76014598343002
3.50015600000000	9.76795537237488
3.55263970000000	9.77205374619518
3.59984100000000	9.73230903017118
3.65339820000000	9.73617650731777
3.70655120000000	9.73945377319268
3.78035570000000	9.74432397100435
3.83426570000000	9.74846265155411
3.88720100000000	9.75642472498066
3.93363100000000	9.75895788794489
4.00491770000000	9.76573244154082
4.07764180000000	9.77085315206564
4.12713570000000	9.77438185065238
4.19811730000000	9.77785832980046
4.27047470000000	9.78240081638874
4.33973450000000	9.78595450335708
4.39089090000000	9.78738490516313
4.43596560000000	9.78794974919534
4.48584050000000	9.78816544317381
4.53294340000000	9.78940246965933
4.60621950000000	9.79017781685006
4.67900380000000	9.78985905564133
4.72843270000000	9.78997783975516
4.77796130000000	9.78972167995295
4.85173500000000	9.78973677758189
4.90590110000000	9.78951883897097
4.95408750000000	9.78984020464496
5.01176930000000	9.78965835710583
5.06257930000000	9.79039138562093
};
\end{axis}
\end{tikzpicture}